\documentclass[dvipsnames]{article} 
\usepackage{iclr2026_conference,times}

\usepackage[utf8]{inputenc} 
\usepackage[T1]{fontenc}    
\usepackage{url}            
\usepackage{booktabs}       
\usepackage{amsfonts}       
\usepackage{nicefrac}       
\usepackage{microtype}      
\usepackage{xcolor}         
\usepackage{multirow}
\usepackage{colortbl}
\usepackage{nicematrix}
\usepackage{tabularx}
\usepackage{tabularray}
\usepackage{listings}
\usepackage[inline]{enumitem}
\usepackage{caption}
\usepackage{subcaption}
\usepackage{graphicx}
\usepackage{glossaries}
\usepackage{adjustbox}
\usepackage{pgfplots}
\usepackage{mleftright}
\usepackage{derivative}
\usepackage{pifont}
\usepackage{titlesec}
\usepackage{tablefootnote}
\usepackage{algorithm,algorithmicx,algpseudocode}
\usepackage{soul}
\usepackage{wrapfig}
\usepackage{bm}
\usepackage[bottom]{footmisc}
\usepackage{afterpage}
\usepackage{siunitx}
\usepackage{setspace}
\usepackage{inconsolata}
\usepackage{tcolorbox}
\usepackage[toc]{appendix}
\usepackage{hyperref}       
\UseTblrLibrary{booktabs}
\UseTblrLibrary{varwidth}
\definecolor{myBlue}{HTML}{0F1A5F}
\definecolor{myRed}{HTML}{721010}
\hypersetup{
  colorlinks,
  linkcolor={myRed},
  citecolor={myBlue},
  urlcolor={blue!80!black}
  }
\glsdisablehyper
\hypersetup{breaklinks=true}
\usepgfplotslibrary{groupplots}
\usepgfplotslibrary{colormaps}
\usepgfplotslibrary{fillbetween}
\usetikzlibrary{plotmarks}

\pgfplotsset{compat=1.14}	 %
\pgfplotsset{compat/show suggested version=false}
\pgfplotsset{every mark/.append style={solid}}
\mleftright

\titlespacing*{\paragraph}{0pt}{0.35\baselineskip}{1em}

\usepackage{amsmath, amsthm, amssymb}
\usepackage{mathtools}

\usepackage[capitalize,noabbrev]{cleveref}
\theoremstyle{plain}
\newtheorem{theorem}{Theorem}[section]

\theoremstyle{definition}

\theoremstyle{remark}

\definecolor{C0}{rgb}{0.121569, 0.466667, 0.705882}
\definecolor{C1}{rgb}{1.000000, 0.498039, 0.054902}
\definecolor{C2}{rgb}{0.172549, 0.627451, 0.172549}
\definecolor{C3}{rgb}{0.839216, 0.152941, 0.156863}
\definecolor{C4}{rgb}{0.580392, 0.403922, 0.741176}
\definecolor{C5}{rgb}{0.549020, 0.337255, 0.294118}
\definecolor{C6}{rgb}{0.890196, 0.466667, 0.760784}
\definecolor{C7}{rgb}{0.498039, 0.498039, 0.498039}
\definecolor{C8}{rgb}{0.737255, 0.741176, 0.133333}
\definecolor{C9}{rgb}{0.090196, 0.745098, 0.811765}

\newcommand{\mycc}{\cellcolor{LightGray}}

\newcommand{\myparallel}{{\mkern3mu\vphantom{\perp}\vrule depth 0pt\mkern2mu\vrule depth 0pt\mkern3mu}}

\newcolumntype{Y}{>{\centering\arraybackslash}X}
\newcolumntype{C}{>{\hsize=.0\hsize\centering\arraybackslash}X}

\colorlet{LightGoldenrod}{White!40!Goldenrod}
\colorlet{LightGray}{White!90!Periwinkle}
\definecolor{LG}{gray}{0.95}

\sethlcolor{LightGoldenrod}

\definecolor{codegreen}{rgb}{0,0.6,0}
  \definecolor{codegray}{rgb}{0.5,0.5,0.5}
  \definecolor{codepurple}{rgb}{0.58,0,0.82}
  \definecolor{backcolour}{rgb}{0.95,0.95,0.92}
  \lstdefinestyle{mystyle}{
    backgroundcolor=\color{backcolour},
    commentstyle=\color{codegreen},
    keywordstyle=\color{magenta},
    numberstyle=\tiny\color{codegray},
    stringstyle=\color{codepurple},
    basicstyle=\ttfamily\footnotesize,
    breakatwhitespace=false,
    breaklines=true,
    captionpos=b,
    keepspaces=true,
    numbers=left,
    numbersep=5pt,
    showspaces=false,
    showstringspaces=false,
    showtabs=false,
    tabsize=2
  }
  
  \lstnewenvironment{mylisting}{\lstset{style=mystyle}}{}

\newcommand{\wcfg}{w_{\textnormal{CFG}}}
\newcommand{\whigs}{w_{\textnormal{\acrshort{method}}}}
\newcommand{\fcfg}{f_{\mathrm{CFG}}}

\newcommand{\pred}[1][t]{D_{\mtheta}(\vz_{#1}, #1, \vy)}
\newcommand{\prednull}[1][t]{D_{\mtheta}(\vz_{#1}, #1)}
\newcommand{\predcfg}[1][t]{{D}_{\textrm{CFG}}(\vz_{#1}, #1, \vy)}

\newcommand{\predcond}[1][t]{D_c(\vz_{#1})}
\newcommand{\preduncond}[1][t]{D_u(\vz_{#1})}
\newcommand{\predguided}[1][t]{{D}_{\textnormal{CFG}}(\vz_{#1})}
\newcommand{\predguidedours}[1][t]{{D}_{\textnormal{\gls{method}}}(\vz_{#1})}

\newcommand{\dpred}[1][t]{\Delta D_{#1}}
\newcommand{\dpredpar}[1][t]{\Delta D_{#1}^{\myparallel}}
\newcommand{\dpredorth}[1][t]{\Delta D_{#1}^{\perp}}

\newcommand{\predi}{D_{\mtheta}(\vz_{t_i}, t_i, \vy)}
\newcommand{\buffer}{\mathcal{H}_k}

\newcommand{\sdthree}{Stable Diffusion 3}

\newcommand{\tmin}{t_{\textnormal{min}}}
\newcommand{\tmax}{t_{\textnormal{max}}}

\usepackage{amsmath,amsfonts,bm, mathtools}

\newcommand{\mc}[1]{\mathcal{#1}}

\def\eqref#1{equation~\ref{#1}}

\def\1{\bm{1}}

\newcommand{\dd}{\mathrm{d}}

\def\vzero{{\bm{0}}}

\def\vx{{\bm{x}}}
\def\vy{{\bm{y}}}
\def\vz{{\bm{z}}}

\def\mI{{\pmb{I}}}

\DeclareMathAlphabet{\mathsfit}{\encodingdefault}{\sfdefault}{m}{sl}
\SetMathAlphabet{\mathsfit}{bold}{\encodingdefault}{\sfdefault}{bx}{n}

\def\mepsilon{{\bm{\epsilon}}}

\def\mtheta{{\bm{\theta}}}

\newcommand{\pdata}{p_{\textnormal{data}}}

\newcommand{\norm}[1]{\left\lVert #1 \right\rVert}

\newcommand{\inner}[2]{\left \langle #1 , #2 \right \rangle}
\newcommand{\normal}[2]{\mc{N}\prn{#1, #2}}
\newcommand{\sg}[1]{\operatorname{sg}\brk{#1}}

\DeclarePairedDelimiterX{\infdivx}[2]{(}{)}{%
  #1\delimsize\|#2%
}

\newcommand{\brk}[1]{\left[ #1 \right]}

\newcommand{\prn}[1]{\left( #1 \right)}

\newcommand{\zero}{\pmb{0}}

\DeclareDocumentCommand{\ex}{m o}{
   \mathbb{E}\IfValueT{#2}{_{#2}}\left[#1\right]
}

\DeclareMathOperator{\grad}{\nabla}

\DeclarePairedDelimiterX\Set[1]{\lbrace}{\rbrace}%
 {  #1 }

\def\ddefloop#1{\ifx\ddefloop#1\else\ddef{#1}\expandafter\ddefloop\fi}
\def\ddef#1{\expandafter\def\csname #1bb\endcsname{\ensuremath{\mathbb{#1}}}}
\ddefloop ABCDEFGHIJKLMNOPQRSTUVWXYZ\ddefloop
\def\ddefloop#1{\ifx\ddefloop#1\else\ddef{#1}\expandafter\ddefloop\fi}
\def\ddef#1{\expandafter\def\csname #1b\endcsname{\ensuremath{\mathbf{#1}}}}
\ddefloop ABCDEFGHIJKLMNOPQRSTUVWXYZ\ddefloop
\def\ddef#1{\expandafter\def\csname #1c\endcsname{\ensuremath{\mathcal{#1}}}}
\ddefloop ABCDEFGHIJKLMNOPQRSTUVWXYZ\ddefloop
\def\ddef#1{\expandafter\def\csname #1hat\endcsname{\ensuremath{\widehat{#1}}}}
\ddefloop ABCDEFGHIJKLMNOPQRSTUVWXYZ\ddefloop
\def\ddef#1{\expandafter\def\csname hc#1\endcsname{\ensuremath{\widehat{\mathcal{#1}}}}}
\ddefloop ABCDEFGHIJKLMNOPQRSTUVWXYZ\ddefloop
\def\ddef#1{\expandafter\def\csname #1til\endcsname{\ensuremath{\widetilde{#1}}}}
\ddefloop ABCDEFGHIJKLMNOPQRSTUVWXYZ\ddefloop
\def\ddef#1{\expandafter\def\csname tc#1\endcsname{\ensuremath{\widetilde{\mathcal{#1}}}}}
\ddefloop ABCDEFGHIJKLMNOPQRSTUVWXYZ\ddefloop
\def\ddef#1{\expandafter\def\csname #1Bar\endcsname{\ensuremath{\bar{#1}}}}
\ddefloop ABCDEFGHIJKLMNOPQRSTUVWXYZ\ddefloop

\newacronym{ldm}{LDM}{latent diffusion model}
\newacronym{vae}{VAE}{variational autoencoder}
\newacronym{sdvae}{SD-VAE}{Stable Diffusion VAE}
\newacronym{method}{HiGS}{history-guided sampling}
\newacronym{sgd}{SGD}{stochastic gradient descent}
\newacronym{nfe}{NFEs}{neural function evaluations}
\newacronym{cfg}{CFG}{classifier-free guidance}
\newacronym{dct}{DCT}{discrete cosine transform}

\newcommand{\figTeaser}{
    \begin{figure}[h]
        \centering
        \vspace{-0.15cm}
        \begin{minipage}{\textwidth}
            \begin{subfigure}{\textwidth}
                \begin{tikzpicture}
                    \node[anchor=south west, inner sep=0] (image) at (0,0) {
                            \includegraphics[width=0.495\textwidth]{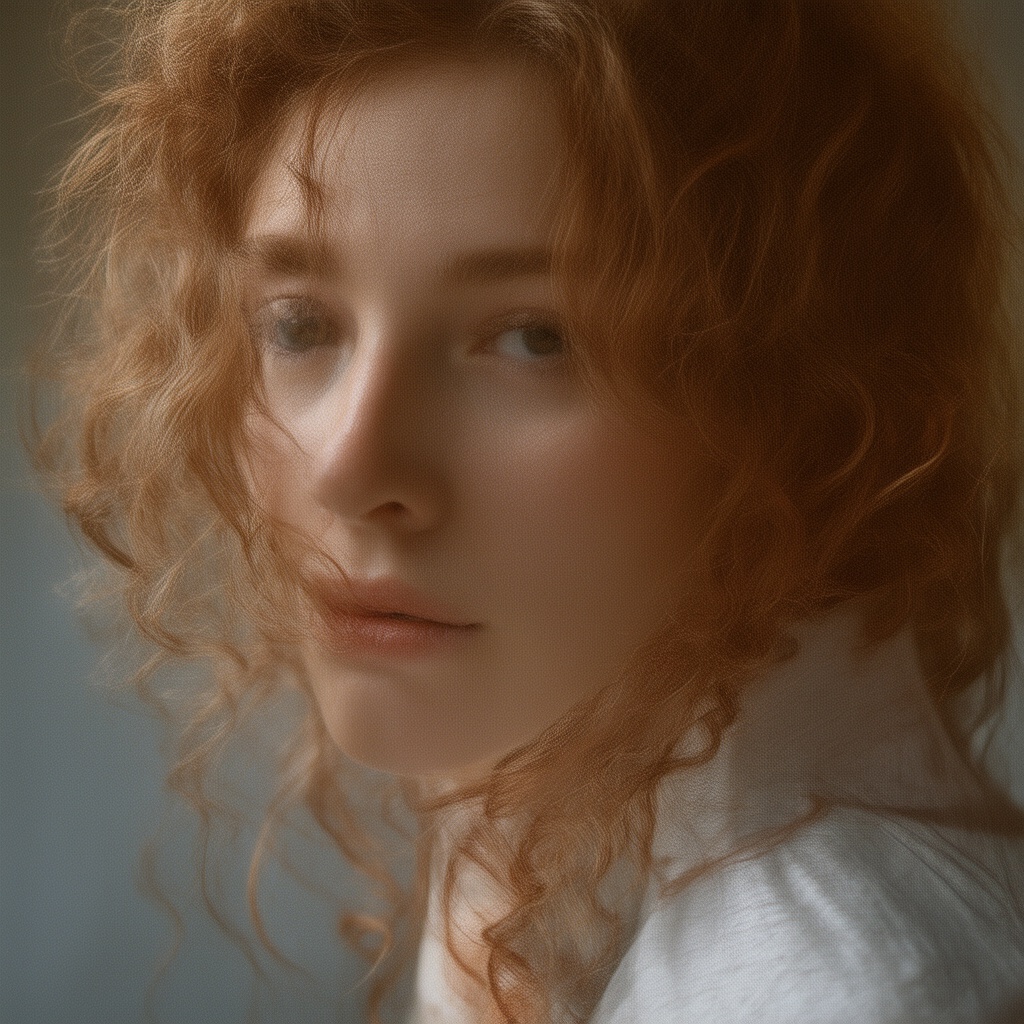}
                    };
                    \node[rotate=0, align=center, font=\normalsize, anchor=base, yshift=4pt] at (image.north) {CFG \strut};
                \end{tikzpicture}
                \hfil
                \begin{tikzpicture}
                    \node[anchor=south west, inner sep=0] (image) at (0,0) {
                \includegraphics[width=0.495\textwidth]{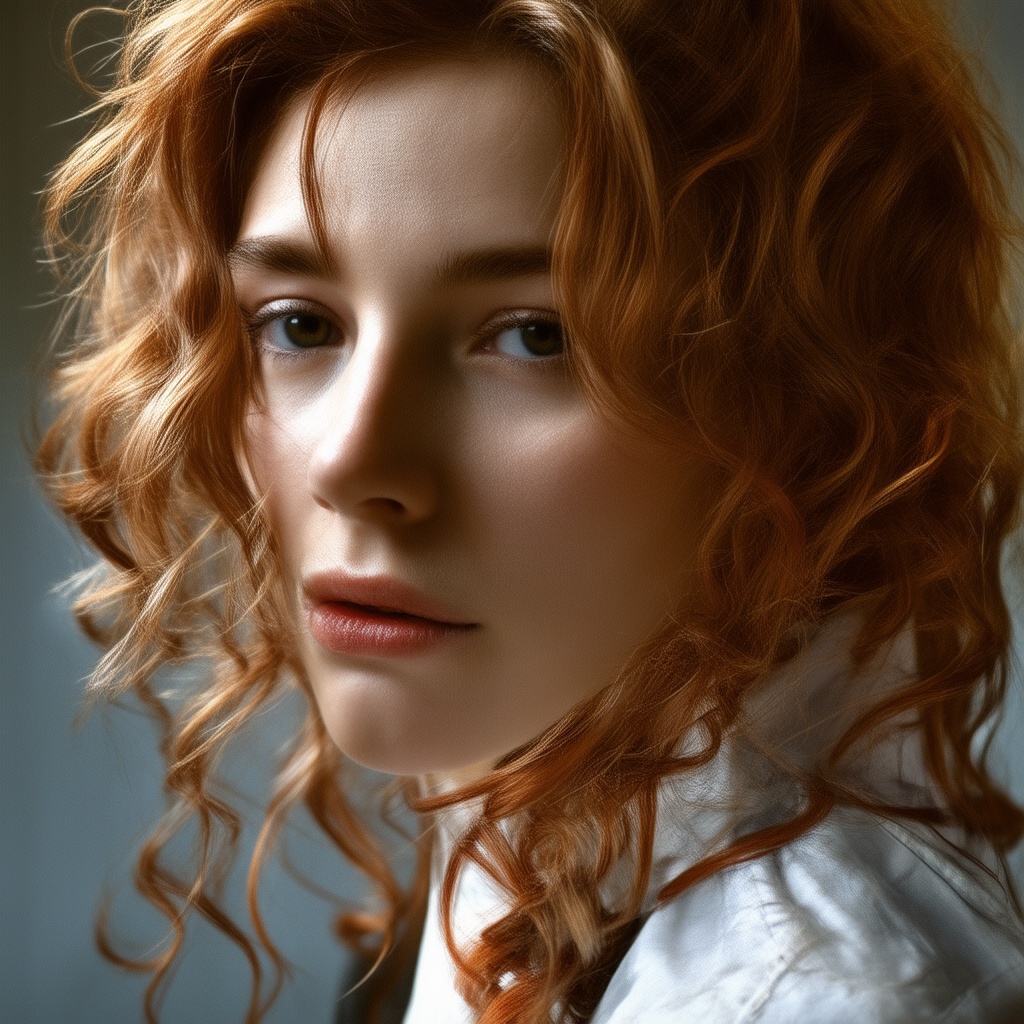}
                    };
                    \node[rotate=0, align=center, font=\normalsize, anchor=base, yshift=4pt] at (image.north) {+ \acrshort{method} (Ours) \strut};
                \end{tikzpicture}
                \caption*{\texttt{Portrait photo of a female with red hair}}
            \end{subfigure}
        \end{minipage}
        \caption{Sampling with diffusion models using fewer steps or lower guidance scales often results in blurry images with artifacts. We propose \gls{method}, a novel sampling method that enhances quality and details in generated images under various sampling budgets and guidance scales by leveraging a weighted average of the diffusion model's past predictions at each inference step. The results shown are generated with Stable Diffusion 3.5 \citep{esser2024scaling} using 10 steps and a guidance scale of 1.2.}
        \label{fig:teaser}
        
    \end{figure}
}

\newcommand{\figMain}{
    \begin{figure}[t]
        \centering
        \vspace{-0.5cm}
        \begin{minipage}{\textwidth}
            \begin{subfigure}{0.495\textwidth}
                \begin{tikzpicture}
                    \node[anchor=south west, inner sep=0] (image) at (0,0) {
                            \includegraphics[width=0.495\textwidth]{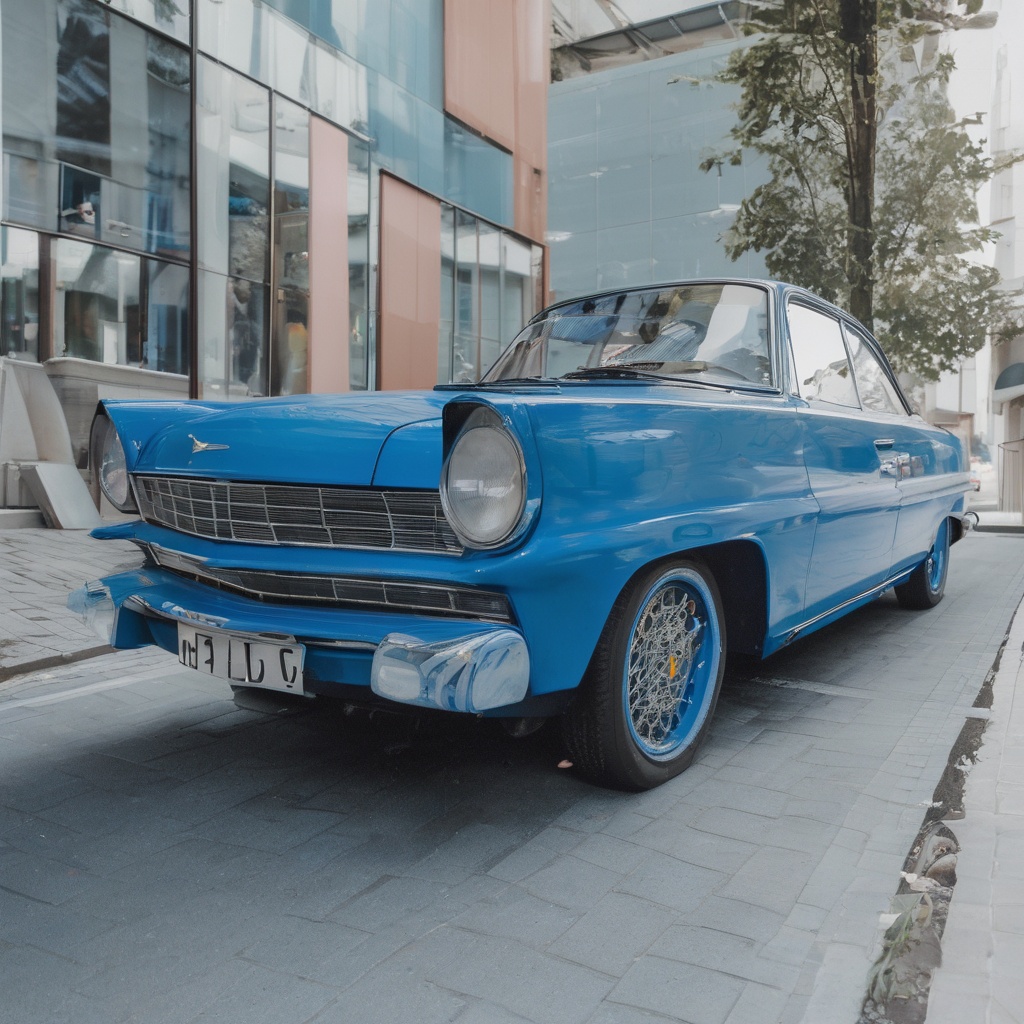}
                    };
                    \node[rotate=0, align=center, font=\normalsize, anchor=base, yshift=4pt] at (image.north) {CFG \strut};
                \end{tikzpicture}
                \hfil
                \begin{tikzpicture}
                    \node[anchor=south west, inner sep=0] (image) at (0,0) {
                \includegraphics[width=0.495\textwidth]{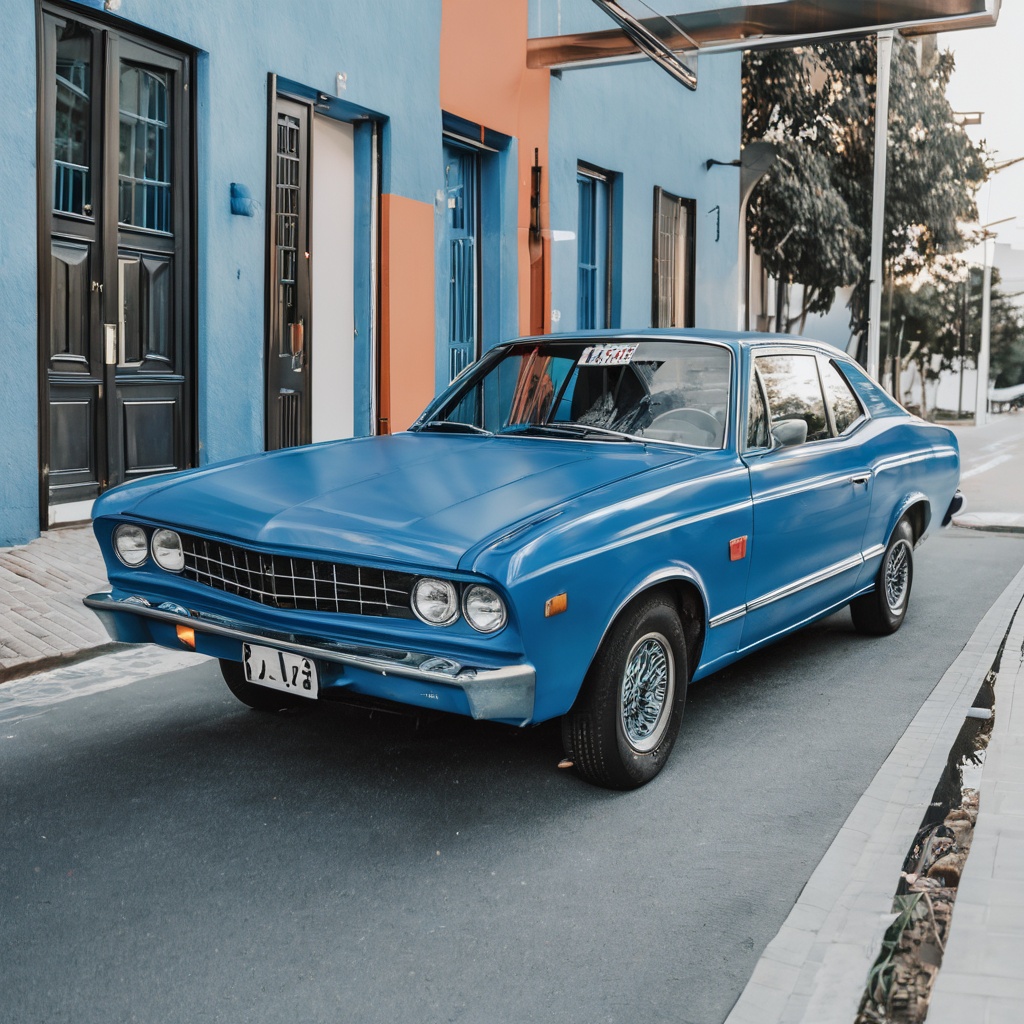}
                    };
                    \node[rotate=0, align=center, font=\normalsize, anchor=base, yshift=4pt] at (image.north) {+ \acrshort{method} (Ours) \strut};
                \end{tikzpicture}
                \caption*{\texttt{A blue car}}
            \end{subfigure}
            \begin{subfigure}{0.495\textwidth}
                \begin{tikzpicture}
                    \node[anchor=south west, inner sep=0] (image) at (0,0) {
                            \includegraphics[width=0.495\textwidth]{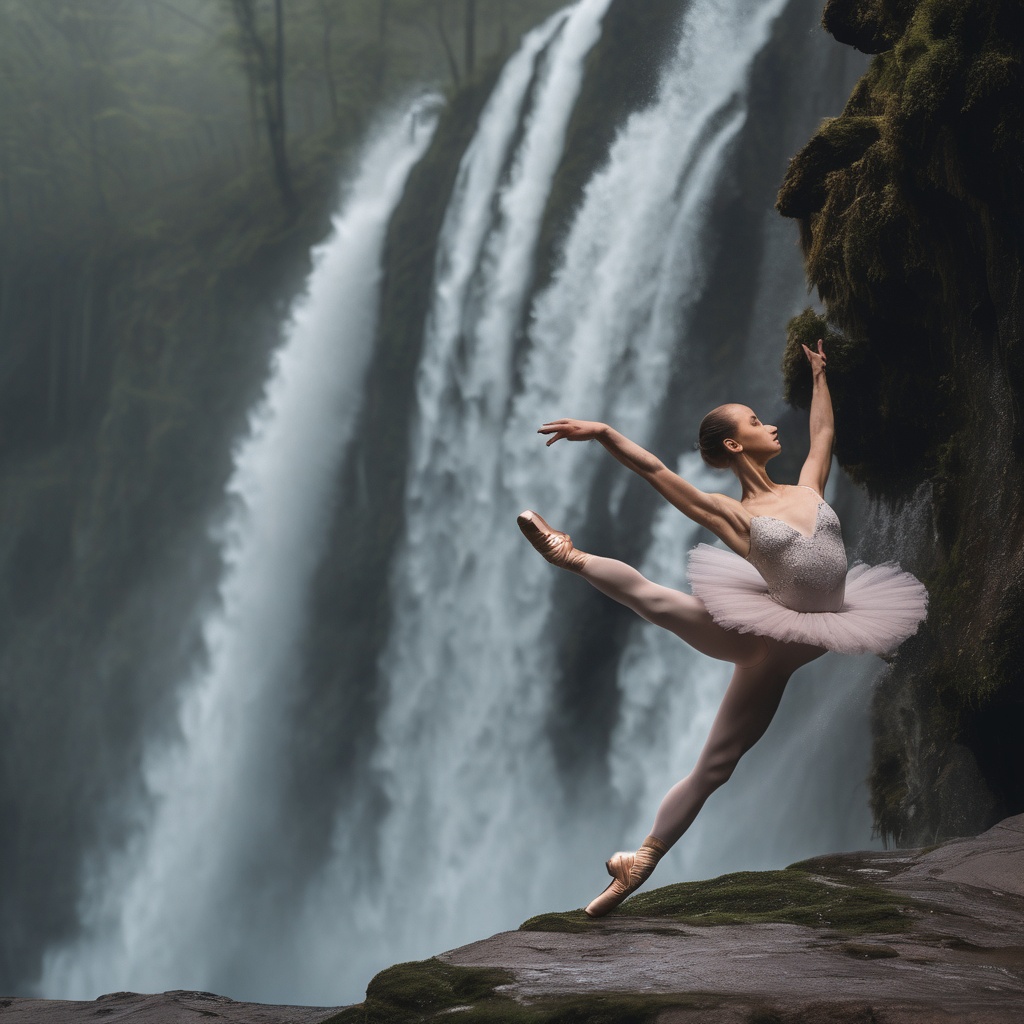}
                    };
                    \node[rotate=0, align=center, font=\normalsize, anchor=base, yshift=4pt] at (image.north) {CFG \strut};
                \end{tikzpicture}
                \hfil
                \begin{tikzpicture}
                    \node[anchor=south west, inner sep=0] (image) at (0,0) {
                \includegraphics[width=0.495\textwidth]{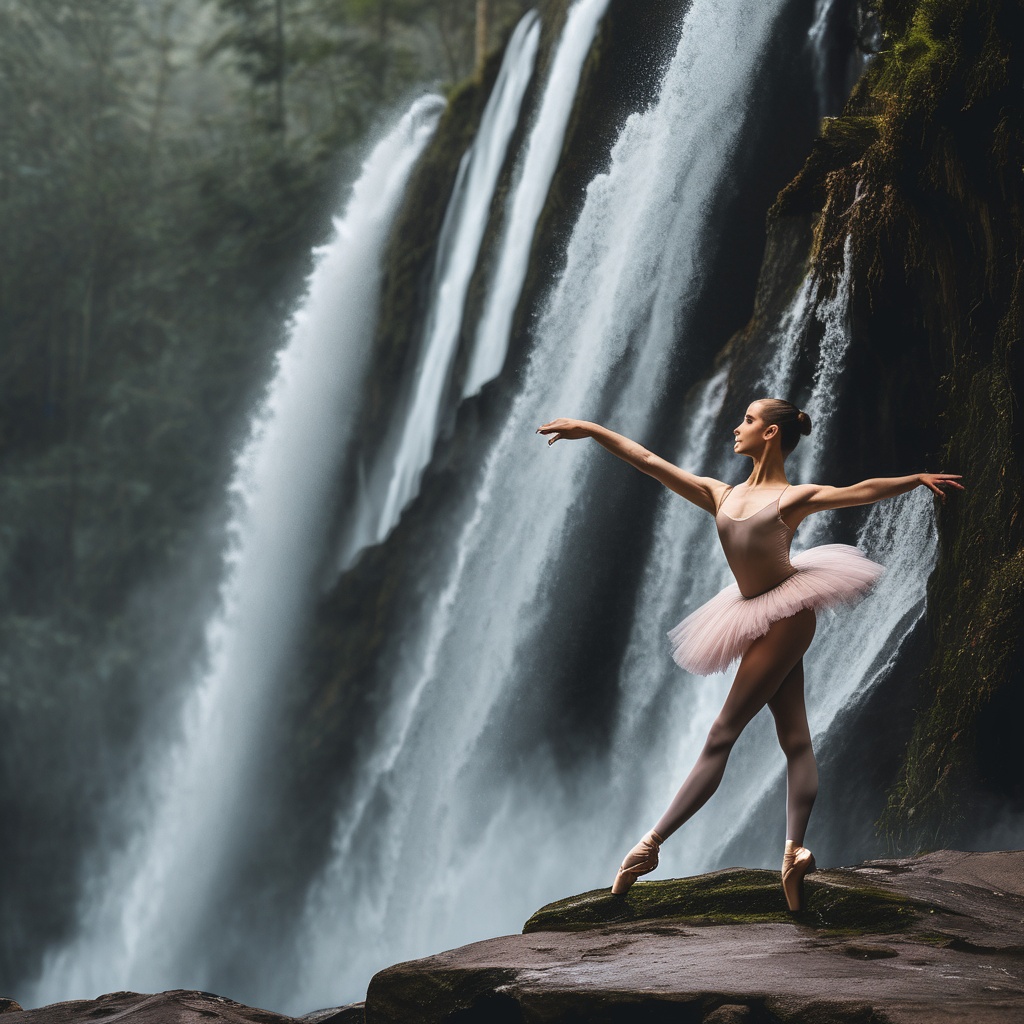}
                    };
                    \node[rotate=0, align=center, font=\normalsize, anchor=base, yshift=4pt] at (image.north) {+ \acrshort{method} (Ours) \strut};
                \end{tikzpicture}
                \caption*{\texttt{A ballet dancer next to a waterfall}}
            \end{subfigure}

            \begin{subfigure}{0.495\textwidth}
                \begin{tikzpicture}
                    \node[anchor=south west, inner sep=0] (image) at (0,0) {
                            \includegraphics[width=0.495\textwidth]{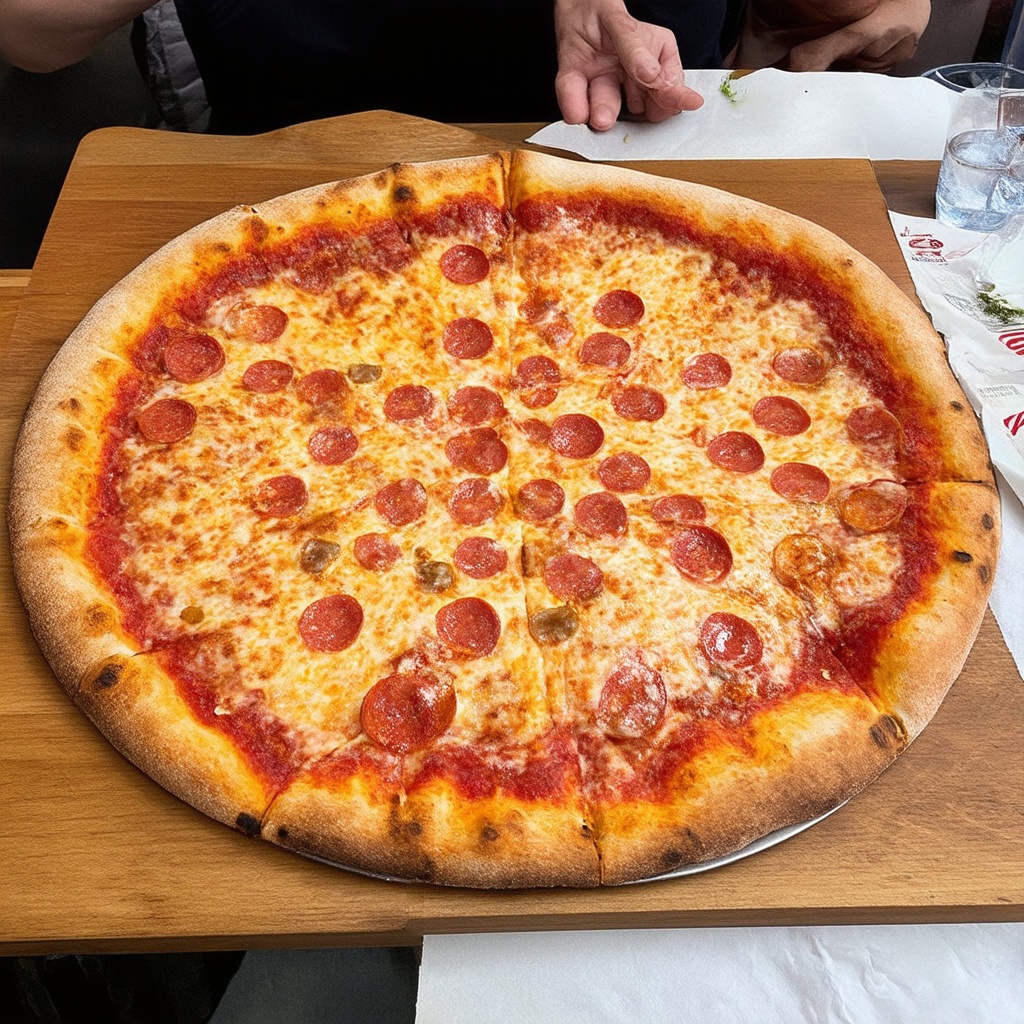}
                    };
                    \node[rotate=0, align=center, font=\normalsize, anchor=base, yshift=4pt] at (image.north) {CFG \strut};
                \end{tikzpicture}
                \hfil
                \begin{tikzpicture}
                    \node[anchor=south west, inner sep=0] (image) at (0,0) {
                \includegraphics[width=0.495\textwidth]{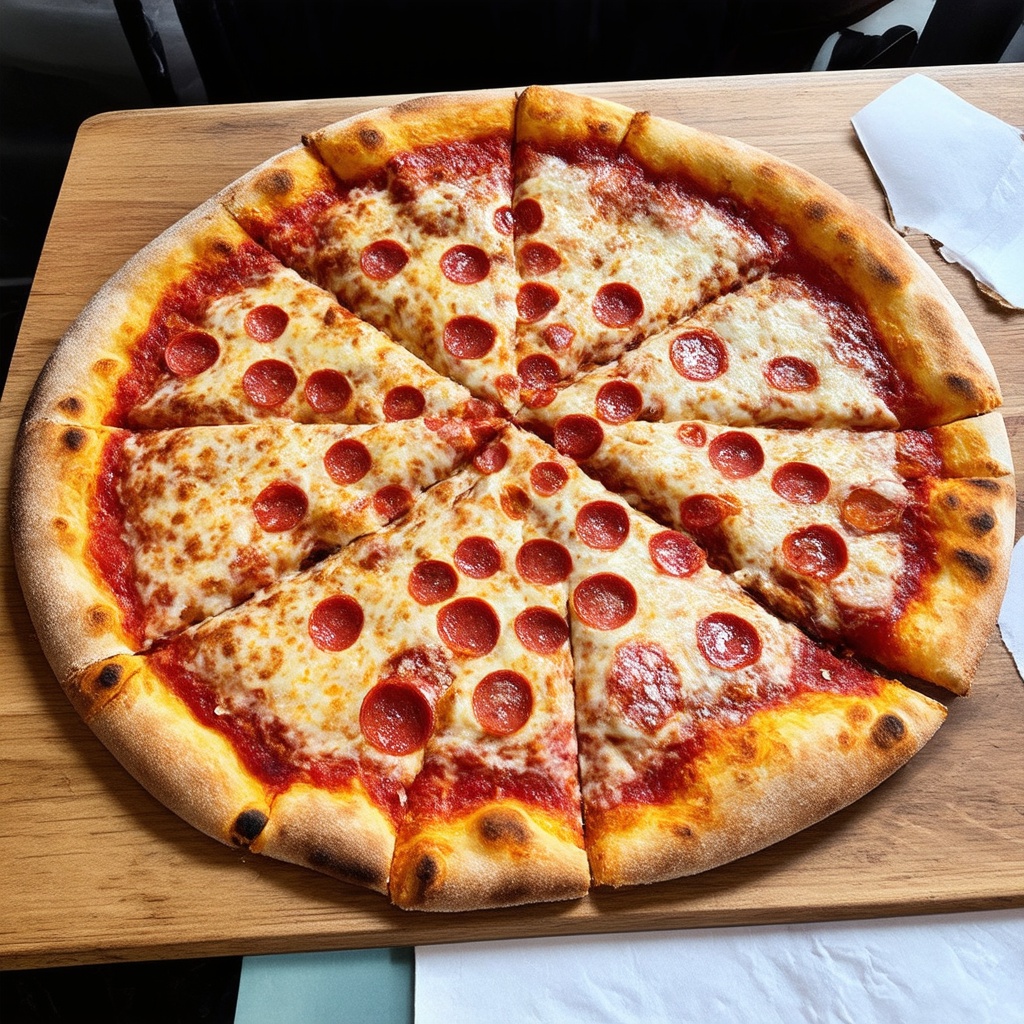}
                    };
                    \node[rotate=0, align=center, font=\normalsize, anchor=base, yshift=4pt] at (image.north) {+ \acrshort{method} (Ours) \strut};
                \end{tikzpicture}
                \caption*{\texttt{A pizza}}
            \end{subfigure}
            \begin{subfigure}{0.495\textwidth}
                \begin{tikzpicture}
                    \node[anchor=south west, inner sep=0] (image) at (0,0) {
                            \includegraphics[width=0.495\textwidth]{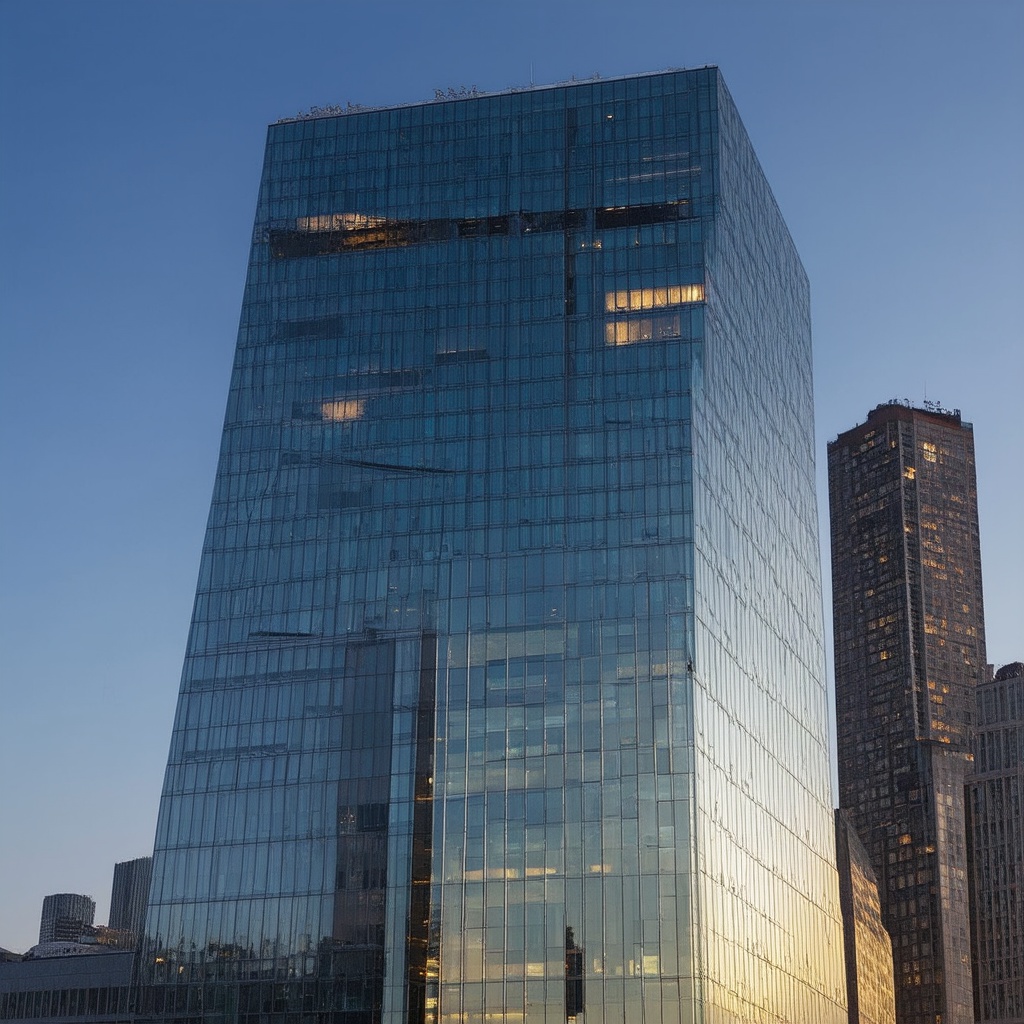}
                    };
                    \node[rotate=0, align=center, font=\normalsize, anchor=base, yshift=4pt] at (image.north) {CFG \strut};
                \end{tikzpicture}
                \hfil
                \begin{tikzpicture}
                    \node[anchor=south west, inner sep=0] (image) at (0,0) {
                \includegraphics[width=0.495\textwidth]{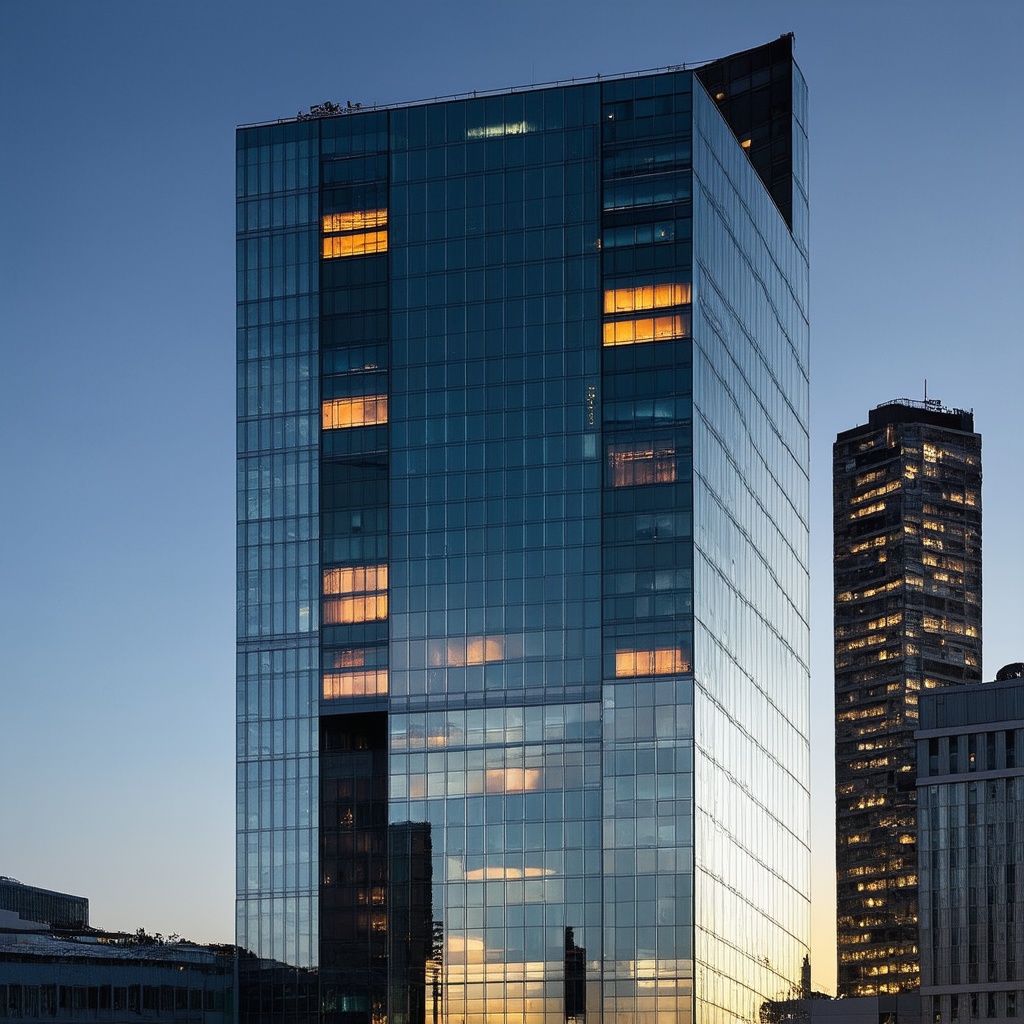}
                    };
                    \node[rotate=0, align=center, font=\normalsize, anchor=base, yshift=4pt] at (image.north) {+ \acrshort{method} (Ours) \strut};
                \end{tikzpicture}
                \caption*{\texttt{A modern glass skyscraper at sunset}}
            \end{subfigure}
        \end{minipage}
        \caption{Effect of \gls{method} on generated samples using standard sampling setups. Compared with CFG outputs, \gls{method} produces sharper images with improved structure and fewer artifacts.}
        \label{fig:qualitative-main}        
    \end{figure}
}

\newcommand{\figMainLowNFE}{
    \begin{figure}[t]
        \centering
        \begin{minipage}{\textwidth}
            \begin{subfigure}{0.495\textwidth}
                \begin{tikzpicture}
                    \node[anchor=south west, inner sep=0] (image) at (0,0) {
                            \includegraphics[width=0.495\textwidth]{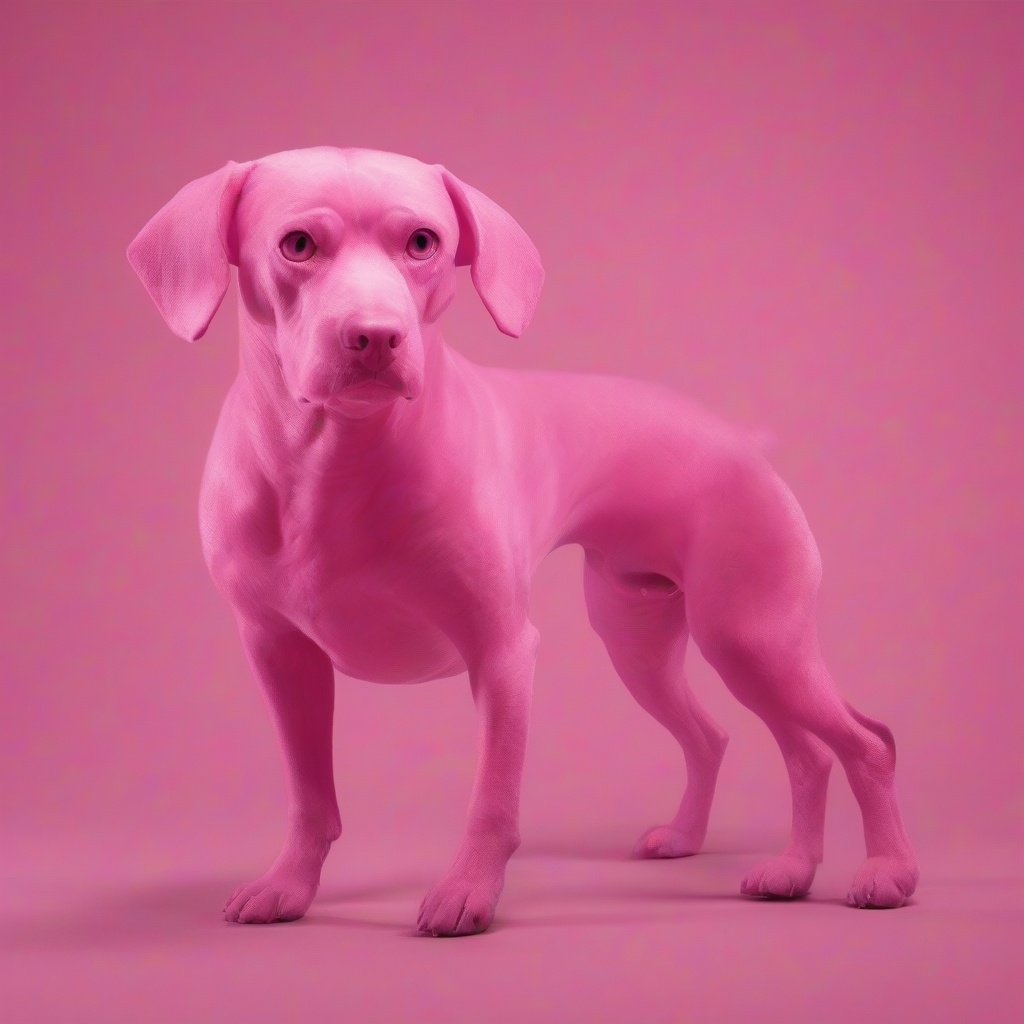}
                    };
                    \node[rotate=0, align=center, font=\normalsize, anchor=base, yshift=4pt] at (image.north) {CFG \strut};
                \end{tikzpicture}
                \hfil
                \begin{tikzpicture}
                    \node[anchor=south west, inner sep=0] (image) at (0,0) {
                \includegraphics[width=0.495\textwidth]{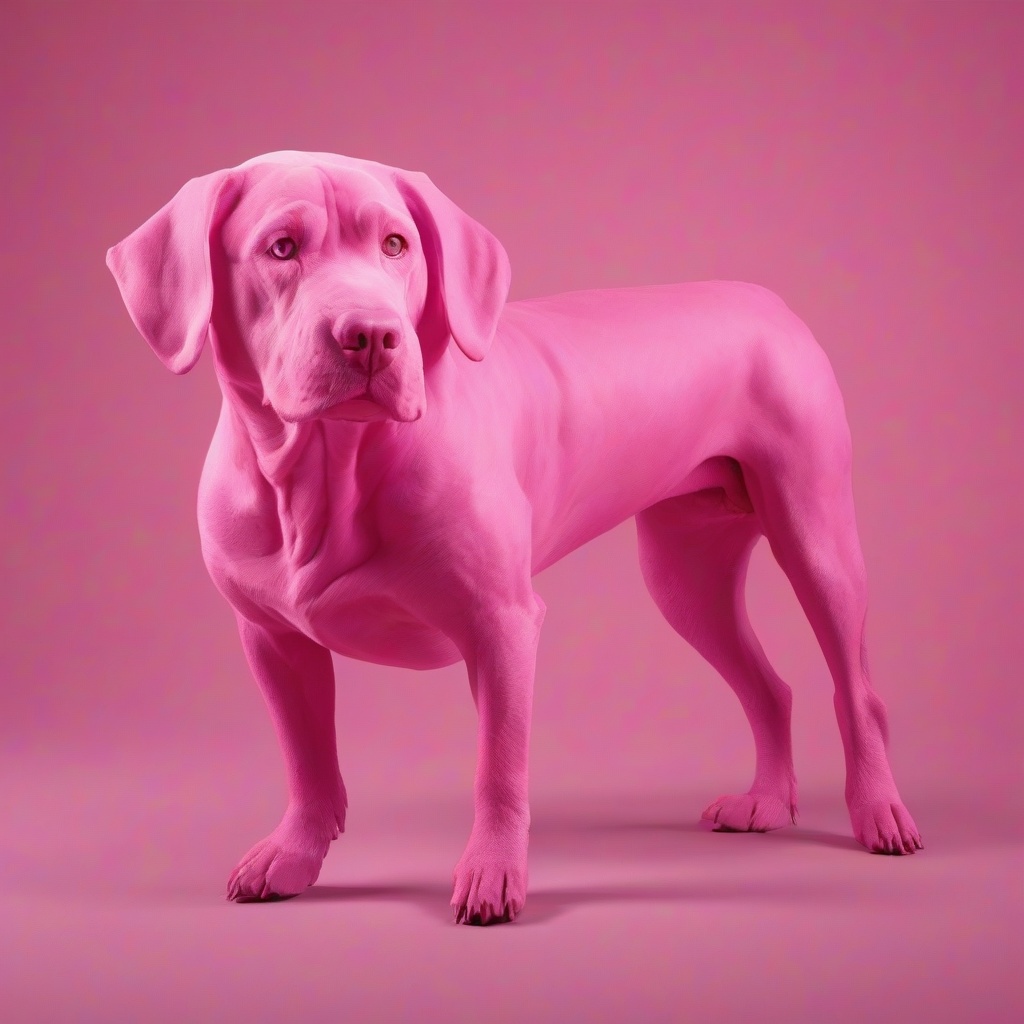}
                    };
                    \node[rotate=0, align=center, font=\normalsize, anchor=base, yshift=4pt] at (image.north) {+ \acrshort{method} (Ours) \strut};
                \end{tikzpicture}
                \caption*{\texttt{A pink dog}}
            \end{subfigure}
            \hfil
            \begin{subfigure}{0.495\textwidth}
                \begin{tikzpicture}
                    \node[anchor=south west, inner sep=0] (image) at (0,0) {
                            \includegraphics[width=0.495\textwidth]{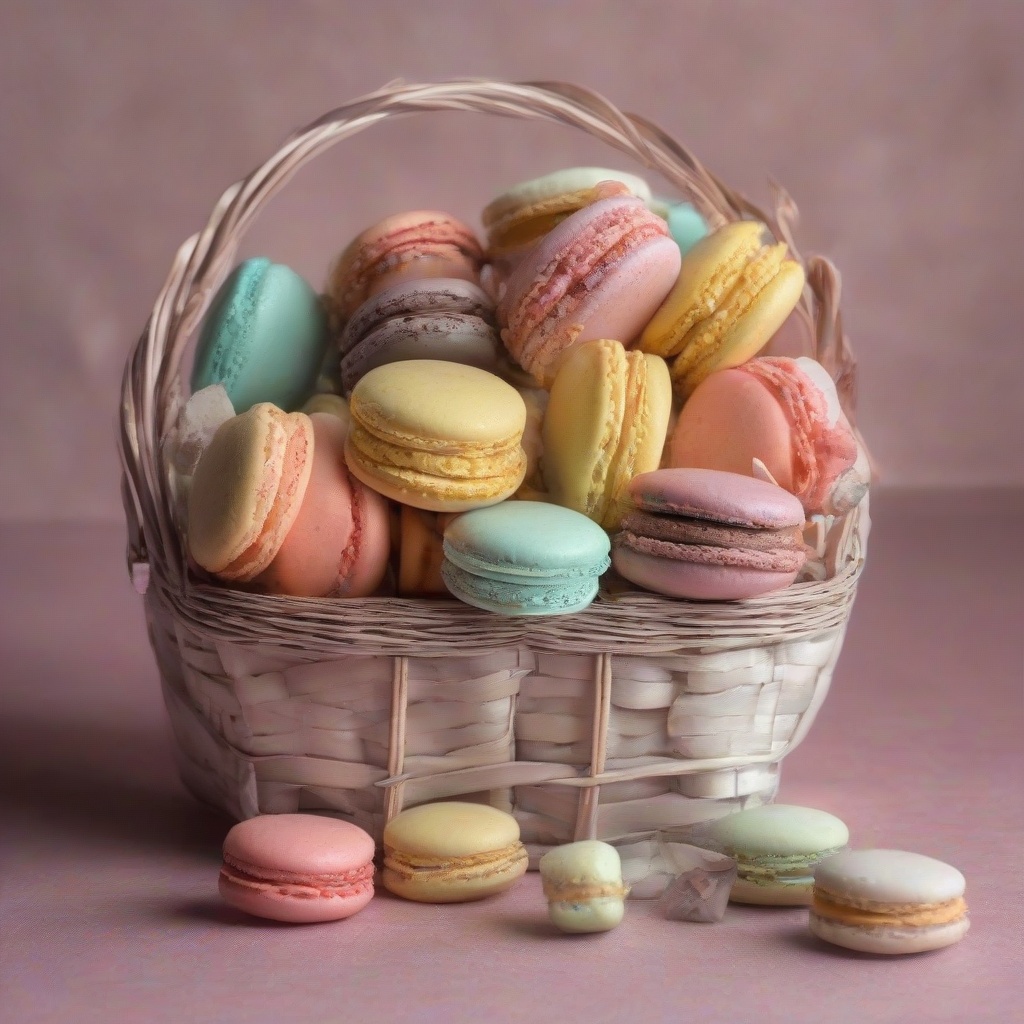}
                    };
                    \node[rotate=0, align=center, font=\normalsize, anchor=base, yshift=4pt] at (image.north) {CFG \strut};
                \end{tikzpicture}
                \hfil
                \begin{tikzpicture}
                    \node[anchor=south west, inner sep=0] (image) at (0,0) {
                \includegraphics[width=0.495\textwidth]{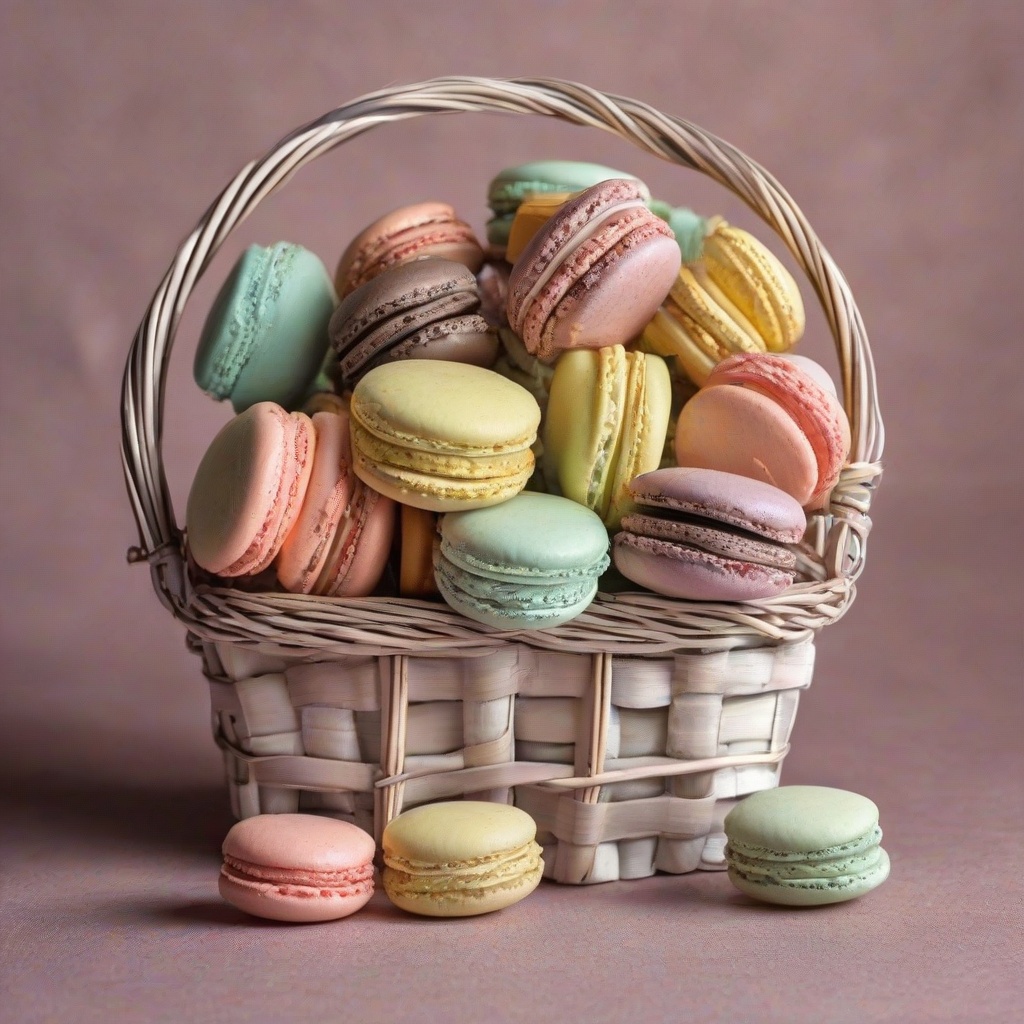}
                    };
                    \node[rotate=0, align=center, font=\normalsize, anchor=base, yshift=4pt] at (image.north) {+ \acrshort{method} (Ours) \strut};
                \end{tikzpicture}
                \caption*{\texttt{A basket of macaroons}}
            \end{subfigure}

            \begin{subfigure}{0.495\textwidth}
                \begin{tikzpicture}
                    \node[anchor=south west, inner sep=0] (image) at (0,0) {
                            \includegraphics[width=0.495\textwidth]{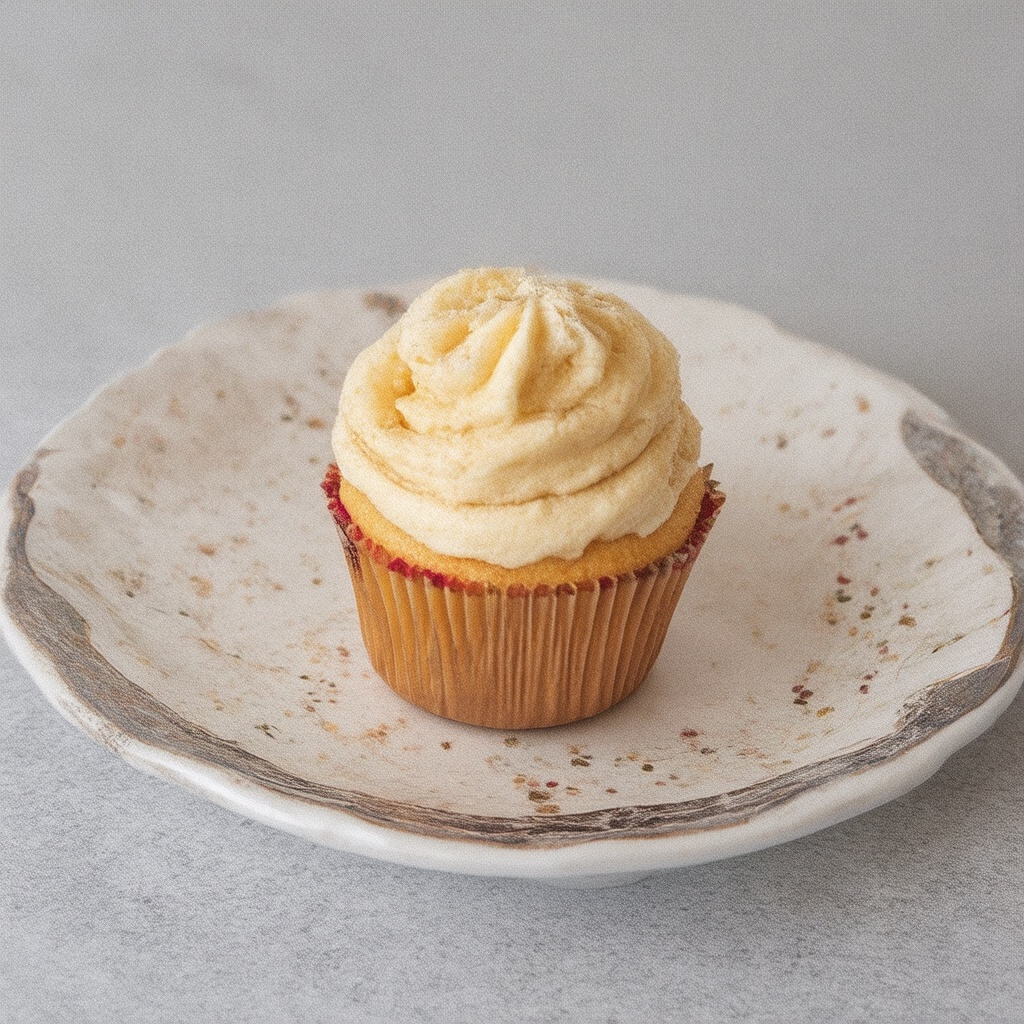}
                    };
                    \node[rotate=0, align=center, font=\normalsize, anchor=base, yshift=4pt] at (image.north) {CFG \strut};
                \end{tikzpicture}
                \hfil
                \begin{tikzpicture}
                    \node[anchor=south west, inner sep=0] (image) at (0,0) {
                \includegraphics[width=0.495\textwidth]{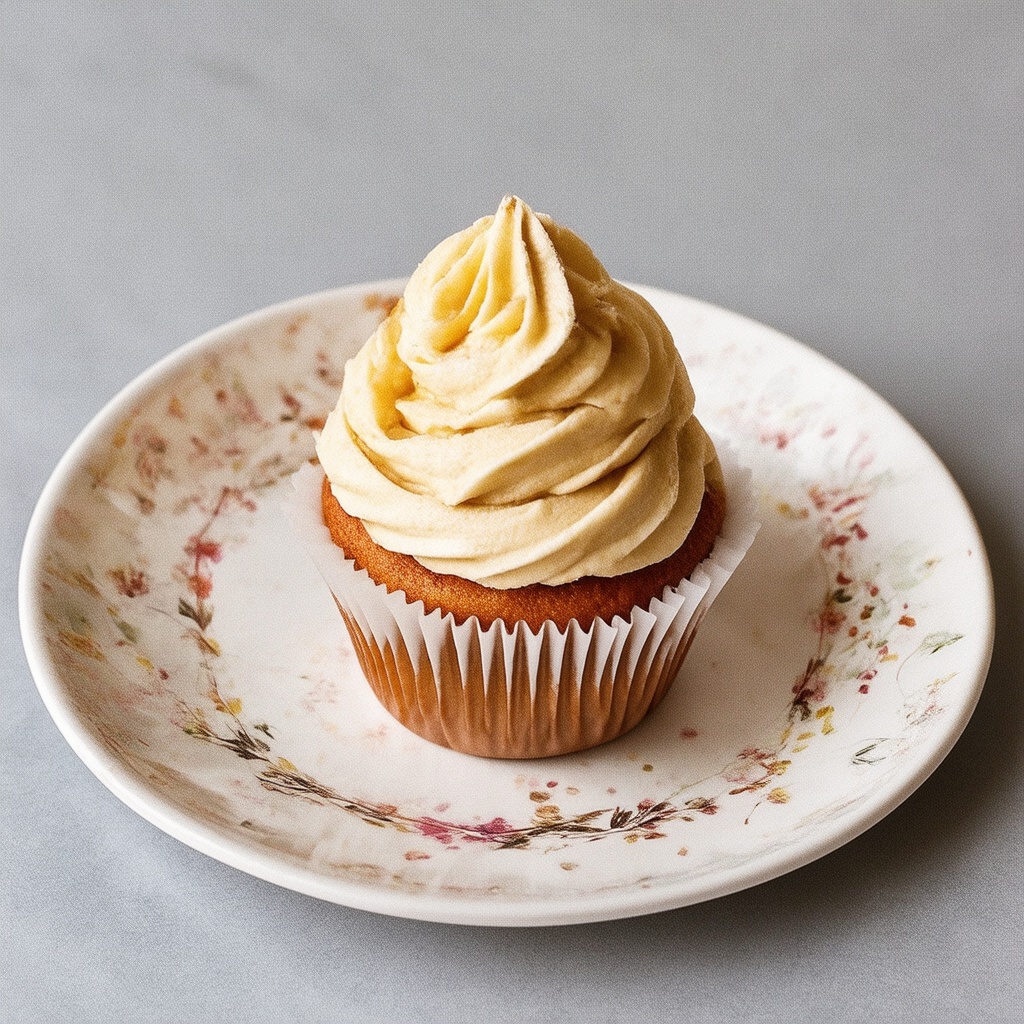}
                    };
                    \node[rotate=0, align=center, font=\normalsize, anchor=base, yshift=4pt] at (image.north) {+ \acrshort{method} (Ours) \strut};
                \end{tikzpicture}
                \caption*{\texttt{Close-up photo of a cupcake}}
            \end{subfigure}
            \hfil
            \begin{subfigure}{0.495\textwidth}
                \begin{tikzpicture}
                    \node[anchor=south west, inner sep=0] (image) at (0,0) {
                            \includegraphics[width=0.495\textwidth]{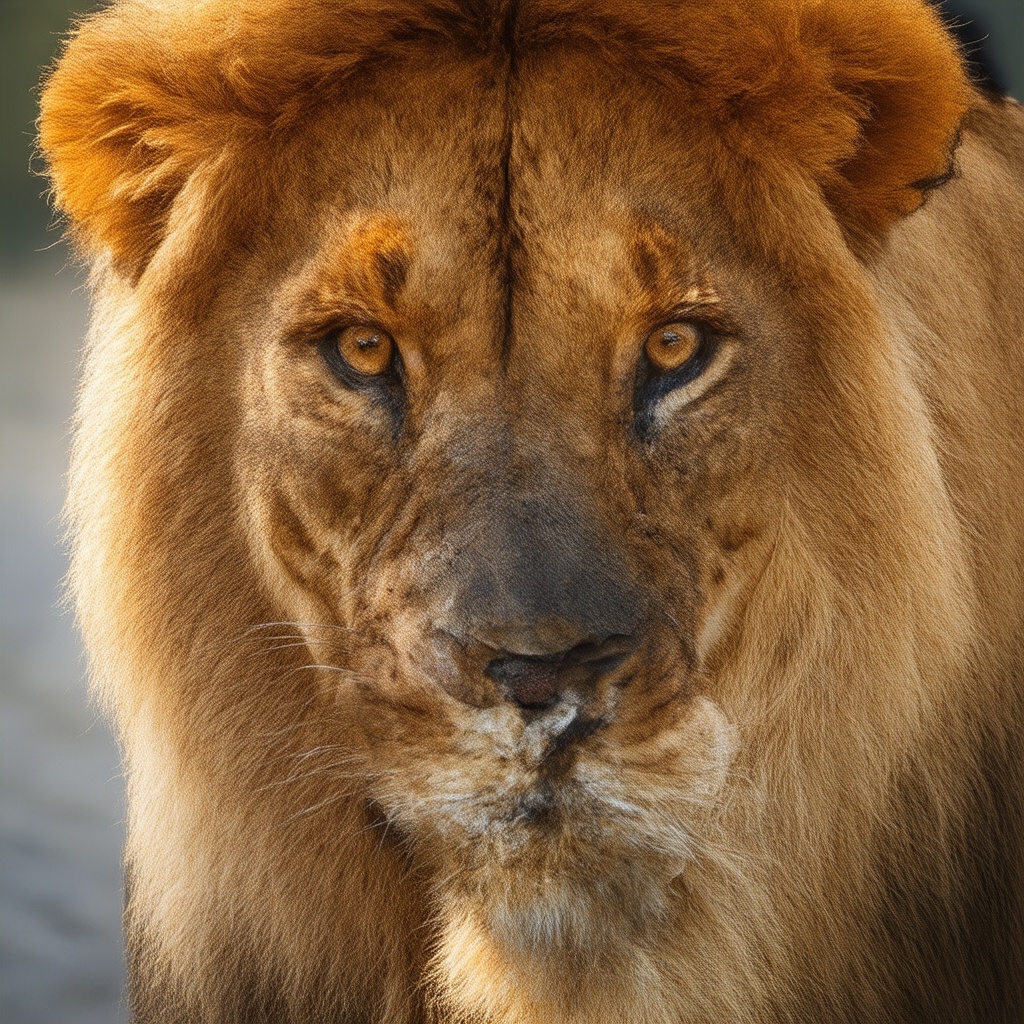}
                    };
                    \node[rotate=0, align=center, font=\normalsize, anchor=base, yshift=4pt] at (image.north) {CFG \strut};
                \end{tikzpicture}
                \hfil
                \begin{tikzpicture}
                    \node[anchor=south west, inner sep=0] (image) at (0,0) {
                \includegraphics[width=0.495\textwidth]{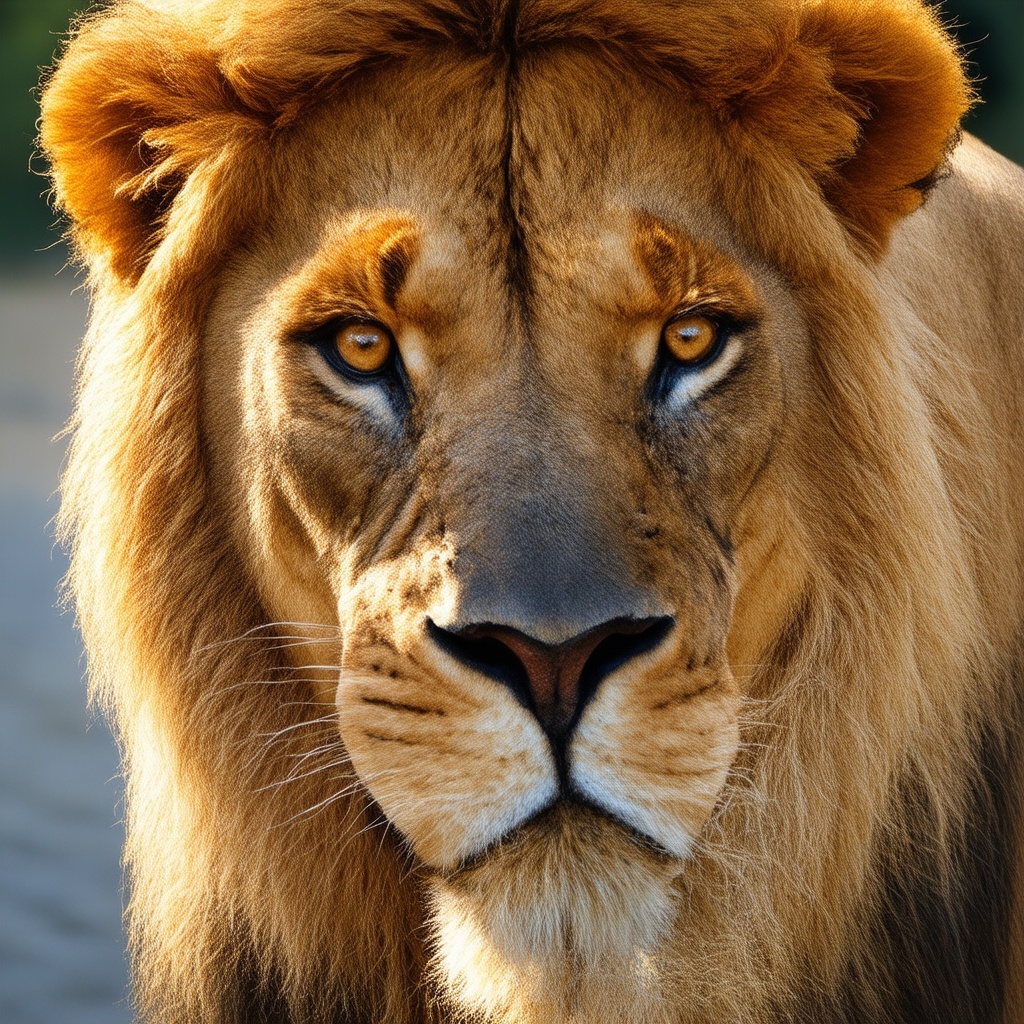}
                    };
                    \node[rotate=0, align=center, font=\normalsize, anchor=base, yshift=4pt] at (image.north) {+ \acrshort{method} (Ours) \strut};
                \end{tikzpicture}
                \caption*{\texttt{Detailed photo of a lion}}
            \end{subfigure}
        \end{minipage}
        \caption{Comparison of standard sampling and \gls{method} under fewer number \gls{nfe}. \gls{method} outputs exhibit significantly better global structure, sharpness, and details than the baseline.}
        \label{fig:qualitative-main-low-nfe}        
    \end{figure}
}

\newcommand{\figMainLowCFG}{
    \begin{figure}[t]
        \centering
        \begin{minipage}{\textwidth}
            \begin{subfigure}{0.495\textwidth}
                \begin{tikzpicture}
                    \node[anchor=south west, inner sep=0] (image) at (0,0) {
                            \includegraphics[width=0.495\textwidth]{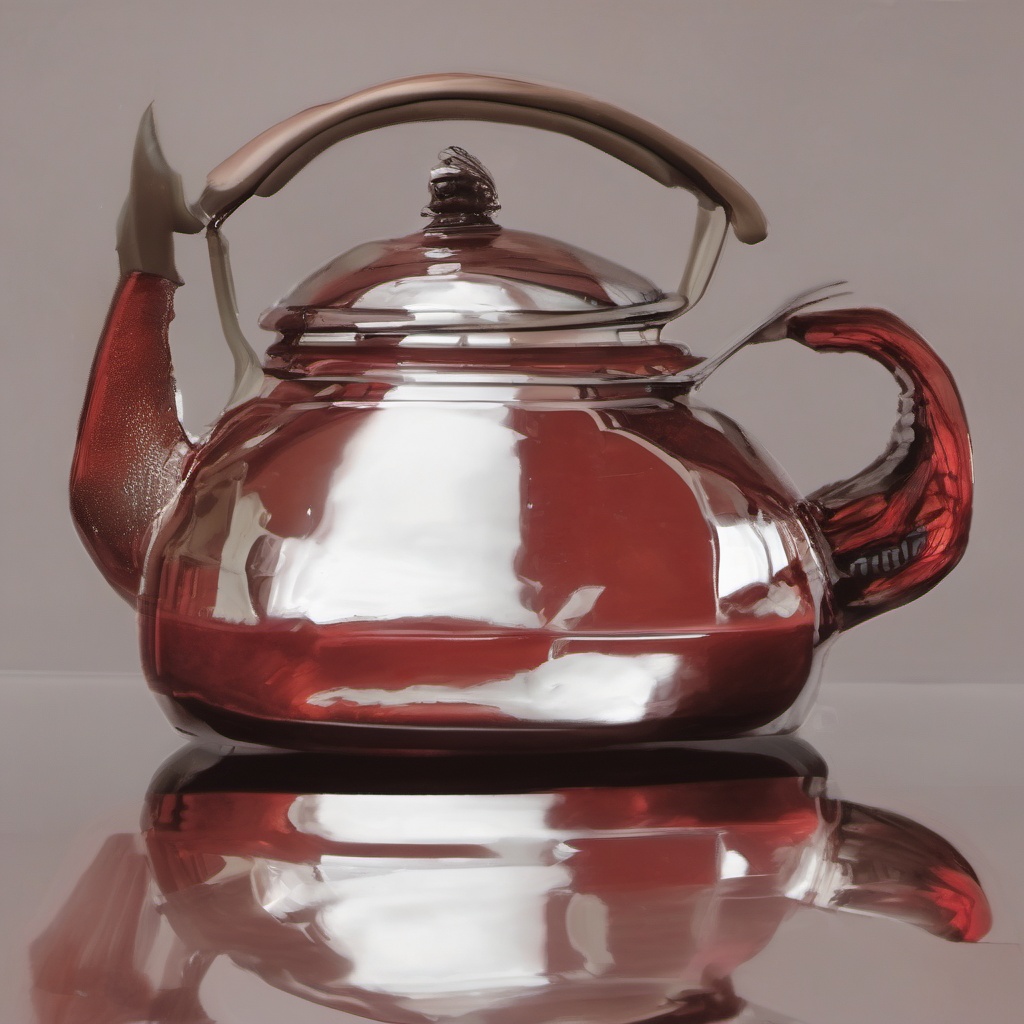}
                    };
                    \node[rotate=0, align=center, font=\normalsize, anchor=base, yshift=4pt] at (image.north) {CFG \strut};
                \end{tikzpicture}
                \hfil
                \begin{tikzpicture}
                    \node[anchor=south west, inner sep=0] (image) at (0,0) {
                \includegraphics[width=0.495\textwidth]{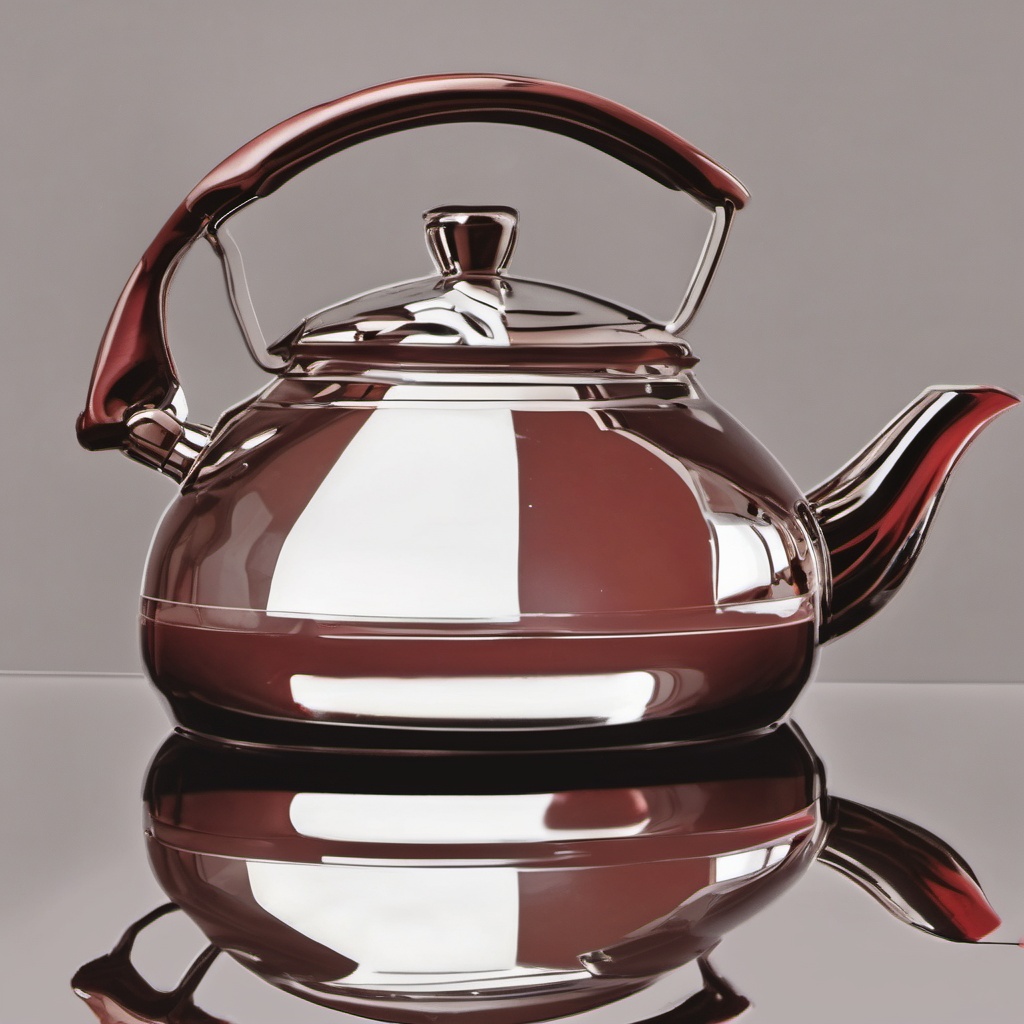}
                    };
                    \node[rotate=0, align=center, font=\normalsize, anchor=base, yshift=4pt] at (image.north) {+ \acrshort{method} (Ours) \strut};
                \end{tikzpicture}
                \caption*{\texttt{Glossy red teapot on a reflective surface}}
            \end{subfigure}
            \hfil
            \begin{subfigure}{0.495\textwidth}
                \begin{tikzpicture}
                    \node[anchor=south west, inner sep=0] (image) at (0,0) {
                            \includegraphics[width=0.495\textwidth]{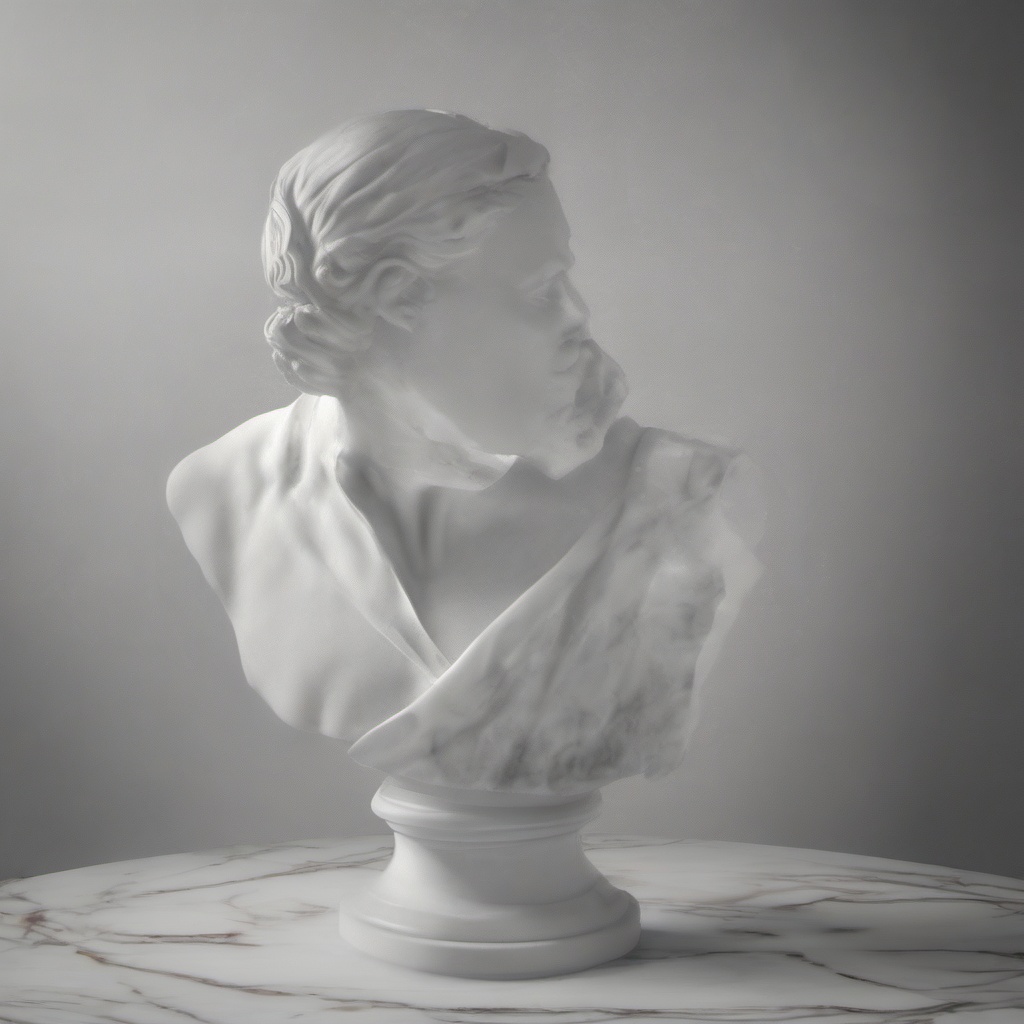}
                    };
                    \node[rotate=0, align=center, font=\normalsize, anchor=base, yshift=4pt] at (image.north) {CFG \strut};
                \end{tikzpicture}
                \hfil
                \begin{tikzpicture}
                    \node[anchor=south west, inner sep=0] (image) at (0,0) {
                \includegraphics[width=0.495\textwidth]{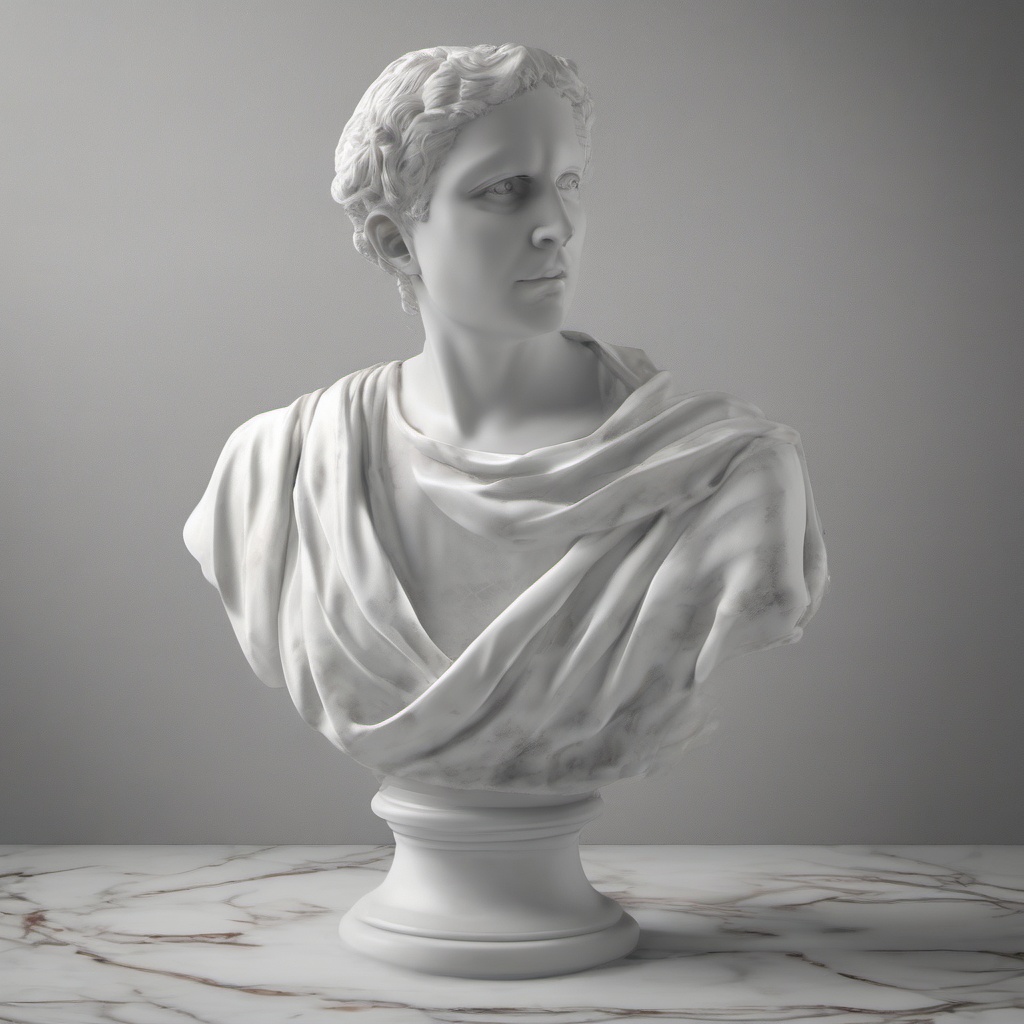}
                    };
                    \node[rotate=0, align=center, font=\normalsize, anchor=base, yshift=4pt] at (image.north) {+ \acrshort{method} (Ours) \strut};
                \end{tikzpicture}
                \caption*{\texttt{Photoreal marble bust, studio render lighting}}
            \end{subfigure}

            \begin{subfigure}[b]{0.495\textwidth}
                \begin{tikzpicture}
                    \node[anchor=south west, inner sep=0] (image) at (0,0) {
                            \includegraphics[width=0.495\textwidth]{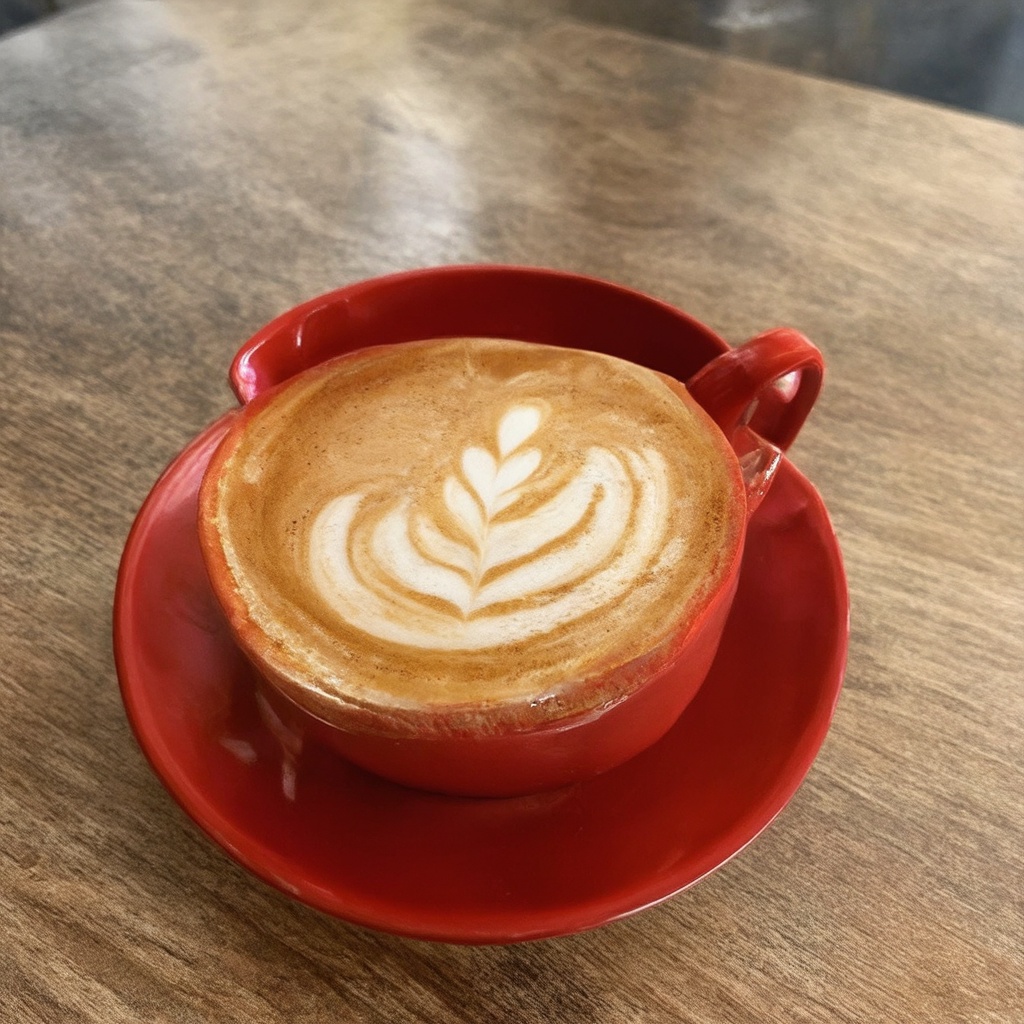}
                    };
                    \node[rotate=0, align=center, font=\normalsize, anchor=base, yshift=4pt] at (image.north) {CFG \strut};
                \end{tikzpicture}
                \hfil
                \begin{tikzpicture}
                    \node[anchor=south west, inner sep=0] (image) at (0,0) {
                \includegraphics[width=0.495\textwidth]{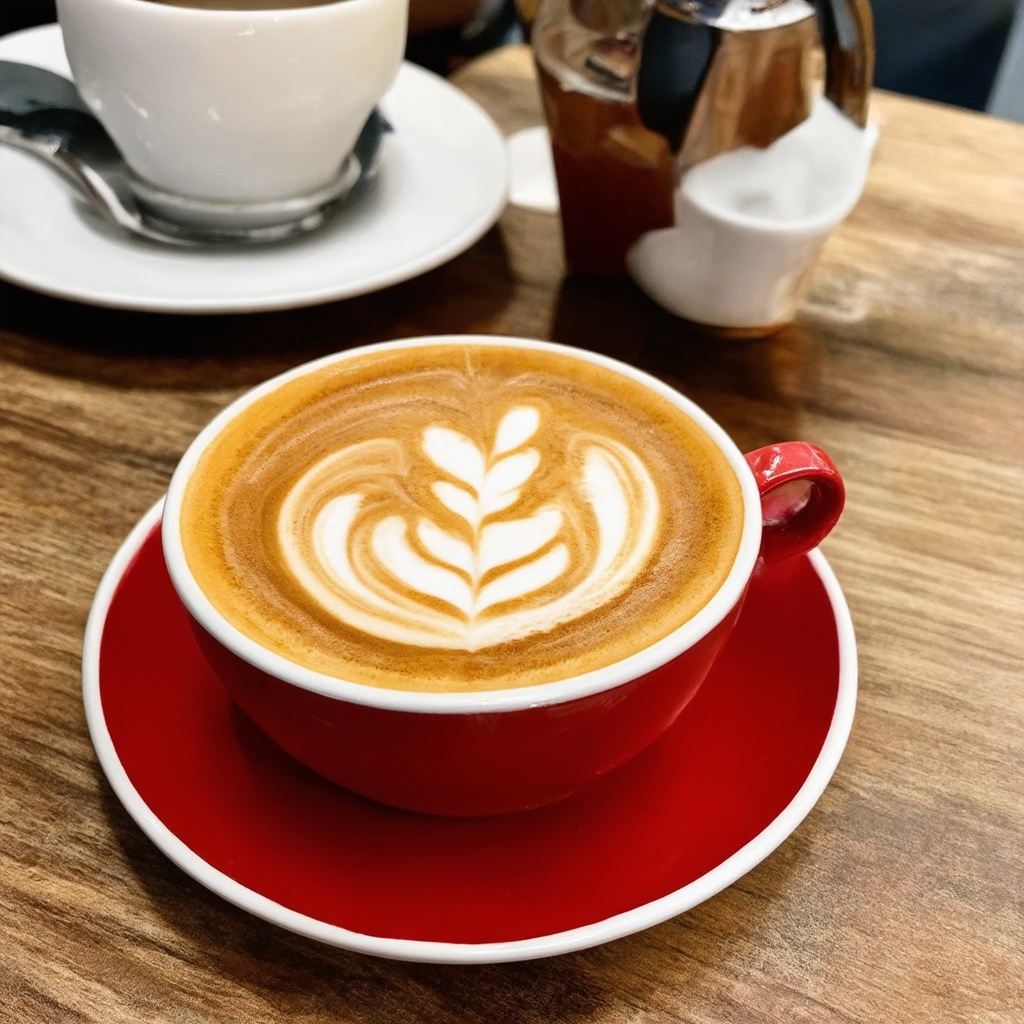}
                    };
                    \node[rotate=0, align=center, font=\normalsize, anchor=base, yshift=4pt] at (image.north) {+ \acrshort{method} (Ours) \strut};
                \end{tikzpicture}
                \caption*{\texttt{Cafe latte in round red cup and saucer}}
            \end{subfigure}
            \hfil
            \begin{subfigure}[b]{0.495\textwidth}
                \begin{tikzpicture}
                    \node[anchor=south west, inner sep=0] (image) at (0,0) {
                            \includegraphics[width=0.495\textwidth]{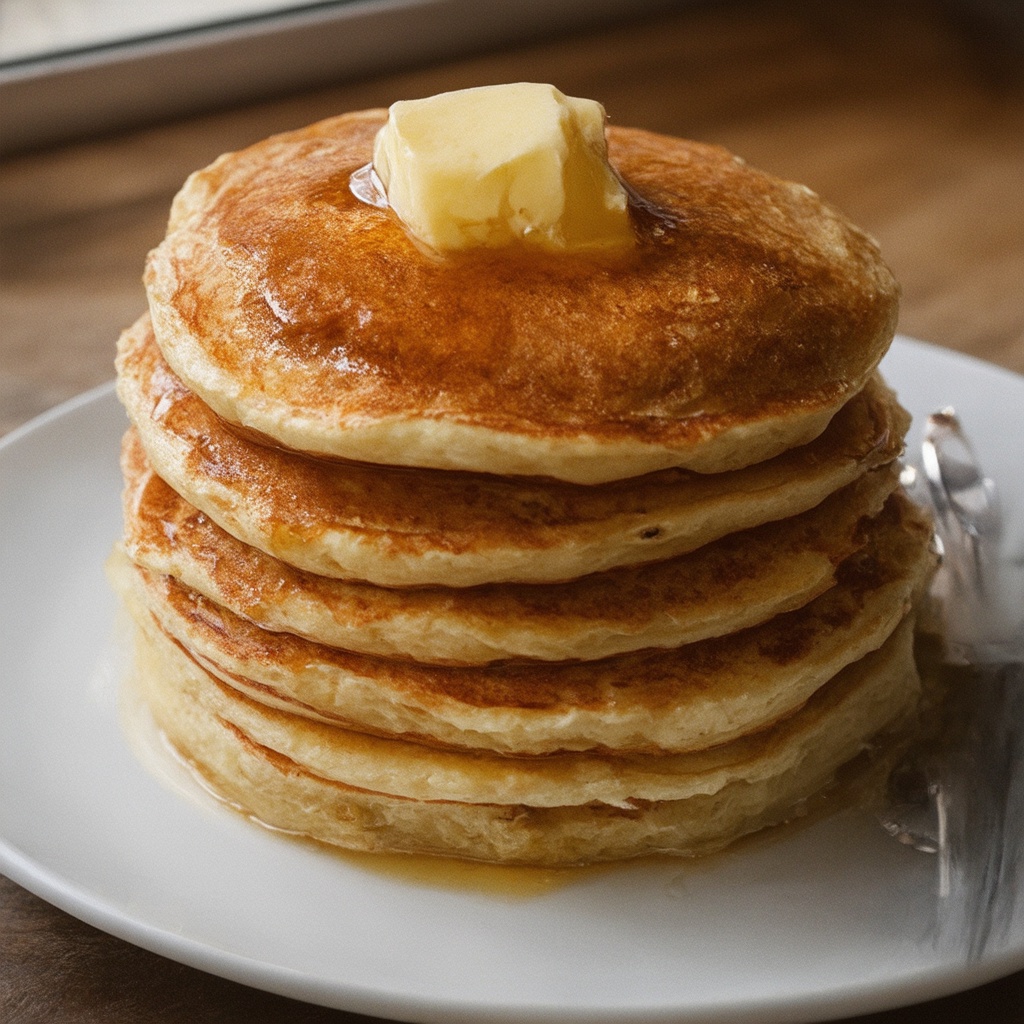}
                    };
                    \node[rotate=0, align=center, font=\normalsize, anchor=base, yshift=4pt] at (image.north) {CFG \strut};
                \end{tikzpicture}
                \hfil
                \begin{tikzpicture}
                    \node[anchor=south west, inner sep=0] (image) at (0,0) {
                \includegraphics[width=0.495\textwidth]{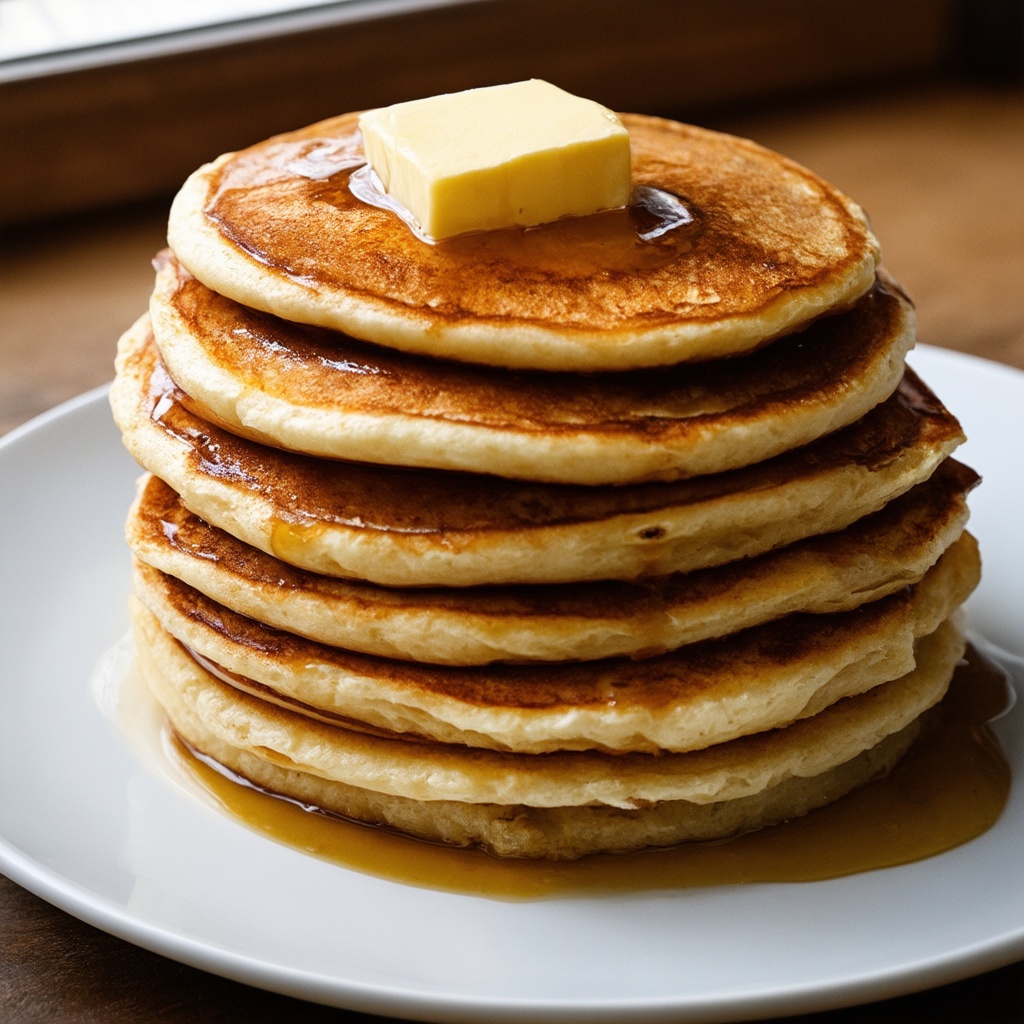}
                    };
                    \node[rotate=0, align=center, font=\normalsize, anchor=base, yshift=4pt] at (image.north) {+ \acrshort{method} (Ours) \strut};
                \end{tikzpicture}
                \caption*{\texttt{A stack of buttermilk pancakes}}
            \end{subfigure}
        \end{minipage}
        \caption{Illustration of the effect of \gls{method} on generated outputs under lower guidance scales. \gls{method} leads to noticeable improvements in image quality and details compared to the baseline.}
        \label{fig:qualitative-main-low-cfg}        
    \end{figure}
}

\newcommand{\figAblationDCT}{
    \begin{figure}[t!]
        \centering
        \begin{minipage}{\textwidth}
            \begin{subfigure}{0.495\textwidth}
                \begin{tikzpicture}
                    \node[anchor=south west, inner sep=0] (image) at (0,0) {
                            \includegraphics[width=0.495\textwidth]{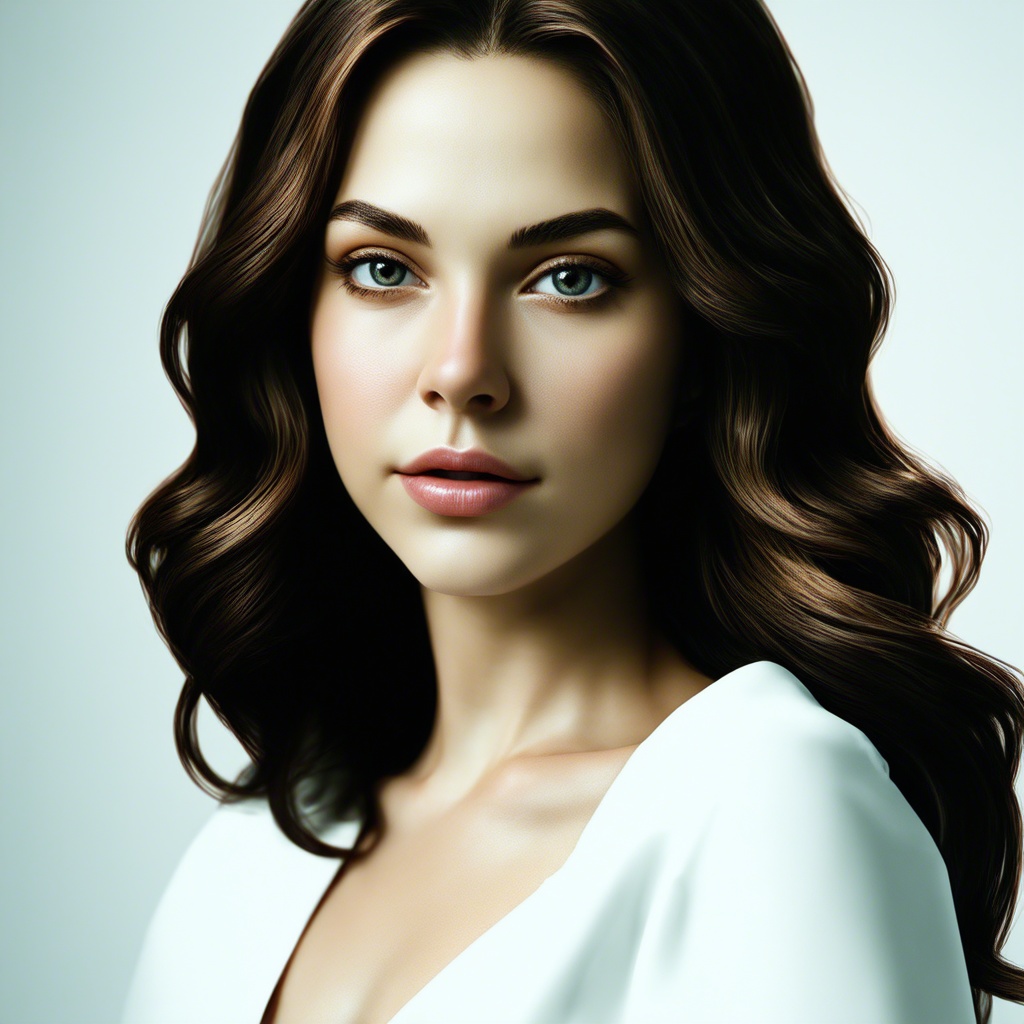}
                    };
                    \node[rotate=0, align=center, font=\normalsize, anchor=north, yshift=0pt] at (image.south) {w/o DCT filtering \strut};
                \end{tikzpicture}
                \hfil
                \begin{tikzpicture}
                    \node[anchor=south west, inner sep=0] (image) at (0,0) {
                \includegraphics[width=0.495\textwidth]{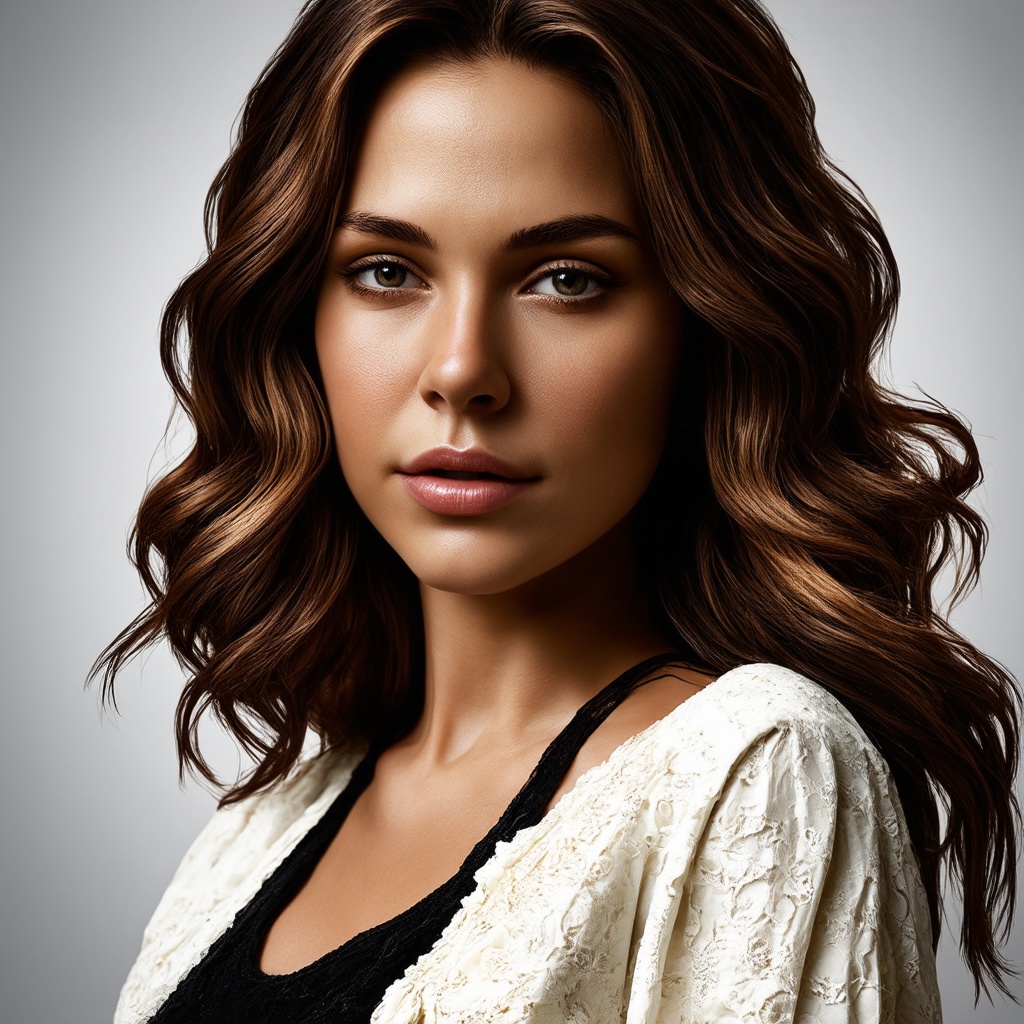}
                    };
                    \node[rotate=0, align=center, font=\normalsize, anchor=north, yshift=0pt] at (image.south) {w/ DCT filtering \strut};
                \end{tikzpicture}
            \end{subfigure}
            \hfil
            \begin{subfigure}{0.495\textwidth}
                \begin{tikzpicture}
                    \node[anchor=south west, inner sep=0] (image) at (0,0) {
                            \includegraphics[width=0.495\textwidth]{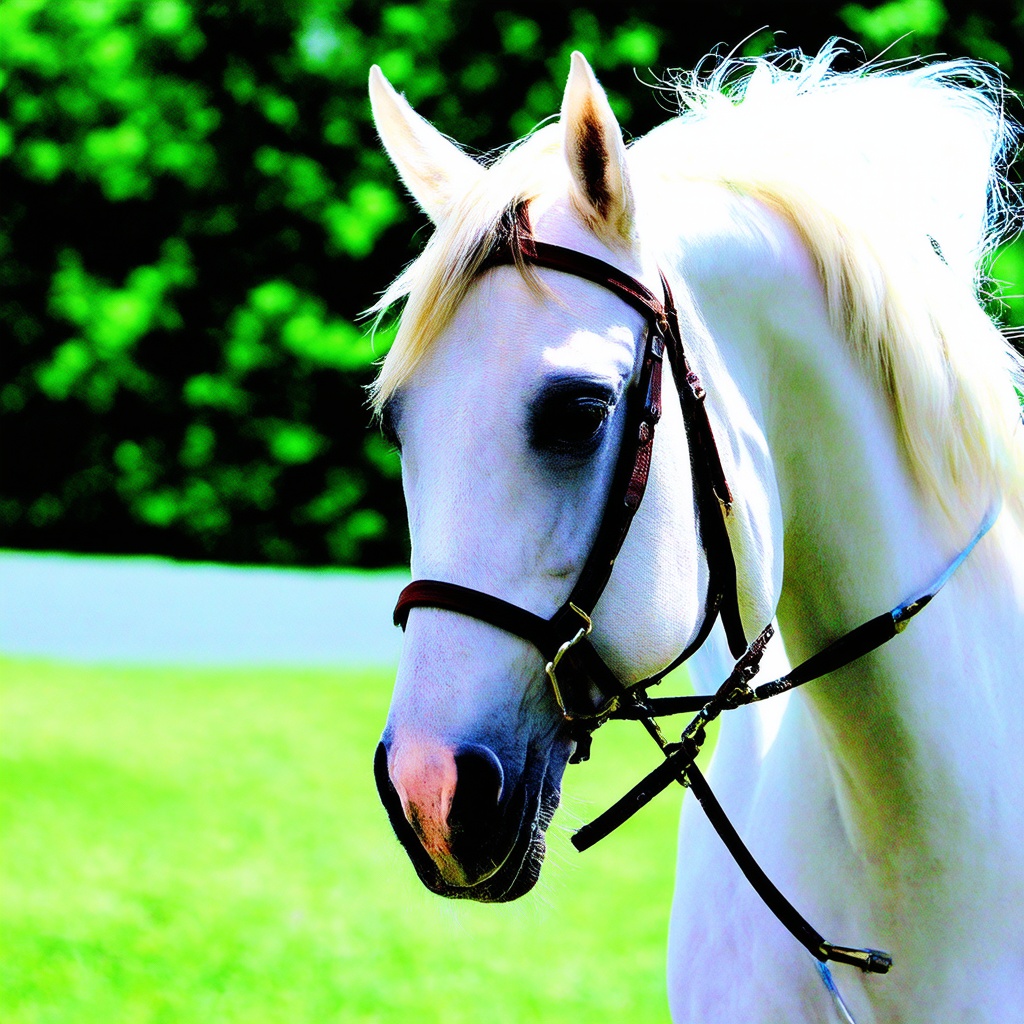}
                    };
                    \node[rotate=0, align=center, font=\normalsize, anchor=north, yshift=0pt] at (image.south) {w/o DCT filtering \strut};
                \end{tikzpicture}
                \hfil
                \begin{tikzpicture}
                    \node[anchor=south west, inner sep=0] (image) at (0,0) {
                \includegraphics[width=0.495\textwidth]{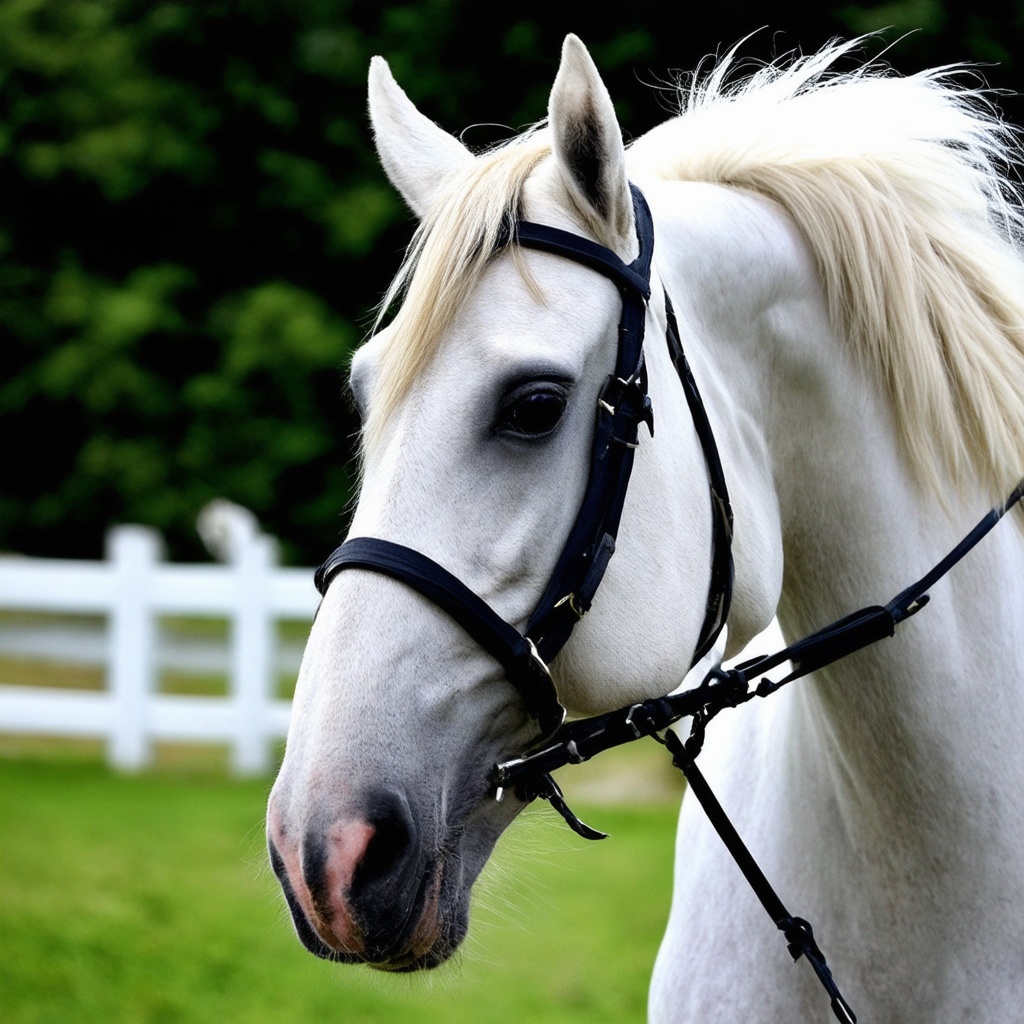}
                    };
                    \node[rotate=0, align=center, font=\normalsize, anchor=north, yshift=0pt] at (image.south) {w/ DCT filtering \strut};
                \end{tikzpicture}
            \end{subfigure}
        \end{minipage}
        \caption{Effect of frequency filtering on generation quality. Without DCT filtering, the results show unrealistic color compositions, which are corrected after applying the filtering operation.}
        \label{fig:ablation-filtering}        
    \end{figure}
}

\newcommand{\figAblationProjection}{
    \begin{figure}[t!]
        \centering
        \begin{minipage}{\textwidth}
            \begin{subfigure}{0.495\textwidth}
                \begin{tikzpicture}
                    \node[anchor=south west, inner sep=0] (image) at (0,0) {
                            \includegraphics[width=0.495\textwidth]{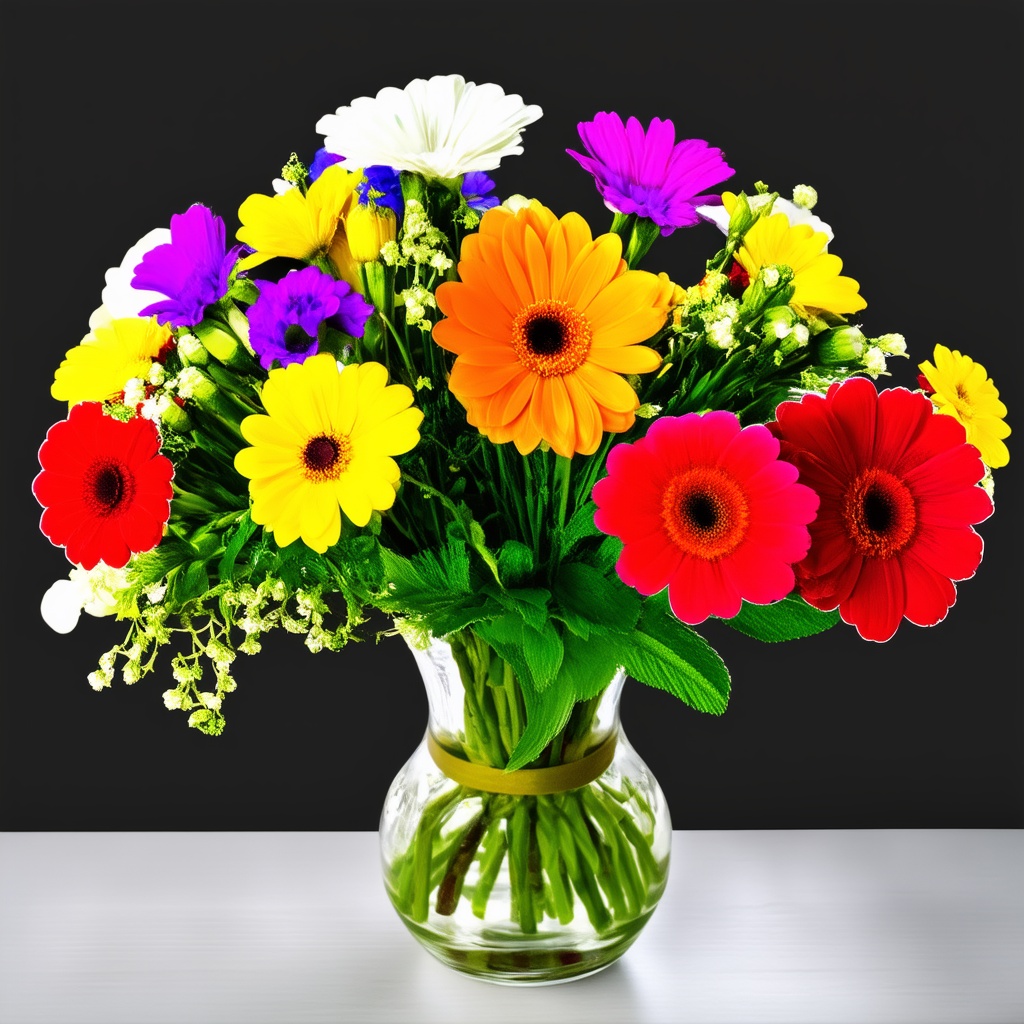}
                    };
                    \node[rotate=0, align=center, font=\normalsize, anchor=north, yshift=0pt] at (image.south) {w/o projection \strut};
                \end{tikzpicture}
                \hfil
                \begin{tikzpicture}
                    \node[anchor=south west, inner sep=0] (image) at (0,0) {
                \includegraphics[width=0.495\textwidth]{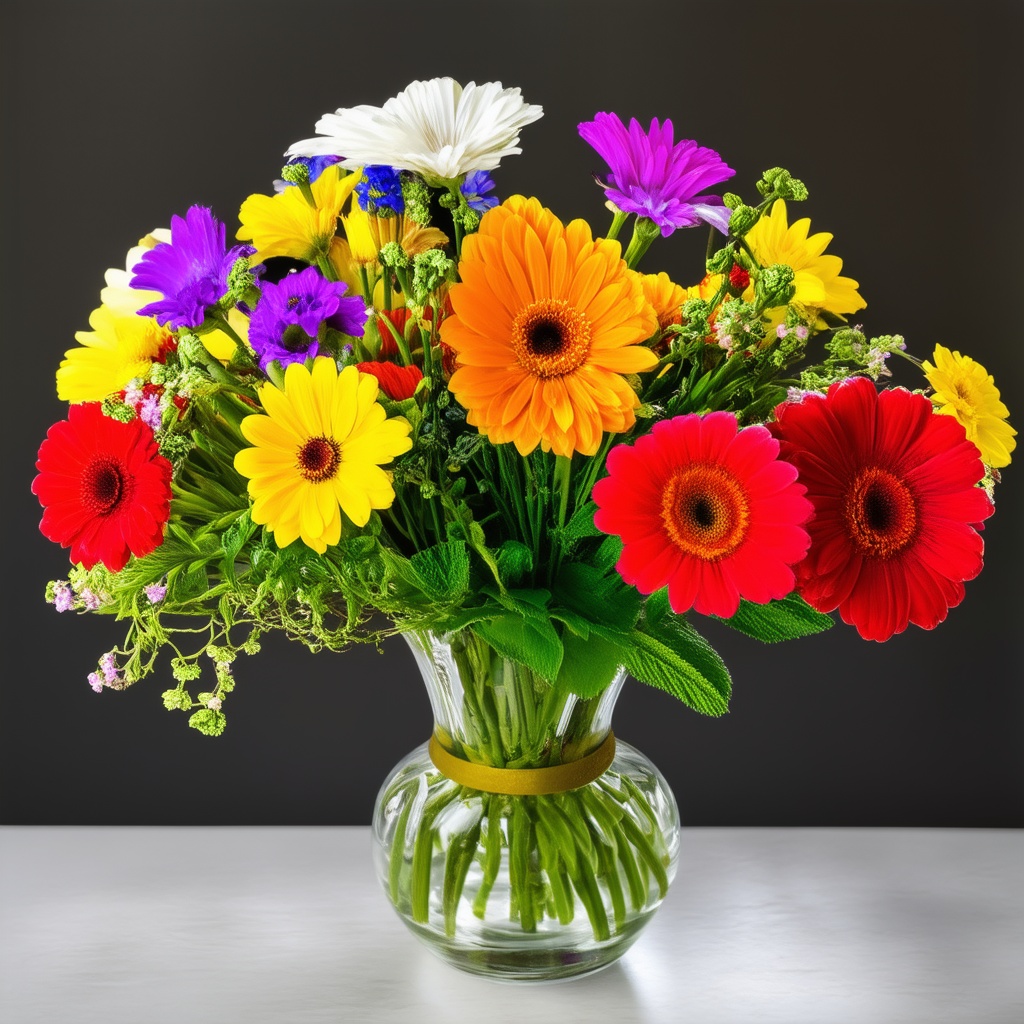}
                    };
                    \node[rotate=0, align=center, font=\normalsize, anchor=north, yshift=0pt] at (image.south) {w/ projection \strut};
                \end{tikzpicture}
            \end{subfigure}
            \hfil
            \begin{subfigure}{0.495\textwidth}
                \begin{tikzpicture}
                    \node[anchor=south west, inner sep=0] (image) at (0,0) {
                            \includegraphics[width=0.495\textwidth]{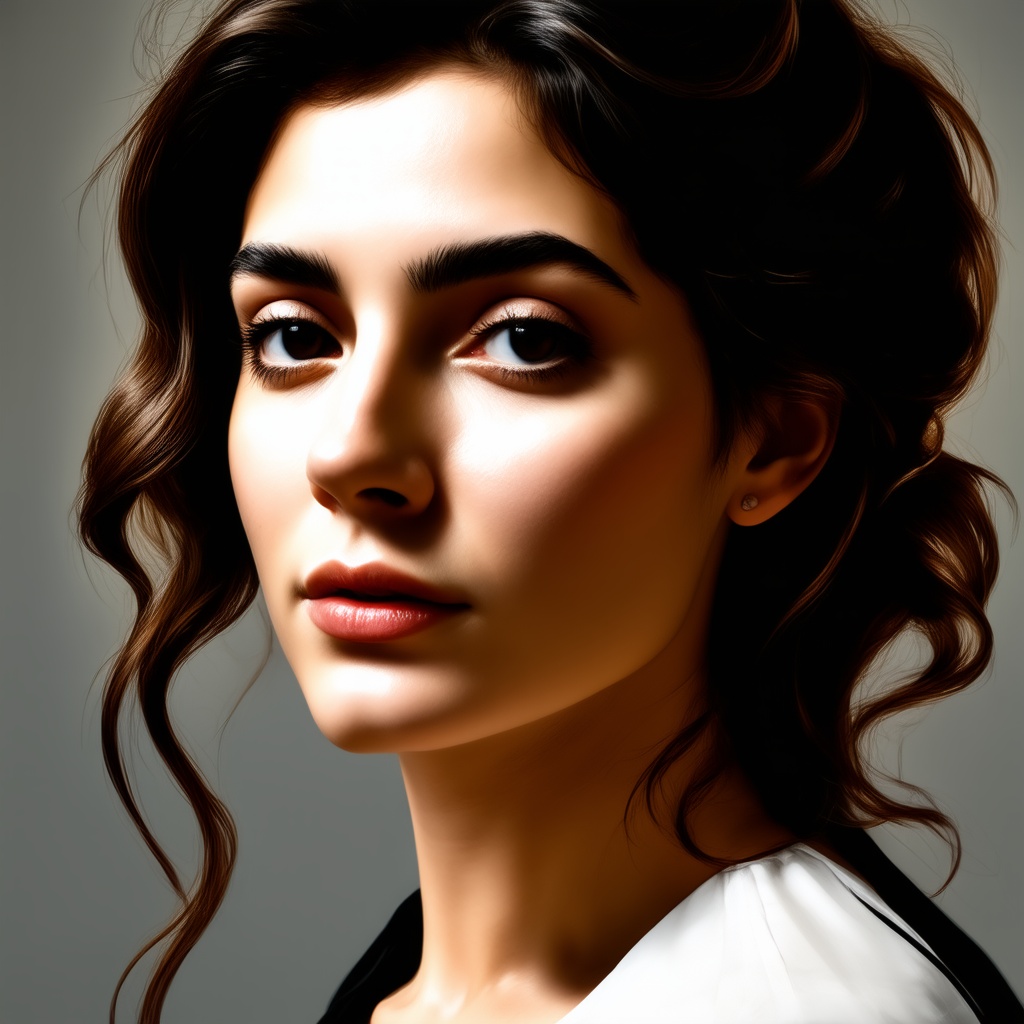}
                    };
                    \node[rotate=0, align=center, font=\normalsize, anchor=north, yshift=0pt] at (image.south) {w/o projection \strut};
                \end{tikzpicture}
                \hfil
                \begin{tikzpicture}
                    \node[anchor=south west, inner sep=0] (image) at (0,0) {
                \includegraphics[width=0.495\textwidth]{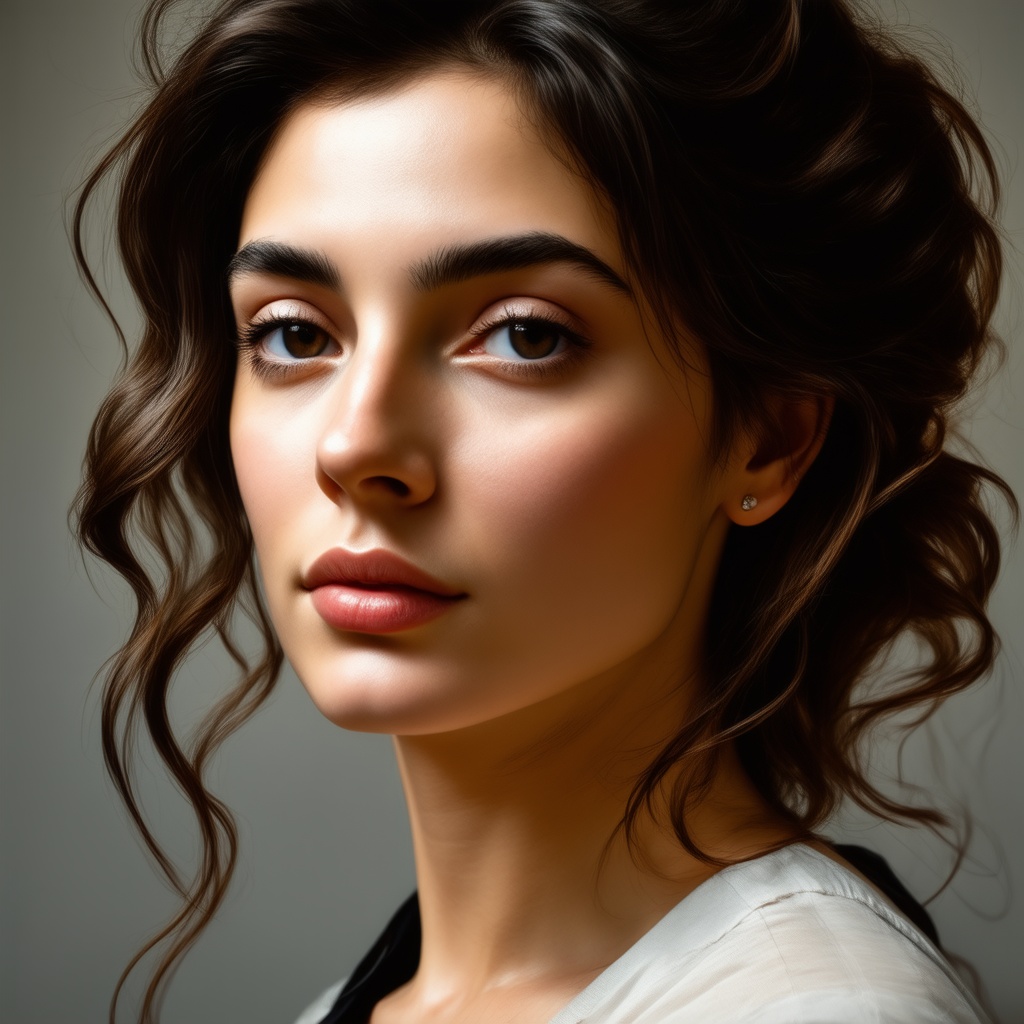}
                    };
                    \node[rotate=0, align=center, font=\normalsize, anchor=north, yshift=0pt] at (image.south) {w/ projection \strut};
                \end{tikzpicture}
            \end{subfigure}
        \end{minipage}
        \caption{Effect of using projection in \gls{method}. Without projection, the generations might produce oversaturated results, which can be mitigated by reducing the strength of the parallel component.}
        \label{fig:ablation-projection}        
    \end{figure}
}

\newcommand{\figREPAEAppendix}{
    \begin{figure}[t]
        \centering
        \begin{minipage}{\textwidth}
            \begin{subfigure}{\textwidth}
                \begin{tikzpicture}
                    \node[anchor=south west, inner sep=0] (image) at (0,0) {
                            \includegraphics[width=0.495\textwidth]{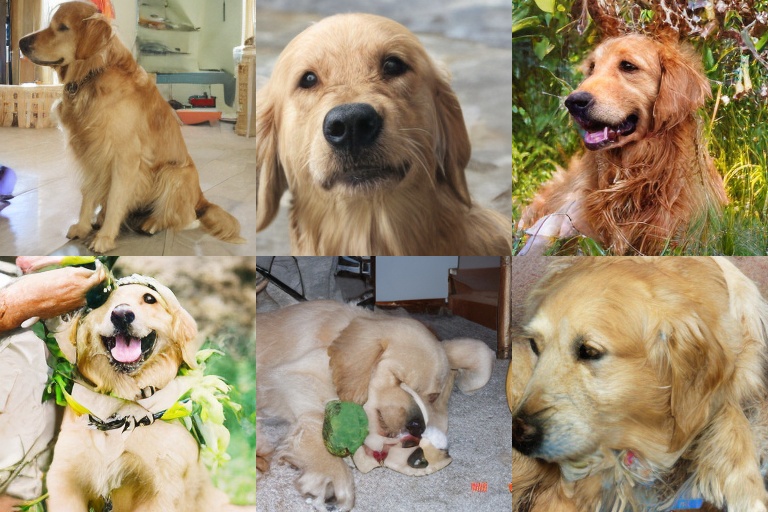}
                    };
                    \node[rotate=0, align=center, font=\normalsize, anchor=base, yshift=4pt] at (image.north) {CFG \strut};
                \end{tikzpicture}
                \hfil
                \begin{tikzpicture}
                    \node[anchor=south west, inner sep=0] (image) at (0,0) {
                \includegraphics[width=0.495\textwidth]{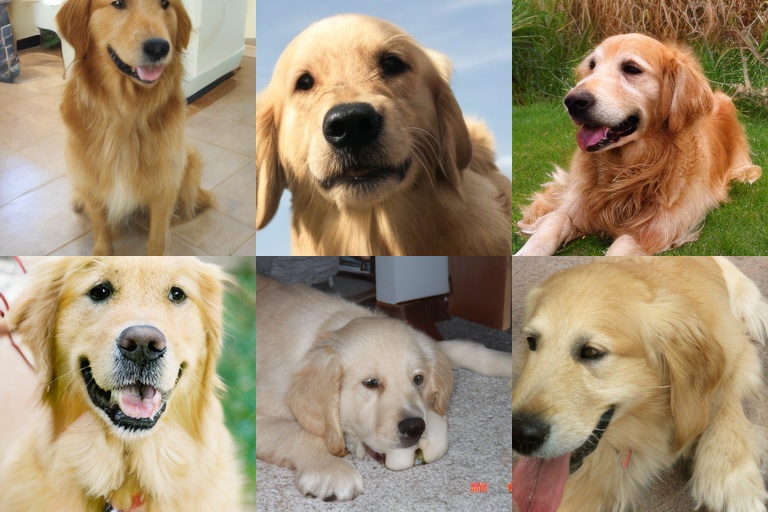}
                    };
                    \node[rotate=0, align=center, font=\normalsize, anchor=base, yshift=4pt] at (image.north) {+ \acrshort{method} (Ours) \strut};
                \end{tikzpicture}
            \end{subfigure}
            \hfil
            \begin{subfigure}{\textwidth}
                \begin{tikzpicture}
                    \node[anchor=south west, inner sep=0] (image) at (0,0) {
                            \includegraphics[width=0.495\textwidth]{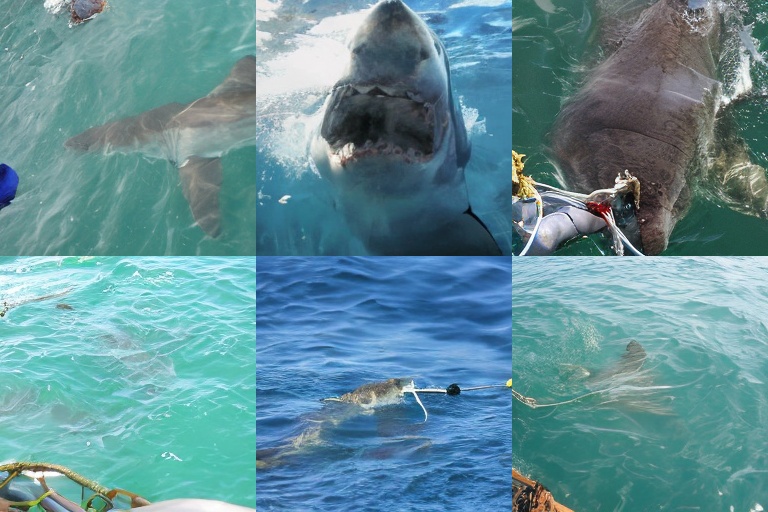}
                    };
                    \node[rotate=0, align=center, font=\normalsize, anchor=base, yshift=4pt] at (image.north) {CFG \strut};
                \end{tikzpicture}
                \hfil
                \begin{tikzpicture}
                    \node[anchor=south west, inner sep=0] (image) at (0,0) {
                \includegraphics[width=0.495\textwidth]{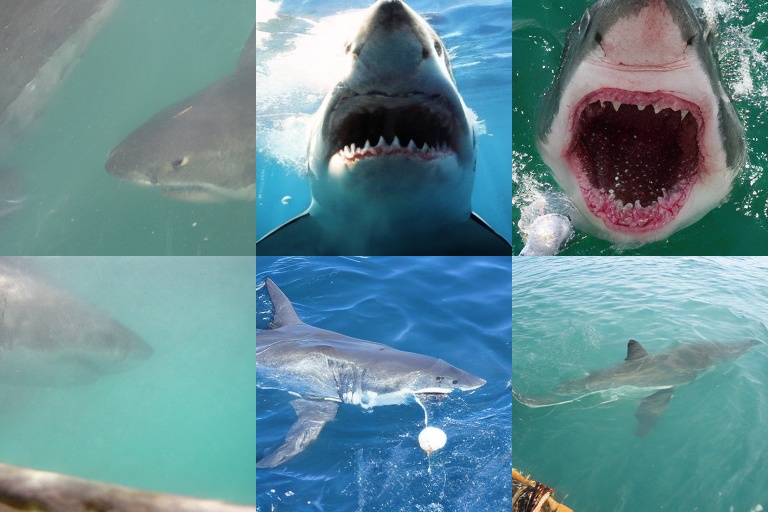}
                    };
                    \node[rotate=0, align=center, font=\normalsize, anchor=base, yshift=4pt] at (image.north) {+ \acrshort{method} (Ours) \strut};
                \end{tikzpicture}
            \end{subfigure}

            \begin{subfigure}[b]{\textwidth}
                \begin{tikzpicture}
                    \node[anchor=south west, inner sep=0] (image) at (0,0) {
                            \includegraphics[width=0.495\textwidth]{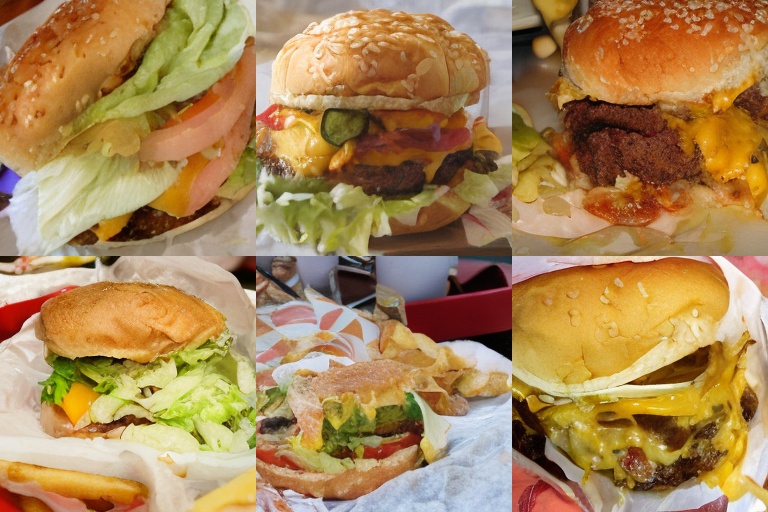}
                    };
                    \node[rotate=0, align=center, font=\normalsize, anchor=base, yshift=4pt] at (image.north) {CFG \strut};
                \end{tikzpicture}
                \hfil
                \begin{tikzpicture}
                    \node[anchor=south west, inner sep=0] (image) at (0,0) {
                \includegraphics[width=0.495\textwidth]{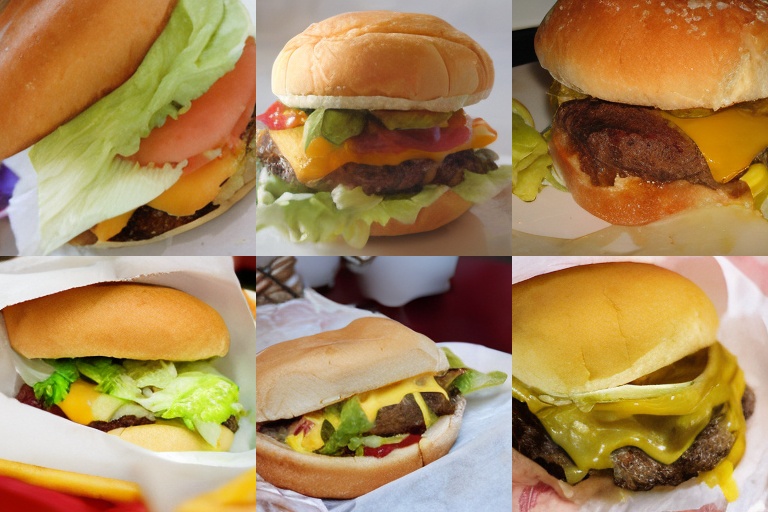}
                    };
                    \node[rotate=0, align=center, font=\normalsize, anchor=base, yshift=4pt] at (image.north) {+ \acrshort{method} (Ours) \strut};
                \end{tikzpicture}
            \end{subfigure}
            \hfil
            \begin{subfigure}[b]{\textwidth}
                \begin{tikzpicture}
                    \node[anchor=south west, inner sep=0] (image) at (0,0) {
                            \includegraphics[width=0.495\textwidth]{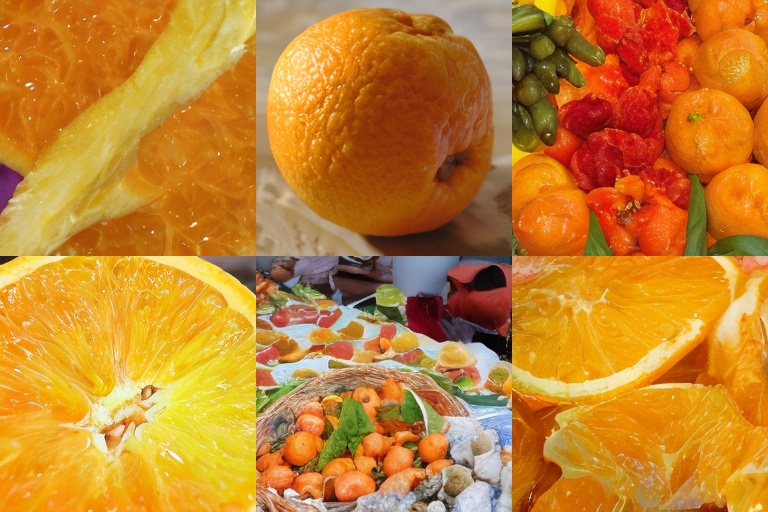}
                    };
                    \node[rotate=0, align=center, font=\normalsize, anchor=base, yshift=4pt] at (image.north) {CFG \strut};
                \end{tikzpicture}
                \hfil
                \begin{tikzpicture}
                    \node[anchor=south west, inner sep=0] (image) at (0,0) {
                \includegraphics[width=0.495\textwidth]{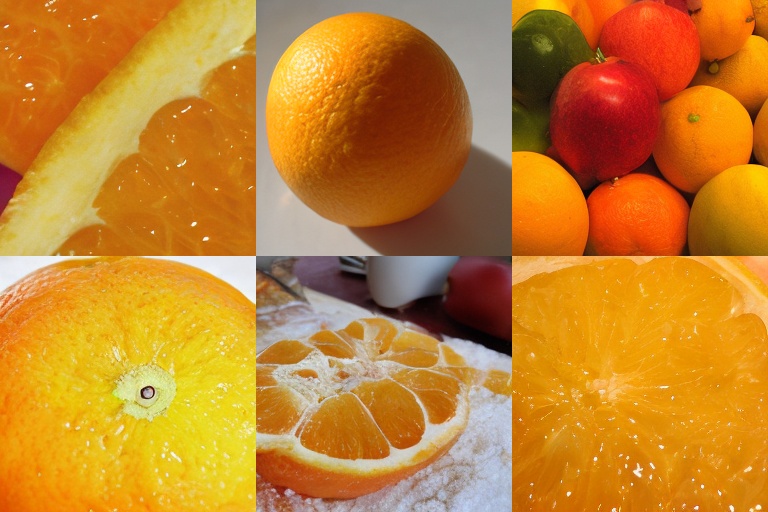}
                    };
                    \node[rotate=0, align=center, font=\normalsize, anchor=base, yshift=4pt] at (image.north) {+ \acrshort{method} (Ours) \strut};
                \end{tikzpicture}
            \end{subfigure}
        \end{minipage}
        \caption{Class-conditional generation with SiT-XL + REPA using 30 steps and $\wcfg=1.8$.}
        \label{fig:sit-appendix}        
    \end{figure}
}

\newcommand{\figAPGAppendix}{
    \begin{figure}[t]
        \centering
        \begin{minipage}{\textwidth}
            \begin{subfigure}{0.495\textwidth}
                \begin{tikzpicture}
                    \node[anchor=south west, inner sep=0] (image) at (0,0) {
                            \includegraphics[width=0.495\textwidth]{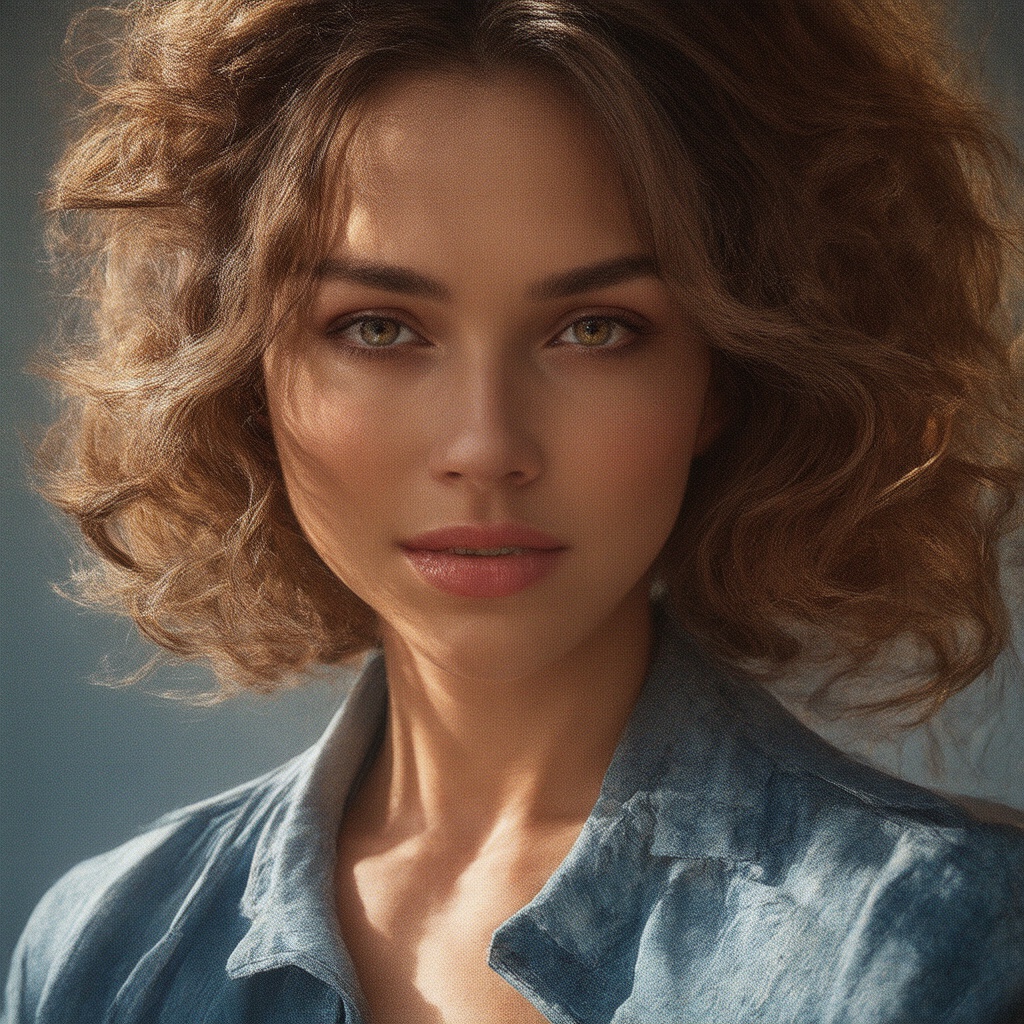}
                    };
                    \node[rotate=0, align=center, font=\normalsize, anchor=base, yshift=4pt] at (image.north) {APG \strut};
                \end{tikzpicture}
                \hfil
                \begin{tikzpicture}
                    \node[anchor=south west, inner sep=0] (image) at (0,0) {
                \includegraphics[width=0.495\textwidth]{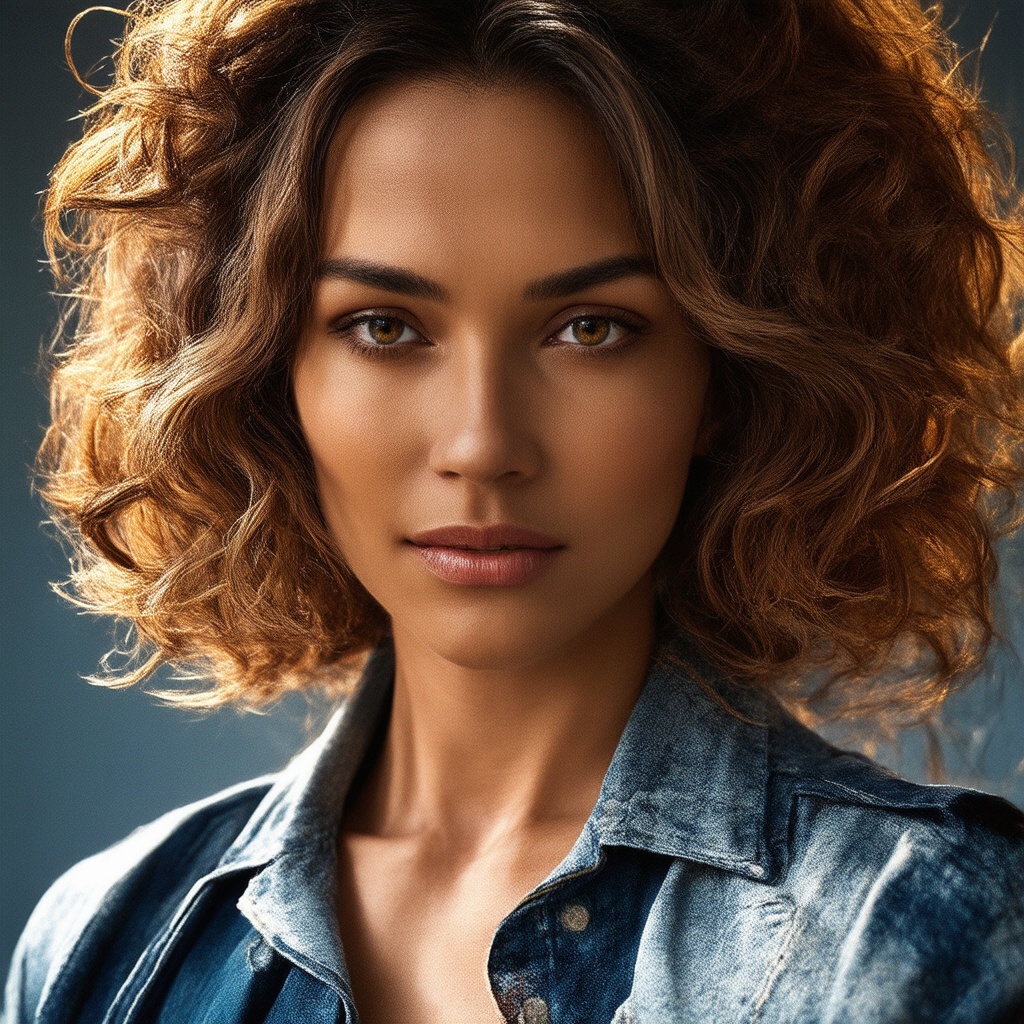}
                    };
                    \node[rotate=0, align=center, font=\normalsize, anchor=base, yshift=4pt] at (image.north) {+ \acrshort{method} (Ours) \strut};
                \end{tikzpicture}
            \end{subfigure}
            \hfil
            \begin{subfigure}{0.495\textwidth}
                \begin{tikzpicture}
                    \node[anchor=south west, inner sep=0] (image) at (0,0) {
                            \includegraphics[width=0.495\textwidth]{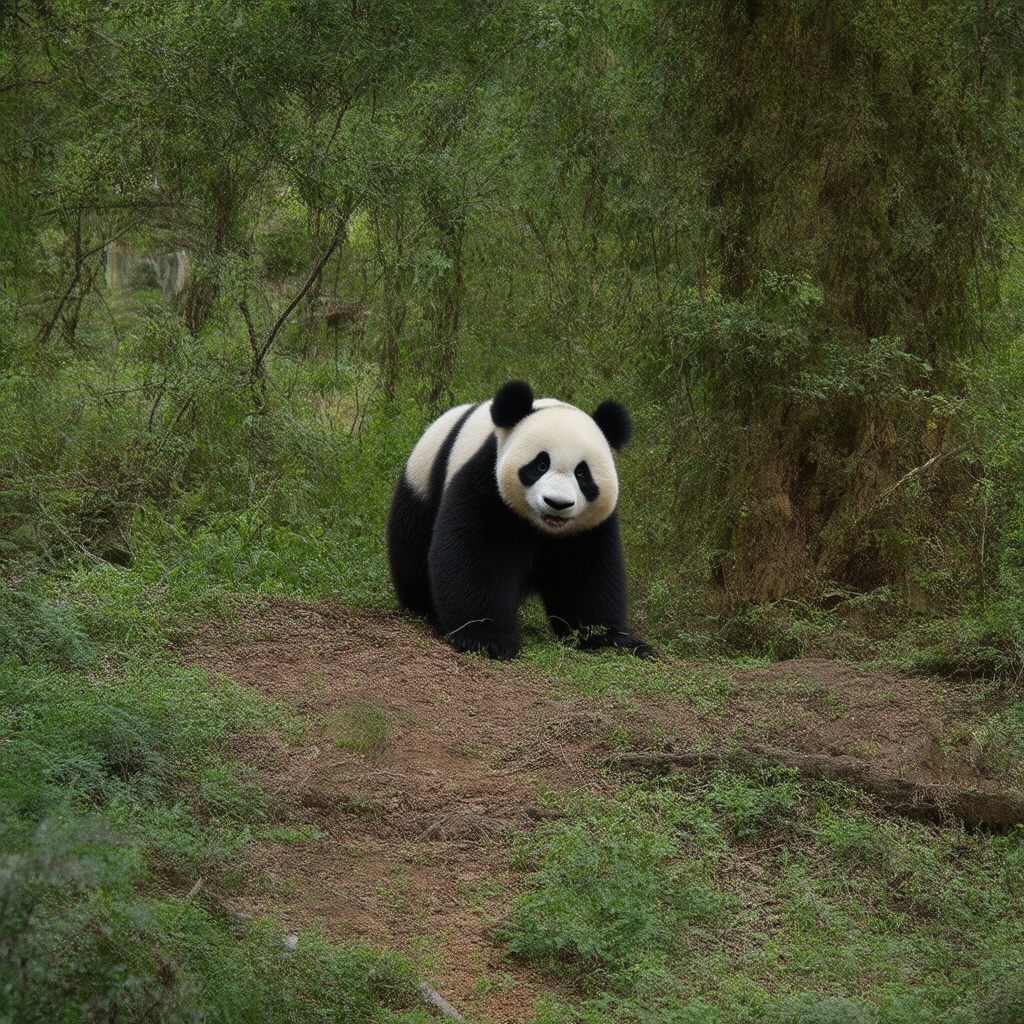}
                    };
                    \node[rotate=0, align=center, font=\normalsize, anchor=base, yshift=4pt] at (image.north) {APG \strut};
                \end{tikzpicture}
                \hfil
                \begin{tikzpicture}
                    \node[anchor=south west, inner sep=0] (image) at (0,0) {
                \includegraphics[width=0.495\textwidth]{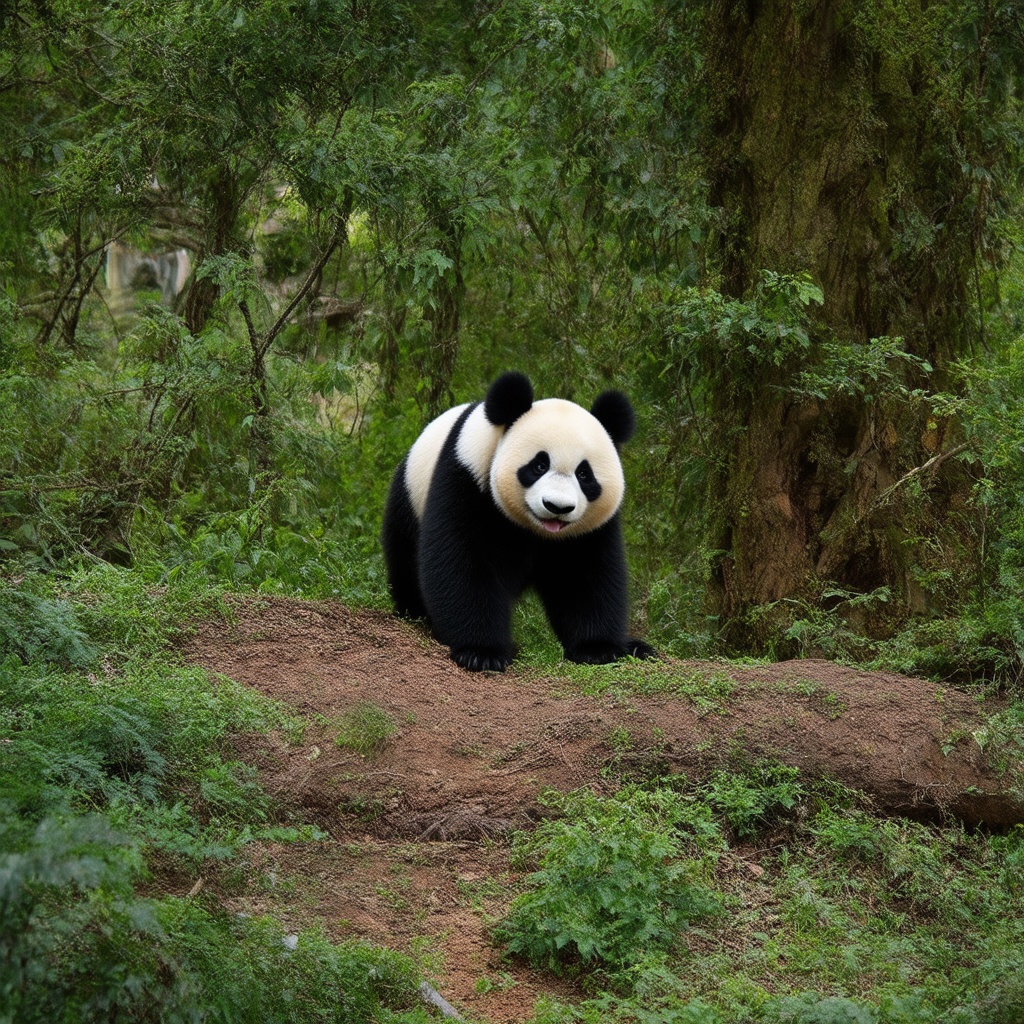}
                    };
                    \node[rotate=0, align=center, font=\normalsize, anchor=base, yshift=4pt] at (image.north) {+ \acrshort{method} (Ours) \strut};
                \end{tikzpicture}
            \end{subfigure}

            \begin{subfigure}{0.495\textwidth}
                \begin{tikzpicture}
                    \node[anchor=south west, inner sep=0] (image) at (0,0) {
                            \includegraphics[width=0.495\textwidth]{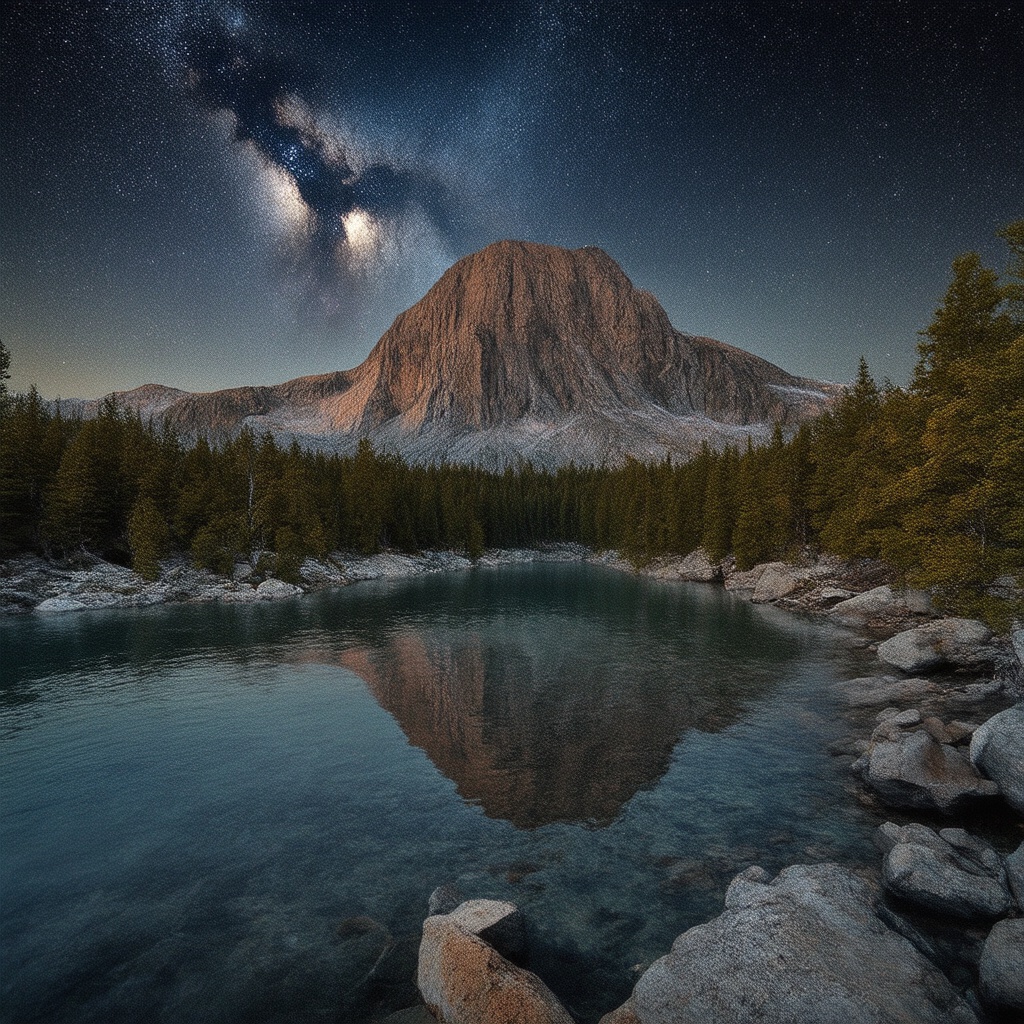}
                    };
                    \node[rotate=0, align=center, font=\normalsize, anchor=base, yshift=4pt] at (image.north) {APG \strut};
                \end{tikzpicture}
                \hfil
                \begin{tikzpicture}
                    \node[anchor=south west, inner sep=0] (image) at (0,0) {
                \includegraphics[width=0.495\textwidth]{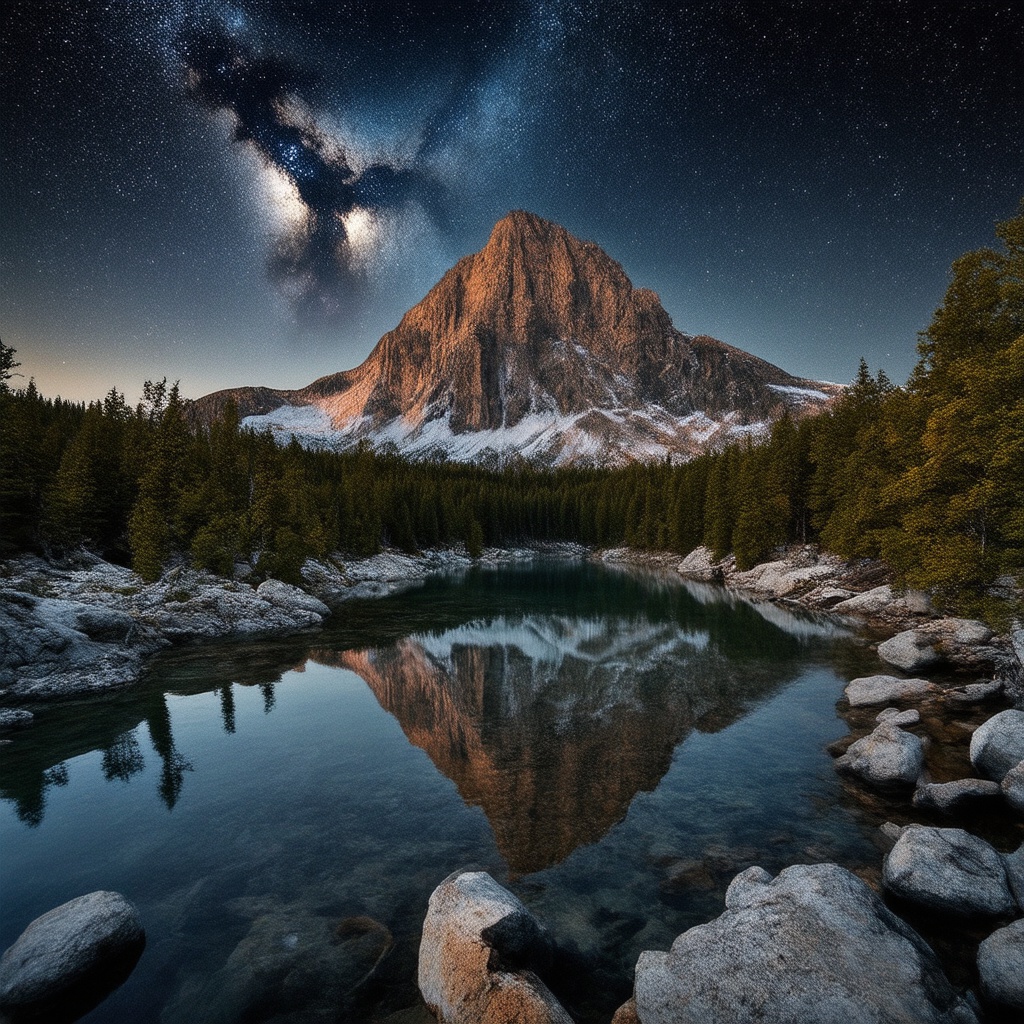}
                    };
                    \node[rotate=0, align=center, font=\normalsize, anchor=base, yshift=4pt] at (image.north) {+ \acrshort{method} (Ours) \strut};
                \end{tikzpicture}
            \end{subfigure}
            \hfil
            \begin{subfigure}{0.495\textwidth}
                \begin{tikzpicture}
                    \node[anchor=south west, inner sep=0] (image) at (0,0) {
                            \includegraphics[width=0.495\textwidth]{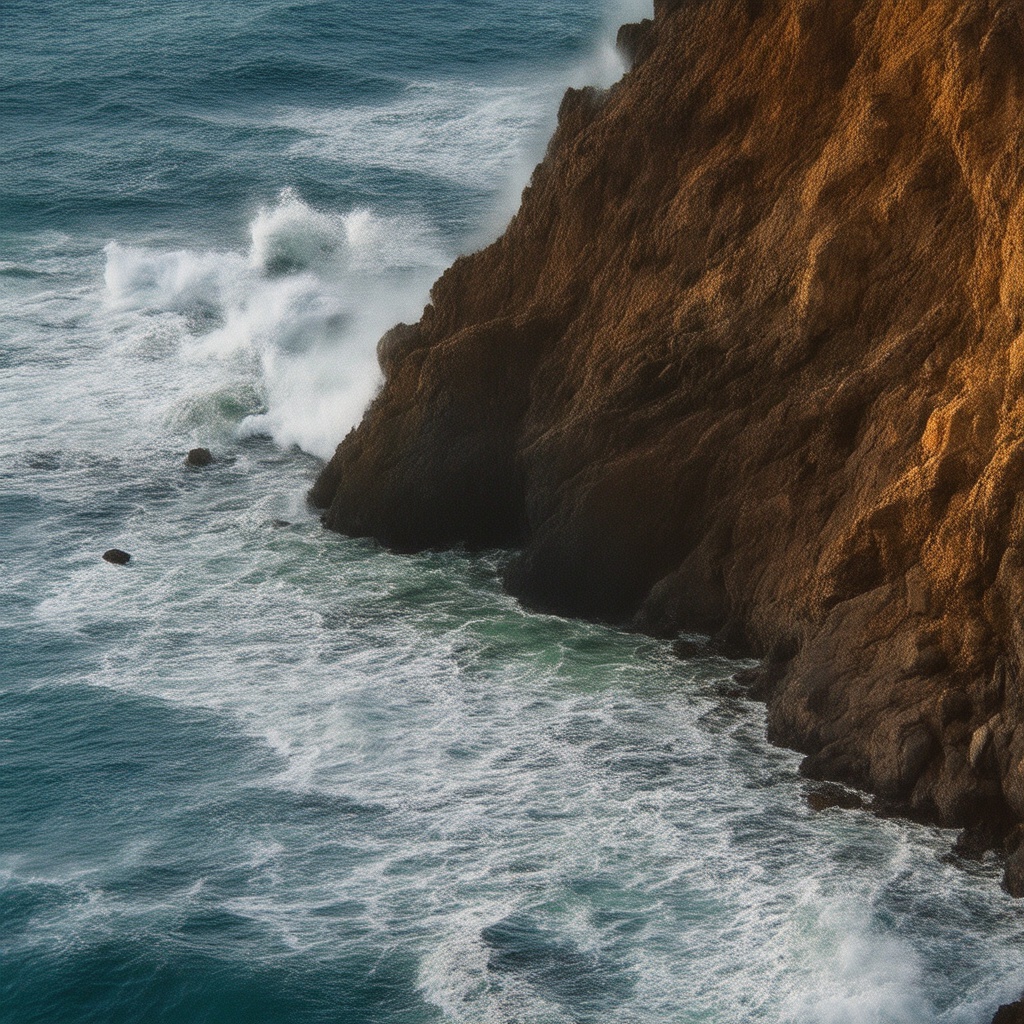}
                    };
                    \node[rotate=0, align=center, font=\normalsize, anchor=base, yshift=4pt] at (image.north) {APG \strut};
                \end{tikzpicture}
                \hfil
                \begin{tikzpicture}
                    \node[anchor=south west, inner sep=0] (image) at (0,0) {
                \includegraphics[width=0.495\textwidth]{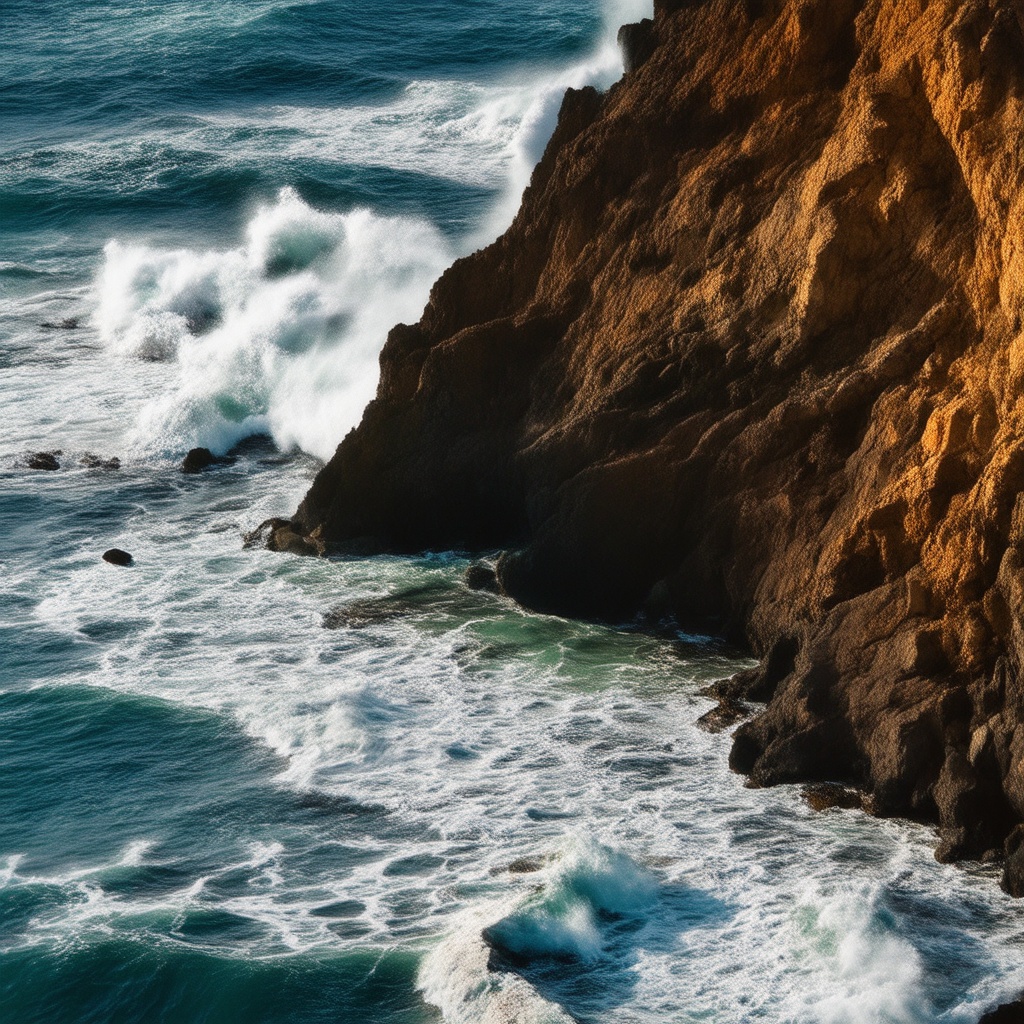}
                    };
                    \node[rotate=0, align=center, font=\normalsize, anchor=base, yshift=4pt] at (image.north) {+ \acrshort{method} (Ours) \strut};
                \end{tikzpicture}
            \end{subfigure}
        \end{minipage}
        \caption{Compatibility of \gls{method} with APG \citep{sadat2025eliminating} on Stable Diffusion 3 \citep{esser2024scaling}. \gls{method} enhances the quality and detail of APG, indicating that its benefits are complementary to various guidance methods.}
        \label{fig:apg-appendix}        
    \end{figure}
}

\newcommand{\figSDXLAppendix}{
    \begin{figure}[t]
        \centering
        \begin{minipage}{\textwidth}
            \begin{subfigure}{0.495\textwidth}
                \begin{tikzpicture}
                    \node[anchor=south west, inner sep=0] (image) at (0,0) {
                            \includegraphics[width=0.495\textwidth]{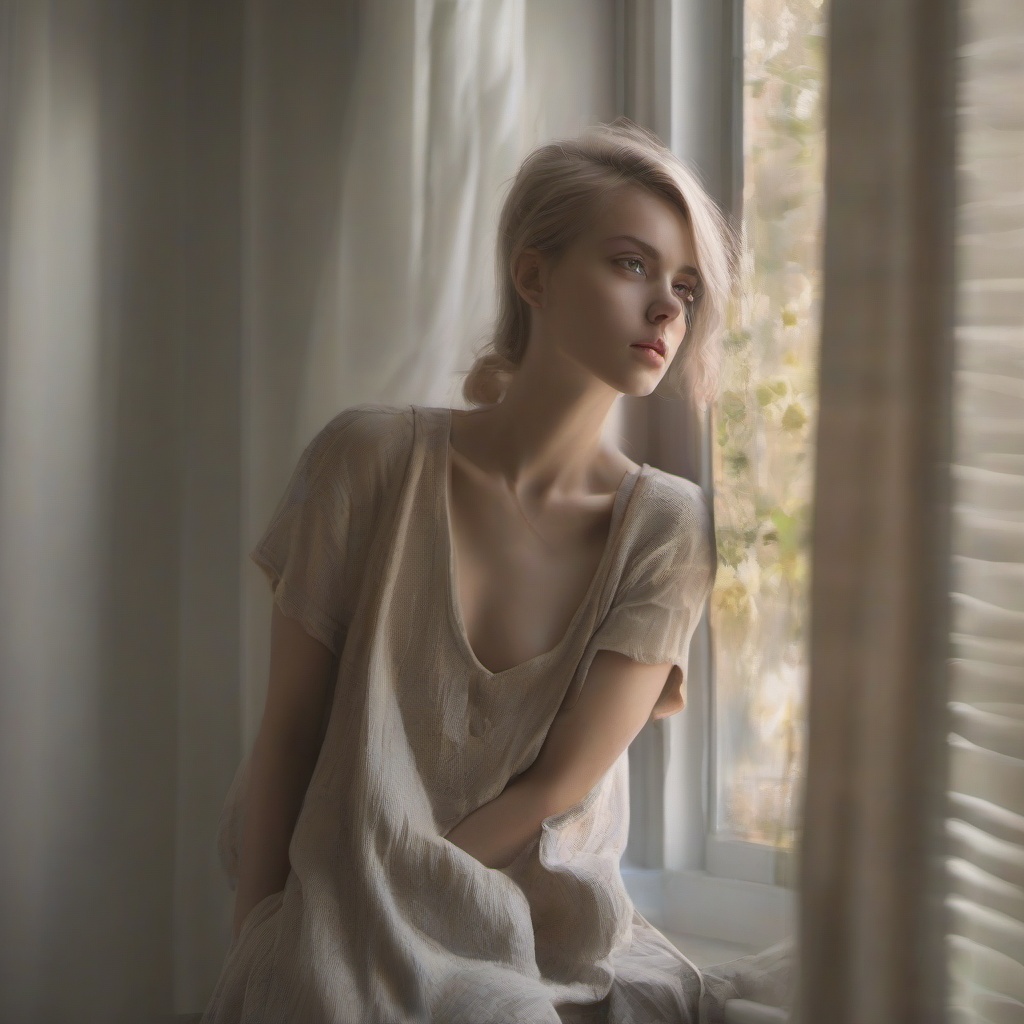}
                    };
                    \node[rotate=0, align=center, font=\normalsize, anchor=base, yshift=4pt] at (image.north) {CFG \strut};
                \end{tikzpicture}
                \hfil
                \begin{tikzpicture}
                    \node[anchor=south west, inner sep=0] (image) at (0,0) {
                \includegraphics[width=0.495\textwidth]{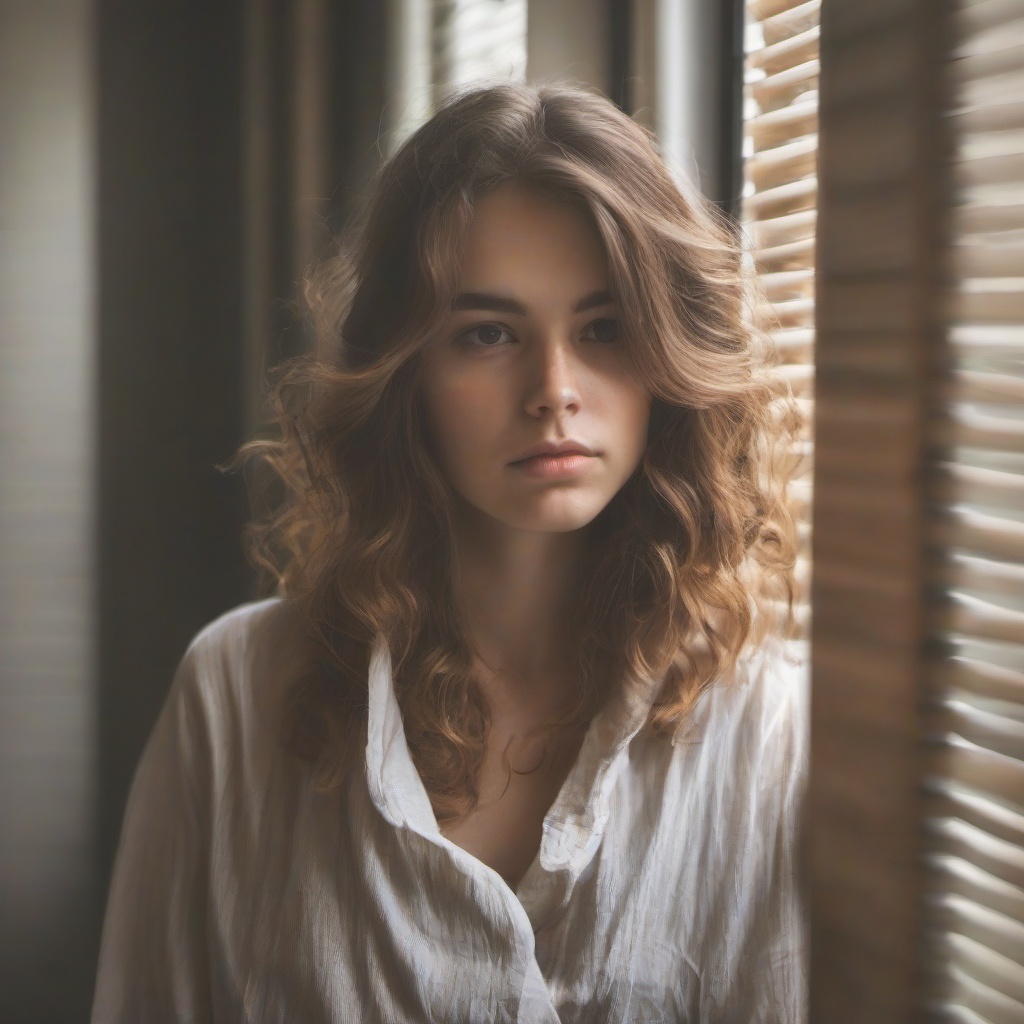}
                    };
                    \node[rotate=0, align=center, font=\normalsize, anchor=base, yshift=4pt] at (image.north) {+ \acrshort{method} (Ours) \strut};
                \end{tikzpicture}
            \end{subfigure}
            \hfil
            \begin{subfigure}{0.495\textwidth}
                \begin{tikzpicture}
                    \node[anchor=south west, inner sep=0] (image) at (0,0) {
                            \includegraphics[width=0.495\textwidth]{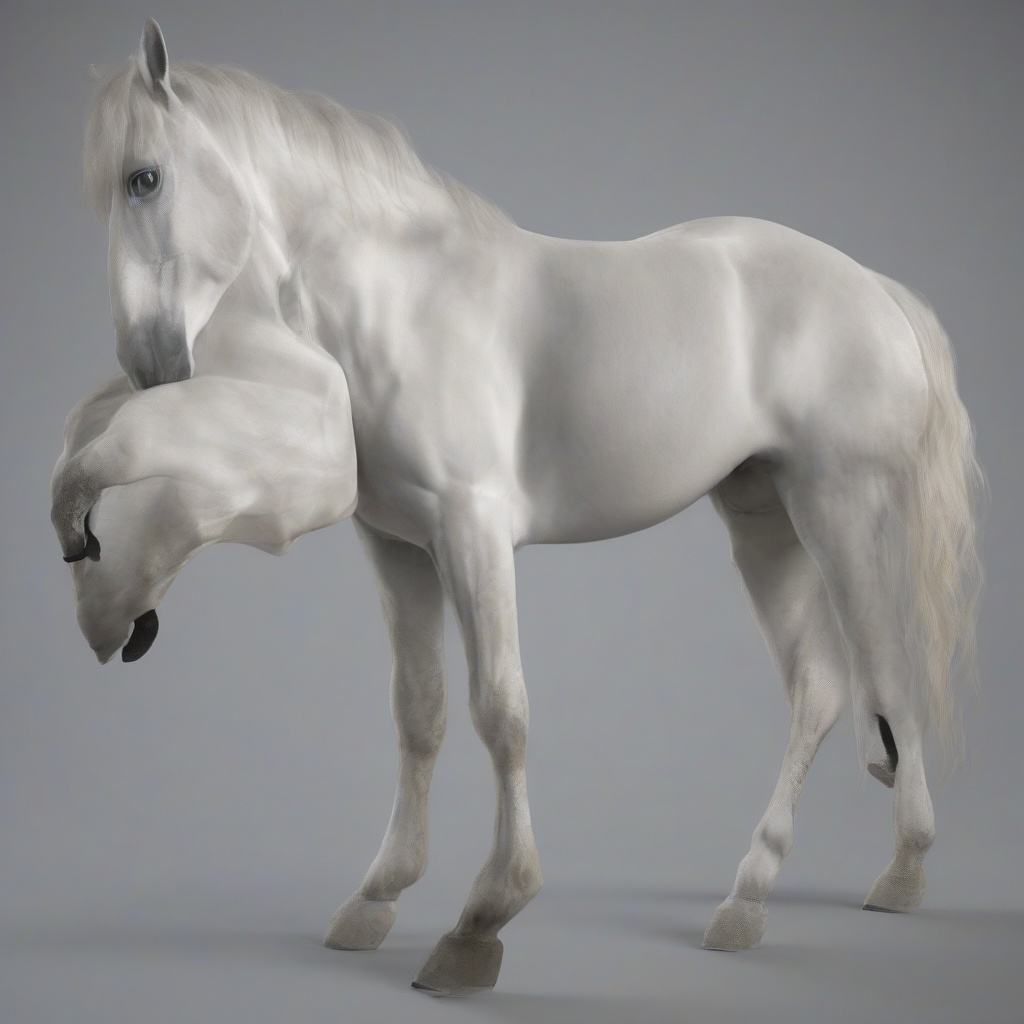}
                    };
                    \node[rotate=0, align=center, font=\normalsize, anchor=base, yshift=4pt] at (image.north) {CFG \strut};
                \end{tikzpicture}
                \hfil
                \begin{tikzpicture}
                    \node[anchor=south west, inner sep=0] (image) at (0,0) {
                \includegraphics[width=0.495\textwidth]{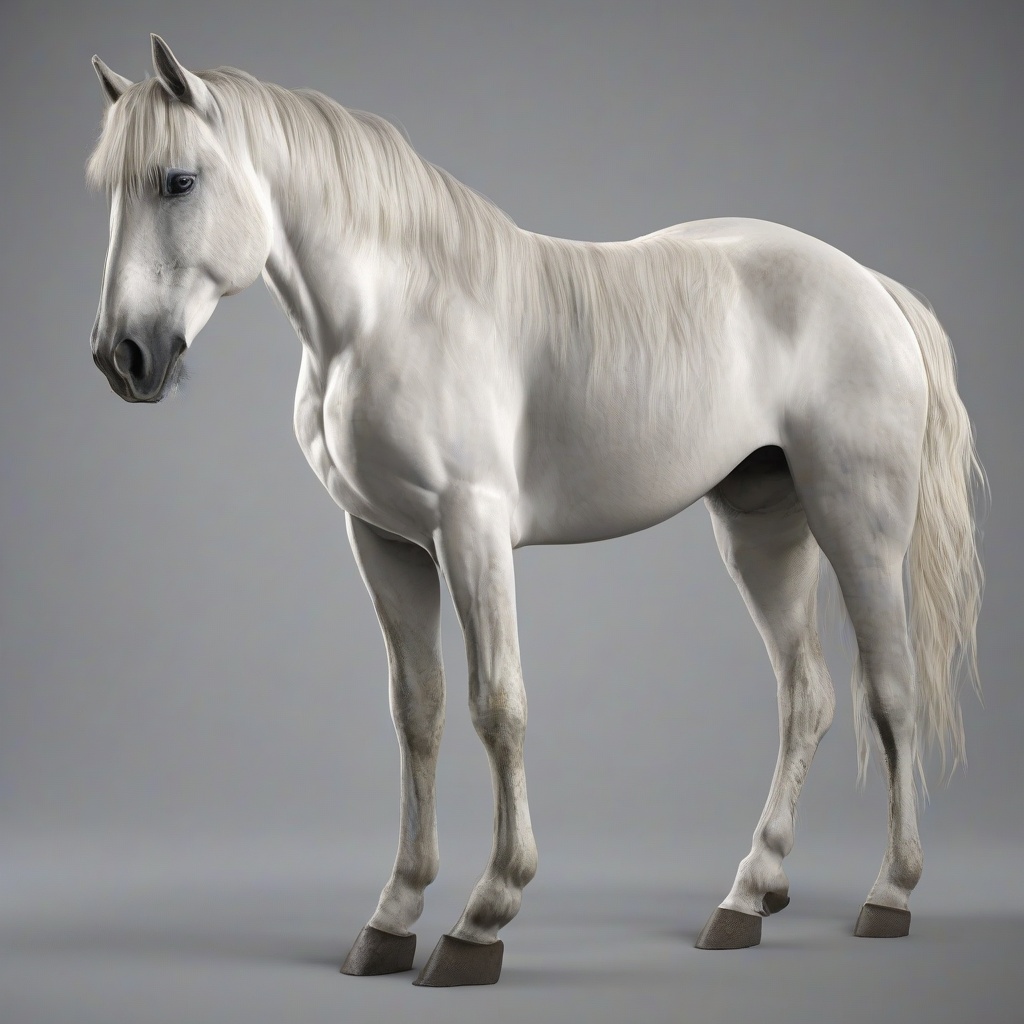}
                    };
                    \node[rotate=0, align=center, font=\normalsize, anchor=base, yshift=4pt] at (image.north) {+ \acrshort{method} (Ours) \strut};
                \end{tikzpicture}
            \end{subfigure}

            \begin{subfigure}[b]{0.495\textwidth}
                \begin{tikzpicture}
                    \node[anchor=south west, inner sep=0] (image) at (0,0) {
                            \includegraphics[width=0.495\textwidth]{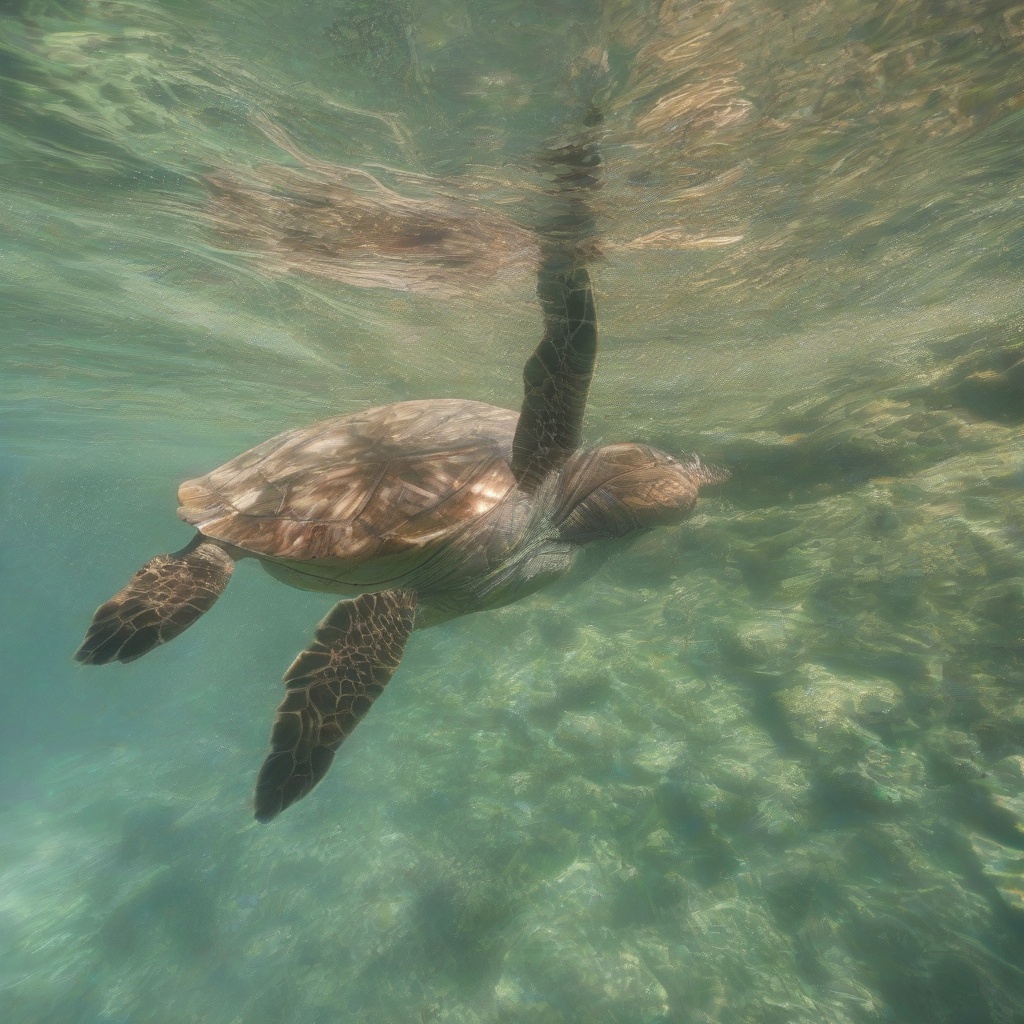}
                    };
                    \node[rotate=0, align=center, font=\normalsize, anchor=base, yshift=4pt] at (image.north) {CFG \strut};
                \end{tikzpicture}
                \hfil
                \begin{tikzpicture}
                    \node[anchor=south west, inner sep=0] (image) at (0,0) {
                \includegraphics[width=0.495\textwidth]{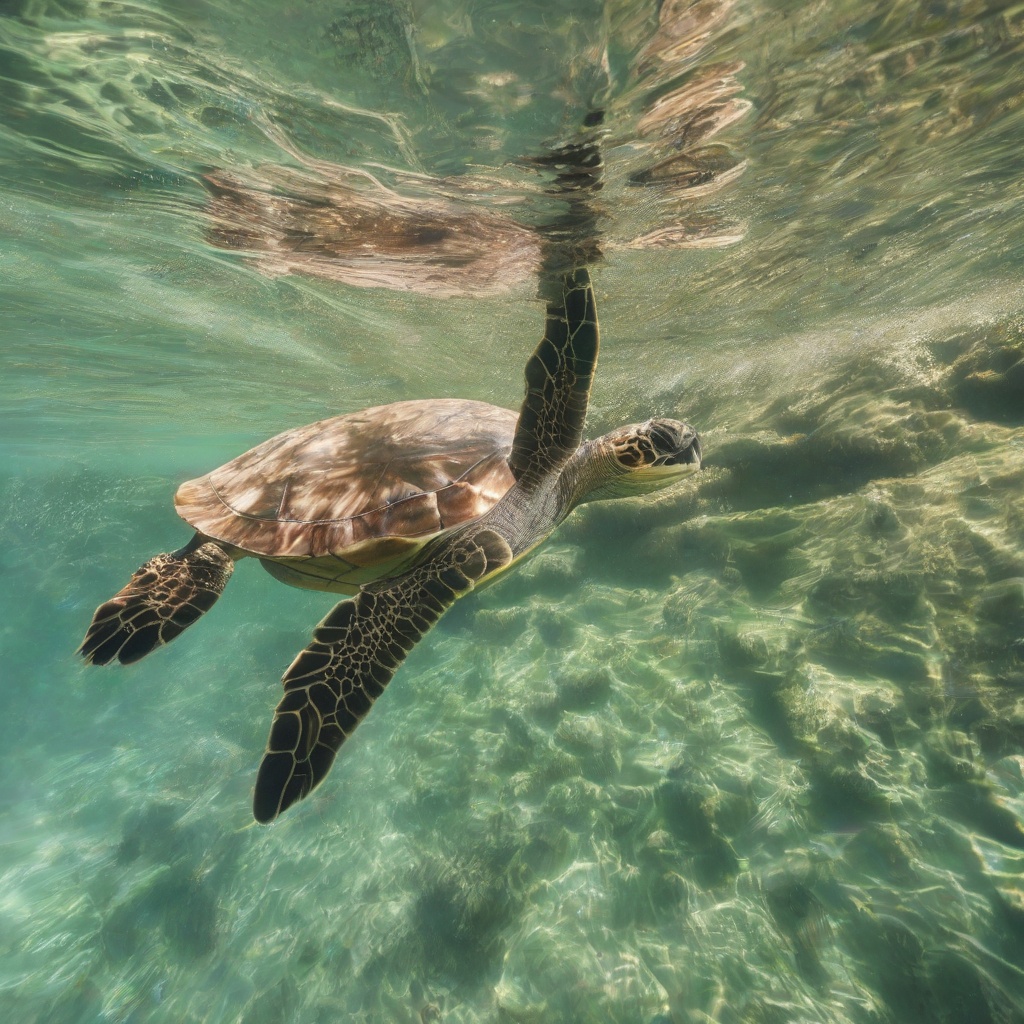}
                    };
                    \node[rotate=0, align=center, font=\normalsize, anchor=base, yshift=4pt] at (image.north) {+ \acrshort{method} (Ours) \strut};
                \end{tikzpicture}
            \end{subfigure}
            \hfil
            \begin{subfigure}[b]{0.495\textwidth}
                \begin{tikzpicture}
                    \node[anchor=south west, inner sep=0] (image) at (0,0) {
                            \includegraphics[width=0.495\textwidth]{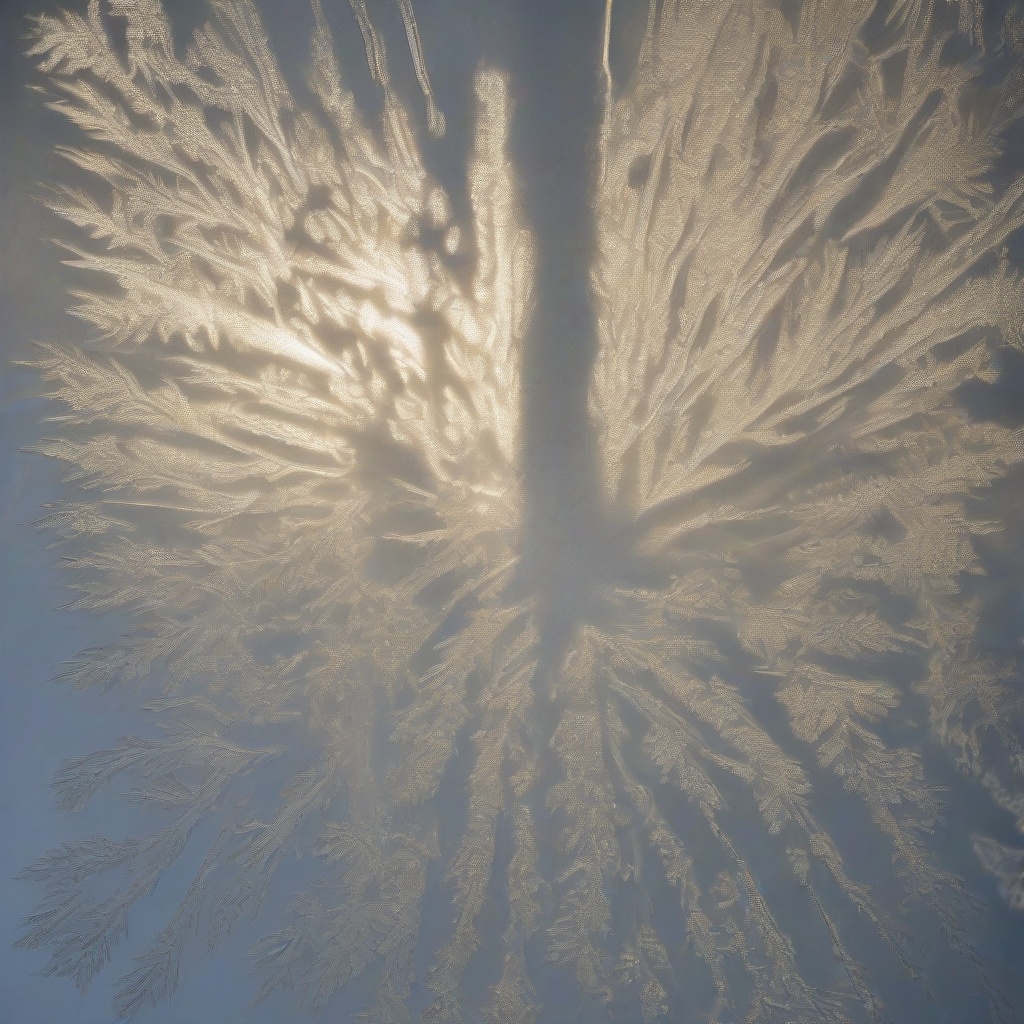}
                    };
                    \node[rotate=0, align=center, font=\normalsize, anchor=base, yshift=4pt] at (image.north) {CFG \strut};
                \end{tikzpicture}
                \hfil
                \begin{tikzpicture}
                    \node[anchor=south west, inner sep=0] (image) at (0,0) {
                \includegraphics[width=0.495\textwidth]{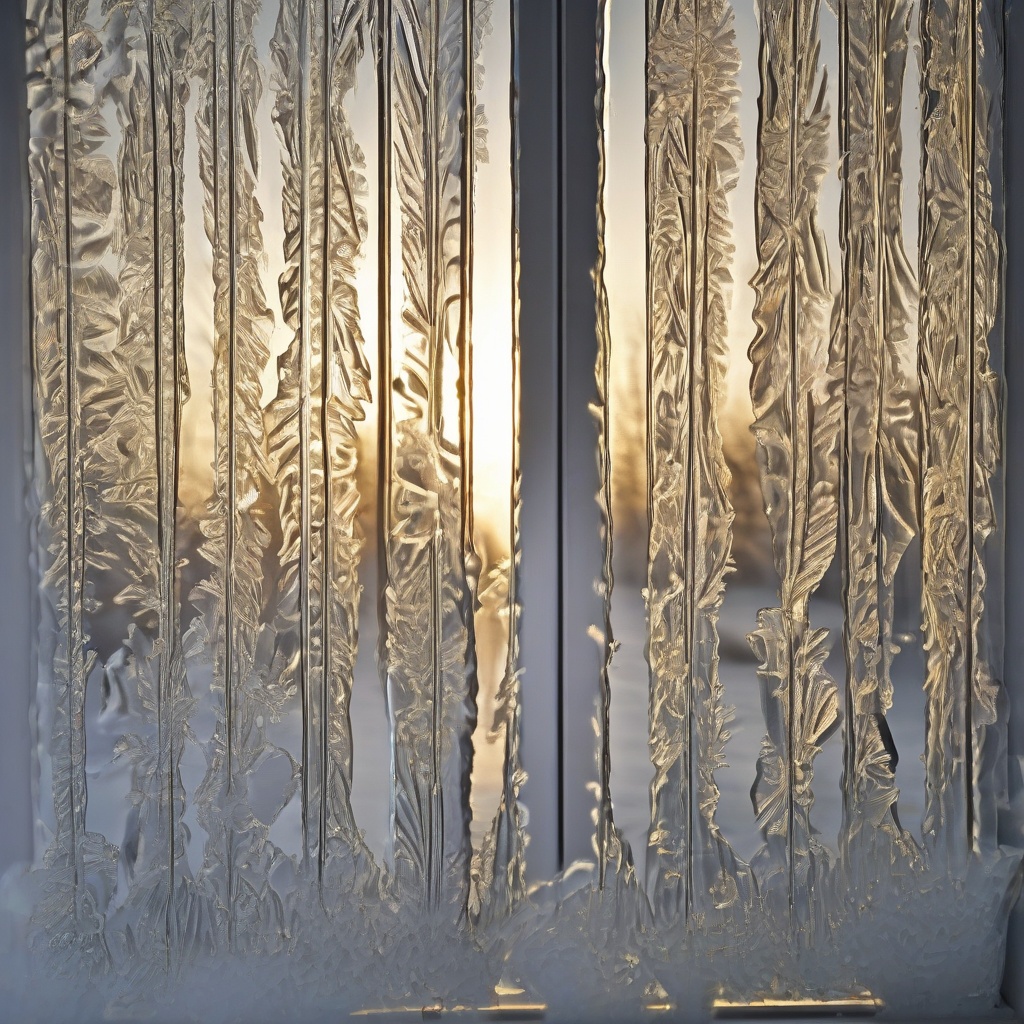}
                    };
                    \node[rotate=0, align=center, font=\normalsize, anchor=base, yshift=4pt] at (image.north) {+ \acrshort{method} (Ours) \strut};
                \end{tikzpicture}
            \end{subfigure}
        \end{minipage}
        \caption{Generated samples using Stable Diffusion XL with 20 steps and $\wcfg=3$.}
        \label{fig:sdxl-appendix}        
    \end{figure}
}

\newcommand{\figFluxAppendix}{
    \begin{figure}[t]
        \centering
        \begin{minipage}{\textwidth}
            \begin{subfigure}{0.495\textwidth}
                \begin{tikzpicture}
                    \node[anchor=south west, inner sep=0] (image) at (0,0) {
                            \includegraphics[width=0.495\textwidth]{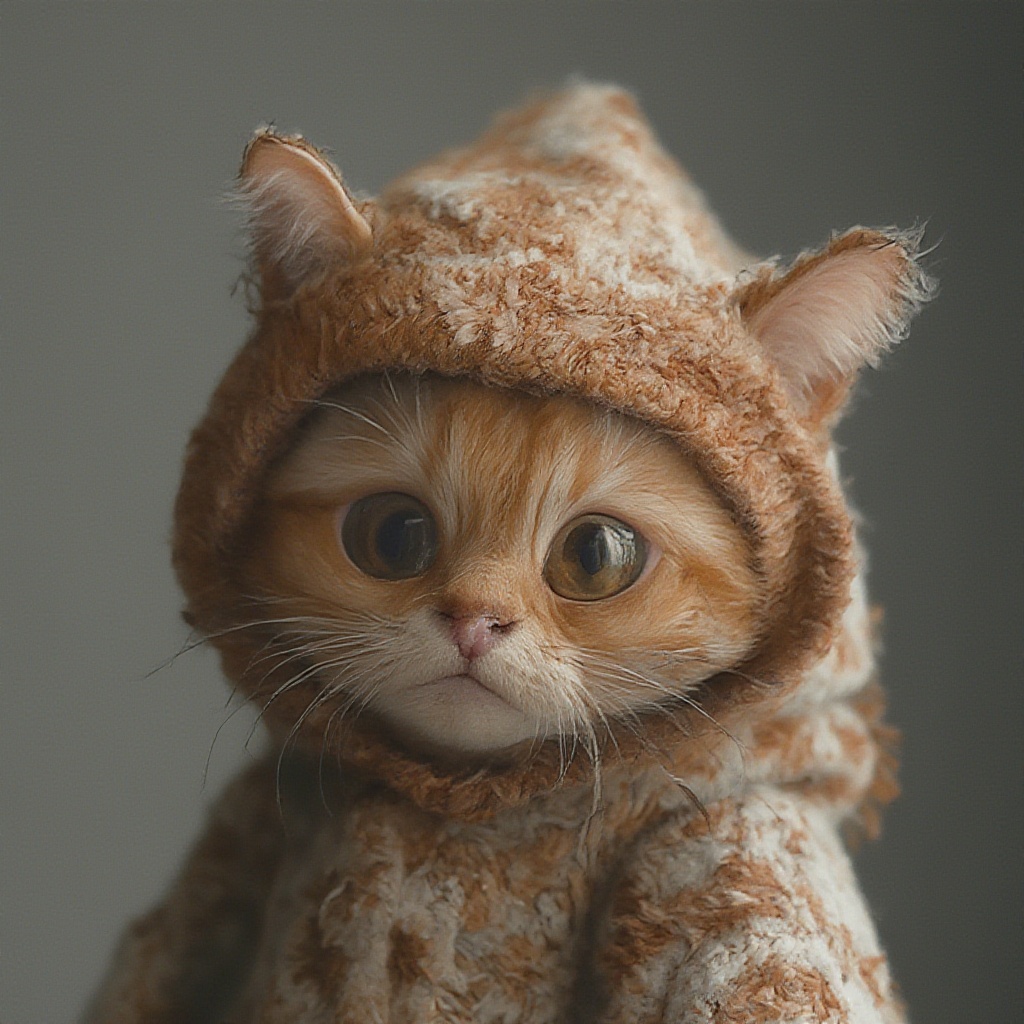}
                    };
                    \node[rotate=0, align=center, font=\normalsize, anchor=base, yshift=4pt] at (image.north) {CFG \strut};
                \end{tikzpicture}
                \hfil
                \begin{tikzpicture}
                    \node[anchor=south west, inner sep=0] (image) at (0,0) {
                \includegraphics[width=0.495\textwidth]{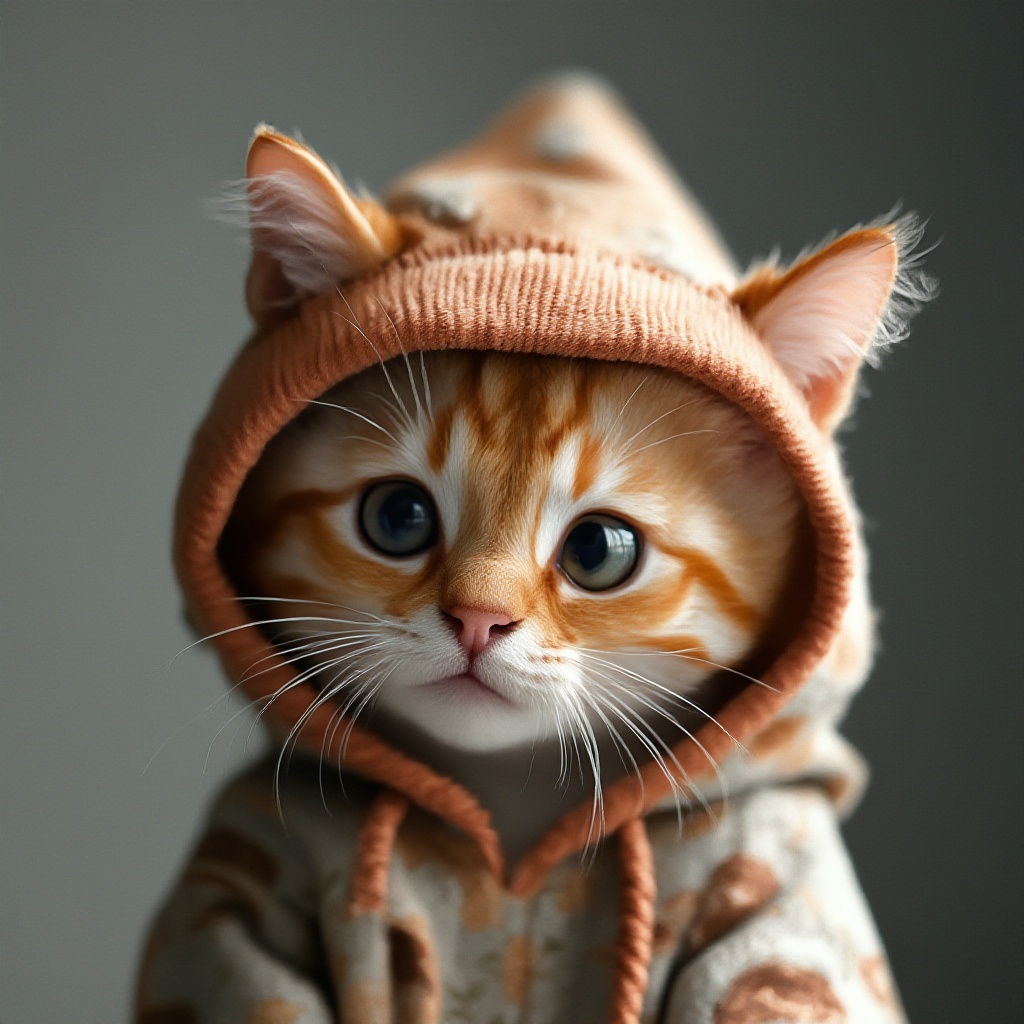}
                    };
                    \node[rotate=0, align=center, font=\normalsize, anchor=base, yshift=4pt] at (image.north) {+ \acrshort{method} (Ours) \strut};
                \end{tikzpicture}
            \end{subfigure}
            \hfil
            \begin{subfigure}{0.495\textwidth}
                \begin{tikzpicture}
                    \node[anchor=south west, inner sep=0] (image) at (0,0) {
                            \includegraphics[width=0.495\textwidth]{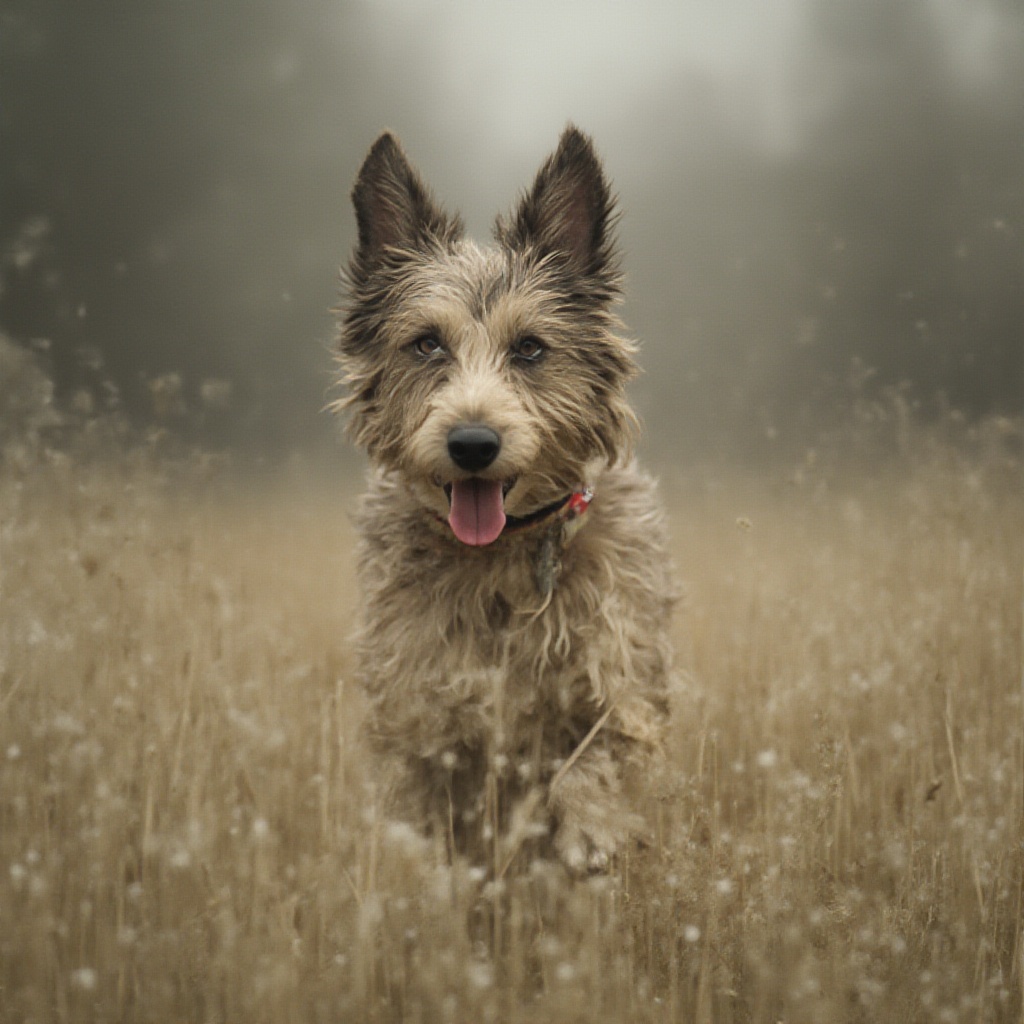}
                    };
                    \node[rotate=0, align=center, font=\normalsize, anchor=base, yshift=4pt] at (image.north) {CFG \strut};
                \end{tikzpicture}
                \hfil
                \begin{tikzpicture}
                    \node[anchor=south west, inner sep=0] (image) at (0,0) {
                \includegraphics[width=0.495\textwidth]{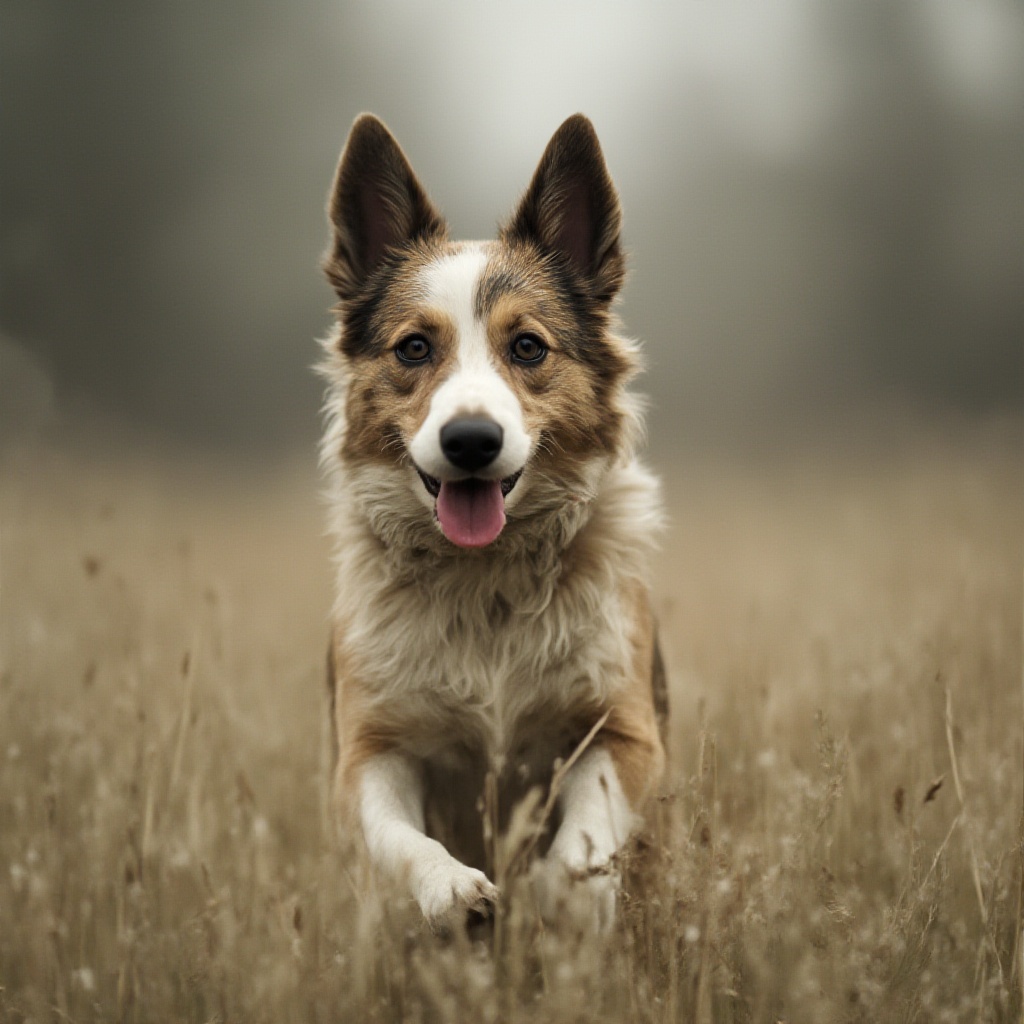}
                    };
                    \node[rotate=0, align=center, font=\normalsize, anchor=base, yshift=4pt] at (image.north) {+ \acrshort{method} (Ours) \strut};
                \end{tikzpicture}
            \end{subfigure}

            \begin{subfigure}[b]{0.495\textwidth}
                \begin{tikzpicture}
                    \node[anchor=south west, inner sep=0] (image) at (0,0) {
                            \includegraphics[width=0.495\textwidth]{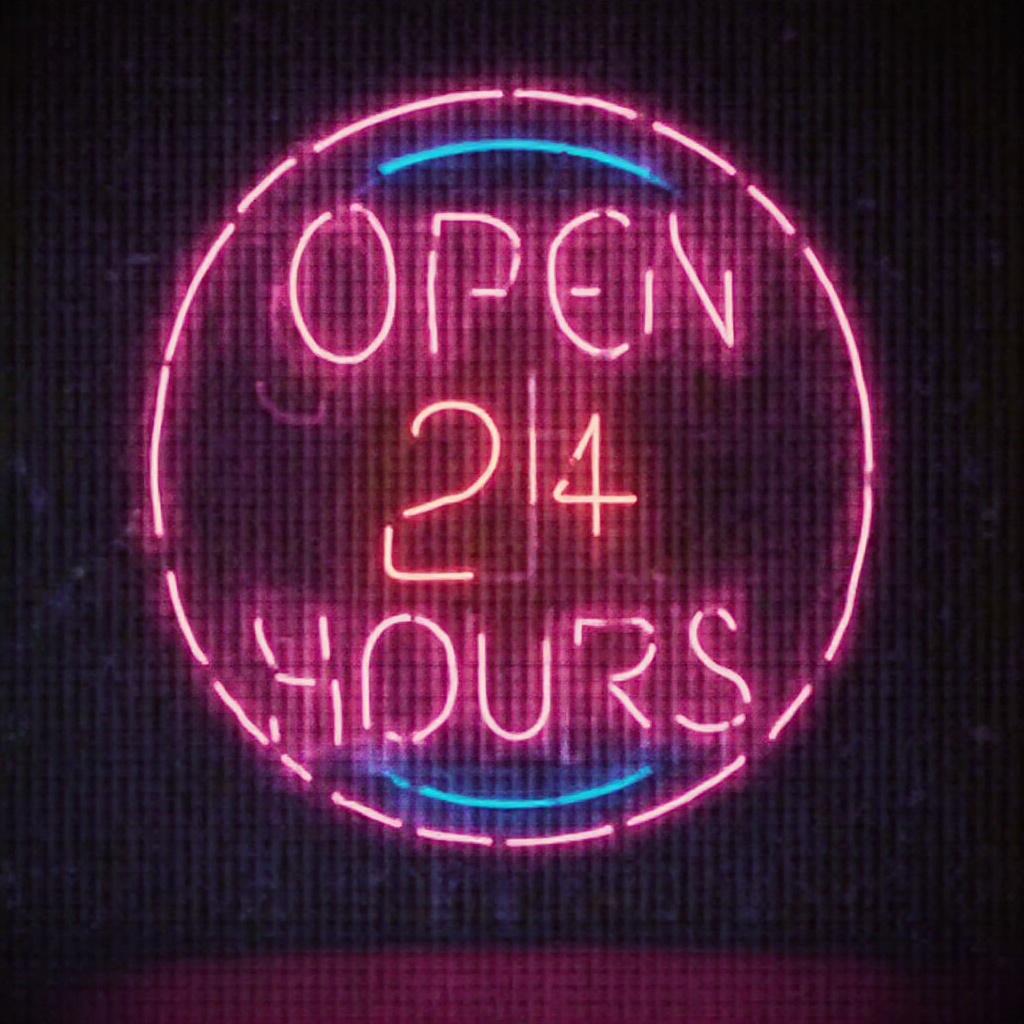}
                    };
                    \node[rotate=0, align=center, font=\normalsize, anchor=base, yshift=4pt] at (image.north) {CFG \strut};
                \end{tikzpicture}
                \hfil
                \begin{tikzpicture}
                    \node[anchor=south west, inner sep=0] (image) at (0,0) {
                \includegraphics[width=0.495\textwidth]{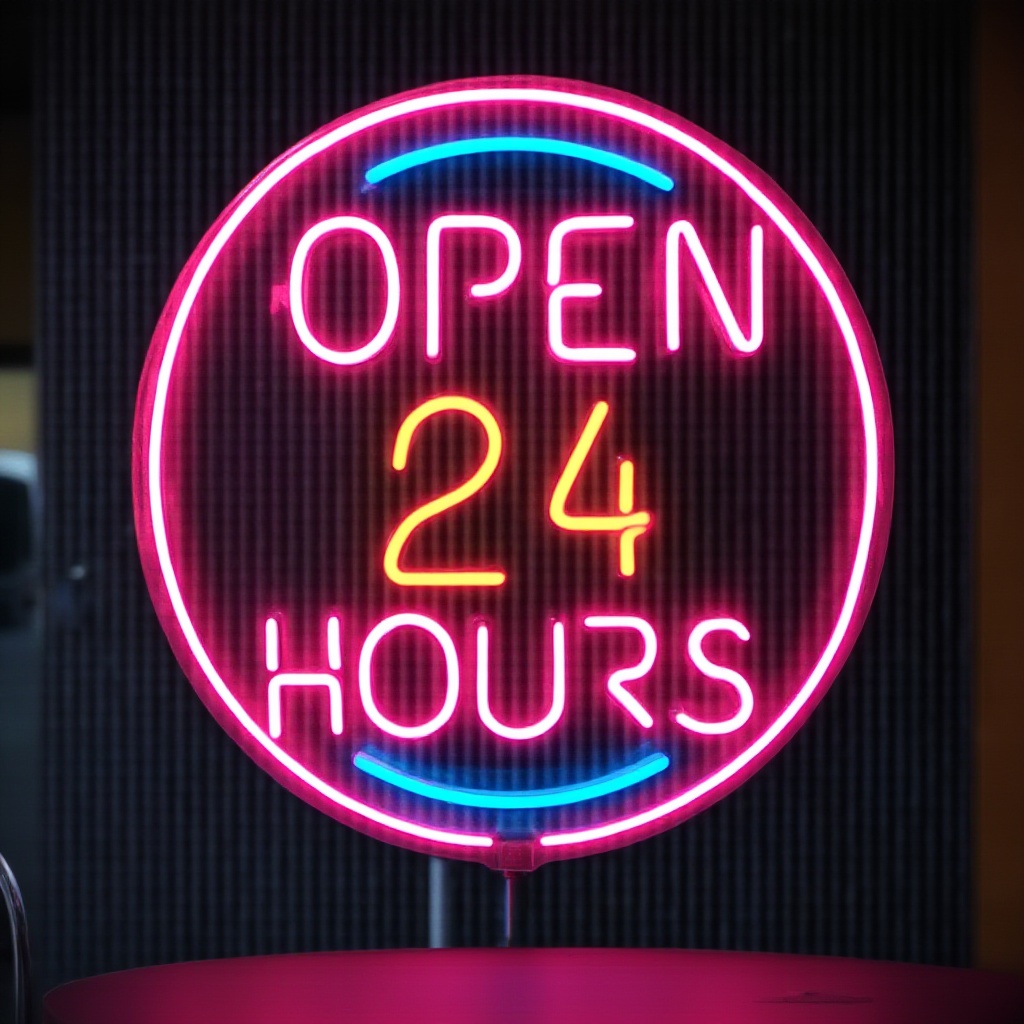}
                    };
                    \node[rotate=0, align=center, font=\normalsize, anchor=base, yshift=4pt] at (image.north) {+ \acrshort{method} (Ours) \strut};
                \end{tikzpicture}
            \end{subfigure}
            \hfil
            \begin{subfigure}[b]{0.495\textwidth}
                \begin{tikzpicture}
                    \node[anchor=south west, inner sep=0] (image) at (0,0) {
                            \includegraphics[width=0.495\textwidth]{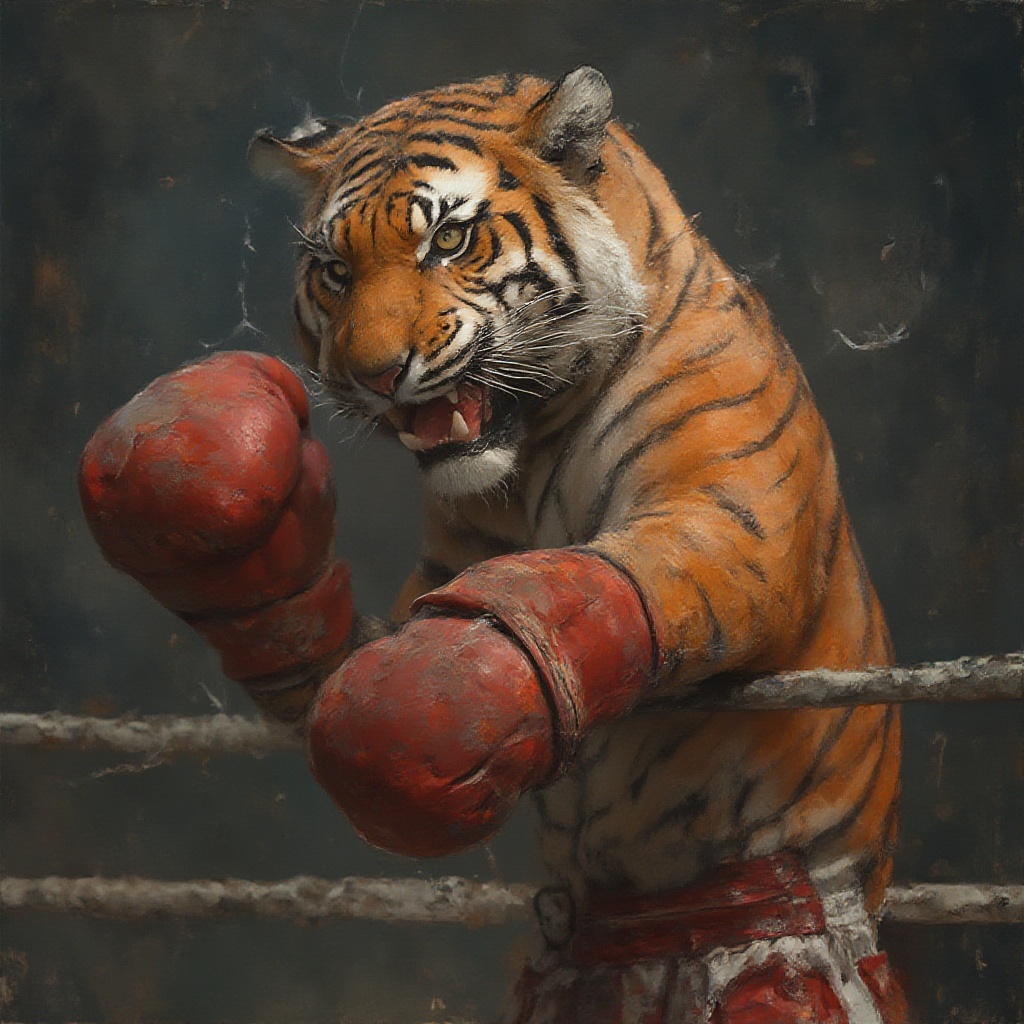}
                    };
                    \node[rotate=0, align=center, font=\normalsize, anchor=base, yshift=4pt] at (image.north) {CFG \strut};
                \end{tikzpicture}
                \hfil
                \begin{tikzpicture}
                    \node[anchor=south west, inner sep=0] (image) at (0,0) {
                \includegraphics[width=0.495\textwidth]{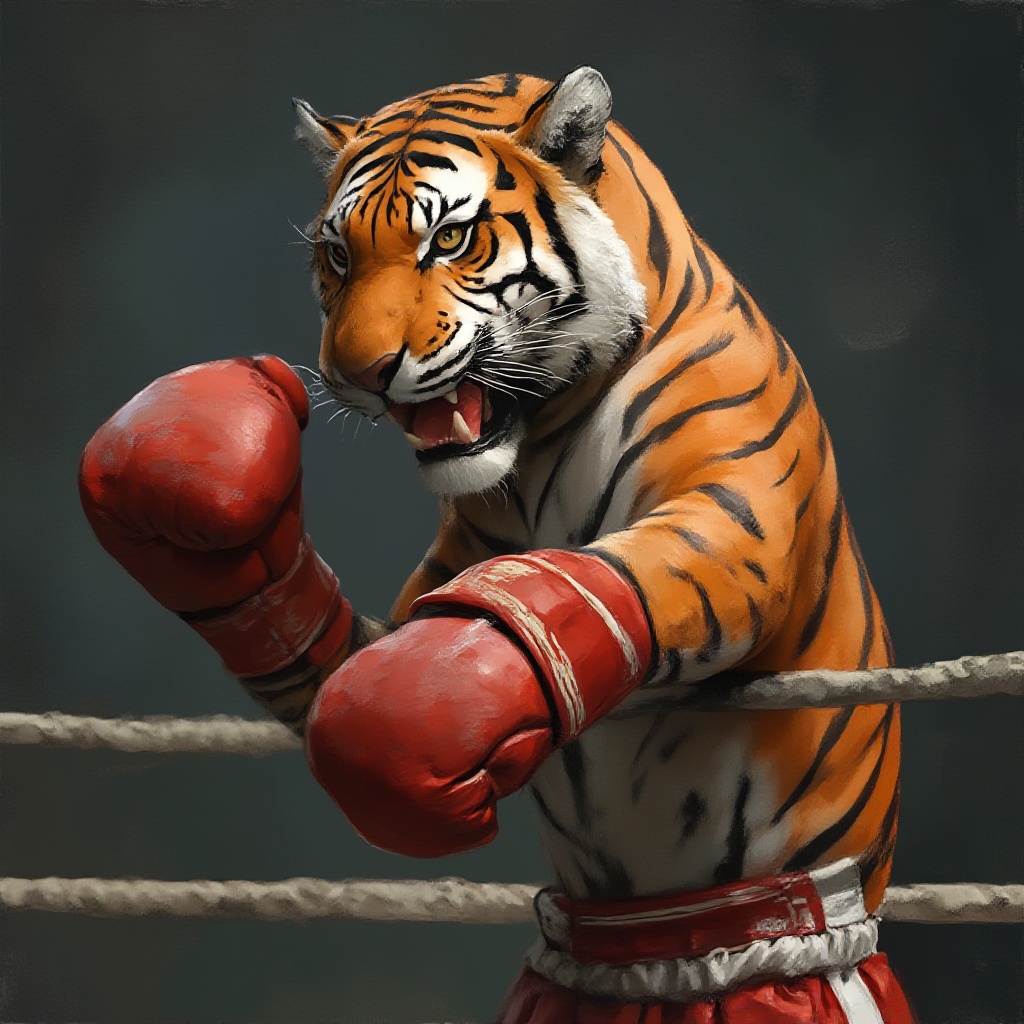}
                    };
                    \node[rotate=0, align=center, font=\normalsize, anchor=base, yshift=4pt] at (image.north) {+ \acrshort{method} (Ours) \strut};
                \end{tikzpicture}
            \end{subfigure}

            \begin{subfigure}[b]{0.495\textwidth}
                \begin{tikzpicture}
                    \node[anchor=south west, inner sep=0] (image) at (0,0) {
                            \includegraphics[width=0.495\textwidth]{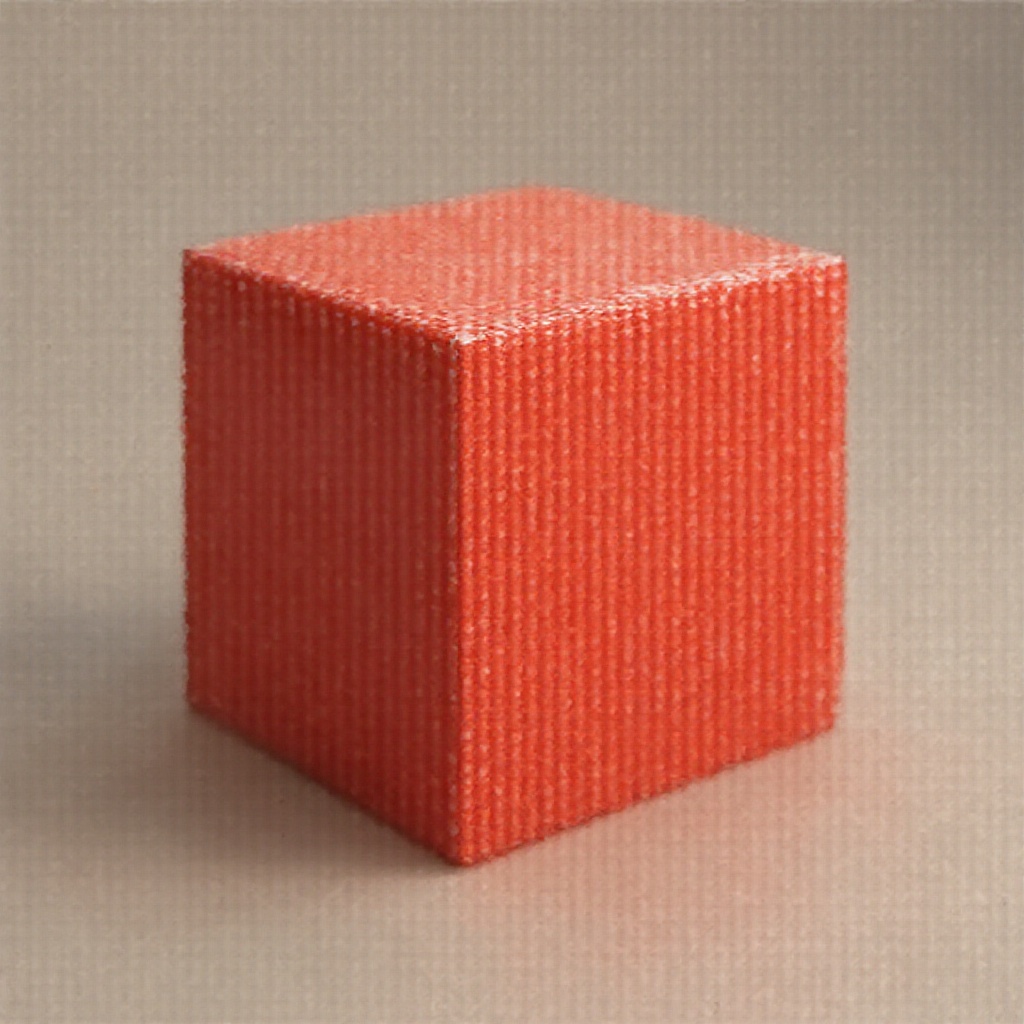}
                    };
                    \node[rotate=0, align=center, font=\normalsize, anchor=base, yshift=4pt] at (image.north) {CFG \strut};
                \end{tikzpicture}
                \hfil
                \begin{tikzpicture}
                    \node[anchor=south west, inner sep=0] (image) at (0,0) {
                \includegraphics[width=0.495\textwidth]{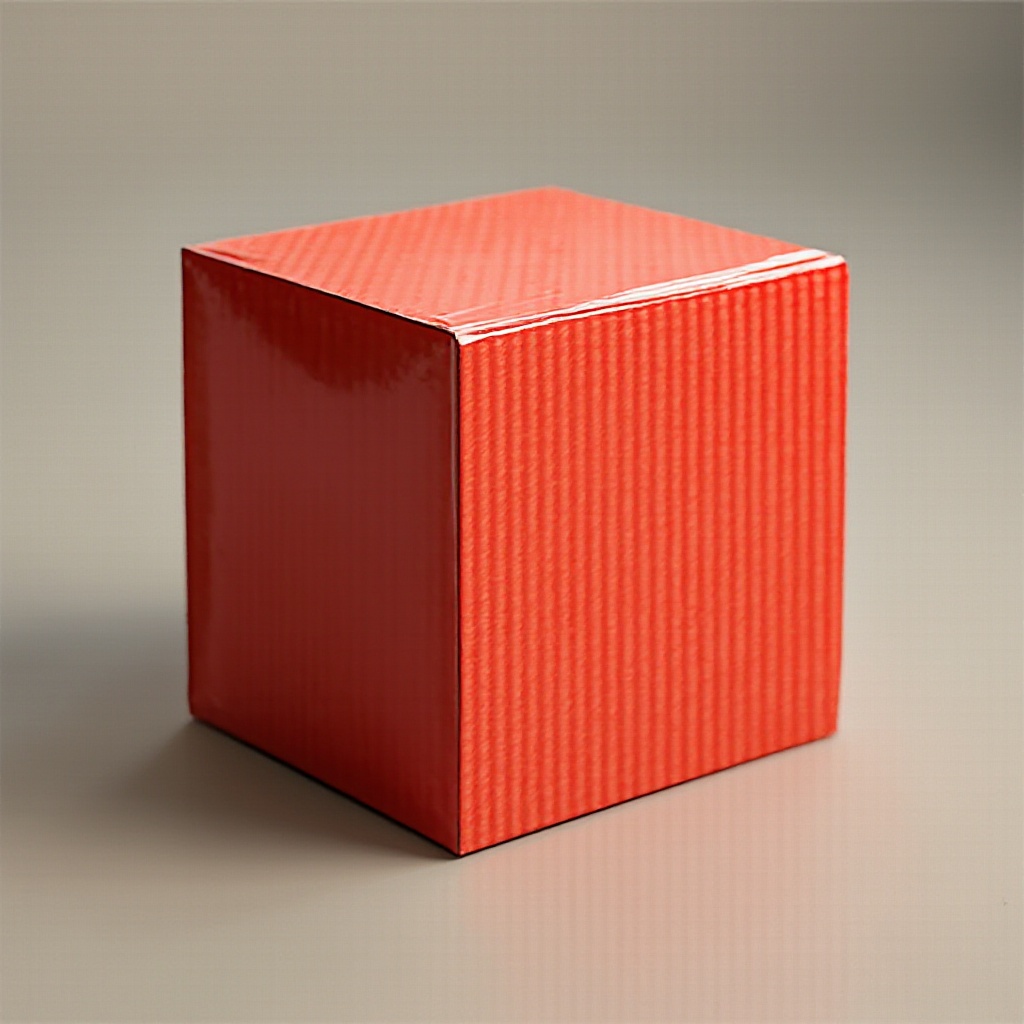}
                    };
                    \node[rotate=0, align=center, font=\normalsize, anchor=base, yshift=4pt] at (image.north) {+ \acrshort{method} (Ours) \strut};
                \end{tikzpicture}
            \end{subfigure}
            \hfil
            \begin{subfigure}[b]{0.495\textwidth}
                \begin{tikzpicture}
                    \node[anchor=south west, inner sep=0] (image) at (0,0) {
                            \includegraphics[width=0.495\textwidth]{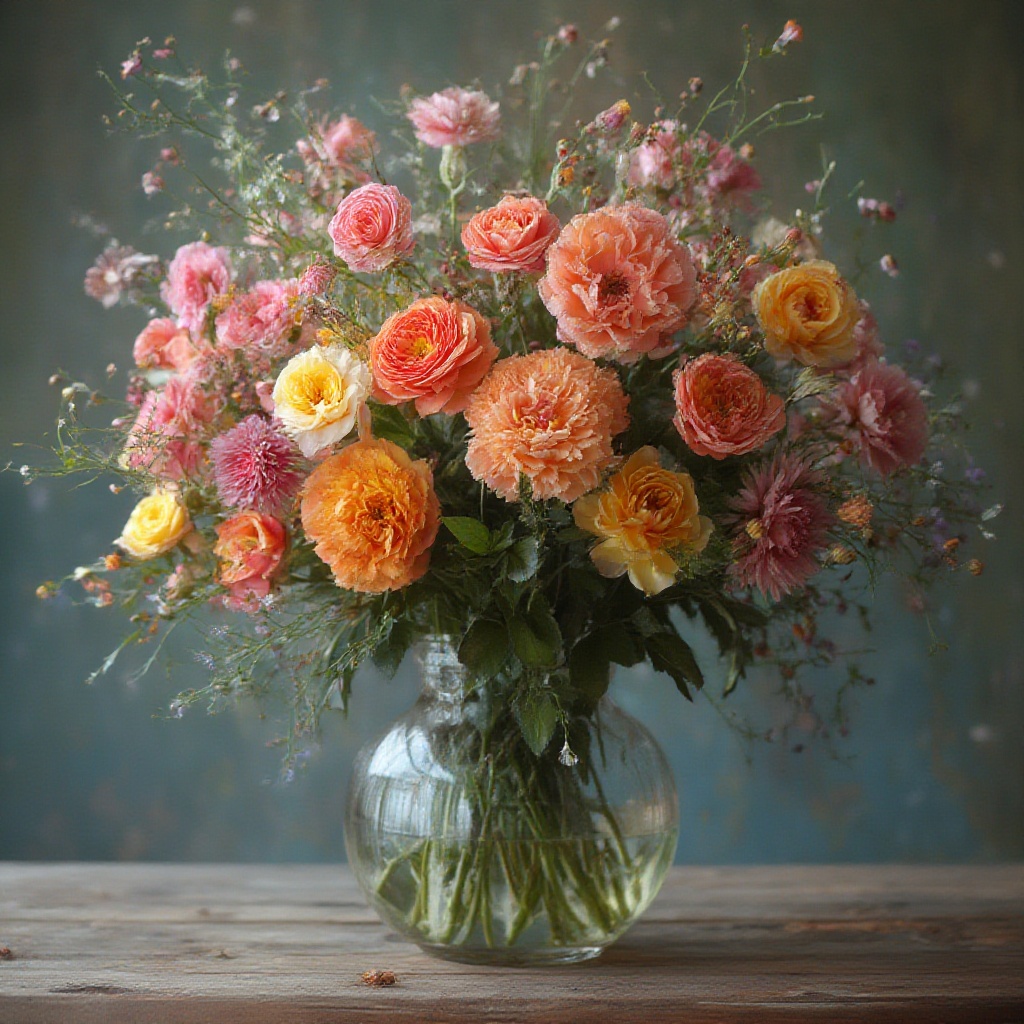}
                    };
                    \node[rotate=0, align=center, font=\normalsize, anchor=base, yshift=4pt] at (image.north) {CFG \strut};
                \end{tikzpicture}
                \hfil
                \begin{tikzpicture}
                    \node[anchor=south west, inner sep=0] (image) at (0,0) {
                \includegraphics[width=0.495\textwidth]{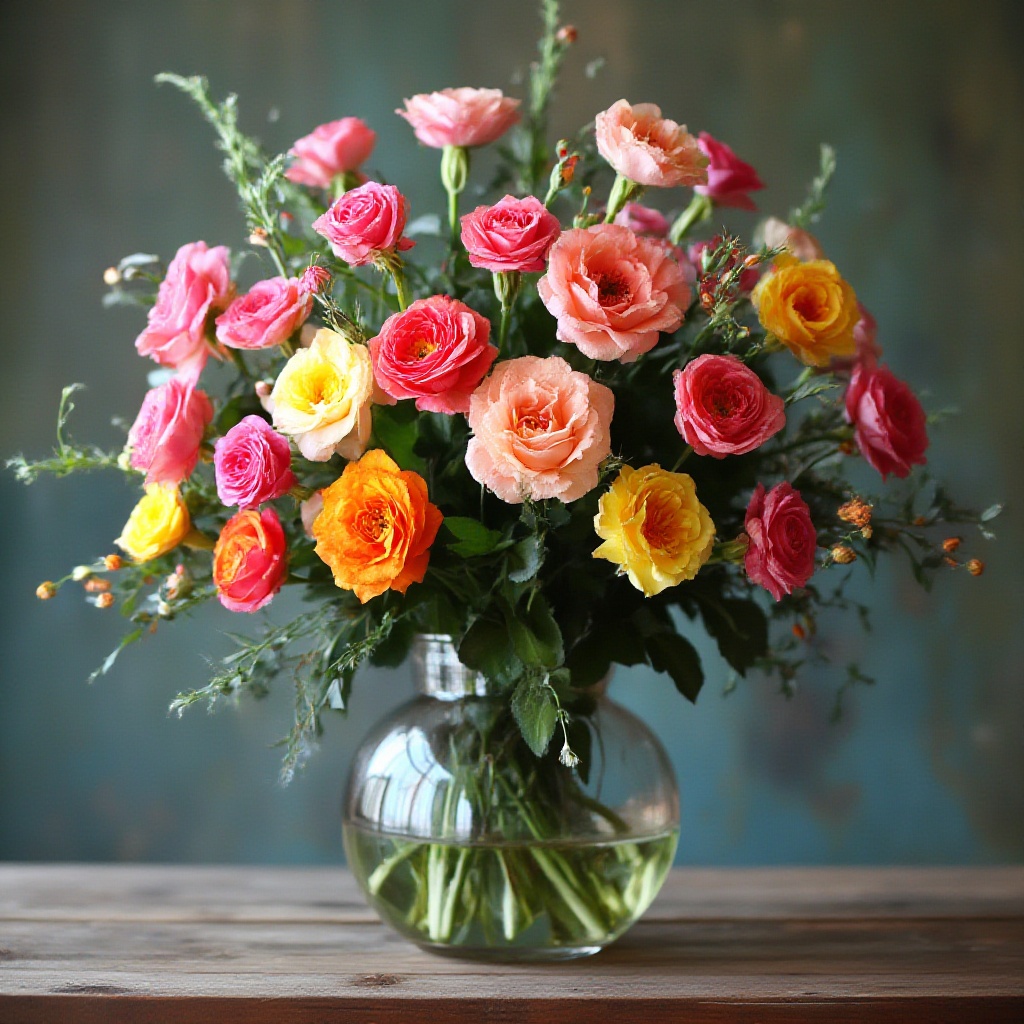}
                    };
                    \node[rotate=0, align=center, font=\normalsize, anchor=base, yshift=4pt] at (image.north) {+ \acrshort{method} (Ours) \strut};
                \end{tikzpicture}
            \end{subfigure}
        \end{minipage}
        \caption{Generated samples using Flux with 15 steps and $\wcfg=1.25$.}
        \label{fig:flux-appendix}        
    \end{figure}
}

\newcommand{\figSDThreeAppendix}{
    \begin{figure}[t]
        \centering
        \begin{minipage}{\textwidth}
            \begin{subfigure}{0.495\textwidth}
                \begin{tikzpicture}
                    \node[anchor=south west, inner sep=0] (image) at (0,0) {
                            \includegraphics[width=0.495\textwidth]{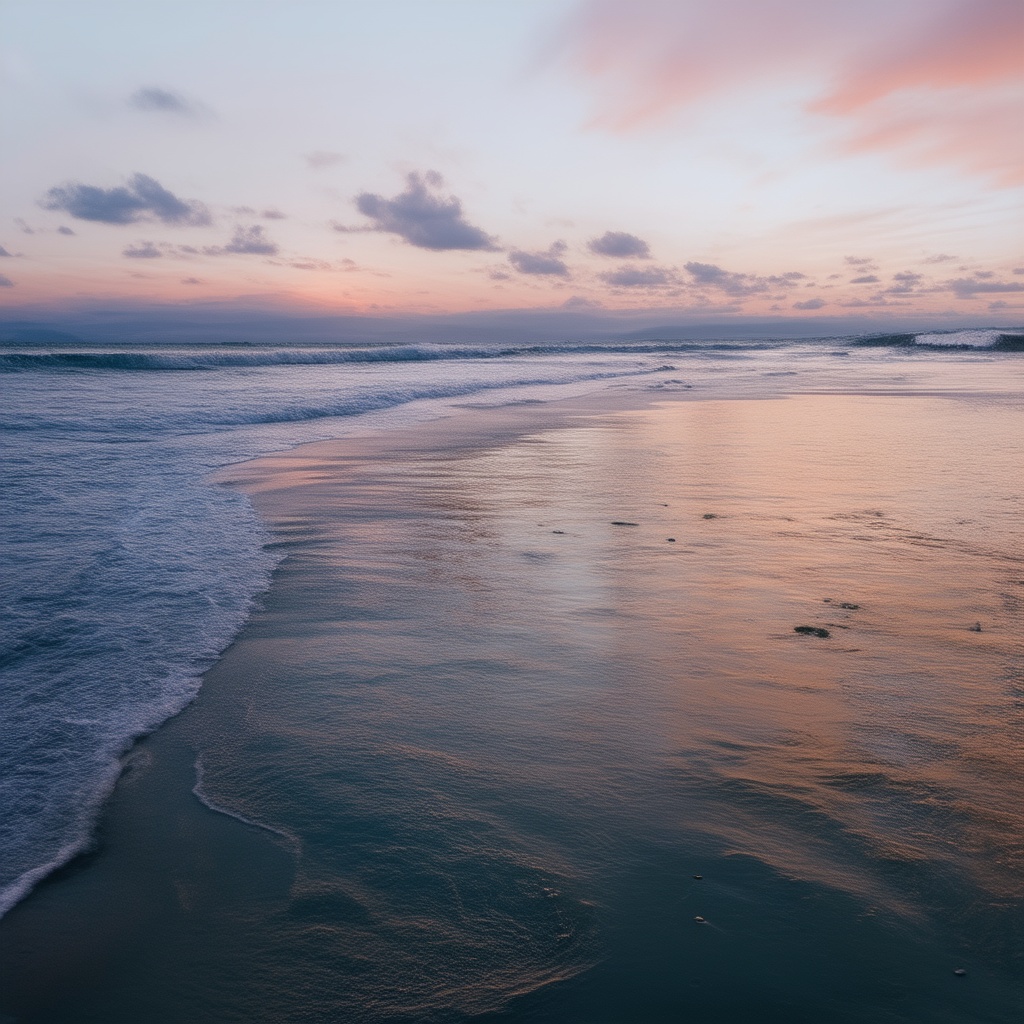}
                    };
                    \node[rotate=0, align=center, font=\normalsize, anchor=base, yshift=4pt] at (image.north) {CFG \strut};
                \end{tikzpicture}
                \hfil
                \begin{tikzpicture}
                    \node[anchor=south west, inner sep=0] (image) at (0,0) {
                \includegraphics[width=0.495\textwidth]{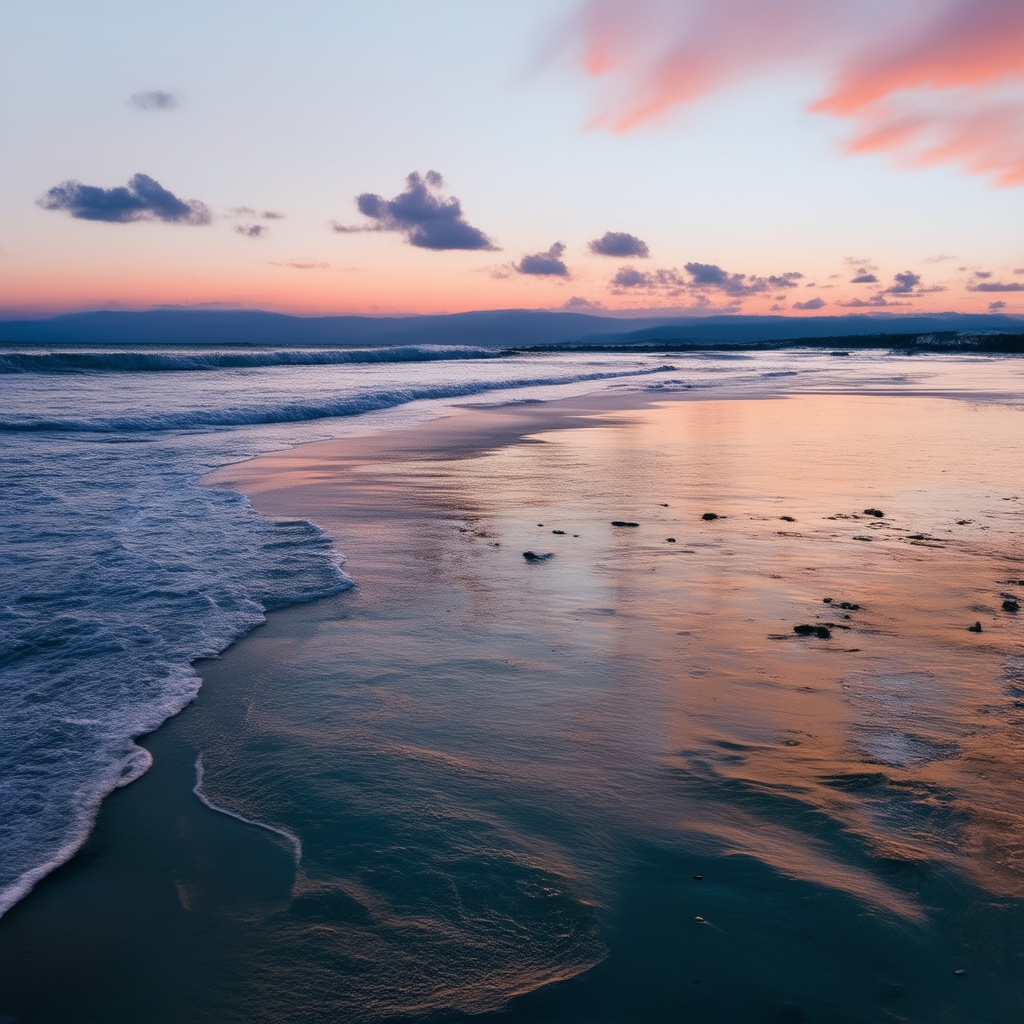}
                    };
                    \node[rotate=0, align=center, font=\normalsize, anchor=base, yshift=4pt] at (image.north) {+ \acrshort{method} (Ours) \strut};
                \end{tikzpicture}
            \end{subfigure}
            \hfil
            \begin{subfigure}{0.495\textwidth}
                \begin{tikzpicture}
                    \node[anchor=south west, inner sep=0] (image) at (0,0) {
                            \includegraphics[width=0.495\textwidth]{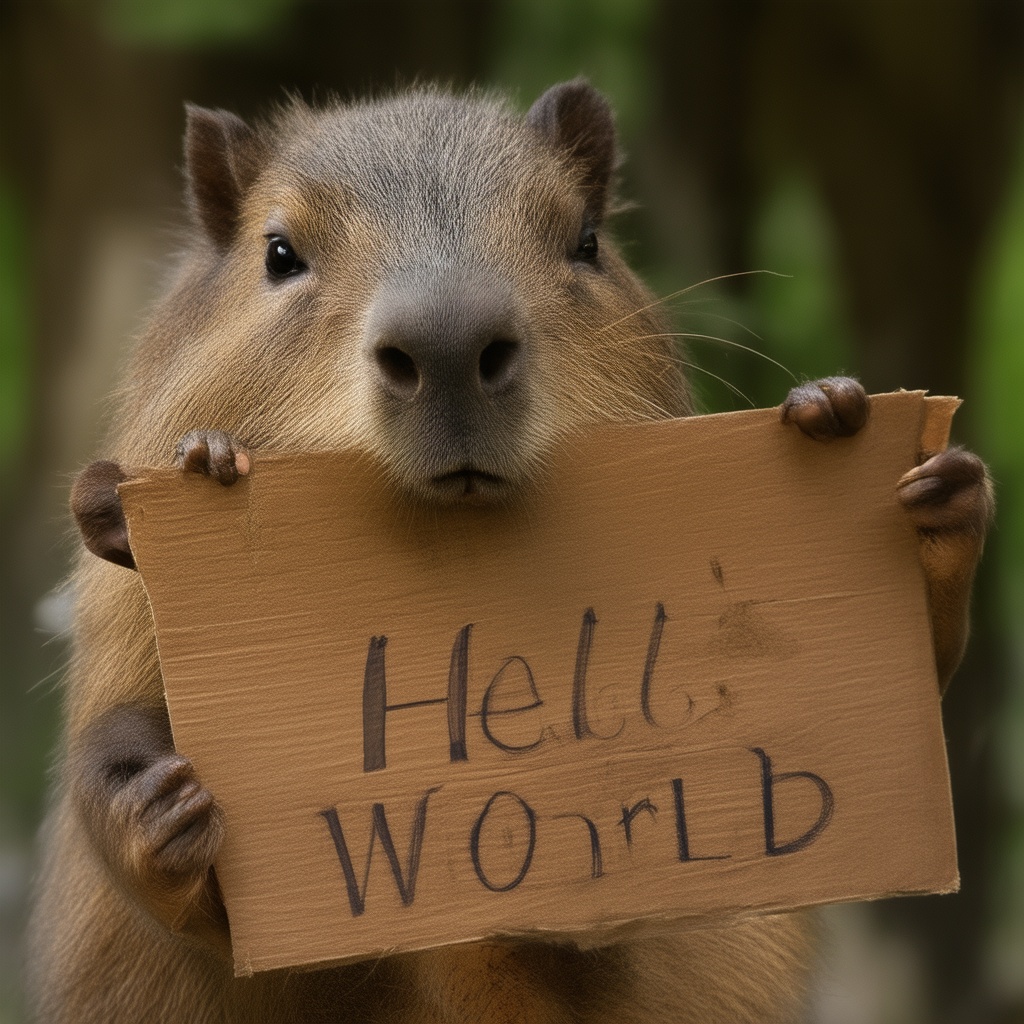}
                    };
                    \node[rotate=0, align=center, font=\normalsize, anchor=base, yshift=4pt] at (image.north) {CFG \strut};
                \end{tikzpicture}
                \hfil
                \begin{tikzpicture}
                    \node[anchor=south west, inner sep=0] (image) at (0,0) {
                \includegraphics[width=0.495\textwidth]{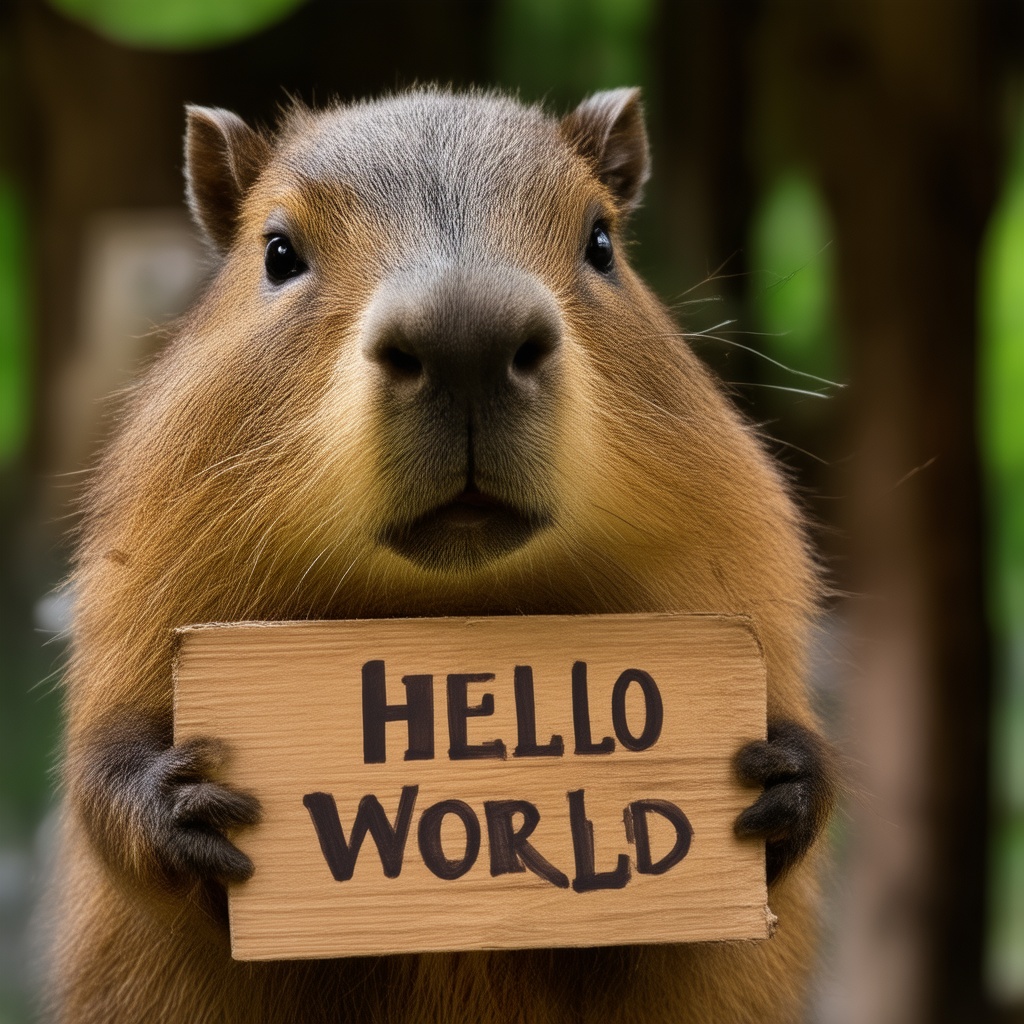}
                    };
                    \node[rotate=0, align=center, font=\normalsize, anchor=base, yshift=4pt] at (image.north) {+ \acrshort{method} (Ours) \strut};
                \end{tikzpicture}
            \end{subfigure}

            \begin{subfigure}[b]{0.495\textwidth}
                \begin{tikzpicture}
                    \node[anchor=south west, inner sep=0] (image) at (0,0) {
                            \includegraphics[width=0.495\textwidth]{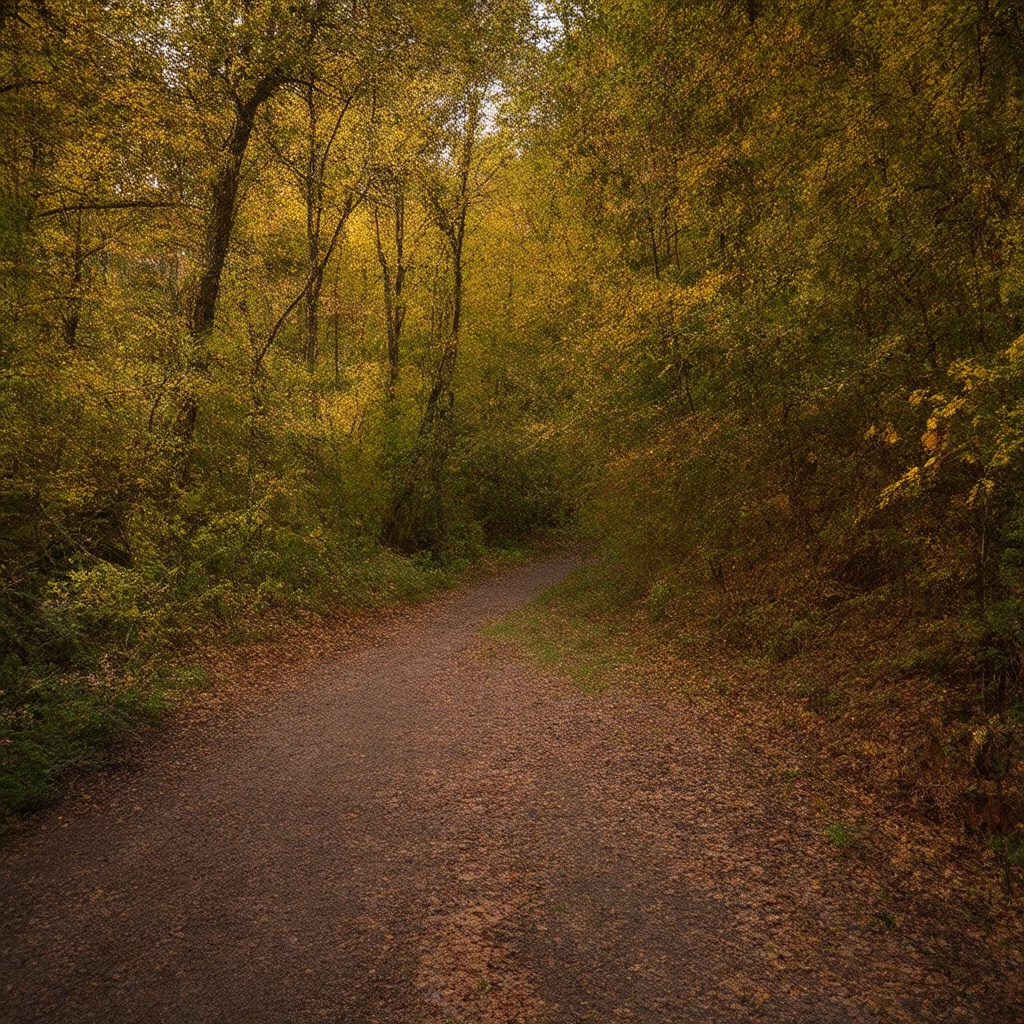}
                    };
                    \node[rotate=0, align=center, font=\normalsize, anchor=base, yshift=4pt] at (image.north) {CFG \strut};
                \end{tikzpicture}
                \hfil
                \begin{tikzpicture}
                    \node[anchor=south west, inner sep=0] (image) at (0,0) {
                \includegraphics[width=0.495\textwidth]{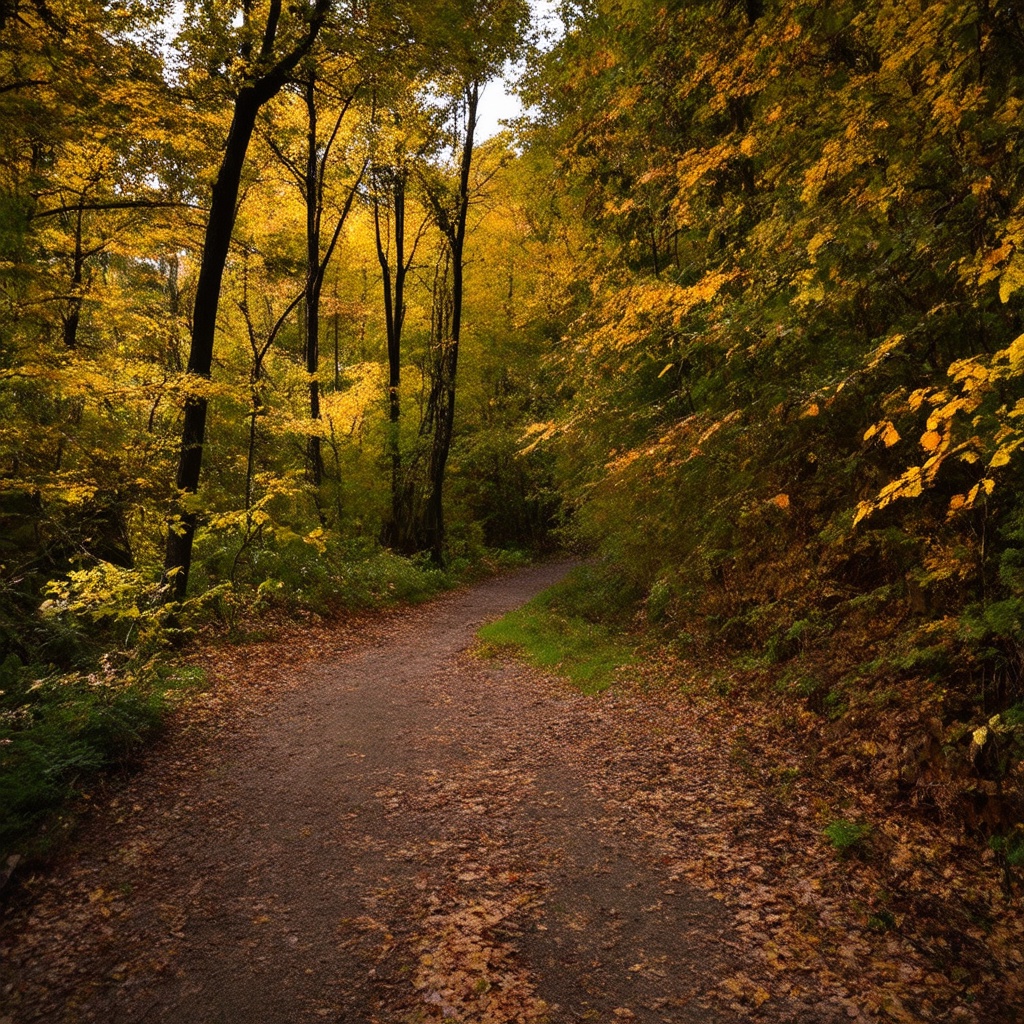}
                    };
                    \node[rotate=0, align=center, font=\normalsize, anchor=base, yshift=4pt] at (image.north) {+ \acrshort{method} (Ours) \strut};
                \end{tikzpicture}
            \end{subfigure}
            \hfil
            \begin{subfigure}[b]{0.495\textwidth}
                \begin{tikzpicture}
                    \node[anchor=south west, inner sep=0] (image) at (0,0) {
                            \includegraphics[width=0.495\textwidth]{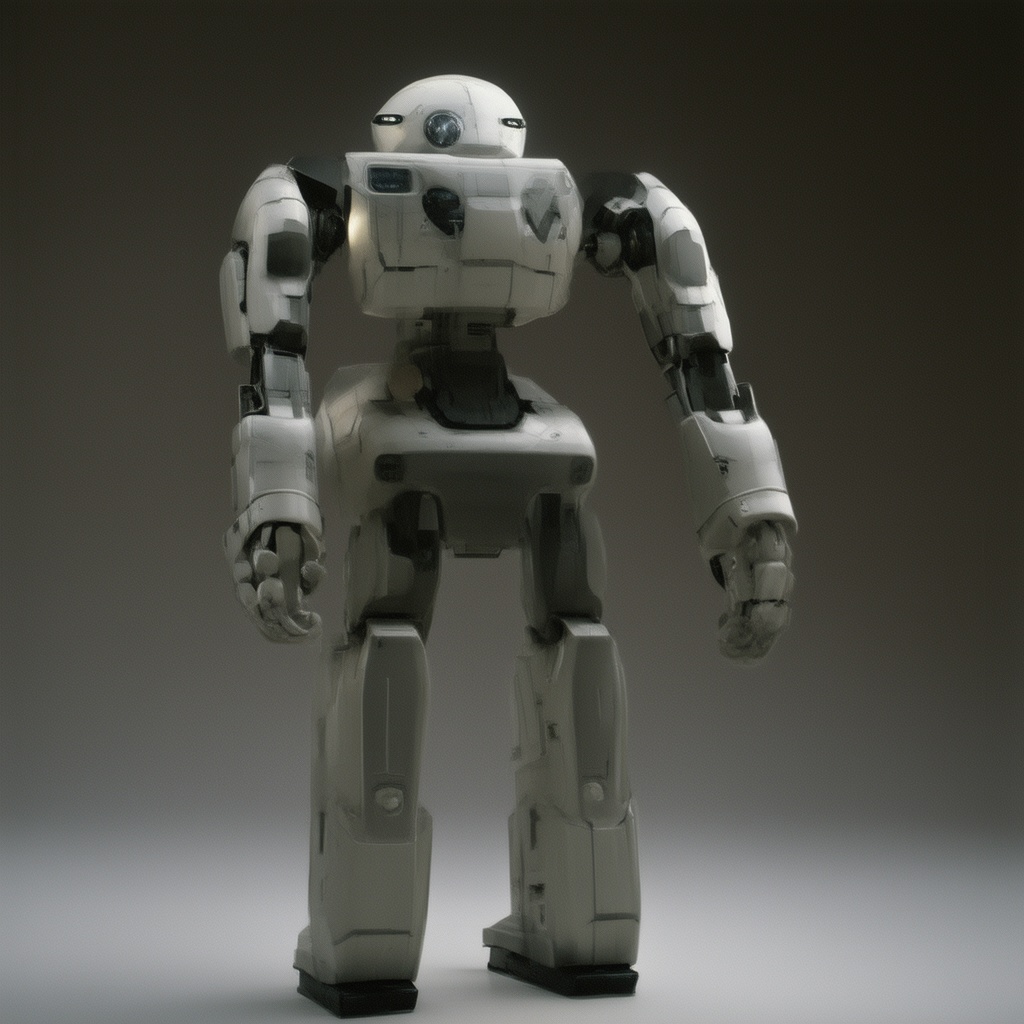}
                    };
                    \node[rotate=0, align=center, font=\normalsize, anchor=base, yshift=4pt] at (image.north) {CFG \strut};
                \end{tikzpicture}
                \hfil
                \begin{tikzpicture}
                    \node[anchor=south west, inner sep=0] (image) at (0,0) {
                \includegraphics[width=0.495\textwidth]{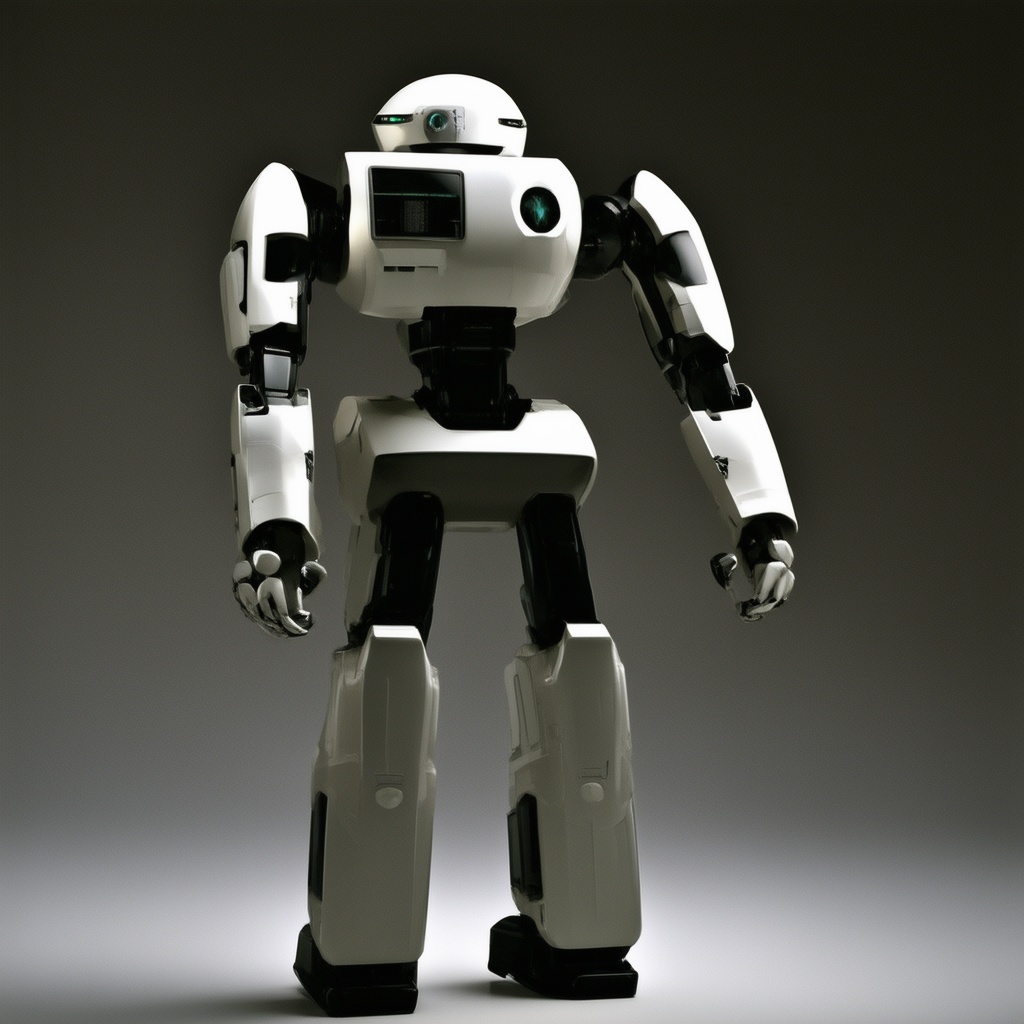}
                    };
                    \node[rotate=0, align=center, font=\normalsize, anchor=base, yshift=4pt] at (image.north) {+ \acrshort{method} (Ours) \strut};
                \end{tikzpicture}
            \end{subfigure}
        \end{minipage}
        \caption{Generated samples using {\sdthree} with 16 steps and $\wcfg=2$.}
        \label{fig:sd3-appendix}        
    \end{figure}

    \begin{figure}[t]
        \centering
        \begin{minipage}{\textwidth}
            \begin{subfigure}{0.495\textwidth}
                \begin{tikzpicture}
                    \node[anchor=south west, inner sep=0] (image) at (0,0) {
                            \includegraphics[width=0.495\textwidth]{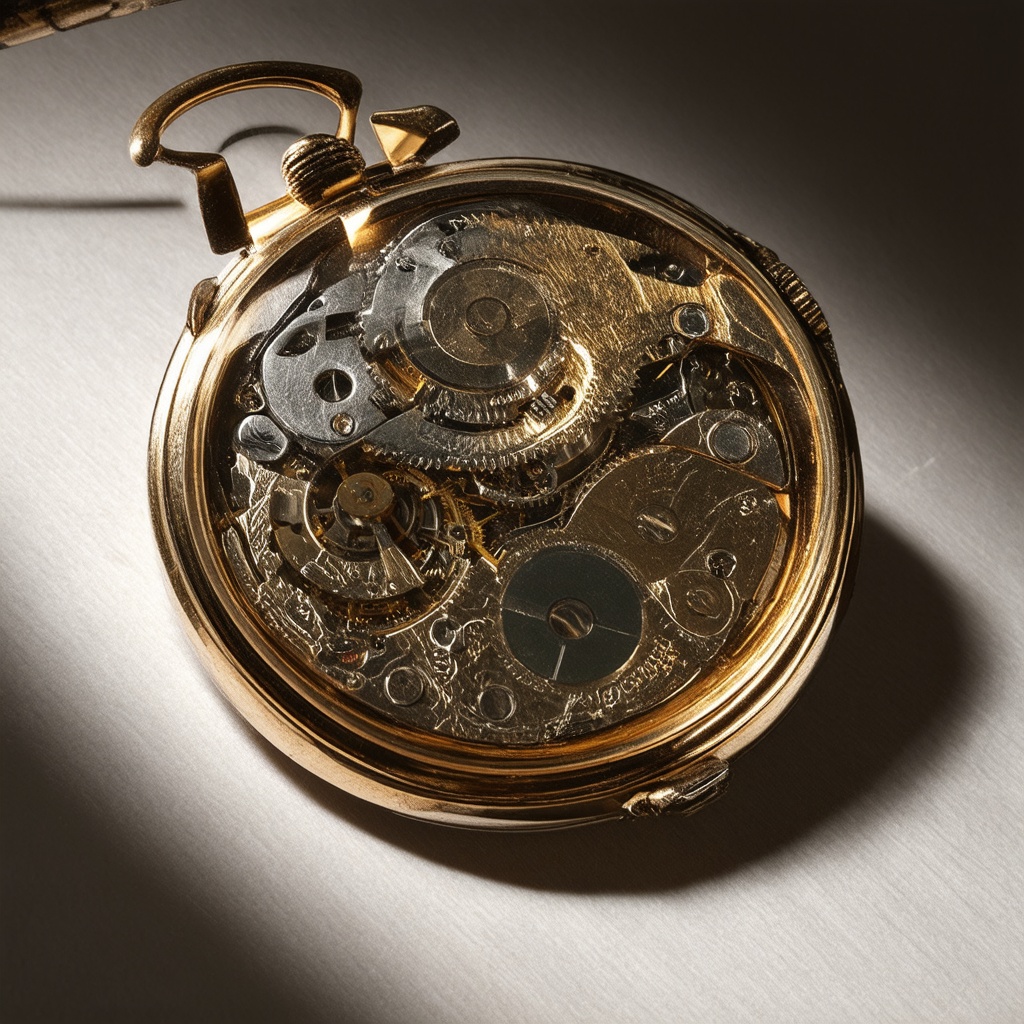}
                    };
                    \node[rotate=0, align=center, font=\normalsize, anchor=base, yshift=4pt] at (image.north) {CFG \strut};
                \end{tikzpicture}
                \hfil
                \begin{tikzpicture}
                    \node[anchor=south west, inner sep=0] (image) at (0,0) {
                \includegraphics[width=0.495\textwidth]{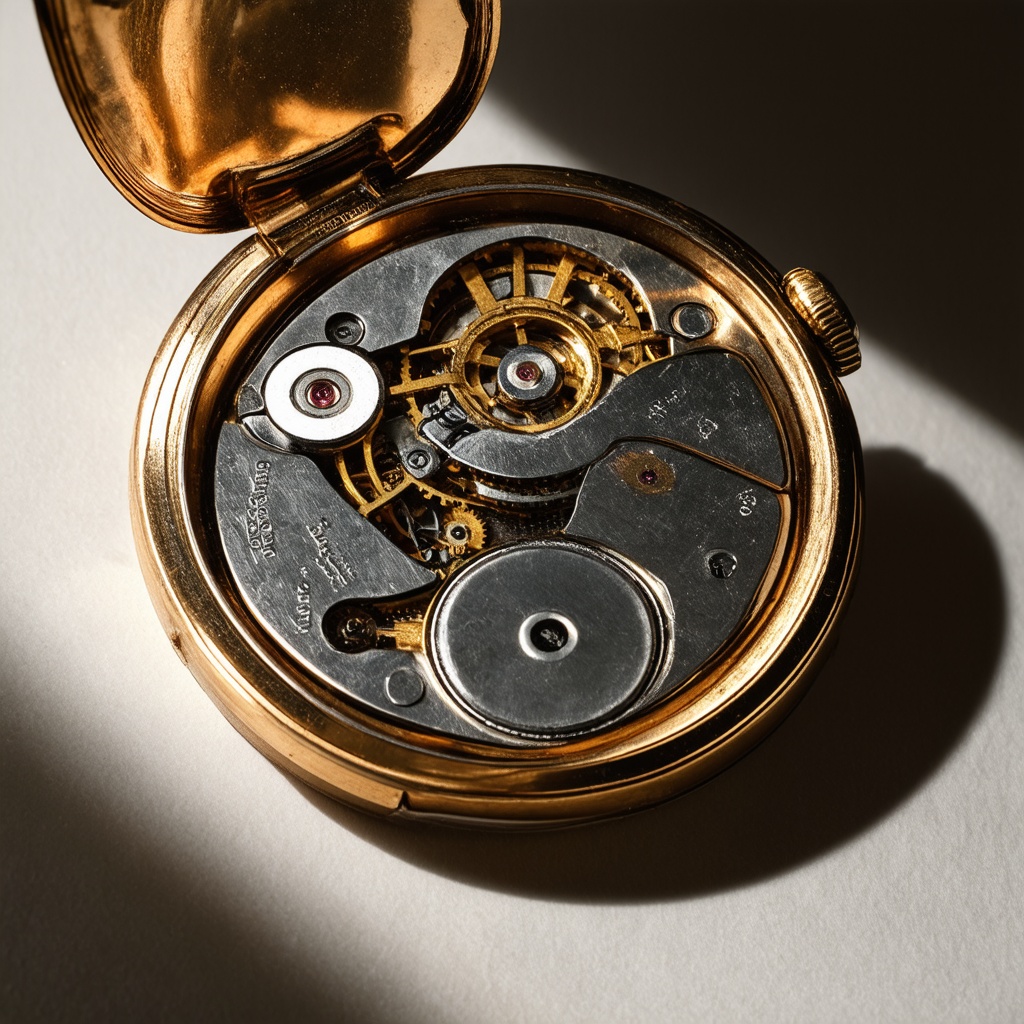}
                    };
                    \node[rotate=0, align=center, font=\normalsize, anchor=base, yshift=4pt] at (image.north) {+ \acrshort{method} (Ours) \strut};
                \end{tikzpicture}
            \end{subfigure}
            \hfil
            \begin{subfigure}{0.495\textwidth}
                \begin{tikzpicture}
                    \node[anchor=south west, inner sep=0] (image) at (0,0) {
                            \includegraphics[width=0.495\textwidth]{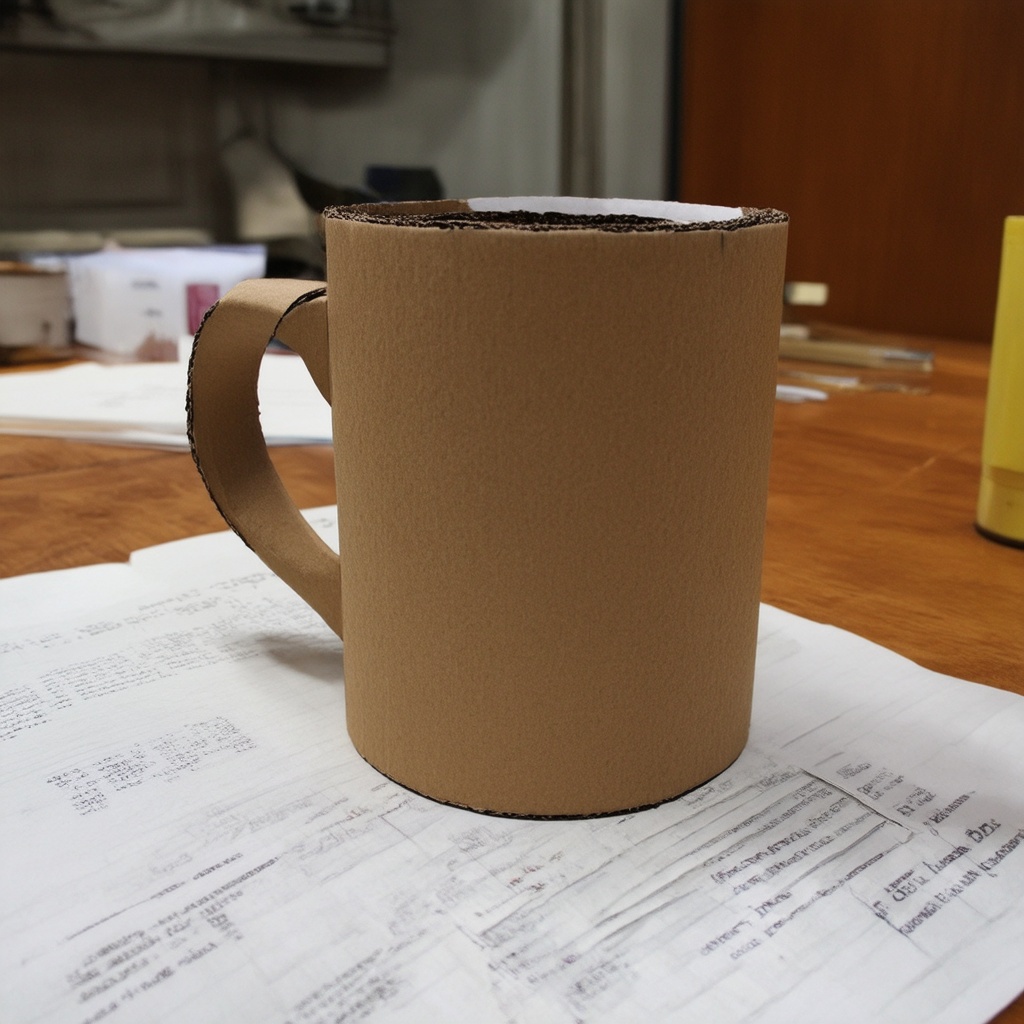}
                    };
                    \node[rotate=0, align=center, font=\normalsize, anchor=base, yshift=4pt] at (image.north) {CFG \strut};
                \end{tikzpicture}
                \hfil
                \begin{tikzpicture}
                    \node[anchor=south west, inner sep=0] (image) at (0,0) {
                \includegraphics[width=0.495\textwidth]{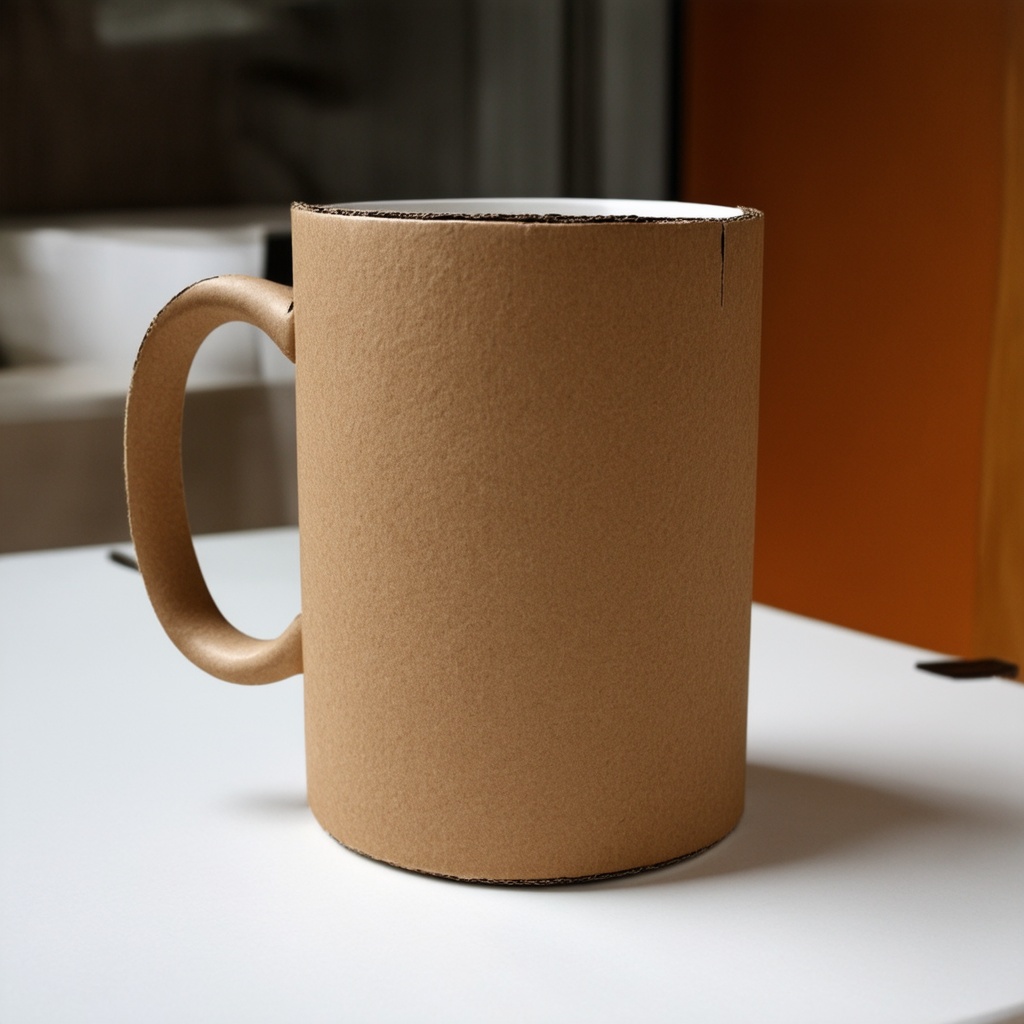}
                    };
                    \node[rotate=0, align=center, font=\normalsize, anchor=base, yshift=4pt] at (image.north) {+ \acrshort{method} (Ours) \strut};
                \end{tikzpicture}
            \end{subfigure}

            \begin{subfigure}[b]{0.495\textwidth}
                \begin{tikzpicture}
                    \node[anchor=south west, inner sep=0] (image) at (0,0) {
                            \includegraphics[width=0.495\textwidth]{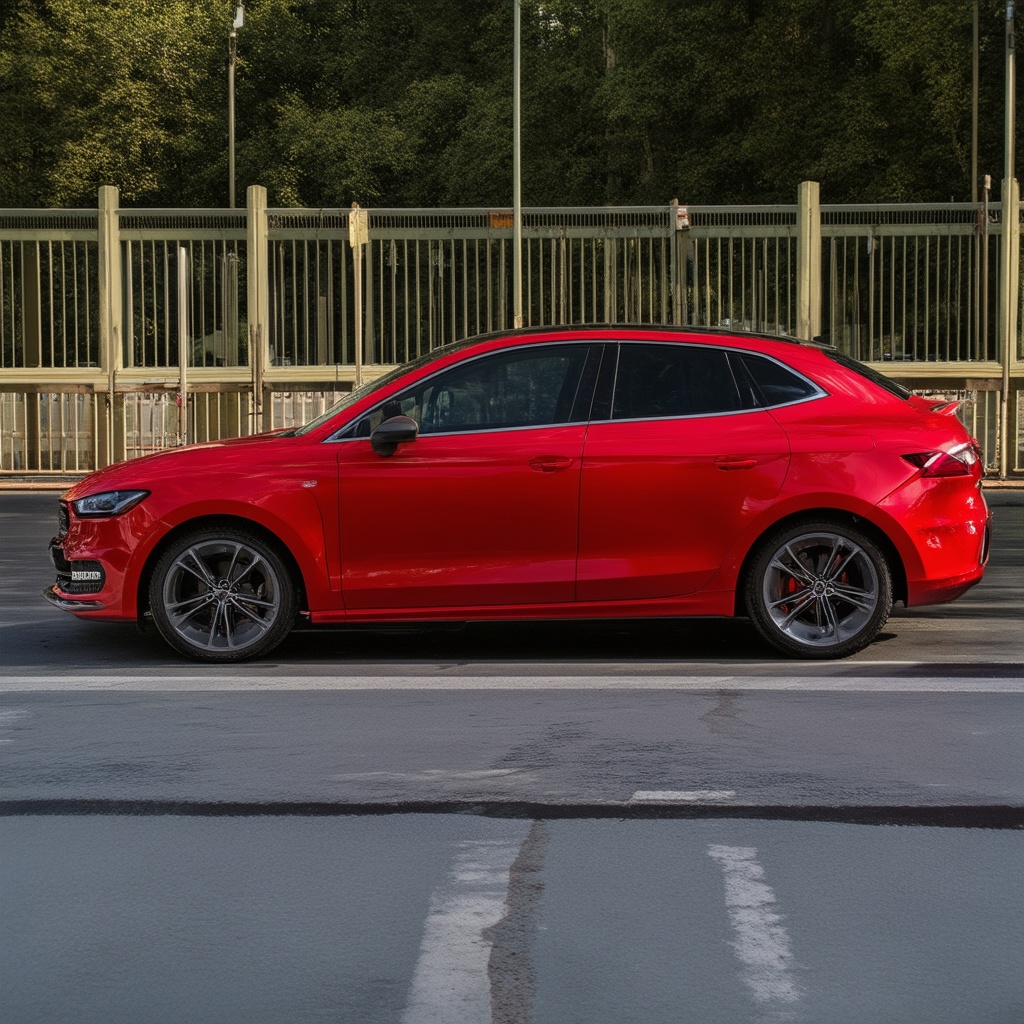}
                    };
                    \node[rotate=0, align=center, font=\normalsize, anchor=base, yshift=4pt] at (image.north) {CFG \strut};
                \end{tikzpicture}
                \hfil
                \begin{tikzpicture}
                    \node[anchor=south west, inner sep=0] (image) at (0,0) {
                \includegraphics[width=0.495\textwidth]{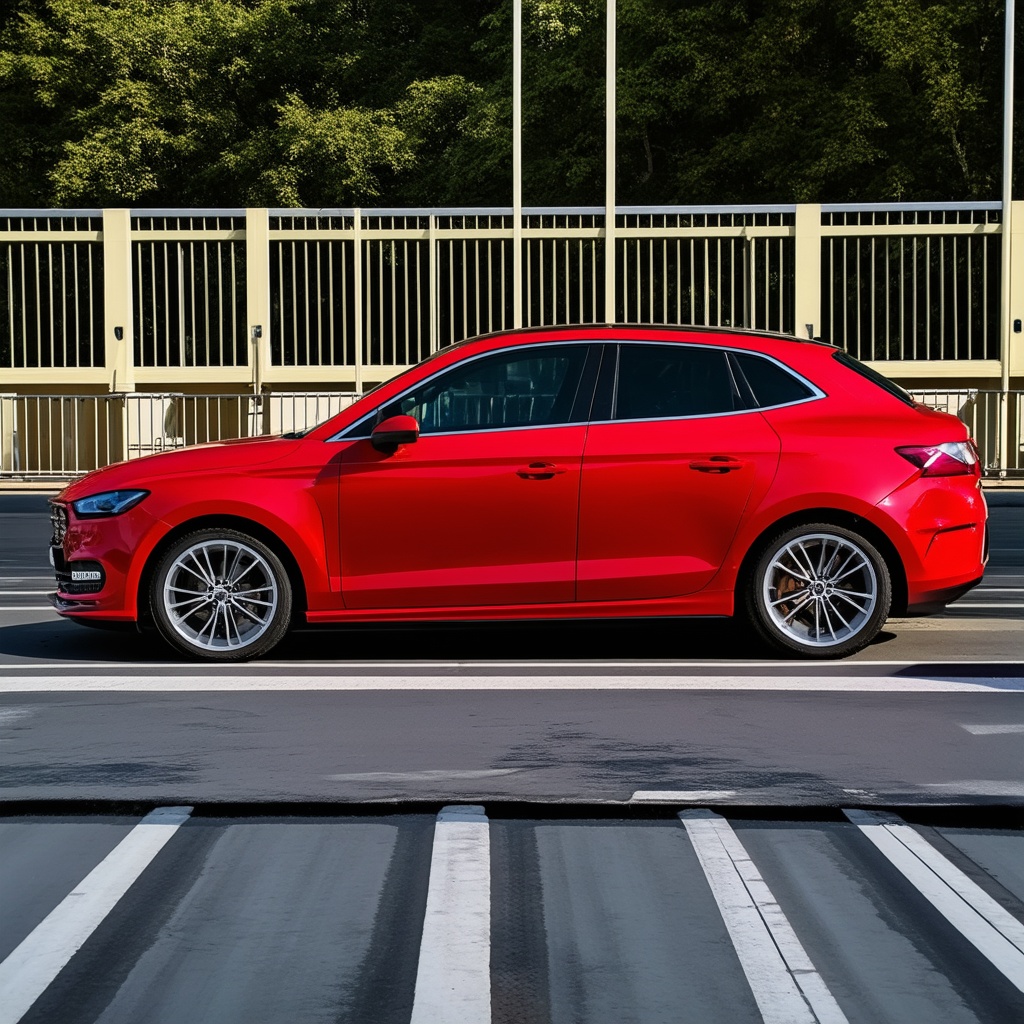}
                    };
                    \node[rotate=0, align=center, font=\normalsize, anchor=base, yshift=4pt] at (image.north) {+ \acrshort{method} (Ours) \strut};
                \end{tikzpicture}
            \end{subfigure}
            \hfil
            \begin{subfigure}[b]{0.495\textwidth}
                \begin{tikzpicture}
                    \node[anchor=south west, inner sep=0] (image) at (0,0) {
                            \includegraphics[width=0.495\textwidth]{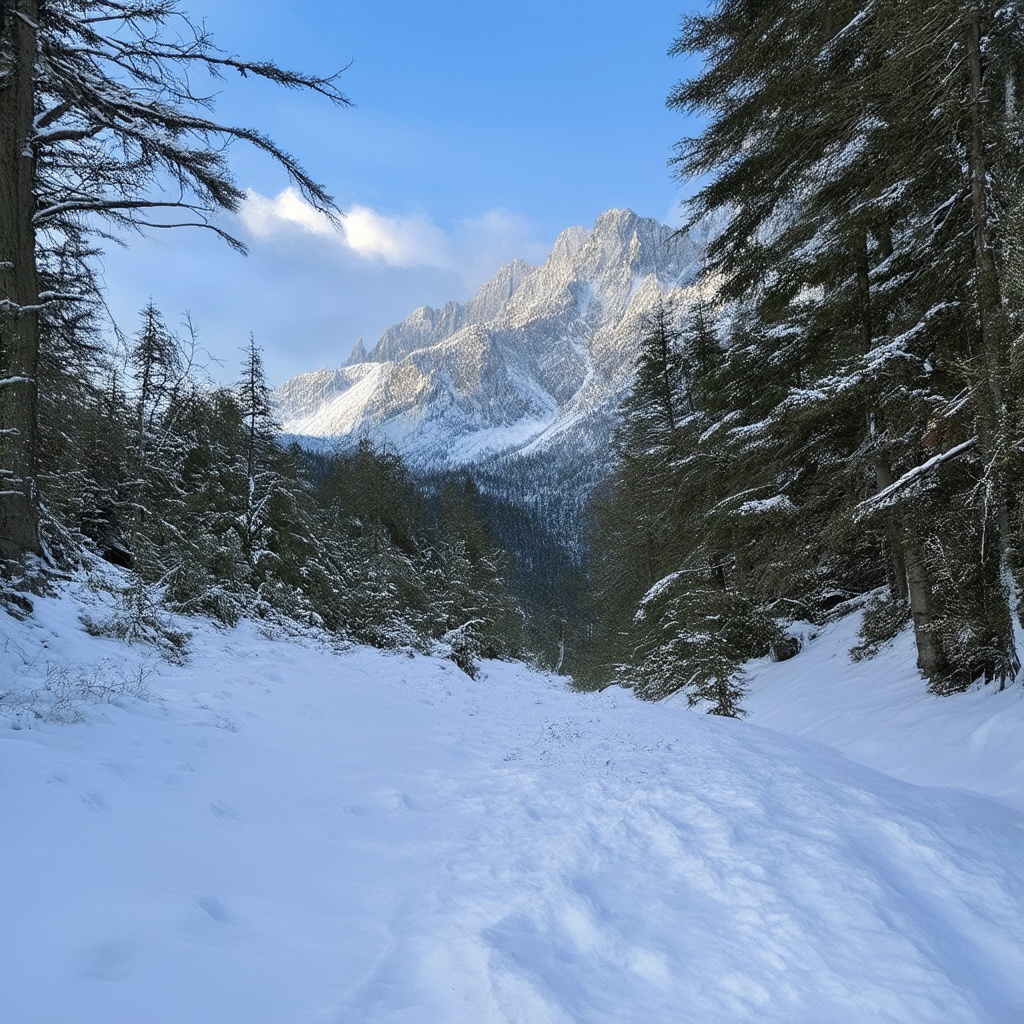}
                    };
                    \node[rotate=0, align=center, font=\normalsize, anchor=base, yshift=4pt] at (image.north) {CFG \strut};
                \end{tikzpicture}
                \hfil
                \begin{tikzpicture}
                    \node[anchor=south west, inner sep=0] (image) at (0,0) {
                \includegraphics[width=0.495\textwidth]{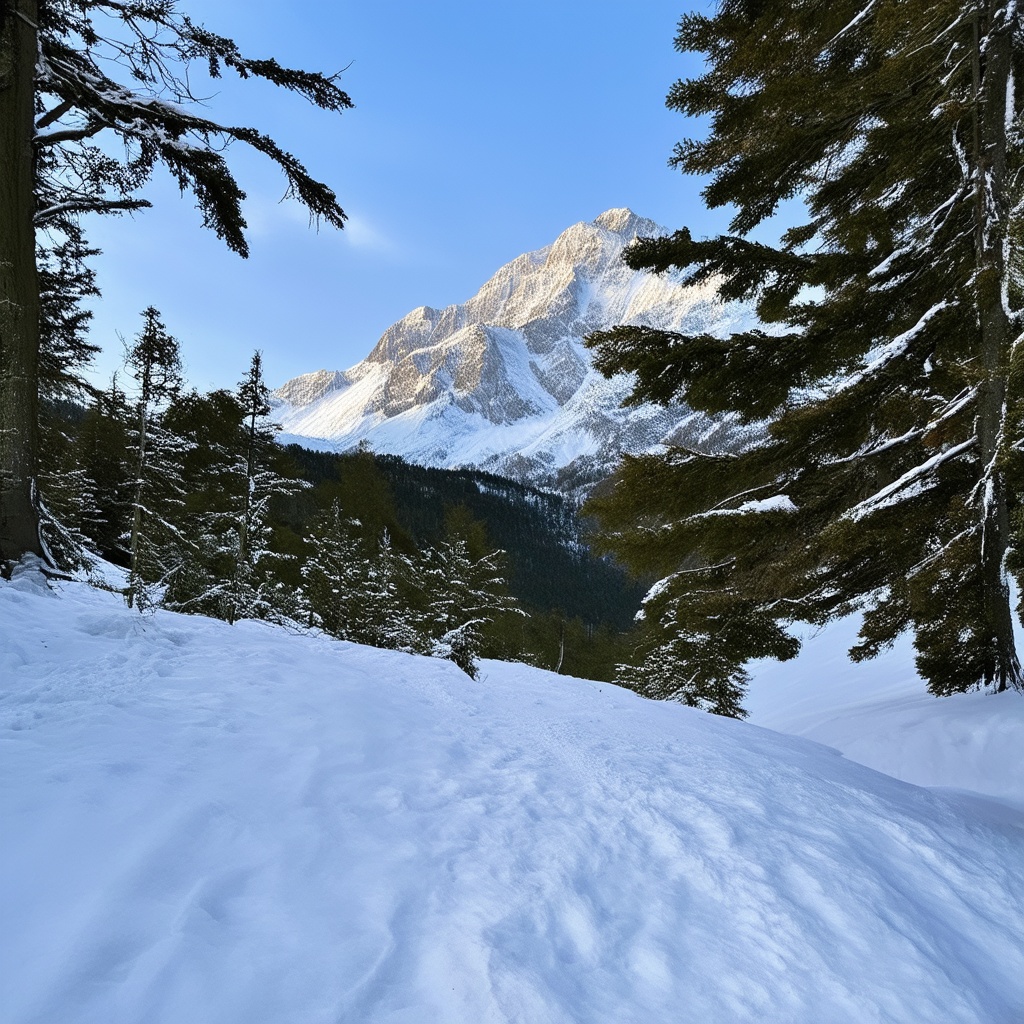}
                    };
                    \node[rotate=0, align=center, font=\normalsize, anchor=base, yshift=4pt] at (image.north) {+ \acrshort{method} (Ours) \strut};
                \end{tikzpicture}
            \end{subfigure}
        \end{minipage}
        \caption{Generated samples using {\sdthree} with 28 steps and $\wcfg=4.5$.}
        \label{fig:sd3-appendix-two}        
    \end{figure}

    \begin{figure}[t]
        \centering
        \begin{minipage}{\textwidth}
            \begin{subfigure}{0.495\textwidth}
                \begin{tikzpicture}
                    \node[anchor=south west, inner sep=0] (image) at (0,0) {
                            \includegraphics[width=0.495\textwidth]{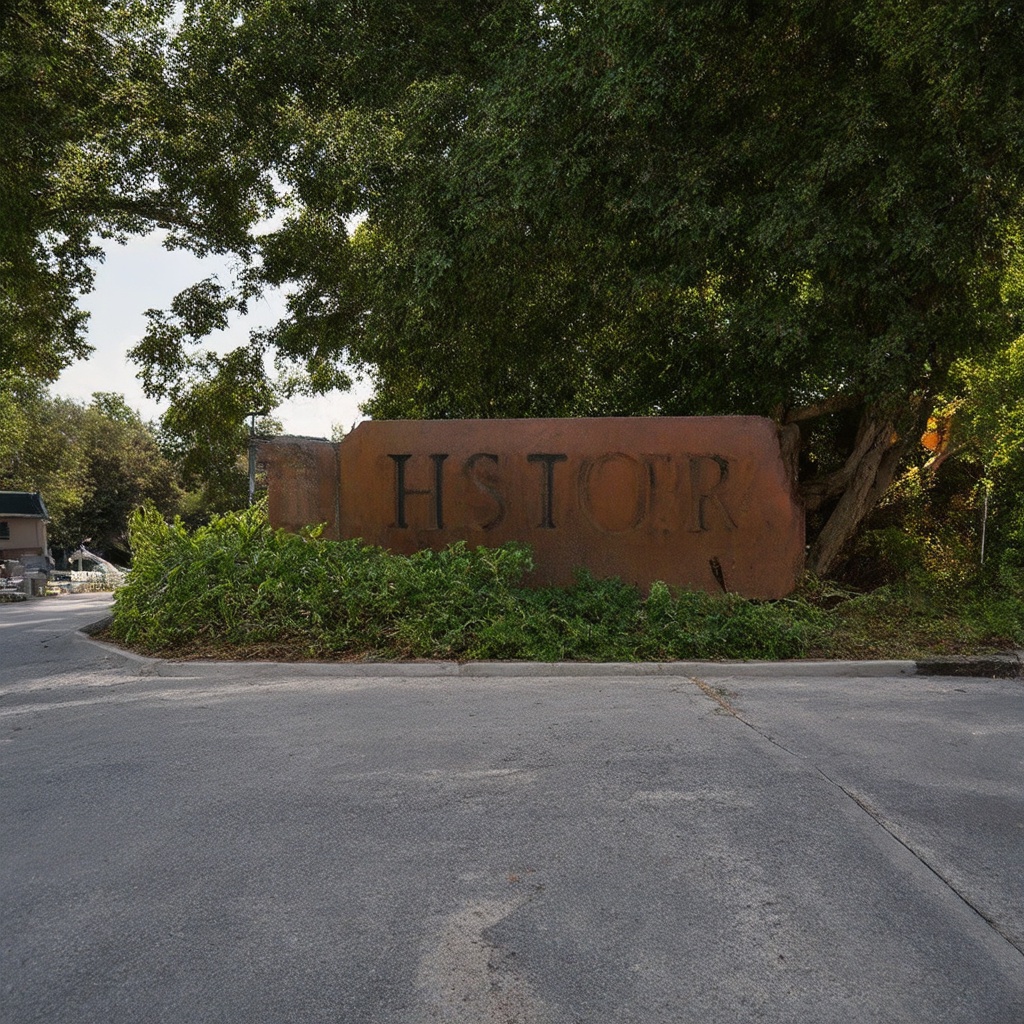}
                    };
                    \node[rotate=0, align=center, font=\normalsize, anchor=base, yshift=4pt] at (image.north) {CFG \strut};
                \end{tikzpicture}
                \hfil
                \begin{tikzpicture}
                    \node[anchor=south west, inner sep=0] (image) at (0,0) {
                \includegraphics[width=0.495\textwidth]{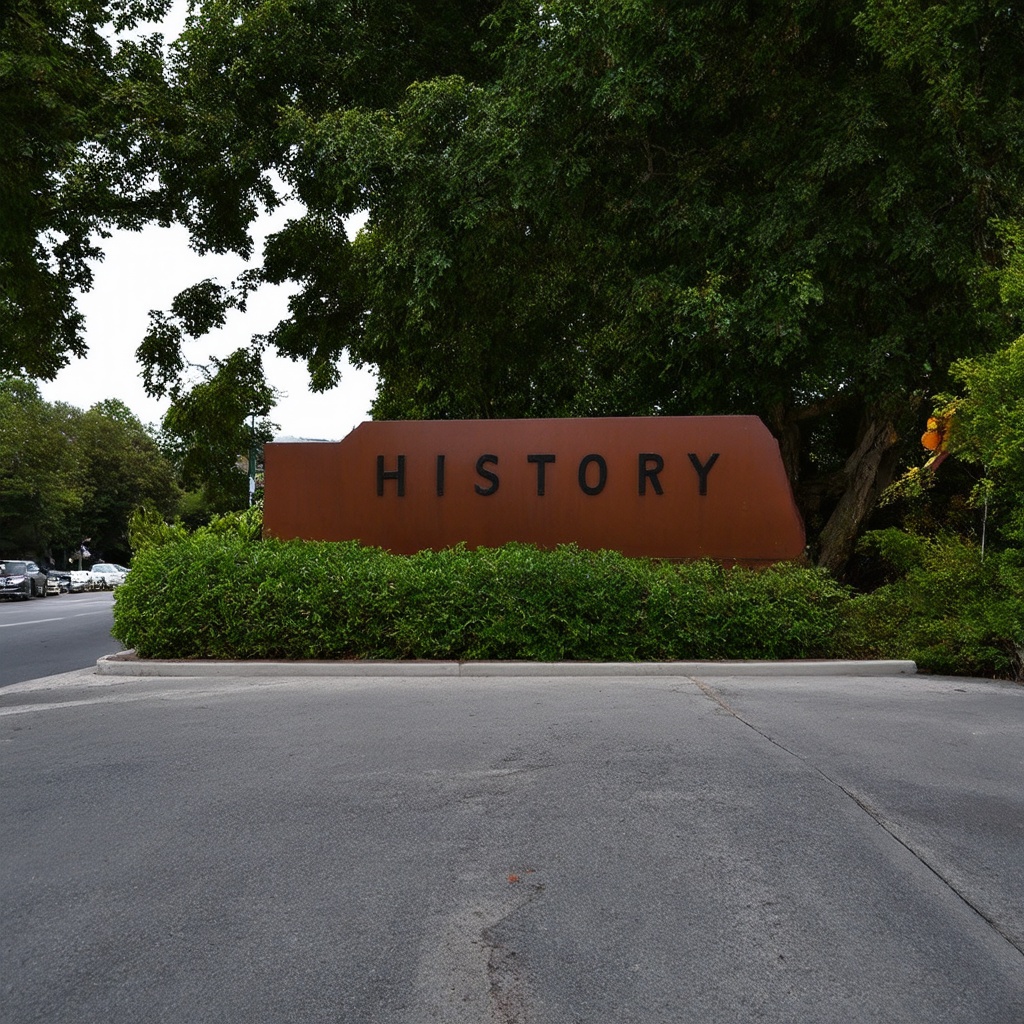}
                    };
                    \node[rotate=0, align=center, font=\normalsize, anchor=base, yshift=4pt] at (image.north) {+ \acrshort{method} (Ours) \strut};
                \end{tikzpicture}
            \end{subfigure}
            \hfil
            \begin{subfigure}{0.495\textwidth}
                \begin{tikzpicture}
                    \node[anchor=south west, inner sep=0] (image) at (0,0) {
                            \includegraphics[width=0.495\textwidth]{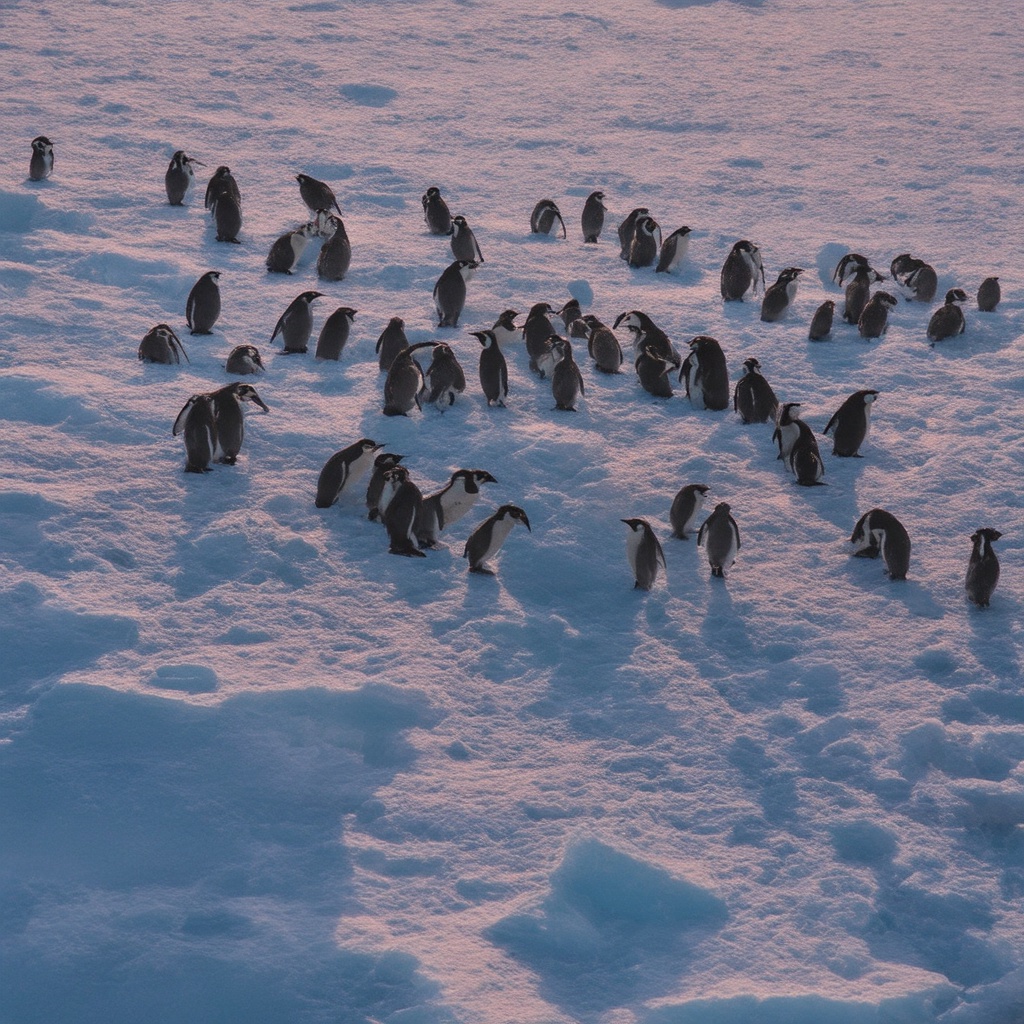}
                    };
                    \node[rotate=0, align=center, font=\normalsize, anchor=base, yshift=4pt] at (image.north) {CFG \strut};
                \end{tikzpicture}
                \hfil
                \begin{tikzpicture}
                    \node[anchor=south west, inner sep=0] (image) at (0,0) {
                \includegraphics[width=0.495\textwidth]{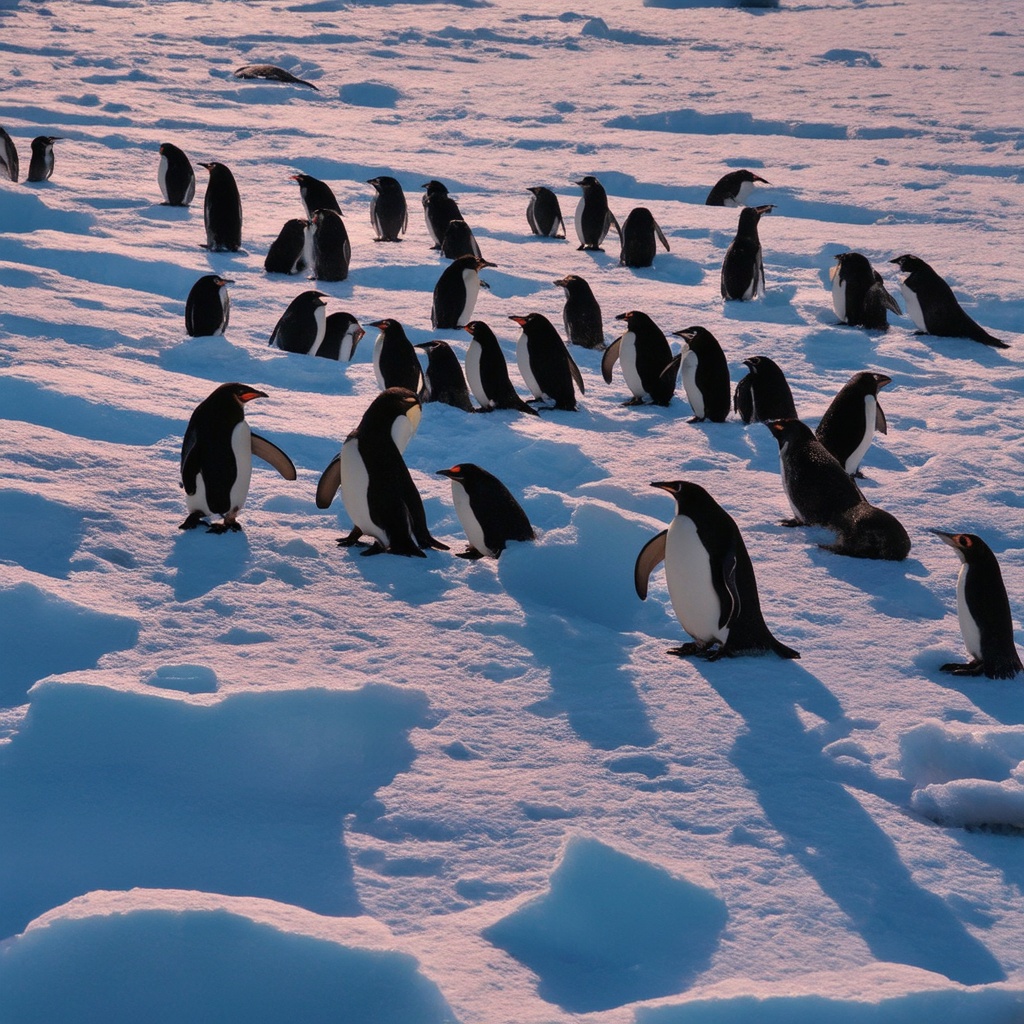}
                    };
                    \node[rotate=0, align=center, font=\normalsize, anchor=base, yshift=4pt] at (image.north) {+ \acrshort{method} (Ours) \strut};
                \end{tikzpicture}
            \end{subfigure}

            \begin{subfigure}[b]{0.495\textwidth}
                \begin{tikzpicture}
                    \node[anchor=south west, inner sep=0] (image) at (0,0) {
                            \includegraphics[width=0.495\textwidth]{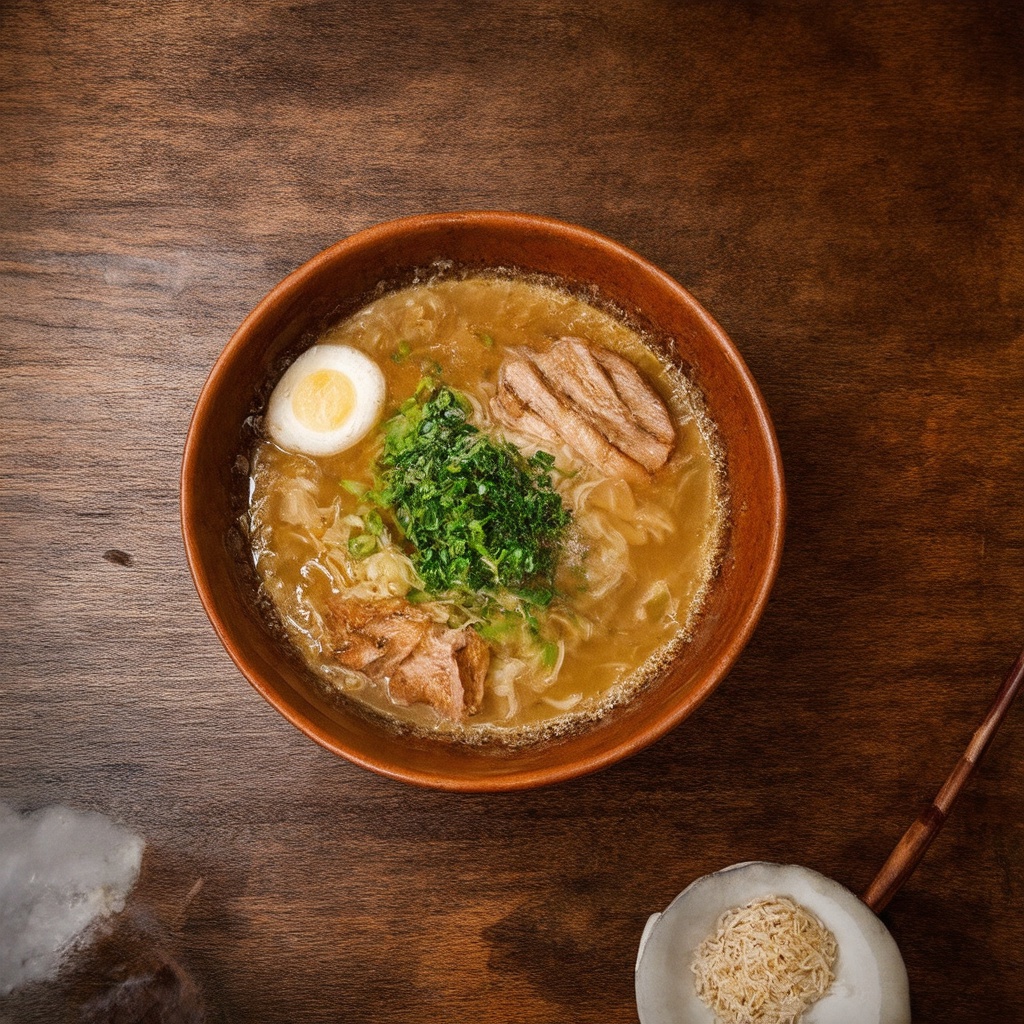}
                    };
                    \node[rotate=0, align=center, font=\normalsize, anchor=base, yshift=4pt] at (image.north) {CFG \strut};
                \end{tikzpicture}
                \hfil
                \begin{tikzpicture}
                    \node[anchor=south west, inner sep=0] (image) at (0,0) {
                \includegraphics[width=0.495\textwidth]{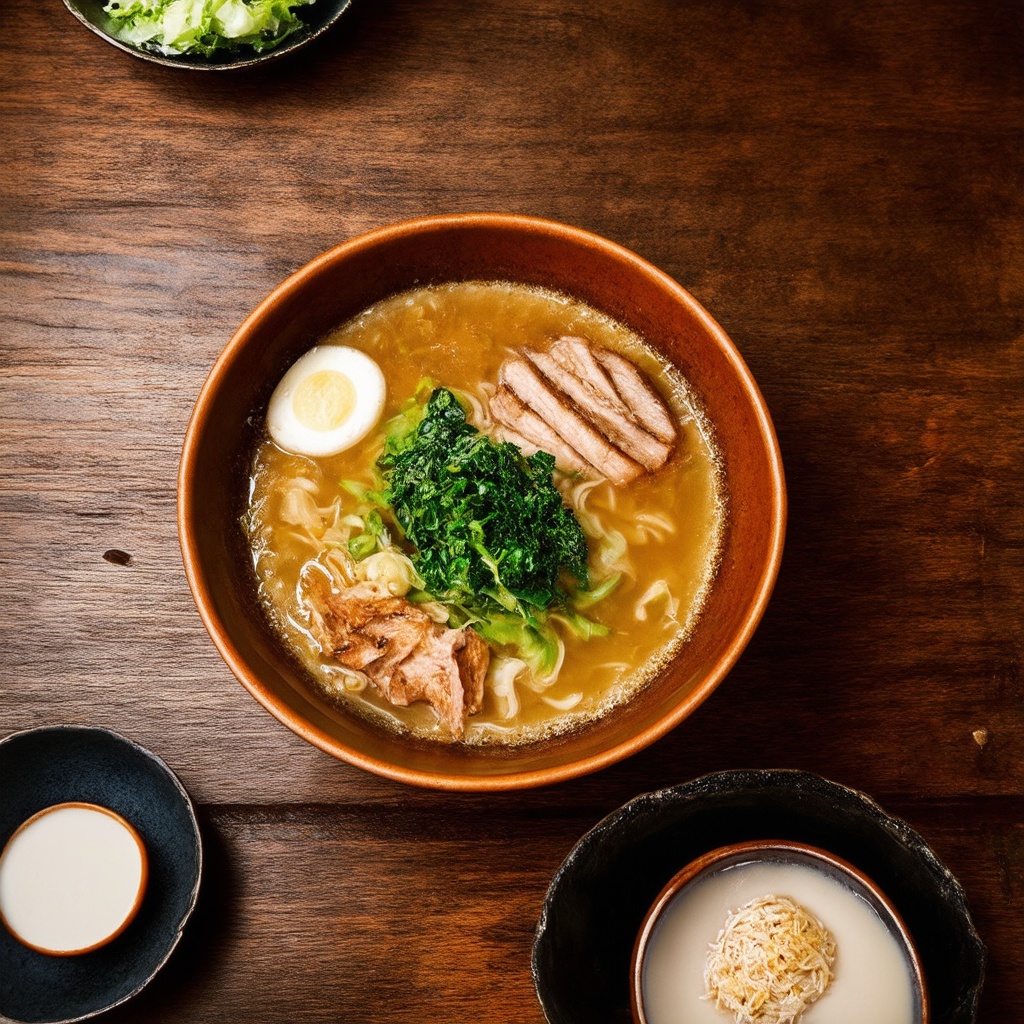}
                    };
                    \node[rotate=0, align=center, font=\normalsize, anchor=base, yshift=4pt] at (image.north) {+ \acrshort{method} (Ours) \strut};
                \end{tikzpicture}
            \end{subfigure}
            \hfil
            \begin{subfigure}[b]{0.495\textwidth}
                \begin{tikzpicture}
                    \node[anchor=south west, inner sep=0] (image) at (0,0) {
                            \includegraphics[width=0.495\textwidth]{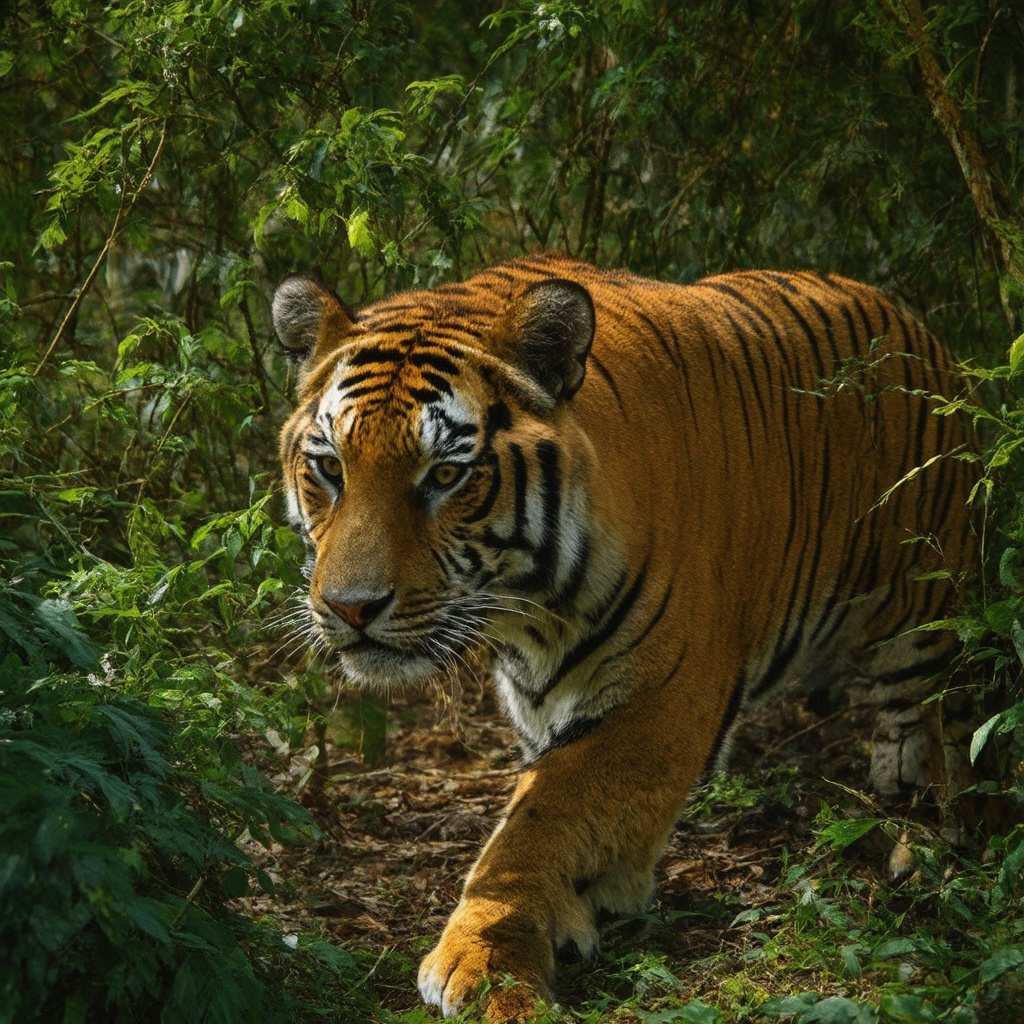}
                    };
                    \node[rotate=0, align=center, font=\normalsize, anchor=base, yshift=4pt] at (image.north) {CFG \strut};
                \end{tikzpicture}
                \hfil
                \begin{tikzpicture}
                    \node[anchor=south west, inner sep=0] (image) at (0,0) {
                \includegraphics[width=0.495\textwidth]{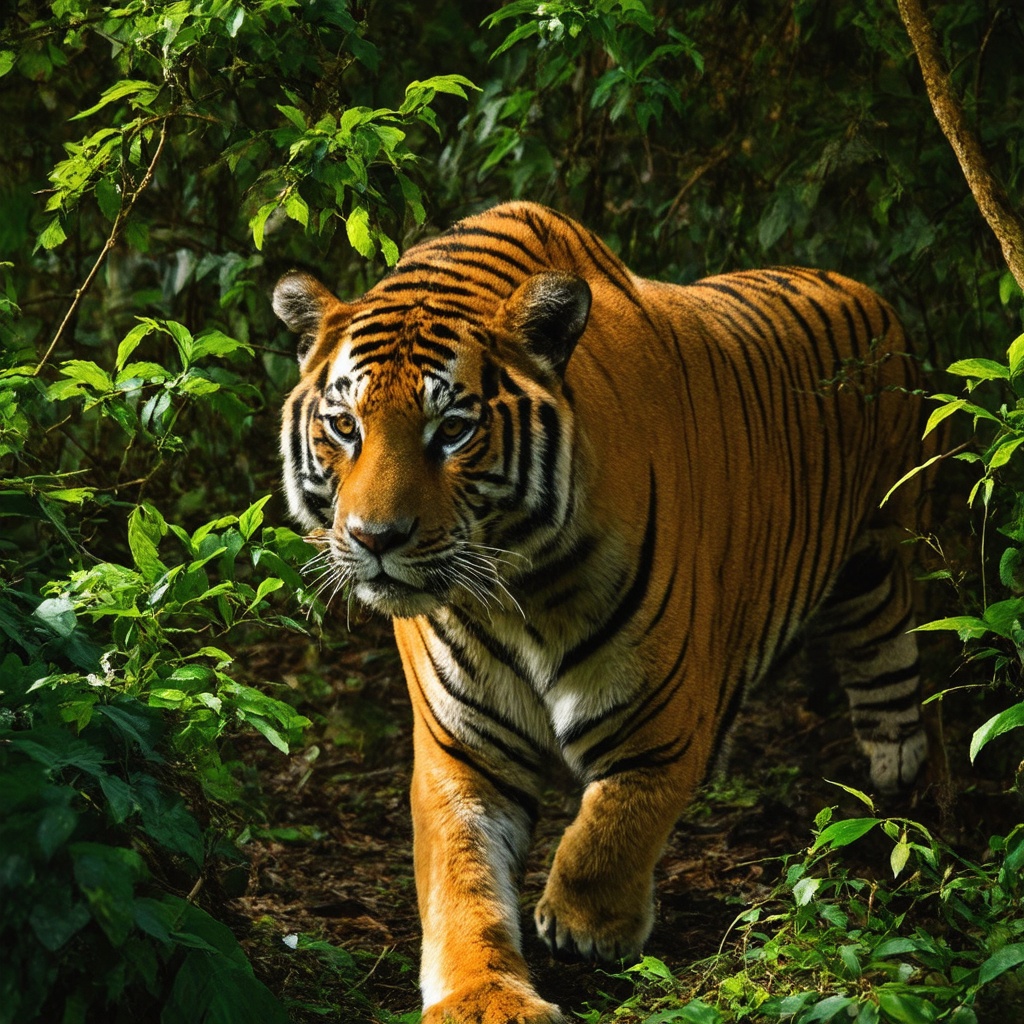}
                    };
                    \node[rotate=0, align=center, font=\normalsize, anchor=base, yshift=4pt] at (image.north) {+ \acrshort{method} (Ours) \strut};
                \end{tikzpicture}
            \end{subfigure}
        \end{minipage}
        \caption{Generated samples using {\sdthree} with 28 steps and $\wcfg=2$.}
        \label{fig:sd3-appendix-three}      
    \end{figure}
}

\newcommand{\figFluxDiversity}{
    \begin{figure}[t]
        \centering
        \begin{minipage}{\textwidth}
            \begin{subfigure}{\textwidth}
                \begin{tikzpicture}
                    \node[anchor=south west, inner sep=0] (image) at (0,0) {
                            \includegraphics[width=\textwidth]{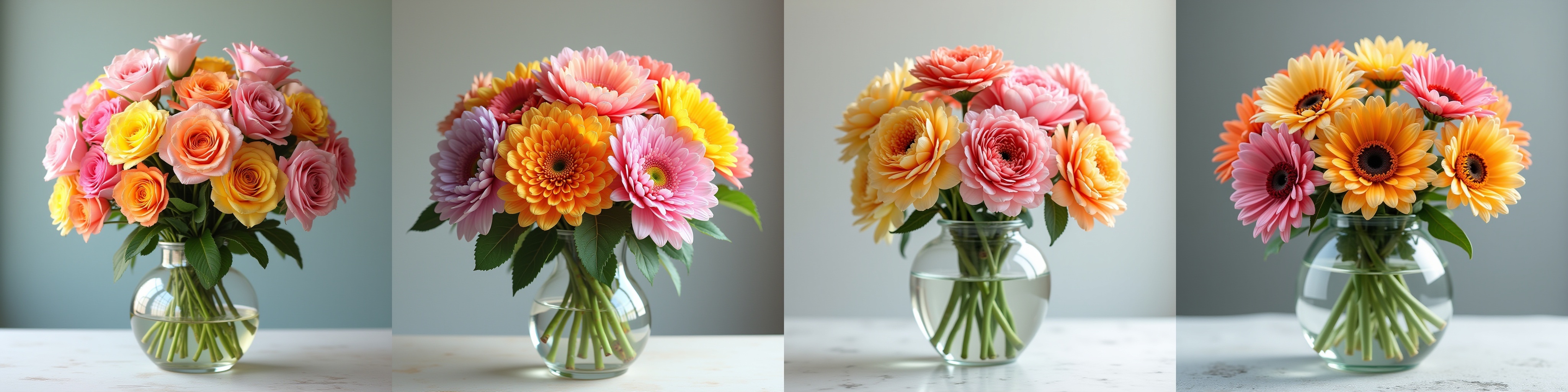}
                    };
                    \node[rotate=0, align=center, font=\normalsize, anchor=base, yshift=4pt] at (image.north) {High CFG \strut};
                \end{tikzpicture}
                \hfil
                \begin{tikzpicture}
                    \node[anchor=south west, inner sep=0] (image) at (0,0) {
                \includegraphics[width=\textwidth]{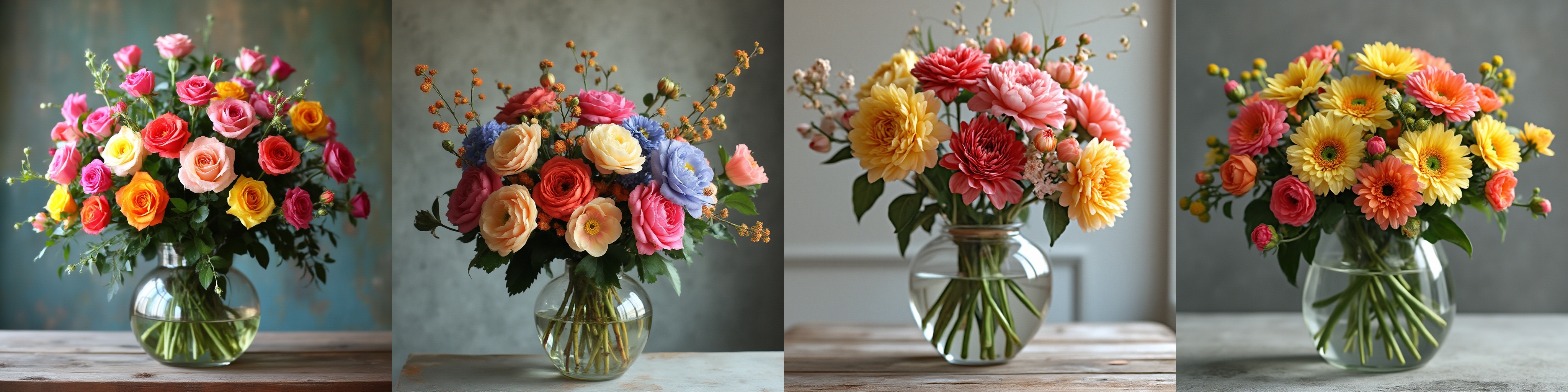}
                    };
                    \node[rotate=0, align=center, font=\normalsize, anchor=base, yshift=4pt] at (image.north) {Low CFG + \acrshort{method} (Ours) \strut};
                \end{tikzpicture}
            \end{subfigure}

            \begin{subfigure}{\textwidth}
                \begin{tikzpicture}
                    \node[anchor=south west, inner sep=0] (image) at (0,0) {
                            \includegraphics[width=\textwidth]{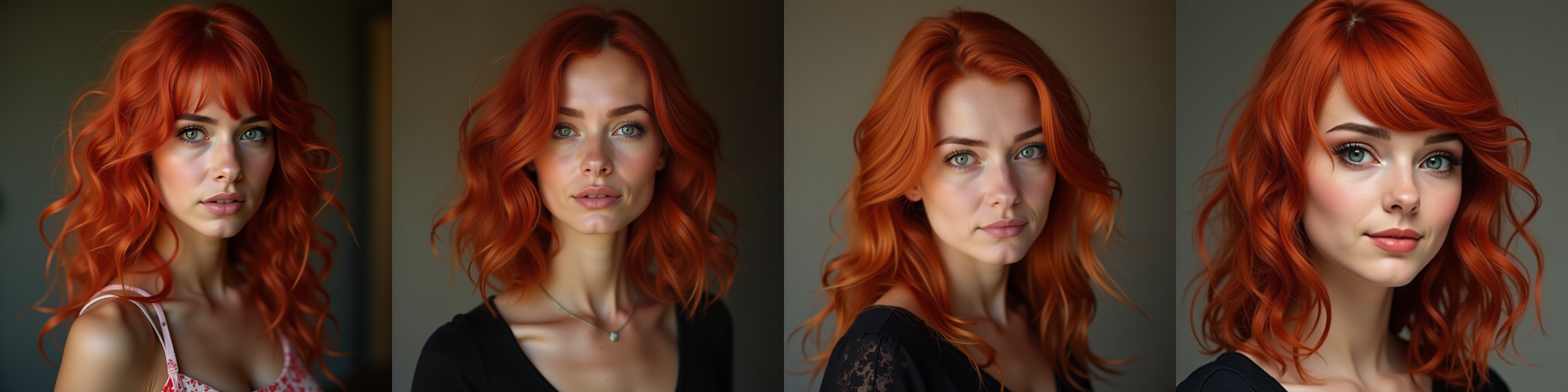}
                    };
                    \node[rotate=0, align=center, font=\normalsize, anchor=base, yshift=4pt] at (image.north) {High CFG \strut};
                \end{tikzpicture}
                \hfil
                \begin{tikzpicture}
                    \node[anchor=south west, inner sep=0] (image) at (0,0) {
                \includegraphics[width=\textwidth]{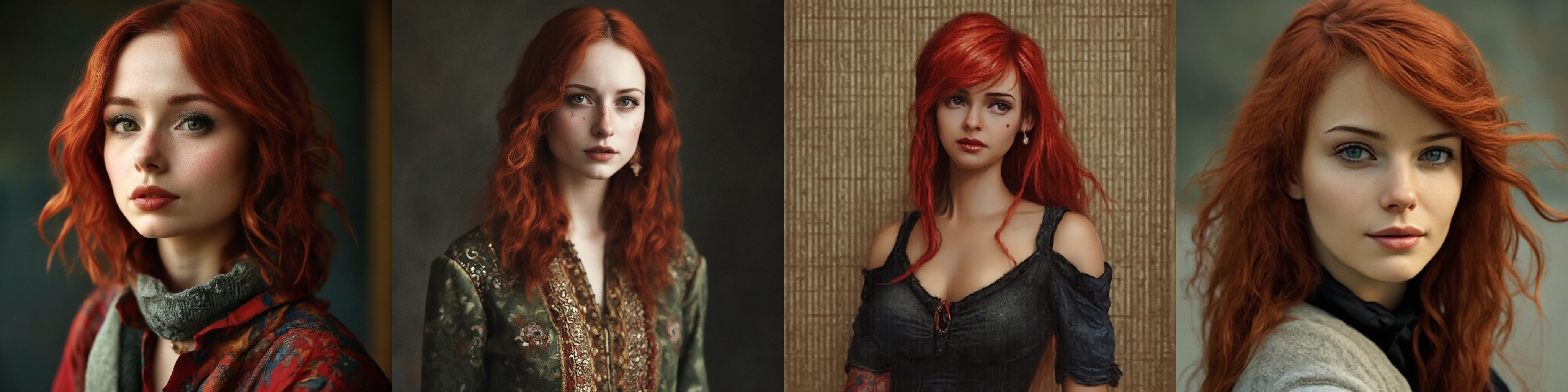}
                    };
                    \node[rotate=0, align=center, font=\normalsize, anchor=base, yshift=4pt] at (image.north) {Low CFG + \acrshort{method} (Ours) \strut};
                \end{tikzpicture}
            \end{subfigure}

        \end{minipage}
        \caption{High CFG scales improve overall structure of the image but lead to lack of diversity, oversaturation, and unrealistic effects. \gls{method} significantly improves the quality of lower CFG scales, leading to more realistic and diverse generations. Samples are generated with Flux \citep{flux2024}.}
        \label{fig:flux-appendix-diversity}        
    \end{figure}

    \begin{figure}[t]
        \centering
        \begin{minipage}{\textwidth}

            \begin{subfigure}{\textwidth}
                \begin{tikzpicture}
                    \node[anchor=south west, inner sep=0] (image) at (0,0) {
                            \includegraphics[width=\textwidth]{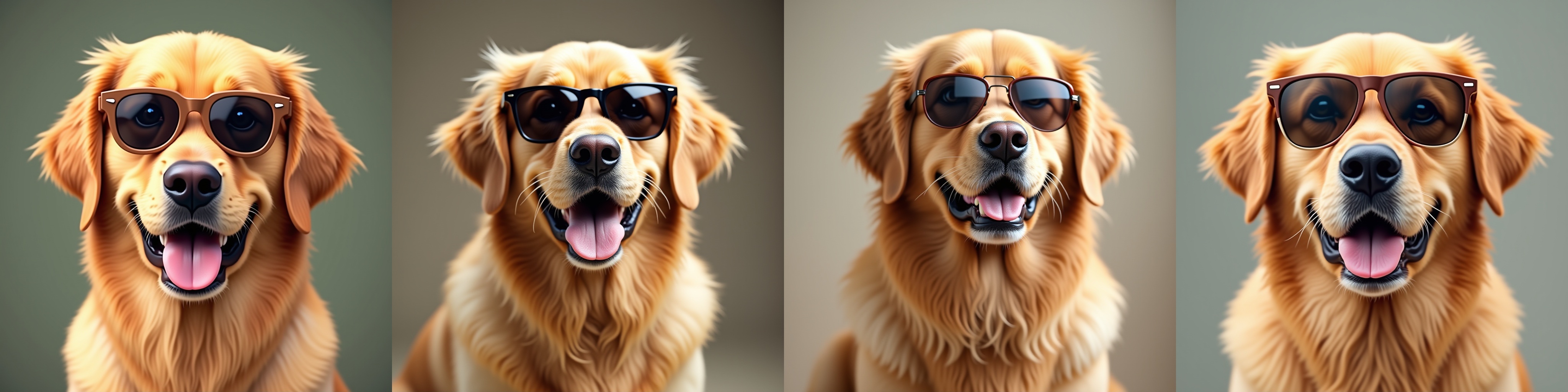}
                    };
                    \node[rotate=0, align=center, font=\normalsize, anchor=base, yshift=4pt] at (image.north) {High CFG \strut};
                \end{tikzpicture}
                \hfil
                \begin{tikzpicture}
                    \node[anchor=south west, inner sep=0] (image) at (0,0) {
                \includegraphics[width=\textwidth]{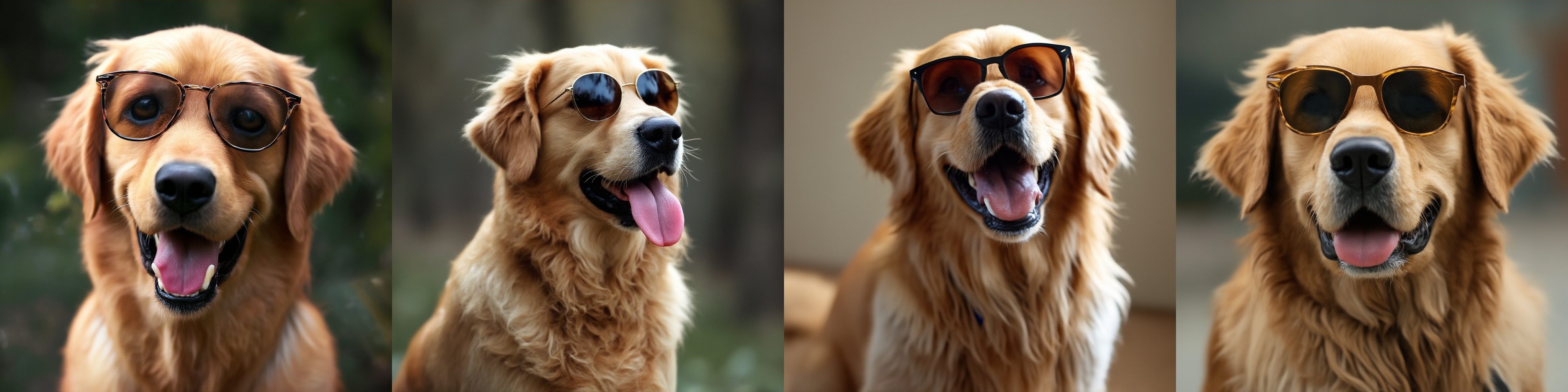}
                    };
                    \node[rotate=0, align=center, font=\normalsize, anchor=base, yshift=4pt] at (image.north) {Low CFG + \acrshort{method} (Ours) \strut};
                \end{tikzpicture}
            \end{subfigure}
            \hfil
            \begin{subfigure}{\textwidth}
                \begin{tikzpicture}
                    \node[anchor=south west, inner sep=0] (image) at (0,0) {
                            \includegraphics[width=\textwidth]{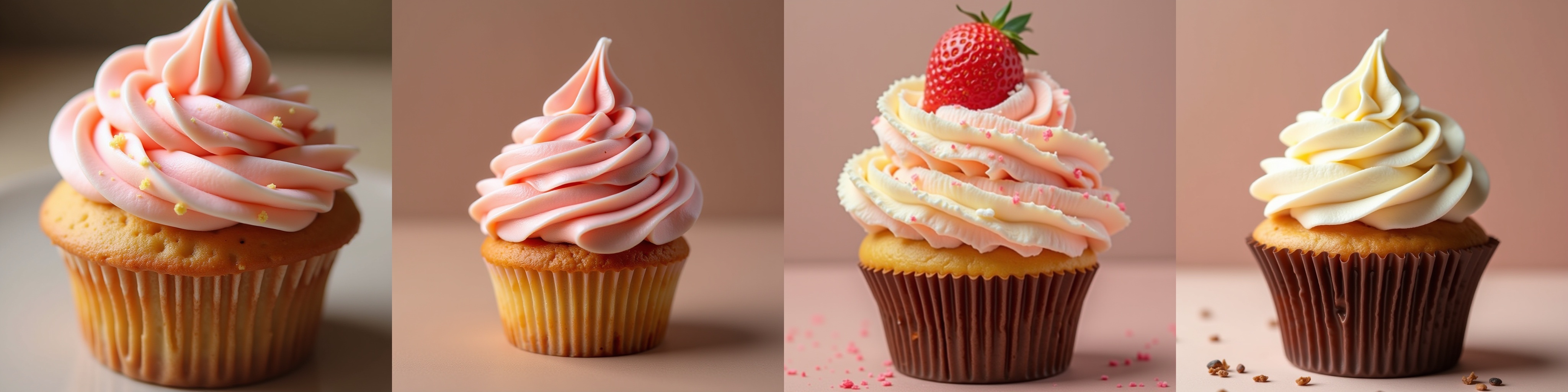}
                    };
                    \node[rotate=0, align=center, font=\normalsize, anchor=base, yshift=4pt] at (image.north) {High CFG \strut};
                \end{tikzpicture}
                \hfil
                \begin{tikzpicture}
                    \node[anchor=south west, inner sep=0] (image) at (0,0) {
                \includegraphics[width=\textwidth]{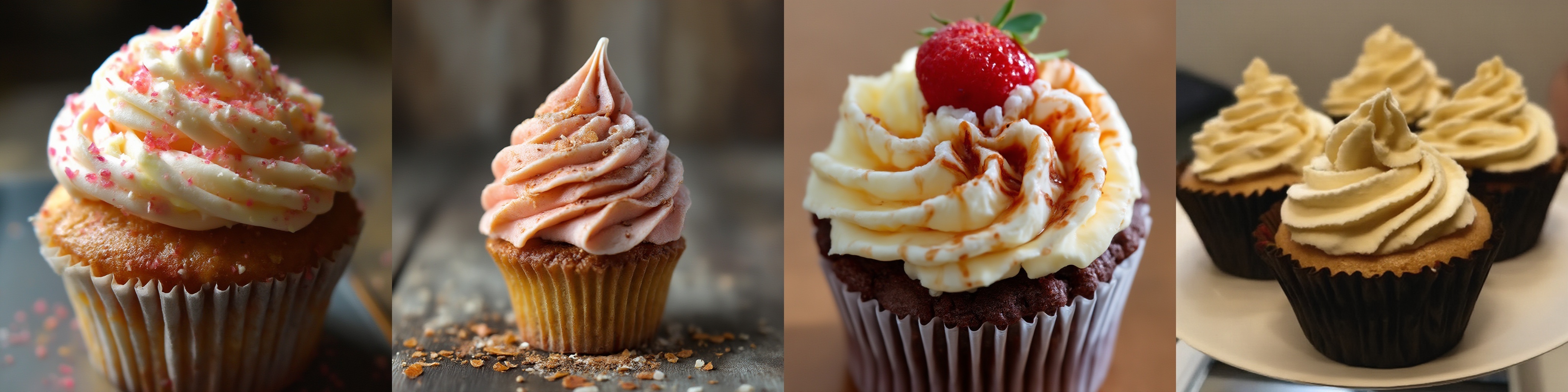}
                    };
                    \node[rotate=0, align=center, font=\normalsize, anchor=base, yshift=4pt] at (image.north) {Low CFG + \acrshort{method} (Ours) \strut};
                \end{tikzpicture}
            \end{subfigure}
        \end{minipage}
        \caption{High CFG scales improve overall structure of the image but lead to lack of diversity, oversaturation, and unrealistic effects. \gls{method} significantly improves the quality of lower CFG scales, leading to more realistic and diverse generations. Samples are generated with Flux \citep{flux2024}.}
        \label{fig:flux-appendix-diversity2}        
    \end{figure}
}
\newcommand{\tabResultSD}{
    \begin{table}[t!]
        \centering
        \caption{Evaluation of using \gls{method} with various Stable Diffusion models. \gls{method} improves all metrics reflecting human preference for image quality and prompt alignment across multiple benchmarks. All comparisons are conducted with the same NFE and CFG scale for fairness.}
        \label{tab:sd-results}
        
        \maxsizebox{\textwidth}{!}{
        \begin{tabular}{llcccc}
        \toprule
        Benchmark & Model & Guidance & ImageReward $\uparrow$ & HPSv2 $\uparrow$ & Win Rate $\uparrow$ \\
        \midrule
        \multirow{6}{*}{DrawBench \citep{saharia2022photorealistic}} 
            & \multirow{2}{*}{Stable Diffusion XL} & CFG & -0.091 & 0.224 & 0.07 \\
            & & \mycc +\gls{method} (Ours) & \mycc \textbf{0.148} & \mycc \textbf{0.249} & \mycc \textbf{0.93} \\
        \cmidrule{2-6}
            & \multirow{2}{*}{Stable Diffusion 3} & CFG & 0.491 & 0.257 & 0.18 \\
            & & \mycc +\gls{method} (Ours) & \mycc \textbf{0.621} & \mycc \textbf{0.272} & \mycc \textbf{0.82} \\
        \cmidrule{2-6}
            & \multirow{2}{*}{Stable Diffusion 3.5} & CFG & 0.621 & 0.258 & 0.21 \\
            & & \mycc +\gls{method} (Ours) & \mycc \textbf{0.702} & \mycc \textbf{0.270} & \mycc \textbf{0.79} \\
        \midrule
        \multirow{6}{*}{Parti Prompts \citep{yu2022scaling}} 
            & \multirow{2}{*}{Stable Diffusion XL} & CFG & 0.191 & 0.239 & 0.08 \\
            & & \mycc +\gls{method} (Ours) & \mycc \textbf{0.360} & \mycc \textbf{0.261} & \mycc \textbf{0.92} \\
        \cmidrule{2-6}
            & \multirow{2}{*}{Stable Diffusion 3} & CFG & 0.843 & 0.273 & 0.19 \\
            & & \mycc +\gls{method} (Ours) & \mycc \textbf{0.919} & \mycc \textbf{0.285} & \mycc \textbf{0.81} \\
        \cmidrule{2-6}
            & \multirow{2}{*}{Stable Diffusion 3.5} & CFG & 0.879 & 0.270 & 0.18 \\
            & & \mycc +\gls{method} (Ours) & \mycc \textbf{0.935} & \mycc \textbf{0.282} & \mycc \textbf{0.82} \\
        \midrule
        \multirow{6}{*}{HPS Prompts \citep{wu2023human}} 
            & \multirow{2}{*}{Stable Diffusion XL} & CFG & 0.327 & 0.245 & 0.04 \\
            & & \mycc +\gls{method} (Ours) & \mycc \textbf{0.515} & \mycc \textbf{0.275} & \mycc \textbf{0.96} \\
        \cmidrule{2-6}
            & \multirow{2}{*}{Stable Diffusion 3} & CFG & 0.820 & 0.279 & 0.21 \\
            & & \mycc +\gls{method} (Ours) & \mycc \textbf{0.901} & \mycc \textbf{0.291} & \mycc \textbf{0.79} \\
        \cmidrule{2-6}
            & \multirow{2}{*}{Stable Diffusion 3.5} & CFG & 0.821 & 0.274 & 0.18 \\
            & & \mycc +\gls{method} (Ours) & \mycc \textbf{0.889} & \mycc \textbf{0.289} & \mycc \textbf{0.82} \\
        \bottomrule
    \end{tabular}
        }
    \end{table}
}

\newcommand{\tabResultDistilled}{
    \begin{table}[t!]
        \centering
        \caption{Effect of adding \gls{method} to the sampling process of distilled models. \gls{method} improves the quality of these models, showing that its effects are complementary to diffusion distillation.}
        \label{tab:distillation}
        
        \maxsizebox{\textwidth}{!}{
        \begin{tabular}{llcccccc}
        \toprule
        Model & Guidance & ImageReward $\uparrow$ & HPSv2 $\uparrow$ & Win Rate $\uparrow$ & CLIP Score $\uparrow$  \\
        \midrule
        \multirow{2}{*}{SDXL-Flash \citep{sdxl-flash}}
             & CFG & 0.774 & 0.273 & 0.03 & 0.332 \\
            & \mycc +\gls{method} (Ours) & \mycc \textbf{0.864} & \mycc \textbf{0.298} & \mycc \textbf{0.97} & \mycc \textbf{0.333} \\
        \midrule
        \multirow{2}{*}{SDXL-Lightning \citep{lin2024sdxllightning}}
             & CFG & 0.63 & 0.277 & 0.18 & \textbf{0.318} \\
            & \mycc +\gls{method} (Ours) & \mycc \textbf{0.66} & \mycc \textbf{0.285} & \mycc \textbf{0.82} & \mycc {0.317} \\
        \bottomrule
    \end{tabular}
        }
    \end{table}
}

\newcommand{\tabResultREPA}{
    \begin{table}[t!]
        \centering
        \caption{Effect of adding \gls{method} to recent state-of-the-art methods for conditional ImageNet generation at 256$\times$256. \gls{method} significantly improves sampling speed, matching the FID of the original models in just 30–40 steps. Moreover, \gls{method} achieves a new state-of-the-art FID of 1.61 for conditional generation without CFG using the SiT-XL + REPA-E model \citep{leng2025repae}.}
        \label{tab:repa}
        
        \maxsizebox{\textwidth}{!}{
    \begin{tabular}{llccccc}
        \toprule
        Model & Guidance & \# Steps $\downarrow$ & FID $\downarrow$ & IS $\uparrow$ & Precision $\uparrow$ & Recall $\uparrow$ \\
        \midrule
        \multirow{4}{*}{REPA-E \citep{leng2025repae}}
        & Unguided & 250 & 1.83 & 217.30 & 0.77 & \textbf{0.66}\\
        & \mycc +\gls{method} (Ours) & \mycc \textbf{30} & \mycc \textbf{1.61} & \mycc \textbf{240.75} & \mycc \textbf{0.81}  & \mycc {0.62} \\
        \cmidrule{2-7}
             & CFG & 250 & 1.26 & \textbf{314.90} & 0.79 & \textbf{0.66} \\
            & \mycc +\gls{method} (Ours) & \mycc \textbf{40} & \mycc {1.32} & \mycc {306.10} & \mycc \textbf{0.80}  & \mycc {0.65} \\
        \midrule
        \multirow{4}{*}{REPA \citep{yu2024repa}}
        & Unguided & 250 & 5.90 & 157.80 & 0.70 & \textbf{0.69}\\
        & \mycc +\gls{method} (Ours) & \mycc \textbf{40} & \mycc \textbf{5.43} & \mycc \textbf{165.91} & \mycc \textbf{0.71}  & \mycc {0.68} \\
        \cmidrule{2-7}
             & CFG & 250 & \textbf{1.42} & 305.70 & 0.80 & \textbf{0.65} \\
            & \mycc +\gls{method} (Ours) & \mycc \textbf{40} & \mycc {1.44} & \mycc \textbf{306.80} & \mycc \textbf{0.80}  & \mycc {0.64} \\
        \bottomrule
    \end{tabular}
        }
    \end{table}
}

\newcommand{\tabResultSDClip}{
    \begin{table}[t!]
        \centering
        \caption{Comparison of CLIP scores across datasets and models. \gls{method} consistently improves generation quality (see \Cref{tab:sd-results}) while maintaining virtually identical CLIP scores. This demonstrates that the quality gains are achieved without sacrificing prompt alignment.}
        \label{tab:sd-results-clip}
        
        \maxsizebox{\textwidth}{!}{
        \begin{tabular}{lllc}
                \toprule
                Benchmark & Model & Guidance & CLIP Score $\uparrow$ \\
                \midrule

                \multirow{6}{*}{DrawBench \citep{saharia2022photorealistic}}
                    & \multirow{2}{*}{Stable Diffusion XL} & CFG & 0.308 \\
                    & & \mycc +\gls{method} (Ours) & \mycc \textbf{0.309} \\
                    \cmidrule{2-4}
                    & \multirow{2}{*}{Stable Diffusion 3} & CFG & 0.329 \\
                    & & \mycc +\gls{method} (Ours) & \mycc 0.327 \\
                    \cmidrule{2-4}
                    & \multirow{2}{*}{Stable Diffusion 3.5} & CFG & \textbf{0.334} \\
                    & & \mycc +\gls{method} (Ours) & \mycc 0.330 \\
                \midrule

                \multirow{6}{*}{Parti Prompts \citep{yu2022scaling}}
                    & \multirow{2}{*}{Stable Diffusion XL} & CFG & 0.317 \\
                    & & \mycc +\gls{method} (Ours) & \mycc \textbf{0.318} \\
                    \cmidrule{2-4}
                    & \multirow{2}{*}{Stable Diffusion 3} & CFG & \textbf{0.327} \\
                    & & \mycc +\gls{method} (Ours) & \mycc 0.325 \\
                    \cmidrule{2-4}
                    & \multirow{2}{*}{Stable Diffusion 3.5} & CFG & \textbf{0.330} \\
                    & & \mycc +\gls{method} (Ours) & \mycc 0.329 \\
                \midrule

                \multirow{6}{*}{HPS Prompts \citep{wu2023human}}
                    & \multirow{2}{*}{Stable Diffusion XL} & CFG & 0.333 \\
                    & & \mycc +\gls{method} (Ours) & \mycc \textbf{0.336} \\
                    \cmidrule{2-4}
                    & \multirow{2}{*}{Stable Diffusion 3} & CFG & \textbf{0.332} \\
                    & & \mycc +\gls{method} (Ours) & \mycc 0.330 \\
                    \cmidrule{2-4}
                    & \multirow{2}{*}{Stable Diffusion 3.5} & CFG & \textbf{0.337} \\
                    & & \mycc +\gls{method} (Ours) & \mycc 0.336 \\
                \bottomrule
            \end{tabular}
        }
    \end{table}
}

\newcommand{\tabResultFID}{
    \begin{table}[t]
        \centering
        \caption{Quantitative evaluation of the effect of \gls{method} on sampling. \gls{method} consistently improves all metrics across different models, demonstrating higher generation quality than the CFG baseline. For fairness, sampling with and without \gls{method} is performed using the same NFE and CFG scale.}
        \label{tab:fid-results}
        \maxsizebox{\textwidth}{!}{
            \small
            \begin{tabular}{llcccc}
                \toprule
                Model & Guidance & FID $\downarrow$ & IS $\uparrow$ & Precision $\uparrow$ & Recall $\uparrow$ \\
                \midrule

                \multirow{2}{*}{SiT-XL + REPA \citep{yu2024repa}} 
                & CFG & 12.08 & 187.11 & 0.68 & \textbf{0.73} \\
                & \mycc \gls{method} (Ours) & \mycc \textbf{4.86} & \mycc \textbf{277.20} & \mycc \textbf{0.80} & \mycc {0.70} \\
                \midrule

                \multirow{2}{*}{DiT-XL/2 \citep{peeblesScalableDiffusionModels2022}} 
                & CFG & 8.73 & 173.21 & 0.72 & 0.68 \\
                & \mycc \gls{method} (Ours) & \mycc \textbf{7.15} & \mycc \textbf{180.05} & \mycc \textbf{0.75} & \mycc \textbf{0.71} \\
                \midrule

                \multirow{2}{*}{Stable Diffusion XL \citep{sdxl}} 
                & CFG & 28.49 & 35.07 & 0.56 & 0.54 \\
                & \mycc +\gls{method} (Ours) & \mycc \textbf{26.18} & \mycc \textbf{36.22} & \mycc \textbf{0.59} & \mycc \textbf{0.57} \\
                \midrule

                \multirow{2}{*}{Stable Diffusion 3 \citep{esser2024scaling}} 
                & CFG & 27.19 & 40.11 & 0.73 & 0.41 \\
                & \mycc +\gls{method} (Ours) & \mycc \textbf{26.84} & \mycc \textbf{40.94} & \mycc \textbf{0.76} & \mycc \textbf{0.42} \\
                
                \bottomrule
            \end{tabular}
        }
    \end{table}
}

\newcommand{\tabSamplers}{
    \begin{table}[t]
        \centering
        \caption{Effect of adding \gls{method} to various popular diffusion samplers using the DiT-XL/2 model with 15 steps and $\wcfg=1.25$. Note that \gls{method} improves the performance of all samplers (including multistep solvers such as DPM++), and hence its effect is complementary to changing the sampler.}
        \label{tab:samplers}
        \maxsizebox{\textwidth}{!}{
        \begin{tabular}{llcccc}
        \toprule
         Sampler & Guidance & FID $\downarrow$ & IS $\uparrow$ & Precision $\uparrow$ & Recall $\uparrow$ \\
        \midrule
             \multirow{2}{*}{DDIM \citep{songDenoisingDiffusionImplicit2022}} 
                 & CFG & 11.87 & 151.40 & 0.69 & \textbf{0.70} \\
                 & \mycc +\gls{method} (Ours) & \mycc \textbf{8.73} & \mycc \textbf{173.21} & \mycc \textbf{0.73} & \mycc 0.69 \\
        \midrule
             \multirow{2}{*}{DPM++ \citep{dpm_solver}} 
                 & CFG & 7.15 & 180.05 & 0.75 & \textbf{0.71} \\
                 & \mycc +\gls{method} (Ours) & \mycc \textbf{6.66} & \mycc \textbf{191.45} & \mycc \textbf{0.76} & \mycc 0.70 \\
        \midrule
             \multirow{2}{*}{DDPM \citep{hoDenoisingDiffusionProbabilistic2020}} 
                 & CFG & 24.65 & 107.25 & 0.58 & \textbf{0.67} \\
                 & \mycc +\gls{method} (Ours) & \mycc \textbf{17.03} & \mycc \textbf{128.38} & \mycc \textbf{0.64} & \mycc 0.64 \\
        \midrule
             \multirow{2}{*}{PLMS \citep{plms}} 
                 & CFG & 6.75 & 183.11 & 0.74 & \textbf{0.72} \\
                 & \mycc +\gls{method} (Ours) & \mycc \textbf{6.13} & \mycc \textbf{191.44} & \mycc \textbf{0.75} & \mycc 0.72 \\
        \midrule
             \multirow{2}{*}{UniPC \citep{zhao2023unipc}} 
                 & CFG & 6.79 & 185.86 & 0.75 & \textbf{0.71} \\
                 & \mycc +\gls{method} (Ours) & \mycc \textbf{6.60} & \mycc \textbf{192.81} & \mycc \textbf{0.76} & \mycc 0.69 \\
        \bottomrule
    \end{tabular}
        }
    \end{table}
}

\newcommand{\tabAblationInput}{
\begin{table}[t]
\centering
\caption{Comparison of using the conditional model prediction vs the CFG prediction as input to \gls{method}. While both options outperform baseline sampling without \gls{method}, using the CFG-guided prediction (when available) leads to better results.}
\label{tab:ablation-input}
\begin{tabular}{lccc}
\toprule
Config & HPSv2 $\uparrow$ & Image Reward $\uparrow$ & CLIP Score $\uparrow$ \\
\midrule
Baseline (with CFG)                & 0.238 & 0.174 & 0.317 \\
\midrule
\mycc +\gls{method} (Conditional)  & \mycc 0.249 & \mycc 0.234 & \mycc 0.315 \\
\mycc +\gls{method} (CFG)          & \mycc \textbf{0.255} & \mycc \textbf{0.371} & \mycc \textbf{0.322} \\
\bottomrule
\end{tabular}
\end{table}
}

\newcommand{\tabAblationBeta}{
\begin{table}[t!]
\centering
\caption{Effect of different weight schedulers on \gls{method}.}
\label{tab:ablation-beta}
\begin{tabular}{lcccc}
\toprule
Weight Scheduler & $\whigs$ & HPSv2 $\uparrow$ & Image Reward $\uparrow$ & CLIP Score $\uparrow$ \\
\midrule
Constant              & 1.75 & \textbf{0.261} & 0.36 & \textbf{0.319} \\
Square-root     & 2.50 & \textbf{0.261} & \textbf{0.39} & \textbf{0.319} \\
Linear         & 3.25 & 0.260 & 0.37 & 0.318 \\
\bottomrule
\end{tabular}
\end{table}
}

\newcommand{\tabAblationG}{
\begin{table}[t!]
\centering
\caption{Effect of different averaging functions $g(\buffer)$ on \gls{method}.}
\label{tab:ablation-g}
\begin{tabular}{lcccc}
\toprule
$g(\buffer)$ & HPSv2 $\uparrow$ & Image Reward $\uparrow$ & CLIP Score $\uparrow$ \\
\midrule
Random            & \textbf{0.261} & 0.362 & 0.318 \\
Average           & \textbf{0.261} & 0.349 & 0.319 \\
Weighted average  & 0.260 & 0.370 & 0.320 \\
EMA average       & 0.255 & \textbf{0.371} & \textbf{0.322} \\
\bottomrule
\end{tabular}
\end{table}
}

\newcommand{\tabParameters}{
    \begin{table}[t!]
        \centering
        \caption{Guidance parameters used for \Cref{tab:fid-results}.}
        \label{tab:parameters1}
        \maxsizebox{\linewidth}{!}{
    \begin{booktabs}{lccccc}
        \toprule
        Model &  \# Steps & $\wcfg$ & $\whigs$ &  $\eta$ & $\tmin$ & $\tmax$ & $\alpha$ & $R_c$  \\
        \midrule
        SiT-XL + REPA & 30 & 1.5 & 1 & 1 & 0.3 & 1 & 0.75 & 0.05  \\ 
        DiT-XL/2 & 15 & 1.25 & 2 & 0 & 0.3 & 1 & 0.75 & 0.05\\
        Stable Diffusion XL & 20 & 2.5 & 1.75 & 0 & 0.4 & 1 & 0.75 & 0.05 \\
        Stable Diffusion 3 & 20 & 2.5 & 1.75 & 0 & 0.4 & 1 & 0.75 & 0.05 \\
        \bottomrule
        \end{booktabs}
        }
    \end{table}

    \begin{table}[t!]
        \centering
        \caption{Guidance parameters used for Table 2.}
        \label{tab:parameters2}
        \maxsizebox{\linewidth}{!}{
    \begin{booktabs}{lccccc}
        \toprule
        Model &  \# Steps & $\wcfg$ & $\whigs$ &  $\eta$ & $\tmin$ & $\tmax$ & $\alpha$ & $R_c$  \\
        \midrule
        Stable Diffusion XL & 20 & 2.5 & 1.75 & 1 & 0.4 & 1 & 0.5 & 0.05 \\
        Stable Diffusion 3 & 20 & 2.5 & 1.75 & 1 & 0.4 & 1 & 0.5 & 0.05 \\
        Stable Diffusion 3.5 & 20 & 2.5 & 1.75 & 1 & 0.4 & 1 & 0.5 & 0.05 \\
        \bottomrule
        \end{booktabs}
        }
    \end{table}

    \begin{table}[t!]
        \centering
        \caption{Guidance parameters used for \Cref{tab:repa}.}
        \label{tab:parameters3}
        \maxsizebox{\linewidth}{!}{
    \begin{booktabs}{lccccc}
        \toprule
        Model &  \# Steps & $\wcfg$ & $\whigs$ &  $\eta$ & $\tmin$ & $\tmax$ & $\alpha$ & $R_c$  \\
        \midrule
        SiT-XL + REPA (Unguided) & 40 & 1 & 1 & 1 & 0.4 & 1 & 0.75 & 0.05  \\ 
        SiT-XL + REPA (with CFG) & 40 & 1.8 & 1 & 1 & 0.35 & 1 & 0.75 & 0.05  \\ 
        SiT-XL + REPA-E (Unguided) & 30 & 1 & 1.25 & 1 & 0.35 & 1 & 0.75 & 0.05  \\ 
        SiT-XL + REPA-E (with CFG) & 40 & 2.5 & 0.75 & 0 & 0.3 & 1 & 0.75 & 0.05  \\ 
        \bottomrule
        \end{booktabs}
        }
    \end{table}
}
\newcommand{\algHigs}{
            \begin{algorithm}[t!]
            \centering
            \caption{Sampling with \gls{method}}
            \label{alg:method}
            \setstretch{1.15}
            \begin{algorithmic}[1]
            \Require Diffusion model $D_\mtheta$, input condition $\vy$
            \Require CFG scale $\wcfg$, \gls{method} schedule $\whigs(t)$, history length $W$
            \Require EMA parameter $\alpha$, projection weight $\eta$, DCT parameters $(R_c, \lambda)$
            \State Initialize latent $\vz_T \sim \mathcal{N}(\vzero, \mI)$, history buffer $\mathcal{H} \gets \emptyset$
            \For{$t_k \in \Set{t_0, t_1, \dotsc, t_T}$}
                \State Compute conditional and unconditional predictions: $\predcond[t_k], \preduncond[t_k]$
                \State Apply CFG: $\predguided[t_k] = \wcfg \predcond[t_k] -  (\wcfg - 1) \preduncond[t_k]$
                \State Compute the history signal $g(\buffer) = \sum_{i\in I_k} \alpha(1-\alpha)^{t-1-i} \predguided[t_i]$
                \State Form the guidance direction: $\dpred[t_k] = \predguided[t_k] - g(\buffer)$
                \If{projection enabled (i.e., $\eta < 1$)}
                    \State Project $\dpred[t_k]$ into orthogonal and parallel components w.r.t. $\predguided[t_k]$
                    \State $\dpred[t_k] \gets \dpredorth[t_k] + \eta \dpredpar[t_k]$
                \EndIf
                \State Apply DCT-based high-pass filter: $\dpred[t_k] \gets \texttt{iDCT}(H(R) \cdot \texttt{DCT}(\dpred[t_k]))$
                \State Apply \gls{method}: $\predguidedours[t_k] = \predguided[t_k] + \whigs(t_k) \dpred[t_k]$
                \State Update history: $\mathcal{H}_{k+1} \gets \mathcal{H}_k \cup \{\predguided[t_k]\}$, truncate oldest if $|\mathcal{H}_{k+1}| > W$
                \State Apply one sampling step: $\vz_{t-1} = \textsc{SamplingStep}(\predguidedours, \vz_t, t)$
            \EndFor
            \State \Return $\vz_0$
            \end{algorithmic}
            \end{algorithm}
}

\newcommand{\CFGScalePlot}{
    \begin{figure}[t!]
        \centering
        \begin{subfigure}{0.495\textwidth}
            \begin{minipage}{\textwidth}
            \resizebox{\textwidth}{!}{
                \begin{tikzpicture}
                    \begin{groupplot}[
                            group style={
                                    group size=4 by 1,
                                    horizontal sep=1.25cm,
                                },
                            xlabel={$\wcfg$},
                            ymajorgrids=true,
                            xmajorgrids=true,
                            grid style=dashed,
                            major grid style = {lightgray},
                            tick label style={font=\normalsize},
                            label style={font=\Large},
                            title style={font=\Large},
                            legend cell align={left},
                            legend style={at={(0.95,0.525)},anchor=east, font=\Large},
                            scale only axis,
                        ]
                        \nextgroupplot[title={HPSv2}]
                        \addplot [C0, mark=*, very thick, mark options={solid}] table [x index=0, y index=2, col sep=space] {data/cfg_base.dat};
                        \addplot [C1, mark=*, very thick, mark options={solid}] table [x index=0, y index=2, col sep=space] {data/cfg_higs.dat};
                        
                        \nextgroupplot[title={Win Rate}]
                        \addplot [C0, mark=*, very thick, mark options={solid}] table [x index=0, y index=4, col sep=space] {data/cfg_base.dat};
                        \addplot [C1, mark=*, very thick, mark options={solid}] table [x index=0, y index=4, col sep=space] {data/cfg_higs.dat};
                    \end{groupplot}
                \end{tikzpicture}
            }
        \end{minipage}
        \caption{Changing the guidance scale}
        \label{fig:cfg-scale}
        \end{subfigure}
        \begin{subfigure}{0.495\textwidth}
            \begin{minipage}{\textwidth}
            \resizebox{\textwidth}{!}{
                \begin{tikzpicture}
                    \begin{groupplot}[
                            group style={
                                    group size=4 by 1,
                                    horizontal sep=1.25cm,
                                },
                            xlabel={\# Steps},
                            ymajorgrids=true,
                            xmajorgrids=true,
                            grid style=dashed,
                            major grid style = {lightgray},
                            tick label style={font=\normalsize},
                            label style={font=\Large},
                            title style={font=\Large},
                            legend cell align={left},
                            legend style={at={(0.95,0.525)},anchor=east, font=\Large},
                            scale only axis,
                        ]
                        \nextgroupplot[title={HPSv2}]
                        \addplot [C0, mark=*, very thick, mark options={solid}] table [x index=0, y index=2, col sep=space] {data/nfe_base.dat};
                        \addplot [C1, mark=*, very thick, mark options={solid}] table [x index=0, y index=2, col sep=space] {data/nfe_higs.dat};
                        
                        \nextgroupplot[title={Win Rate}]
                        \addplot [C0, mark=*, very thick, mark options={solid}] table [x index=0, y index=4, col sep=space] {data/nfe_base.dat};
                        \addplot [C1, mark=*, very thick, mark options={solid}] table [x index=0, y index=4, col sep=space] {data/nfe_higs.dat};

                        \legend{CFG, \gls{method}}
                    \end{groupplot}                    
                \end{tikzpicture}
            }
        \end{minipage}
        \caption{Changing the number of sampling steps}
        \label{fig:nfe}
        \end{subfigure}
        \caption{Effect of \gls{method} across guidance scales and sampling budgets using Stable Diffusion 3 \citep{esser2024scaling}. \gls{method} consistently outperforms standard sampling with CFG in all settings, highlighting its effectiveness in improving generation quality under varying guidance scales and sampling budgets.}
    \end{figure}
}

\newcommand{\AblationWhigsPlot}{
    \begin{figure}[t]
        \centering
        \begin{minipage}{\textwidth}
            \resizebox{\textwidth}{!}{
                \begin{tikzpicture}
                    \begin{groupplot}[
                            group style={
                                    group size=4 by 1,
                                    horizontal sep=1.25cm,
                                },
                            xlabel={$\whigs$},
                            ymajorgrids=true,
                            xmajorgrids=true,
                            grid style=dashed,
                            major grid style = {lightgray},
                            tick label style={font=\normalsize},
                            label style={font=\Large},
                            title style={font=\Large},
                            legend cell align={left},
                            legend style={at={(0.95,0.525)},anchor=east, font=\Large},
                            scale only axis,
                        ]
                        \nextgroupplot[title={HPSv2}, ymin=0.23]
                        \addplot [C0, mark=*, very thick, mark options={solid}] table [x index=0, y index=1, col sep=space] {data/ablation_whigs.dat};
                        \draw[darkgray,dashed,very thick]
  (axis cs:\pgfkeysvalueof{/pgfplots/xmin},0.238)
  -- (axis cs:\pgfkeysvalueof{/pgfplots/xmax},0.238) node [pos=0.12, below, font=\Large] {CFG};
                        
                        \nextgroupplot[title={ImageReward}, ymin=0.13]
                        \addplot [C0, mark=*, very thick, mark options={solid}] table [x index=0, y index=2, col sep=space] {data/ablation_whigs.dat};
                        \draw[darkgray,dashed,very thick]
  (axis cs:\pgfkeysvalueof{/pgfplots/xmin},0.174)
  -- (axis cs:\pgfkeysvalueof{/pgfplots/xmax},0.174) node [pos=0.12, below, font=\Large] {CFG};

                        \nextgroupplot[title={CLIP Score}, ymin=0.315]
                        \addplot [C0, mark=*, very thick, mark options={solid}] table [x index=0, y index=3, col sep=space] {data/ablation_whigs.dat};

                        \draw[darkgray,dashed,very thick]
  (axis cs:\pgfkeysvalueof{/pgfplots/xmin},0.317)
  -- (axis cs:\pgfkeysvalueof{/pgfplots/xmax},0.317) node [pos=0.12, below, font=\Large] {CFG};

                    \end{groupplot}

                \end{tikzpicture}
            }
        \end{minipage}
            
        \caption{Effect of varying $\whigs$ on different quality metrics with Stable Diffusion XL \citep{sdxl}. \gls{method} consistently outperforms the CFG baseline across a wide range of $\whigs$ scales. In practice, we found that setting $\whigs \leq 3$ generally leads to good performance across models.}
        \label{fig:ablation-whigs}
    \end{figure}
}

\newcommand{\AblationEMAPlot}{
    \begin{figure}[t]
        \centering
        \begin{minipage}{\textwidth}
            \resizebox{\textwidth}{!}{
                \begin{tikzpicture}
                    \begin{groupplot}[
                            group style={
                                    group size=4 by 1,
                                    horizontal sep=1.25cm,
                                },
                            xlabel={EMA $\alpha$},
                            ymajorgrids=true,
                            xmajorgrids=true,
                            grid style=dashed,
                            major grid style = {lightgray},
                            tick label style={font=\normalsize},
                            label style={font=\Large},
                            title style={font=\Large},
                            legend cell align={left},
                            legend style={at={(0.95,0.525)},anchor=east, font=\Large},
                            scale only axis,
                        ]
                        \nextgroupplot[title={HPSv2}, ymin=0.15, ymax=0.3]
                        \addplot [C0, mark=*, very thick, mark options={solid}] table [x index=0, y index=1, col sep=space] {data/ablation_ema.dat};
                        
                        \nextgroupplot[title={ImageReward}, ymin=-1, ymax=1]
                        \addplot [C0, mark=*, very thick, mark options={solid}] table [x index=0, y index=2, col sep=space] {data/ablation_ema.dat};

                        \nextgroupplot[title={CLIP Score}, ymin=0.25, ymax=0.35]
                        \addplot [C0, mark=*, very thick, mark options={solid}] table [x index=0, y index=3, col sep=space] {data/ablation_ema.dat};
                    \end{groupplot}

                \end{tikzpicture}
            }
        \end{minipage}
            
        \caption{Effect of varying the EMA parameter $\alpha$ on different quality metrics for Stable Diffusion XL. The results indicate that performance remains consistent across choices of $\alpha$.}
        \label{fig:ablation-ema}
    \end{figure}
}

\newcommand{\AblationDCTPlot}{
    \begin{figure}[t]
        \centering
        \begin{minipage}{\textwidth}
            \resizebox{\textwidth}{!}{
                \begin{tikzpicture}
                    \begin{groupplot}[
                            group style={
                                    group size=3 by 1,
                                    horizontal sep=1.25cm,
                                },
                            xlabel={DCT threshold},
                            ymajorgrids=true,
                            xmajorgrids=true,
                            grid style=dashed,
                            major grid style = {lightgray},
                            tick label style={font=\normalsize},
                            label style={font=\Large},
                            title style={font=\Large},
                            legend cell align={left},
                            legend style={at={(0.95,0.525)},anchor=east, font=\Large},
                            scale only axis,
                        ]
                        \nextgroupplot[title={HPSv2}, ymin=0.15, ymax=0.3]
                        \addplot [C0, mark=*, very thick, mark options={solid}] table [x index=0, y index=1, col sep=space] {data/ablation_rdct.dat};
                        
                        \nextgroupplot[title={ImageReward}, ymin=-1, ymax=1]
                        \addplot [C0, mark=*, very thick, mark options={solid}] table [x index=0, y index=2, col sep=space] {data/ablation_rdct.dat};

                        \nextgroupplot[title={CLIP Score}, ymin=0.25, ymax=0.35]
                        \addplot [C0, mark=*, very thick, mark options={solid}] table [x index=0, y index=3, col sep=space] {data/ablation_rdct.dat};
                    \end{groupplot}
                \end{tikzpicture}
            }
        \end{minipage}
            
        \caption{Effect of varying the DCT threshold on different quality metrics for Stable Diffusion XL. The results indicate that performance remains consistent across different threshold choices.}
        \label{fig:ablation-dct}
    \end{figure}
}

\newcommand{\AblationMaxTPlot}{
    \begin{figure}[t]
        \centering
        \begin{minipage}{\textwidth}
            \resizebox{\textwidth}{!}{
                \begin{tikzpicture}
                    \begin{groupplot}[
                            group style={
                                    group size=3 by 1,
                                    horizontal sep=1.25cm,
                                },
                            xlabel={$\tmax$},
                            ymajorgrids=true,
                            xmajorgrids=true,
                            grid style=dashed,
                            major grid style = {lightgray},
                            tick label style={font=\normalsize},
                            label style={font=\Large},
                            title style={font=\Large},
                            legend cell align={left},
                            legend style={at={(0.95,0.525)},anchor=east, font=\Large},
                            scale only axis,
                        ]
                        \nextgroupplot[title={HPSv2}]
                        \addplot [C0, mark=*, very thick, mark options={solid}] table [x index=0, y index=1, col sep=space] {data/ablation_maxt.dat};
                        
                        \nextgroupplot[title={ImageReward}]
                        \addplot [C0, mark=*, very thick, mark options={solid}] table [x index=0, y index=2, col sep=space] {data/ablation_maxt.dat};

                        \nextgroupplot[title={CLIP Score}]
                        \addplot [C0, mark=*, very thick, mark options={solid}] table [x index=0, y index=3, col sep=space] {data/ablation_maxt.dat};
                    \end{groupplot}
                \end{tikzpicture}
            }
        \end{minipage}
            
        \caption{Effect of varying $\tmax$ on different quality metrics for Stable Diffusion XL. The results show that reducing $\tmax$ often leads to insufficient guidance and degraded quality. We recommend setting $\tmax \in [0.9, 1]$ for all models.}
        \label{fig:ablation-maxt}
    \end{figure}
}

\newcommand{\AblationMinTPlot}{
    \begin{figure}[t]
        \centering
        \begin{minipage}{\textwidth}
            \resizebox{\textwidth}{!}{
                \begin{tikzpicture}
                    \begin{groupplot}[
                            group style={
                                    group size=3 by 1,
                                    horizontal sep=1.25cm,
                                },
                            xlabel={$\tmin$},
                            ymajorgrids=true,
                            xmajorgrids=true,
                            grid style=dashed,
                            major grid style = {lightgray},
                            tick label style={font=\normalsize},
                            label style={font=\Large},
                            title style={font=\Large},
                            legend cell align={left},
                            legend style={at={(0.95,0.525)},anchor=east, font=\Large},
                            scale only axis,
                        ]
                        \nextgroupplot[title={HPSv2}]
                        \addplot [C0, mark=*, very thick, mark options={solid}] table [x index=0, y index=1, col sep=space] {data/ablation_mint.dat};
                        
                        \nextgroupplot[title={ImageReward}]
                        \addplot [C0, mark=*, very thick, mark options={solid}] table [x index=0, y index=2, col sep=space] {data/ablation_mint.dat};

                        \nextgroupplot[title={CLIP Score}]
                        \addplot [C0, mark=*, very thick, mark options={solid}] table [x index=0, y index=3, col sep=space] {data/ablation_mint.dat};
                    \end{groupplot}
                \end{tikzpicture}
            }
        \end{minipage}
            
        \caption{Effect of varying $\tmin$ on different quality metrics for Stable Diffusion XL. The results show that performance degrades when $\tmin$ is set too low or too high. We find that $\tmin \in [0.3, 0.5]$ yields good results across all models.}
        \label{fig:ablation-mint}
    \end{figure}
}

\title{HiGS: \textcolor{Plum}{Hi}story-\textcolor{Plum}{G}uided \textcolor{Plum}{S}ampling for Plug-and-Play Enhancement of Diffusion Models}

\author{Seyedmorteza Sadat\textsuperscript{1}, Farnood Salehi\textsuperscript{2}, Romann M.\ Weber\textsuperscript{2} \\
\textsuperscript{1}ETH Z\"urich, \textsuperscript{2}DisneyResearch\textbar{}Studios\\
\texttt{\{seyedmorteza.sadat\}@inf.ethz.ch} \\
\texttt{\{farnood.salehi, romann.weber\}@disneyresearch.com}
}

\iclrfinalcopy

\begin{document}

\maketitle
\vspace{-0.15cm}

\begin{abstract}
While diffusion models have made remarkable progress in image generation, their outputs can still appear unrealistic and lack fine details, especially when using fewer number of \gls{nfe} or lower guidance scales. To address this issue, we propose a novel momentum-based sampling technique, termed \gls{method}, which enhances quality and efficiency of diffusion sampling by integrating recent model predictions into each inference step. Specifically, \gls{method} leverages the difference between the current prediction and a weighted average of past predictions to steer the sampling process toward more realistic outputs with better details and structure. Our approach introduces practically no additional computation and integrates seamlessly into existing diffusion frameworks, requiring neither extra training nor fine-tuning. Extensive experiments show that \gls{method} consistently improves image quality across diverse models and architectures and under varying sampling budgets and guidance scales. Moreover, using a pretrained SiT model, \gls{method} achieves a new state-of-the-art FID of 1.61 for unguided ImageNet generation at 256$\times$256 with only 30 sampling steps (instead of the standard 250). We thus present \gls{method} as a plug-and-play enhancement to standard diffusion sampling that enables faster generation with higher fidelity.
\end{abstract}
\figTeaser

\section{Introduction}
Diffusion models \citep{sohl2015deep,hoDenoisingDiffusionProbabilistic2020,score-sde} are a class of generative models that learn the data distribution by reversing a forward noising process that gradually transforms data points into Gaussian noise. Although in theory, the reverse process should yield high-quality samples from the target distribution, diffusion models can produce outputs that are blurry or lack details due to optimization errors and inaccurate estimations of the data distribution at intermediate time steps. This issue becomes more pronounced when using fewer sampling steps or lower classifier-free guidance scales \citep{karras2022elucidating}.

The generation process in diffusion models typically involves multiple \acrfull{nfe}, which are computationally expensive, especially for large-scale diffusion models containing billions of parameters \citep{esser2024scaling}. Reducing the number of \gls{nfe} reduces the sampling cost but leads to outputs that lack clarity, detail, and coherent global structures, as seen in \Cref{fig:teaser}. Although various sampling methods have been proposed to reduce the required sampling steps \citep{dpm_solver, karras2022elucidating}, existing samplers still rely on relatively high step counts (e.g., 50) to achieve satisfactory results \citep{sdxl}. Finding training-free methods that enable generating high-quality samples with fewer \gls{nfe} remains an open research question.

In addition to iterative sampling, modern diffusion models use high guidance scales to enhance sample quality and prompt alignment \citep{sdxl, glide}. \Gls{cfg} \citep{hoCascadedDiffusionModels2021} has proven essential for reducing outliers and improving generation quality \citep{karras2024guiding}. However, \gls{cfg} doubles the network forward passes per sampling step, thereby increasing the computational cost of inference, and higher guidance scales often lead to oversaturation and reduced diversity \citep{sadat2025eliminating,sadat2024cads}. Consequently, further research is needed to enhance generation quality when using lower guidance scales combined with fewer \gls{nfe}.

In this paper, We investigate the sampling process in diffusion models and propose a training-free, momentum-based modification of inference that consistently improves global structure, detail, and sharpness across different sampling budgets and guidance scales. Inspired by momentum-based variance reduction in stochastic optimization \citep{cutkosky2019momentum}, we introduce \acrfull{method}, a novel method for improving the quality of diffusion models by leveraging the history of predictions made by the network. We demonstrate that this prediction history defines an effective guidance direction for steering the sampling trajectory toward higher-quality regions of the data distribution, especially under fewer \gls{nfe} or lower guidance scales.  \gls{method} enables higher quality outputs and more efficient sampling from pretrained models while also improving the generation quality at lower guidance scales to avoid the drawbacks of high CFG scales.

\gls{method} introduces no additional overhead to the sampling process and can be seamlessly integrated into existing diffusion models and samplers without extra training. Through extensive experiments across a range of models and setups {(including distilled diffusion models)}, we show that \gls{method} consistently improves sample quality, particularly in scenarios involving fewer sampling steps or lower guidance scales. Our results indicate that \gls{method} achieves higher quality metrics compared to standard sampling methods, establishing it as a universal enhancement for pretrained diffusion models under various sampling budgets and guidance scales. Furthermore, using a pretrained SiT model, \gls{method} achieves a state-of-the-art FID of 1.61 for {unguided (i.e., without CFG)} ImageNet generation at 256$\times$256 while using only 30 sampling steps instead of the standard 250.
\section{Related work}
Score-based diffusion models \citep{DBLP:conf/nips/SongE19,score-sde,sohl2015deep,hoDenoisingDiffusionProbabilistic2020} learn complex data distributions by inverting a forward process that gradually adds Gaussian noise to the data. These models have rapidly advanced generative modeling, surpassing prior approaches in both fidelity and diversity \citep{nichol2021improved,dhariwalDiffusionModelsBeat2021}. They have demonstrated state-of-the-art performance across a wide range of tasks, including unconditional and conditional image synthesis \citep{dhariwalDiffusionModelsBeat2021,karras2022elucidating,yu2024representation,karras2024guiding}, text-to-image generation \citep{sdxl,esser2024scaling,flux2024,qin2025lumina}, video generation \citep{blattmann2023align,stableVideoDiffusion,bar2024lumiere,wan2025wan}, image-to-image translation \citep{saharia2022palette,liu20232i2sb,xia2023diffi2i}, and audio generation \citep{WaveGrad,huang2023noise2music,liu2023audioldm,tian2025audiox}.

Since the introduction of DDPM \citep{hoDenoisingDiffusionProbabilistic2020}, the field has seen substantial progress, with improvements spanning network architectures \citep{hoogeboom2023simple,karras2023analyzing,peeblesScalableDiffusionModels2022,dhariwalDiffusionModelsBeat2021}, sampling techniques \citep{songDenoisingDiffusionImplicit2022,karras2022elucidating,plms,dpm_solver,salimansProgressiveDistillationFast2022}, and training strategies \citep{nichol2021improved,karras2022elucidating,score-sde,salimansProgressiveDistillationFast2022,rombachHighResolutionImageSynthesis2022}. Despite these advancements, sampling from diffusion models still require relatively high step counts, and various guidance mechanisms---such as classifier guidance \citep{dhariwalDiffusionModelsBeat2021} and classifier-free guidance \citep{hoClassifierFreeDiffusionGuidance2022}---remain critical for enhancing image quality and ensuring strong prompt alignment \citep{glide}.

A recent line of work has focused on developing better ODE solvers for the diffusion sampling process, often combined with improved training techniques to make the sampling ODE more linear \citep{dpm_solver,karras2022elucidating,esser2024scaling}. Despite these advances, most state-of-the-art samplers still require a relatively high number of sampling steps (e.g., 50 steps for Stable Diffusion XL \citep{sdxl}). Another direction of research explores distilling the diffusion model into a student network capable of sampling with fewer steps \citep{salimansProgressiveDistillationFast2022,song2023consistency,sauer2024fast}. However, step distillation remains computationally expensive, requiring long training on advanced hardware. We show that \gls{method} serves as a training-free method to improve generation quality across various sampling budgets and networks (including distilled diffusion models).

Modern diffusion models often rely on high guidance scales to achieve strong image quality and prompt alignment. However, CFG doubles inference cost, and excessive CFG scales reduce diversity and cause oversaturation \citep{sadat2024cads, sadat2025eliminating}. On the other hand, sampling with lower CFG scales and \gls{nfe} avoid these issues but typically yield blurry images lacking fine detail and coherent structure. While several methods mitigate the drawbacks of high CFG scales \citep{intervalGuidance,sadat2025eliminating,wang2024analysis}, little attention has been given to improving quality under low CFG scales and limited sampling steps. We show that \gls{method} enhances generation across both low and high CFG regimes, offering benefits under varying CFG scales and sampling budgets.

In summary, sampling from diffusion models is an expensive iterative process that might produce unrealistic outputs under certain configurations. We propose a training-free method that improves generation quality across different sampling budgets and guidance scales, particularly in low-NFE and low-CFG scenarios. Thus, we present \gls{method} as a plug-and-play enhancement for diffusion sampling.
\section{Background}
This section provides a brief overview of diffusion models. Let \(\vx \sim \pdata(\vx)\) denote a data sample, and let \(t \in [0,T]\) represent continuous time. The forward diffusion process gradually corrupts data by adding Gaussian noise: \(\vz_t = \vx + \sigma(t) \mepsilon\), where \(\mepsilon \sim \mathcal{N}(\mathbf{0}, \mI)\) and \(\sigma(t)\) defines a monotonically increasing noise schedule, satisfying \(\sigma(0) = 0\) and \(\sigma(T) = \sigma_{\textnormal{max}} \gg \sigma_{\textnormal{data}}\). As shown by \citet{karras2022elucidating}, this forward process can be described by the following ordinary differential equation (ODE):
\begin{equation}\label{eq:diffusion-ode}
    \dd\vz = - \dot{\sigma}(t) \sigma(t) \grad_{\vz_t} \log p_t(\vz_t) \dd t = -t \grad_{\vz_t} \log p_t(\vz_t) \dd t,
\end{equation}
where we choose \(\sigma(t) = t\), and \(p_t(\vz_t)\) is the distribution of noisy data at time \(t\), with \(p_0 = \pdata\) and \(p_T = \mathcal{N}(\mathbf{0}, \sigma_{\textnormal{max}}^2 \mI)\). If the time-dependent score function \(\grad_{\vz_t} \log p_t(\vz_t)\) is known, one can sample from \(\pdata\) by integrating the ODE (or its stochastic counterpart) in reverse, from \(t = T\) to \(t = 0\).

In practice, the score function is approximated using a neural denoiser \(D_{\mtheta}(\vz_t, t)\), trained to recover the clean data \(\vx\) from the noisy input \(\vz_t\). Conditional generation is supported by extending the denoiser to take an additional input \(\vy\) (e.g., class labels or text), resulting in \(D_{\mtheta}(\vz_t, t, \vy)\).

\paragraph{Classifier-free guidance (CFG)}  
Classifier-free guidance (CFG) is a technique for enhancing sample quality during inference by interpolating between unconditional and conditional model predictions \citep{hoClassifierFreeDiffusionGuidance2022}. Given an unconditional prediction \(\prednull\), the guided denoiser output at each sampling step is computed as:
\begin{equation}\label{eq:cfg}
    \predcfg = \wcfg \pred - (\wcfg - 1)\prednull,
\end{equation}
where \(w_{\textnormal{CFG}} = 1\) corresponds to no guidance, and larger values increase conditioning strength. The unconditional model is typically trained by randomly dropping the condition $\vy$ with some probability \(p \in [0.1, 0.2]\) during training. Alternatively, a separate network or the conditional model itself can be used to estimate the unconditional score \citep{karras2023analyzing,sadat2025no}. While CFG is known to improve perceptual quality, it often comes at the cost of increased oversaturation and reduced sample diversity \citep{sadat2025eliminating,sadat2024cads}.

\section{Sampling with prediction history}\label{sec:method}

Let ${t_0 > t_1 > \ldots > t_M}$ be the sampling time grid with $M{+}1$ steps. 
The model prediction at time step $t_k$ is denoted by $\pred[t_k]$. 
Given a window of size $W \geq 1$, we define the history $\mathcal{H}_k$ at step $k$ as the set of past predictions prior to $t_k$, i.e., $\mathcal{H}_k := \{\predi\}_{i \in I_k}$ for $I_k := \{\max(0, k - W), \ldots, k - 1\}$. We next show how leveraging this history enhances sampling quality. For simplicity, we define $\predcond[t_k] \doteq \pred[t_k]$, $\preduncond[t_k] \doteq \prednull[t_k]$, and $\predguided[t_k] \doteq \predcfg[t_k]$ to represent the conditional, unconditional, and CFG outputs, respectively.

\subsection{Motivation} 
We first claim that the Euler sampler for diffusion models can be interpreted as performing \gls{sgd} on a time-varying energy function. To show this, consider the ODE in \Cref{eq:diffusion-ode} and a discretization $\Set{t_0, t_1, \ldots, t_M}$. A single Euler step at time $t_k$ can be written as
\begin{align}\label{eq:energy}
\vz_{t_{k+1}} &= \vz_{t_{k}} + t_{k} (t_{k} - t_{k+1}) \grad_{\vz_{t_{k}}} \log p_{t_{k}}(\vz_{t_{k}}) \\
&= \vz_{t_k} - t_k (t_{k} - t_{k+1}) \grad_{\vz_{t_k}} E_{t_k}(\vz_{t_k}),
\end{align}
where the energy function $E_t(\vz_t)$ is defined via $p_t(\vz_t) = \tfrac{1}{Z} \exp\prn{-E_t(\vz_t)}$ for some normalization constant $Z$. This shows that each Euler step in diffusion sampling corresponds to a step of gradient descent on the time-dependent energy $E_t(\vz_t)$ with learning rate $t_k (t_{k} - t_{k+1})$.

Motivated by this new perspective, we argue that we can enhance this gradient estimate to improve the efficiency and quality of the sampling process in diffusion models. One promising approach is to augment the gradient with an additional momentum-based term similar to STORM \citep{cutkosky2019momentum}, a variance-reduction method for non-convex optimization. Specifically, given a differentiable function $f$, the momentum term in STORM is defined as $\grad f(\vz_{t_k}) - \grad f(\vz_{t_{k-1}})$, which incorporates information from consecutive steps for more stable updates. Applying this to \Cref{eq:energy}, we obtain $\grad E_{t_k}(\vz_{t_k}) - \grad E_{t_{k-1}}(\vz_{t_{k-1}})$ as the momentum term. Equivalently, this can be seen as incorporating the residual of past score terms or model outputs, since the score function corresponds to the energy gradient. In the following, we further generalize this idea to incorporate multiple past predictions rather than only the previous step (similar to multistep ODE solvers \citep{atkinson2009numerical}). An alternative perspective is given in \Cref{sec:CFG-SGD}, where we connect classifier-free guidance to gradient ascent on a specific objective. In \Cref{sec:error-theory}, we show that the history terms reduce the Euler solver’s local truncation error from $\mc{O}(h_k^2)$ to $\mc{O}(h_k^3)$, where $h_k = t_k - t_{k+1}$, thereby improving the global error from $\mc{O}(h)$ to $\mc{O}(h^2)$ with $h = \max_k h_k$.

\subsection{\texorpdfstring{\Acrlong{method}}{History-guided sampling}}
Building on the insights discussed above, we propose a novel sampling method for diffusion models, termed \acrfull{method}, which integrates past predictions into the current sampling step. We leverage the set of past predictions $\buffer$ at each sampling step $t_k$ to improve generation quality such as sharpness, details and global structure. We detail each step of \gls{method} below.

\paragraph{Buffer input} When the sampling is done with CFG, we can choose to keep track of the CFG-guided predictions $\predguided[t_k]$ or the conditional outputs $\predcond[t_k]$ as the history. We found that the CFG-guided predictions are more effective in improving the quality of sampling. Thus, we use $\predguided[t_k]$ as the inputs to $\buffer$ when CFG is enabled, leading to $\mathcal{H}_k := \{\predguided[t_i]\}_{i \in I_k}$.

\paragraph{Incorporating the history}
The next design choice is how to use $\buffer$ during sampling. Specifically, we want a function $g$ that determines the influence of past predictions on the current step. By generalizing the momentum update rule in STORM, we adopt an EMA-style weighted average that emphasizes recent predictions:
\begin{equation}
g(\buffer) = \sum_{i\in I_k} \alpha(1 - \alpha)^{k-1-i} \predguided[t_{i}],
\end{equation}
where $\alpha \in (0,1)$ is a hyperparameter. This formulation integrates information from history while prioritizing recent steps for more informative guidance. We later show that several alternative definitions of $g$, such as simple averaging, are viable options (see \Cref{sec:ablation} for more details). 

Let $\dpred[t_k] = \predguided[t_k] - g(\buffer)$ denote the guidance term in \gls{method}. A straightforward strategy for using $\dpred[t_k]$ at inference is to combine this update with the current output, analogous to CFG, via $\predguided[t_k] + \whigs \dpred[t_k]$ with scale $\whigs$. In the following, we introduce several improvements over this naive baseline that we found crucial to substantially boost performance. \Cref{sec:ablation} presents extensive ablation studies to support various design choices of \gls{method}.

\paragraph{Scheduling the guidance weight}
In our experiments, the benefits of \gls{method} were most evident during the early and middle sampling steps. We noticed that as sampling progresses, the update term introduces diminishing improvements and may even cause noisy artifacts. To address this, we use a weight schedule $\whigs(t_k)$ that adapts the scale of the guidance term according to the time step $t_k$. We employ a square-root schedule given by 
\begin{equation}
    \whigs(t) = \begin{cases}
        0 & t \leq \tmin, \\
        \whigs \sqrt{\frac{t - \tmin}{\tmax - \tmin}} & \tmin < t \leq \tmax, \\
        0 & t > \tmax.
    \end{cases}
\end{equation}

\paragraph{Optional orthogonal projection} Additionally, we found that it is sometimes beneficial to  project the update vector $\dpred[t_k]$ on \( \predguided[t_k] \) and downweight its parallel component to prevent oversaturation and color artifacts (especially at higher values of $\whigs$) \citep{sadat2025eliminating}. The parallel component in $\dpred[t_k]$ can be computed via
\begin{equation}
    \dpredpar[t_k] = \frac{\inner{\dpred[t_k]}{\predguided[t_k]}}{\inner{\predguided[t_k]}{\predguided[t_k]}} \predguided[t_k],
\end{equation}
and the orthogonal component is computed via $\dpredorth[t_k] = \predguided[t_k] - \dpredpar[t_k]$. We use $\dpred[t_k](\eta) = \dpredorth[t_k] + \eta \dpredpar[t_k]$ as the projected update direction, where $\eta \in [0, 1]$ is a hyperparameter. As demonstrated in \Cref{fig:ablation-projection} (appendix), the orthogonal projection step can mitigate oversaturation and color artifacts in the generated outputs.

\paragraph{Frequency-domain filtering}
We further observed that using $\dpred[t_k](\eta)$ as the guidance term generally leads to unrealistic color compositions in generations (see \Cref{fig:ablation-dct} in the appendix). To solve this, we employ frequency-based high-pass filtering using the \gls{dct}. Since color composition corresponds to low-frequency contents of an image, we apply \gls{dct} to the update term $\dpred[t_k](\eta)$ and attenuate its low-frequency signals with a sigmoid high-pass filter:
\begin{equation}
H(R) = \mathrm{Sigmoid}(\lambda(R - R_c)),
\end{equation}
where $R = \sqrt{f_x^2 + f_y^2}$ is the radial frequency at coordinates $(x,y)$, $R_c$ is the cutoff threshold, and $\lambda$ controls the sharpness of the transition. This procedure effectively removes the color shifts, leading to more realistic and visually consistent outputs. Accordingly, the final update rule for \gls{method} is
\begin{equation}
    \predguidedours[t_k] = \predguided[t_k] + \whigs(t_k)\ \texttt{iDCT}\!\big(H(R) \cdot \texttt{DCT}( \dpred[t_k](\eta))\big).
\end{equation}
\section{Experiments}
We now provide extensive qualitative and quantitative comparisons between standard sampling and sampling with \gls{method} to show that \gls{method} enhances the performance of diffusion models under various setups. Further experiments, ablations, and implementation details are provided in the appendix.

\paragraph{Setup} We evaluate our method primarily on text-to-image generation using Stable Diffusion models \citep{rombachHighResolutionImageSynthesis2022, sdxl, esser2024scaling}, and on class-conditional generation with ImageNet \citep{imagenet} using DiT-XL/2 \citep{peeblesScalableDiffusionModels2022}, and SiT-XL + REPA \citep{yu2024repa,leng2025repae}. For each model, we employ the default sampling algorithms (e.g., the Euler solver for Stable Diffusion XL) and rely on the official pretrained checkpoints and publicly released codebases to ensure alignment with the original implementations. Additional details regarding the experimental configurations and hyperparameters are provided in \Cref{sec:imp-detail}.  

\paragraph{Evaluation metrics} For text-to-image generation, we primarily use HPSv2 \citep{wu2023human} as our main quality and prompt alignment metric, as we found it to be the most aligned with human judgment. We also report ImageReward \citep{xu2023imagereward} and the HPSv2 win rate as additional preference scores, along with the CLIP Score \citep{Hessel2021CLIPScore} for completeness. For class-conditional models, we use the Fréchet Inception Distance (FID) \citep{fid} as the primary metric to assess both image fidelity and diversity, given its strong alignment with human visual perception in this setting. To ensure fair comparisons, all FID evaluations are conducted under a standardized setup to minimize sensitivity to implementation differences. We also report Inception Score (IS) \citep{inceptionScore}, Precision \citep{improvedPR}, and Recall \citep{improvedPR} as complementary metrics to separately evaluate sample quality and diversity.

\subsection{Main results} 

\paragraph{Qualitative comparisons} We first qualitatively evaluate the impact of \gls{method} on generation quality under three different regimes: the normal setup using practical CFG scales with high \gls{nfe} (\cref{fig:qualitative-main}), sampling with fewer number of \gls{nfe} (\Cref{fig:qualitative-main-low-nfe}), and sampling under low CFG scales (\Cref{fig:qualitative-main-low-cfg}). We note that \gls{method} is able to improve generation quality under all these settings, showing that its benefits cover a wide range of CFG scales and sampling budgets.
\figMain
\figMainLowNFE
\figMainLowCFG

\paragraph{\gls{method} vs CFG scale}
Next, we quantitatively evaluate the effect of adding \gls{method} to the sampling process on generation quality across different CFG scales $\wcfg$. \Cref{fig:cfg-scale} shows that \gls{method} is beneficial across all guidance scales, maintaining a significant margin in terms of HPS win rate. 

\paragraph{\gls{method} vs the number of sampling steps}
 Similarly, \Cref{fig:nfe} shows that \gls{method} outperforms the CFG baseline in generation quality across all \gls{nfe}. This demonstrates that \gls{method} is beneficial for all sampling budgets, and its advantages are not limited to a specific number of sampling steps.

\paragraph{Comparison with different models} Furthermore, we evaluate the impact of \gls{method} on output quality across different models and metrics in \Cref{tab:fid-results,tab:sd-results}, using a fixed guidance scale and sampling steps per baseline for both sampling with and without \gls{method} to ensure a fair comparison. As before, \gls{method} consistently improves generation quality across diverse setups and metrics, indicating that it serves as a universal enhancement to standard diffusion sampling.
\CFGScalePlot
\tabResultFID
\tabResultSD

\subsection{ImageNet benchmark} 
\Cref{tab:repa} shows that \gls{method} improves the performance of recent state-of-the-art models on conditional ImageNet generation at 256$\times$256 resolution. For unguided generation (i.e., without CFG), \gls{method} reduces the state-of-the-art FID from 1.83 to 1.61 using only 30 sampling steps, without any retraining of the base model. Moreover, by improving sample quality at lower \gls{nfe}, \gls{method} provides a training-free way to accelerate sampling from existing models. For guided generation with CFG, \gls{method} matches the performance of the base models in just 40 sampling steps, compared to 250 steps with the default sampler. Finally, this experiment shows that \gls{method} is compatible with guidance interval \citep{intervalGuidance}, as both CFG baselines incorporate this technique in their samplers.
\tabResultREPA

\subsection{Compatibility with Distilled Models} 
Finally, we also demonstrate that \gls{method} is compatible with distilled diffusion models that use fewer sampling steps. Diffusion distillation reduces sampling steps by training a student model to replicate the performance of the base model. As shown in \Cref{tab:distillation}, \gls{method} enhances the sampling quality of these distilled models, achieving a significant win rate according to the HPS score. This shows that the benefits of \gls{method} are complementary to diffusion distillation, and that \gls{method} can be applied as a training-free method to further enhance the quality of distilled diffusion models.  
\tabResultDistilled

\section{Discussion and conclusion}\label{sec:conclusion}
In this work, we presented \acrfull{method}, a simple, training-free modification of diffusion sampling that leverages the history of model predictions to steer the reverse process toward higher quality and more coherent images. By applying a history-informed correction that emphasizes the deviation between the current prediction and a weighted average of past predictions, \gls{method} mitigates blur and artifacts that often arise during sampling, particularly with fewer number of \gls{nfe} or lower CFG scales. \gls{method} requires no additional network evaluations, integrates seamlessly with existing samplers and architectures, and is fully plug-and-play. Our experiments demonstrated that \gls{method} consistently improves perceptual quality and fidelity across diverse models and sampling budgets, with especially strong benefits in low-NFE and low-CFG regimes. While highly effective, \gls{method} still inherits, albeit to a lesser extent, the biases and some limitations of the underlying diffusion models, and addressing these challenges remains a promising direction for future work.

\clearpage
\subsubsection*{Broader Impact Statement}
Our method has the potential to improve the realism and quality of outputs produced by diffusion models without requiring expensive retraining, making it practically valuable for applications in visual content creation. However, as generative modeling technologies continue to evolve, the ease of producing and distributing synthetic or misleading content also increases. While such advances can significantly boost creativity and productivity, they also raise important ethical and societal concerns. It is therefore essential to raise awareness about the potential misuse of generative models and to carefully consider the broader societal implications of their deployment. For an in-depth discussion on ethics and creativity in generative modeling, we refer readers to \cite{lin2025comprehensive}.

\subsubsection*{Reproducibility Statement}
Our work builds upon the official implementations of the pretrained models cited in the main text. \Cref{alg:method} gives the algorithm for applying \gls{method} at inference, and the pseudocode for \gls{method} is presented in \Cref{alg:code-utility,alg:code-main}. Additional implementation details, including the hyperparameters used in the main experiments, are given in \Cref{sec:imp-detail}.


\bibliography{ref}
\bibliographystyle{iclr2026_conference}

\clearpage
\appendix
\begin{appendices}
  \crefalias{section}{appendix}
  \crefalias{subsection}{appendix}
\section{CFG as gradient ascent}\label{sec:CFG-SGD}
This section provides an alternative intuition behind \gls{method}. Recall that CFG update at time step $t_k$ can be written in the alternate form 
\begin{align}
\predguided[t_k] &= \predcond[t_k] + (\wcfg - 1)\, (\predcond[t_k] - {\preduncond[t_k]}),
\end{align}
which can be interpreted as a single gradient ascent step on the objective
\begin{equation}\label{eq:cfg-objective}
\fcfg(\predcond[t_k], \preduncond[t_k]) = \frac{1}{2} \norm{\predcond[t_k] - \sg{\preduncond[t_k]}}^2
\end{equation}
w.r.t. the model output $\predcond[t_k]$ \citep{sadat2025eliminating}. Here, $\sg{\cdot}$ denotes the stop-gradient operation. This suggests that we can augment the gradient update in CFG with a history-based momentum term similar to STORM \citep{cutkosky2019momentum}. For a differentiable function $f$, STORM uses the gradient difference $\grad f(\vz_{t_k}) - \grad f(\vz_{t_{k-1}})$ to incorporate information from consecutive steps.
Applying this to the CFG objective from \Cref{eq:cfg-objective}, we have
\begin{align}
    \grad \fcfg(\vz_{t_k}) &= \predcond[t_k] - \preduncond[t_k],\\
    \grad \fcfg(\vz_{t_{k-1}}) &= \predcond[t_{k-1}] - \preduncond[t_{k}],
\end{align} 
so the STORM momentum reduces to
\begin{equation}\label{eq:storm}
\grad \fcfg(\vz_{t_k}) - \grad \fcfg(\vz_{t_{k-1}}) = \predcond[t_k] - \predcond[t_{k-1}].
\end{equation}
This suggests enhancing the guidance direction with a history term along $\predcond[t_k] - \predcond[t_{k-1}]$. In \Cref{sec:method}, we generalized this idea to incorporate multiple past predictions rather than only $\predcond[t_{k-1}]$ and introduced several refinements that further improve generation quality.

\section{Error analysis of \texorpdfstring{\gls{method}}{HiGS}}\label{sec:error-theory}
We now show that adding \gls{method} to the Euler solver for diffusion models can improve the convergence rate of its local truncation error, and hence leading to better global estimates with fewer sampling steps. Let $u(\vz_t, t) = t \grad_{\vz_t} \log p_t(\vz_t)$. The sampling ODE for diffusion models is then equal to $\dd\vz = - u(\vz_t, t) dt$. The Euler step for this ODE at time step $t_k$ is equal to 
\begin{align}\label{eq:ode-euler}
    \hat{\vz}_{t_{k+1}} &= \hat{\vz}_{t_k} + (t_k - t_{k+1}) u(\hat{\vz}_{t_k}, t_k) \\
    & = \hat{\vz}_{t_k} + h_k u(\hat{\vz}_{t_k}, t_k),
\end{align} 
where $h_k = t_k - t_{k+1}$ is the step size, and $\hat{\vz}_{t_0} \sim \normal{\zero}{\sigma_{\max}^2}$. Now, assume that instead of $u(\vz_{t_k}, t_k)$, we follow an update similar to \gls{method} based on the previous update $u(\vz_{t_{k-1}}, t_{k-1})$, i.e., we use 
\begin{equation}\label{eq:ode-higs}
    \tilde{u}(\vz_{t_k}, t_k) = u(\vz_{t_k}, t_k) + w_k (u(\vz_{t_{k}}, t_{k})-u(\vz_{t_{k-1}}, t_{k-1}))
\end{equation}
for some time dependent weight schedule $w_k = w(t_k)$.
\begin{theorem}
Let $u(\vz,t)$ be sufficiently smooth in both arguments with bounded derivatives, and let $h_k = t_k - t_{k+1}$ denote the variable step size for a time grid ${t_0 > t_1 > \ldots > t_M}$ with $M+1$ steps. 
Then, at each time step $t_k$, the Euler update in \Cref{eq:ode-euler} has local truncation error $\mathcal{O}(h_k^2)$. In contrast, there exists a weight $w_k \in \mathbb{R}$ such that the modified update rule of \Cref{eq:ode-higs} used in \gls{method}, yields local truncation error $\mathcal{O}(h_k^3)$. Consequently, the global error improves from $\mathcal{O}(h)$ for Euler to $\mathcal{O}(h^2)$ with \gls{method}, where $h=\max_k(t_k - t_{k+1})$.
\end{theorem}

\begin{proof}
    Let $\vz(t)$ denote the ground truth solution to the sampling ODE (derived by integration). For simplicity, we define $\vz_k = \vz(t_k)$. Using Taylor expansion, we get 
    \begin{equation}
        \vz_{k+1} = \vz_{k} + \odv{\vz}{t}(t_k) (t_{k+1} - t_k) + \frac{1}{2}\odv[order=2]{\vz}{t}(\zeta_1)(t_{k+1} - t_k)^2
    \end{equation}
    for some $\zeta_1 \in [t_{k+1}, t_k]$. Since we have $\odv{\vz}{t} = -u(\vz, t)$, we get 
     \begin{align}
        \vz_{k+1} &= \vz_k - (t_{k+1} - t_k)u(\vz_k, t_k) + \frac{1}{2}\odv[order=2]{\vz}{t}(\zeta_1)(t_{k+1} - t_k)^2 \\
        &= \vz_k + h_k u(\vz_k, t_k) - \frac{h_k^2}{2}\odv[order=1]{u}{t}(\vz(\zeta_1), \zeta_1).
     \end{align}
     Accordingly, the local truncation error $L_k$ of the euler method  is given by
     \begin{align}
        L_k &= \norm{\vz_{k+1} - (\vz_k + h_k u(\vz_k, t_k))}
          = \frac{h_k^2}{2}\norm{\odv[order=1]{u}{t}(\vz(\zeta_1), \zeta_1)} \leq {Ch_k^2}
     \end{align}
     for some constant $C$. This shows that the local truncation error for the Euler solver is equal to $\mathcal{O}(h_k^2)$. We next show that using the history-based update in \Cref{eq:ode-higs} improves this rate to $\mathcal{O}(h_k^3)$. If we use the Taylor expansion of $u(\vz_{k-1}, t_{k-1})$, we get 
     \begin{align}
        u(\vz_{k-1}, t_{k-1}) = u(\vz_{k}, t_{k}) + h_{k-1} \odv{u}{t}(\vz_{k}, t_{k}) + \frac{h_{k-1}^2}{2}\odv[order=2]{u}{t}(\vz({\zeta_2}), \zeta_2)
     \end{align} 
     for some $\zeta_2 \in [t_{k}, t_{k-1}]$. This leads to 
     \begin{equation}
        \tilde{u}(\vz_{k}, t_k) = u(\vz_{k}, t_{k}) - w_k h_{k-1} \odv{u}{t}(\vz_{k}, t_{k}) - \frac{w_kh_{k-1}^2}{2}\odv[order=2]{u}{t}(\vz({\zeta_2}) \zeta_2). 
     \end{equation}
     Using another Taylor expansion for $\vz(t)$ (this time up to third order), we get 
     \begin{align}
        \vz_{k+1} &= \vz_k + \odv{\vz}{t}(t_k) (t_{k+1} - t_k) + \frac{1}{2}\odv[order=2]{\vz}{t}(t_k)(t_{k+1} - t_k)^2 + \frac{1}{6}\odv[order=3]{\vz}{t}(\zeta_3)(t_{k+1} - t_k)^3 \\
         &= \vz_k + h_k u(\vz_k, t_k) - \frac{h_k^2}{2}\odv[order=1]{u}{t}(\vz_k, t_k) + \frac{h_k^3}{6}\odv[order=2]{u}{t}(\vz(\zeta_3), \zeta_3)
     \end{align} 
     for some $\zeta_3 \in [t_{k+1}, t_k]$. Accordingly, the new truncation error $\tilde{L}_k$ is given by
     \begin{align}
       \tilde{L}_k &= \norm{\vz_{k+1} - (\vz_k + h_k \tilde{u}(\vz_k, t_k))} \\
        &= \norm{{h_k}(w_k h_{k-1} -\frac{h_k}{2}) \odv[order=1]{u}{t}(\vz_k, t_k) + \frac{h_k^3}{6}\odv[order=2]{u}{t}(\vz(\zeta_3), \zeta_3) + \frac{w_k h_k h_{k-1}^2}{2}\odv[order=2]{u}{t}(\vz({\zeta_2}), \zeta_2)}.
     \end{align}
     Now, if we set $w_k = \frac{h_k}{2h_{k-1}}$, the first term vanishes, and we get
     \begin{equation}
        \tilde{L}_k = \norm{\frac{h_k^3}{6}\odv[order=2]{u}{t}(\vz(\zeta_3), \zeta_3) + \frac{h_k^2 h_{k-1}}{4}\odv[order=2]{u}{t}(\vz({\zeta_2}), \zeta_2)}.
     \end{equation}
     Since the step sizes are bounded, there is a constant $A$ such that $h_{k-1} \leq A h_k$. We therefore have
     \begin{equation}
        \tilde{L}_k \leq \frac{h_k^3}{6}\norm{\odv[order=2]{u}{t}(\vz(\zeta_3) \zeta_3)} + \frac{Ah_k^3}{4}\norm{\odv[order=2]{u}{t}(\vz({\zeta_2}), \zeta_2)}. 
     \end{equation}  
     Assuming that $u(\vz, t)$ is sufficiently smooth, its derivatives should be bounded on $[t_M, t_0]$. Thus,
     \begin{equation}
        \tilde{L}_k \leq C'h_k^3 = \mc{O}(h_k^3)
     \end{equation} 
     for some constant $C'$. Therefore, \Cref{theorem:global-error} implies that for $h=\max_k(t_k - t_{k+1})$, the global error is $\mathcal{O}(h)$ for the Euler update and $\mathcal{O}(h^2)$ for the history-based update in \gls{method}.
\end{proof}

\begin{theorem}\label{theorem:global-error}
Let $u(\vz,t)$ be sufficiently smooth in both arguments, and consider a decreasing 
time grid $t_0 > t_1 > \cdots > t_M$ with variable step sizes $h_k = t_k - t_{k+1}$. Suppose a numerical ODE solver produces updates with local truncation error $L_k$ of order $\mc{O}(h_k^p)$,
for some $p \ge 1$. Then the final global error of the solver $E_{t_0}$ satisfies $E_{t_0} = \mc{O}(h^p)$ for $h=\max_k(t_k - t_{k+1})$.
\end{theorem}

\begin{proof}
   A general proof is given in Chapter II.3, Theorem 3.4 of \citet{hairer1993solving}.
\end{proof}

\section{Relation to autoguidance}
Autoguidance \citep{karras2024guiding} leverages a weaker diffusion model (a smaller or less trained variant of the base model) to improve generation quality. However, since it requires training an additional model, it cannot be applied out of the box to enhance pretrained models without extra training. In contrast, the history term introduced by \gls{method} can be interpreted as an explicit negative signal derived from the previous predictions of the diffusion model. Specifically, the prediction $\predcond[t_k]$ at each step corresponds to a denoised version of the current input $\vz_{t_k}$, which initially tends to be blurry and of lower quality. As sampling progresses, successive predictions become sharper, and thus the update term $\predcond[t_k] - g(\buffer)$ roughly captures the difference between the base model’s prediction and that of a “worse” version. Because the negative signal $g(\buffer)$ is defined entirely from past predictions, \gls{method} can be applied to pretrained models without any additional training. Moreover, \gls{method} can be combined with autoguidance by replacing the CFG prediction $\predguided[t_k]$ in \gls{method} with the corresponding prediction from autoguidance.

\section{Additional experiments}
\paragraph{Sampling time} Since \gls{method} does not require an additional forward pass, its sampling time is effectively identical to the CFG baseline. To verify this, we measured inference performance on Stable Diffusion 3 using the same GPU, with and without \gls{method}. In both cases, sampling achieved about 6.50 iterations per second with identical memory usage, confirming that \gls{method} has negligible impact on runtime or memory.

\paragraph{Compatibility with adaptive projected guidance} We next demonstrate that \gls{method} is compatible with other CFG variants such as adaptive projected guidance (APG) \citep{sadat2025eliminating}. \Cref{fig:apg-appendix} shows that \gls{method} also improves the quality and details of the APG samples. This indicates that our approach is complementary to various CFG modifications.
\figAPGAppendix   

\paragraph{CLIP scores} While HPSv2 and ImageReward both account for image quality and prompt alignment, we also report CLIP scores in this section for completeness. \Cref{tab:sd-results-clip} shows that \gls{method} achieves the same CLIP score as the baseline, indicating that it preserves the prompt alignment of CFG. We note, however, that CLIP scores may not fully reflect human judgment, as they often show limited variation across models \citep{sdxl}. Therefore, we primarily relied on HPSv2 as our main metric for evaluating both quality and prompt alignment.
\tabResultSDClip

\paragraph{Compatibility with different samplers} The main results of this paper examined the effect of adding \gls{method} to various models with their default samplers. In this experiment, we explicitly demonstrate that \gls{method} is compatible with different diffusion sampling techniques using DiT-XL/2 as the base model. \Cref{tab:samplers} shows that \gls{method} improves the quality across all samplers (including multistep solvers such as DPM++), indicating that its benefits are complementary to changing the sampling method.
\tabSamplers

\paragraph{Avoiding the issues of high CFG scales} While increasing CFG generally improves image quality at various \gls{nfe}, high CFG scales often result in oversaturation and reduced diversity. In contrast, \gls{method} enhances image quality at lower CFG scales while inherently avoiding these drawbacks. \Cref{fig:flux-appendix-diversity,fig:flux-appendix-diversity2} demonstrate that applying \gls{method} at lower CFG scales yields more diverse and realistic generations compared to sampling with high CFG. Furthermore, as shown in the main paper, \gls{method} consistently outperforms CFG across a wide range of scales.
\figFluxDiversity

\paragraph{Changing the scale in \gls{method}}
\Cref{fig:ablation-whigs} shows the performance of \gls{method} as the scale $\whigs$ varies. We observe that performance improves as $\whigs$ increases, but degrades when the scale becomes too large, while all settings still outperform the CFG baseline. Overall, we find that $\whigs \leq 3$ provides consistently strong performance across models.
\AblationWhigsPlot

\section{Ablation studies}\label{sec:ablation}

\paragraph{The choice for buffer input}
\Cref{tab:ablation-input} shows that both the conditional prediction $\predcond$ and the guided prediction $\predguided$ are viable inputs to \gls{method}, but using $\predguided$ as the history leads to better performance. Accordingly, we used $\predguided$ in our experiments whenever possible.
\tabAblationInput

\paragraph{Frequency filtering} \Cref{fig:ablation-filtering} shows that applying frequency filtering is essential for avoiding color artifacts in our method. Without this step (both with and without projection), the images often exhibit unrealistic patterns and unnatural color compositions. These issues are effectively mitigated through our DCT-based filtering, and we later show that performance remains fairly robust to different choices of hyperparameters used in the DCT filter.
\figAblationDCT

\paragraph{Projection} We next demonstrate that projection can also reduce color artifacts in the generated images in some situations. \Cref{fig:ablation-projection} shows that after projection, the generations appear more realistic with fewer oversaturated regions. Thus, incorporating projection into \gls{method} may lead to more realistic outputs when oversaturation exists.
\figAblationProjection

\paragraph{Effect of the thresholds in the weight schedule} We next analyze the impact of $\tmin$ and $\tmax$ in the weight scheduler. \Cref{fig:ablation-mint} shows that setting $\tmin$ too low or too high leads to either excessive or insufficient guidance, resulting in suboptimal performance. Similarly, \Cref{fig:ablation-maxt} indicates that reducing $\tmax$ generally causes insufficient guidance and degraded results. Overall, we found that setting $\tmin \in [0.3, 0.5]$ and $\tmax \in [0.9, 1.0]$ yields consistently good performance across models.
\AblationMinTPlot
\AblationMaxTPlot

\paragraph{Effect of EMA value}
\Cref{fig:ablation-ema} shows that \gls{method} performs well for a wide range of EMA values $\alpha$. We have found that setting $\alpha=0.5$ or $\alpha=0.75$ gives good results across all models and architectures.
\AblationEMAPlot

\paragraph{Effect of the DCT threshold}
We show in this section that \gls{method} is relatively robust to the choice of the \gls{dct} threshold $R_c$ as long as it is sufficiently low. \Cref{fig:ablation-dct} shows that the metrics are more or less the same for lower $R_c$ values. We set $R_c \approx 0.05$ for all models.
\AblationDCTPlot

\paragraph{Effect of weight scheduling} We next show that alternative choices for the weight scheduler are possible. Specifically, we test a constant function for $\whigs(t)$ applied over an interval, i.e.,
\begin{equation}
\whigs(t) = \begin{cases}
0 & t \leq t_{\textnormal{min}}, \\
\whigs & t_{\textnormal{min}} < t \leq t_{\textnormal{max}}, \\
0 & t > t_{\textnormal{max}},
\end{cases}
\end{equation}
as well as a linear schedule given by
\begin{equation}
\whigs(t) = \begin{cases}
0 & t \leq t_{\textnormal{min}}, \\
\whigs \cdot \frac{t - t_{\textnormal{min}}}{t_{\textnormal{max}} - t_{\textnormal{min}}} & t_{\textnormal{min}} < t \leq t_{\textnormal{max}}, \\
0 & t > t_{\textnormal{max}}.
\end{cases}
\end{equation}
\Cref{tab:ablation-beta} shows that all three options are viable. We chose the square-root function in \Cref{sec:method}, as we empirically found that it produces more visually appealing results.
\tabAblationBeta

\paragraph{Various options for the history function}
In \Cref{tab:ablation-g}, we show that different choices of $g$ yield similar quality metrics, indicating that \gls{method} is robust to this design decision. We adopt the EMA option, as it produced slightly more realistic outputs in our visual evaluations. Moreover, it enables computing the average on the fly without storing all past predictions in memory (see \Cref{alg:code-utility}).
\tabAblationG

\section{Implementation details}\label{sec:imp-detail}
The detailed algorithmic procedure for \gls{method} is provided in \Cref{alg:method}, with pseudocode given in \Cref{alg:code-utility,alg:code-main}. Since \gls{method} reuses past predictions without requiring additional forward passes, its computational cost is equivalent to standard CFG. Other operations, such as \gls{dct}, add negligible overhead. Thus, \gls{method} improves quality without increasing the overall sampling cost or memory as can be seen in the code and our experiments.

We always scale the time step $t$ such that $t \in [0, 1]$. For the frequency filter, we found that $\lambda=50$ and $R_c \approx 0.05$ work well across all models. Similarly, setting $\tmin \in [0.3, 0.5]$ and $\tmax \in [0.9, 1.0]$ lead to consistently strong results in our tests. For EMA values, we observed that $\alpha = 0.5$ or $\alpha = 0.75$ worked well across all experiments. We used all past predictions as the history length $W$, which allows us to compute the EMA average on the fly without buffering all previous predictions in memory. The hyperparameters used for each experiment are given in \Cref{tab:parameters1,tab:parameters2,tab:parameters3}.
\tabParameters

For quantitative evaluation, FID scores for class-conditional models in \Cref{tab:fid-results} are reported using \num{10000} generated samples and the full ImageNet training set. The ImageNet results in \Cref{tab:repa} are based on \num{50000} generated samples to ensure fair comparison with prior work. For text-to-image models, we used the entire validation set of the COCO 2017 dataset \citep{lin2014microsoft} as ground truth text-image pairs. We followed the ADM evaluation framework \citep{dhariwalDiffusionModelsBeat2021} for computing FID, IS, Precision, and Recall to maintain consistency across all evaluations.
\algHigs

\begin{algorithm}[t!]
    \caption{Utility functions used in the implementation of \acrshort{method}.}
        \label{alg:code-utility}
        \centering
\begin{tcolorbox}[
    colback=lightgray!10,
    colframe=gray!50,
    boxrule=0.5pt,
    arc=2pt,
    left=3pt,
    right=3pt,
    top=3pt,
    bottom=3pt,
    fontupper=\footnotesize\ttfamily,
]
\begin{lstlisting}[
    language=Python,
    basicstyle=\footnotesize\ttfamily,
    keywordstyle=\color{blue}\bfseries,
    commentstyle=\color{green!60!black},
    stringstyle=\color{red},
    showstringspaces=false,
    tabsize=4,
    breaklines=true,
    columns=flexible
]
import torch
import torch_dct as dct


def project(
    v0: torch.Tensor,  # [B, C, H, W]
    v1: torch.Tensor,  # [B, C, H, W]
):
    dtype = v0.dtype
    v0, v1 = v0.double(), v1.double()
    v1 = torch.nn.functional.normalize(v1, dim=[-1, -2, -3])
    v0_parallel = (v0 * v1).sum(dim=[-1, -2, -3], keepdim=True) * v1
    v0_orthogonal = v0 - v0_parallel
    return v0_parallel.to(dtype), v0_orthogonal.to(dtype)


def square_root_schedule(t, start=0, end=1):
    if t > end or t <= start:
        return 0.0
    return ((t - start) / (end - start)) ** 0.5


def dct2(x):
    return dct.dct_2d(x, norm="ortho")


def idct2(x):
    return dct.idct_2d(x, norm="ortho")


def apply_high_freq_dct_mask(diff, threshold=0.05, sharpness=50):
    B, C, H, W = diff.shape
    device = diff.device
    X = dct2(diff)
    u = torch.arange(H, device=device).view(H, 1) / H
    v = torch.arange(W, device=device).view(1, W) / W
    d = torch.sqrt(u**2 + v**2)  # normalized distance from top-left (DC)
    mask = torch.sigmoid((d - threshold) * sharpness)
    X_filtered = X * mask  # broadcast over (B, C)
    diff_filtered = idct2(X_filtered).to(diff.dtype)
    return diff_filtered


class HistoryBuffer:
    def __init__(self, ema_alpha=0.75):
        self.ema = None
        self.ema_alpha = ema_alpha

    def add(self, current_pred):
        if self.ema is None:
            self.ema = torch.zeros_like(current_pred)
        self.ema = self.ema_alpha * current_pred + (1 - self.ema_alpha) * self.ema
\end{lstlisting}
\end{tcolorbox}
\end{algorithm}

\begin{algorithm}[t!]
    \caption{PyTorch implementation of \acrshort{method}.}
        \label{alg:code-main}
        \centering
\begin{tcolorbox}[
    colback=lightgray!10,
    colframe=gray!50,
    boxrule=0.5pt,
    arc=2pt,
    left=3pt,
    right=3pt,
    top=3pt,
    bottom=3pt,
    fontupper=\footnotesize\ttfamily,
]
\begin{lstlisting}[
    language=Python,
    basicstyle=\footnotesize\ttfamily,
    keywordstyle=\color{blue}\bfseries,
    commentstyle=\color{green!60!black},
    stringstyle=\color{red},
    showstringspaces=false,
    tabsize=4,
    breaklines=true,
    columns=flexible
]
class HiGSGuidance:
    def __init__(
        self,
        w_higs,
        t_min=0.4,
        t_max=1.0,
        eta=1.0,
        ema_alpha=0.75,
        dct_threshold=0.05,
    ):
        self.history = HistoryBuffer(ema_alpha=ema_alpha)
        self.weight = w_higs
        self.min_t = t_min
        self.max_t = t_max
        self.parallel_weight = eta
        self.dct_threshold = dct_threshold

    def step(self, current_pred, timestep=None):
        """
        Compute the HiGS guidance step.
        """
        # current_pred can be either CFG-guided or conditional predictions
        if self.history.ema is None:
            self.history.add(current_pred)
            return torch.zeros_like(current_pred)

        diff = current_pred - self.history.ema

        # compute the projection of the difference
        diff_par, diff_orth = project(diff, current_pred)
        diff = diff_orth + diff_par * self.parallel_weight

        # Compute the scaling factor based on the current timestep
        gamma = square_root_schedule(timestep, self.min_t, self.max_t)
        scale = self.weight * gamma

        # Update the history with the current prediction
        self.history.add(current_pred)

        # Apply the high-frequency DCT mask to the difference
        if self.dct_threshold >= 0:
            diff = apply_high_freq_dct_mask(diff, threshold=self.dct_threshold)

        # Return the scaled difference
        return scale * diff
\end{lstlisting}
\end{tcolorbox}
\end{algorithm}

\section{Additional visual examples}
We provide additional visual examples in this section to demonstrate the effectiveness of \gls{method} in enhancing the quality of various models across a wide range of guidance scales and sampling setups. \Cref{fig:sdxl-appendix,fig:sd3-appendix,fig:sd3-appendix-two,fig:sd3-appendix-three} show further text-to-image generation results using Stable Diffusion models \citep{sdxl,esser2024scaling}. In addition, \Cref{fig:flux-appendix} presents visual results for applying \gls{method} to Flux \citep{flux2024}. Finally, \Cref{fig:sit-appendix} provides class-conditional generation with SiT-XL + REPA \citep{yu2024repa}. In all cases, \gls{method} consistently improves quality over the baseline.

\clearpage
\figSDXLAppendix
\figSDThreeAppendix
\figFluxAppendix
\figREPAEAppendix

\end{appendices}

\end{document}